\pdfoutput=1

\documentclass{article}

\let\savedaddcontentsline\addcontentsline
\usepackage{iclr2027_conference,times}
\let\addcontentsline\savedaddcontentsline
\iclrfinalcopy

\usepackage[T1]{fontenc}
\usepackage{microtype}
\usepackage{graphicx}
\usepackage{subcaption}
\usepackage{booktabs}
\usepackage{tabularx}
\usepackage{array}
\newcolumntype{Y}{>{\raggedright\arraybackslash}X}

\definecolor{citationblue}{RGB}{0,32,128}
\usepackage[breaklinks,colorlinks=true,citecolor=citationblue]{hyperref}       %

\usepackage{amsmath}
\usepackage{amssymb}
\usepackage{mathtools}
\usepackage{amsthm}
\usepackage{algorithm}
\usepackage{algorithmic}
\usepackage[most]{tcolorbox}

\makeatletter
\renewcommand{\paragraph}{%
  \@startsection{paragraph}{4}%
  {\z@}{0ex}{-1em}%
  {\normalfont\normalsize\bfseries}%
}
\makeatother

\newtcolorbox{defbox}{
  colback=white,
  colframe=black,
  boxrule=0.4pt,
  arc=0pt,
  left=4pt,
  right=4pt,
  top=2pt,
  bottom=2pt
}
\usepackage{my_math}

\newcommand{\norm}[1]{\left\|#1\right\|}

\newcommand{\inner}[2]{\left\langle{#1}, {#2}\right\rangle}

\usepackage{amsmath,amsfonts,bm}

\def\1{\bm{1}}

\def\vzero{{\bm{0}}}

\def\va{{\bm{a}}}
\def\vb{{\bm{b}}}

\def\vd{{\bm{d}}}

\makeatletter
\newcommand{\ve}{\@ifnextchar\bgroup{\velong}{{\bm{e}}}}
\newcommand{\velong}[1]{{\bm{#1}}}
\makeatother

\def\vf{{\bm{f}}}
\def\vg{{\bm{g}}}
\def\vh{{\bm{h}}}

\def\vm{{\bm{m}}}
\def\vn{{\bm{n}}}

\def\vp{{\bm{p}}}
\def\vq{{\bm{q}}}
\def\vr{{\bm{r}}}

\def\vu{{\bm{u}}}
\def\vv{{\bm{v}}}
\def\vw{{\bm{w}}}
\def\vx{{\bm{x}}}
\def\vy{{\bm{y}}}
\def\vz{{\bm{z}}}

\def\mA{{\bm{A}}}
\def\mC{{\bm{C}}}
\def\mD{{\bm{D}}}
\def\mE{{\bm{E}}}
\def\mF{{\bm{F}}}
\def\mG{{\bm{G}}}
\def\mH{{\bm{H}}}
\def\mI{{\bm{I}}}
\def\mJ{{\bm{J}}}
\def\mK{{\bm{K}}}

\def\mM{{\bm{M}}}
\def\mN{{\bm{N}}}

\def\mP{{\bm{P}}}
\def\mQ{{\bm{Q}}}

\def\mS{{\bm{S}}}

\def\mW{{\bm{W}}}
\def\mX{{\bm{X}}}
\def\mY{{\bm{Y}}}
\def\mZ{{\bm{Z}}}

\def\mPhi{{\bm{\Phi}}}

\def\mDelta{{\bm{\Delta}}}

\def\mSigma{{\bm{\Sigma}}}

\DeclareMathAlphabet{\mathsfit}{\encodingdefault}{\sfdefault}{m}{sl}
\SetMathAlphabet{\mathsfit}{bold}{\encodingdefault}{\sfdefault}{bx}{n}

\newcommand{\E}{\mathbb{E}}
\newcommand{\R}{\mathbb{R}}

\providecommand{\explain}[1]{\text{\textcolor[gray]{0.45}{#1}}}

\newcommand{\vvareps}{\boldsymbol{\epsilon}}
\newcommand{\lrhat}{\bar\lr_\star}
\newcommand{\lrstar}{\lr_\star}
\newcommand{\lrP}{\lr_\star^{\mathsf P}}
\newcommand{\cstar}{c_\star}
\newcommand{\bstar}{b_\star}
\newcommand{\bstarP}{{b_\star^{\mathsf P}}}
\newcommand{\mustar}{\mu_\star}
\newcommand{\mustarP}{{\mu_\star^{\mathsf P}}}
\newcommand{\XstarP}{{\mX_\star^{\mathsf P}}}
\newcommand{\dcov}{\delta_{\mathrm{cov}}}
\newcommand{\rhoT}{\rho_T}

\newcommand{\WCupdate}[1]{{#1}}

\newcommand{\HWupdate}[1]{{#1}}

\usepackage[capitalize,noabbrev]{cleveref}

\IfFormatAtLeastTF{2026-06-01}{%
  \newcommand{\newaliasedtheorem}[2]{%
    \newtheorem{#1}[theorem]{#2}%
  }
}{%
  \usepackage{aliascnt}
  \newcommand{\newaliasedtheorem}[2]{%
    \newaliascnt{#1}{theorem}%
    \newtheorem{#1}[#1]{#2}%
    \aliascntresetthe{#1}%
  }
}
\theoremstyle{plain}
\newtheorem{theorem}{Theorem}[section]
\newaliasedtheorem{proposition}{Proposition}
\newaliasedtheorem{observation}{Observation}
\crefname{observation}{observation}{observations}
\Crefname{observation}{Observation}{Observations}
\newaliasedtheorem{lemma}{Lemma}
\newaliasedtheorem{corollary}{Corollary}
\theoremstyle{definition}
\newaliasedtheorem{definition}{Definition}
\newaliasedtheorem{assumption}{Assumption}
\crefname{assumption}{assumption}{assumptions}
\Crefname{assumption}{Assumption}{Assumptions}
\theoremstyle{remark}
\newaliasedtheorem{remark}{Remark}

\usepackage[textsize=tiny]{todonotes}

\usepackage{enumitem}
\setlist[itemize]{
  leftmargin=1.5em,
  topsep=0.12em,
  itemsep=0.12em,
  parsep=0.12em,
  partopsep=0.12em,
}
\setlist[enumerate]{
  leftmargin=1.5em,
  topsep=0.12em,
  itemsep=0.12em,
}

\title{Does LLM Pre-Training Typically Occur at the Edge of Stability?}

\author{%
  \begin{tabular}[t]{@{}*{2}{>{\centering\arraybackslash}p{\dimexpr(\textwidth-2\tabcolsep)/2\relax}@{}}}
  Yuhang Cai\textsuperscript{*\textdagger} & Haodong Wen\textsuperscript{*\textdaggerdbl}\\
  \normalfont\small\texttt{willcai@berkeley.edu} & \normalfont\small\texttt{whd25@mails.tsinghua.edu.cn}\\[1.5ex]
  Haofeng Huang\textsuperscript{*\textdaggerdbl} & Deyi Liu\textsuperscript{\textsection}\\
  \normalfont\small\texttt{huanghf22@mails.tsinghua.edu.cn} & \normalfont\small\texttt{deyiliu@bytedance.com}\\[1.5ex]
  Yiyuan Ma\textsuperscript{\textsection} & Kaifeng Lyu\textsuperscript{\textdaggerdbl}\\
  \normalfont\small\texttt{mayiyuan.unicorn@bytedance.com} & \normalfont\small\texttt{klyu@mail.tsinghua.edu.cn}
  \end{tabular}%
}

\begin{document}

\maketitle

\lhead{}
\renewcommand{\headrulewidth}{0pt}

\begingroup
\renewcommand{\thefootnote}{\fnsymbol{footnote}}
\footnotetext[1]{Equal contribution.}
\footnotetext[2]{Affiliations: \textdagger{} UC Berkeley; \textdaggerdbl{} Tsinghua University; \textsection{} ByteDance Seed.}
\endgroup

\begin{abstract}
  Quadratic approximations provide a common lens for studying neural network optimization, but recent evidence challenges their predictive validity. In full-batch gradient descent with learning rate (LR) $\eta$, \citet{cohen2021gradient} observed the Edge of Stability (EoS), where the largest Hessian eigenvalue concentrates near $2/\eta$, seemingly at odds with classical stability conditions predicted by quadratic approximations.
  In this work, we revisit the fidelity of quadratic approximations as a model of neural network training dynamics, with a particular focus on their failure modes in LLM training.
	We first highlight a distinct failure mode of the quadratic approximation, where persistent negative curvature causes the quadratic dynamics to diverge regardless of the LR, even though the real dynamics remain stable.
  We term this phenomenon the \emph{Edge of Convexity}~(EoC).
  After removing the effect of this LR-independent instability, we then propose a generalized definition of \textit{Edge of Stability}~(EoS) for stochastic training with adaptive optimizers, based on whether increasing the LR induces significant divergence in the quadratic dynamics while leaving the real dynamics largely unchanged.
	Across different LLM pretraining runs with various model sizes up to $787$M, we find that (1) EoC is always observed across LLM pretraining, and (2) EoS is also prevalent but not universal: it disappears when the LR becomes sufficiently small (e.g., after decay) or when the batch size falls below a critical threshold that has a positive correlation with the critical batch size.
	Together, these findings characterize when and how quadratic approximations fail and provide a foundation for future work on understanding training dynamics.
\end{abstract}

\section{Introduction}

In optimization theory, the dynamics of gradient-based methods are commonly analyzed through a local \textit{quadratic approximation} of the loss function, given by the second-order Taylor expansion around the current iterate:
\begin{equation}
    \loss(\param_{t+1}) \approx \loss(\param_t) + (\param_{t+1} - \param_t)^\top \nabla \loss(\param_t) + \frac{1}{2} (\param_{t+1} - \param_t)^\top \nabla^2 \loss(\param_t) (\param_{t+1} - \param_t).
\end{equation}
This approximation yields the well-known \textit{descent lemma} for gradient descent, which guarantees a decrease in loss when the LR $\lr$ satisfies $\lr < 2/\lambda_{\max}$. Here, $\lambda_{\max}$ is the maximum eigenvalue of the Hessian, which serves as a measure of the local curvature of the loss function.
A standard analysis pipeline in optimization theory then proceeds by bounding quantities that appear in the lemma and
deriving guarantees on the loss convergence.

However, in deep learning, the quadratic approximation often fails to accurately capture the true training dynamics.
Notably,~\citet{cohen2021gradient} show that full-batch gradient descent (GD) typically operates at the \textit{Edge of Stability} (EoS) when training neural networks.
For a fixed LR $\lr$, the maximum Hessian eigenvalue $\lambda_{\max}$ tends to increase during training and eventually oscillates around~$2/\lr$, thus $\lr \approx 2/\lambda_{\max}$.
Under the quadratic approximation, this places the dynamics exactly at the stability boundary of gradient descent, since the quadratic model predicts loss divergence once $\lr > 2/\lambda_{\max}$.
Consequently, the descent lemma can no longer guarantee a persistent decrease of the loss.
Yet in practice, the loss continues to decrease over long time scales.
Similar EoS behavior has also been observed for full-batch adaptive optimizers such as Adam and RMSProp~\citep{cohen2022adaptive}, when curvature is measured using an appropriately preconditioned Hessian.

\begin{figure}[t]
\vspace{-0.3in}
    \centering
    \begin{subfigure}[t]{0.49\textwidth}
    \centering
    \includegraphics[width=\linewidth]{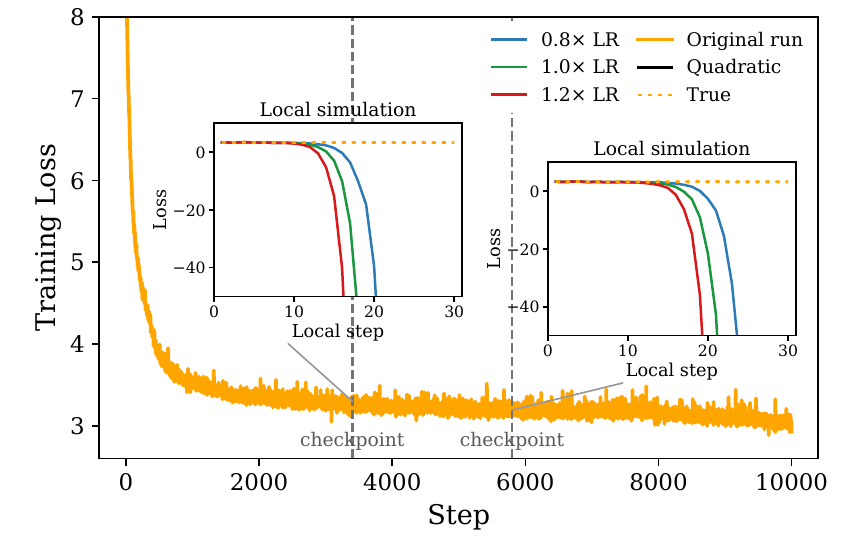}
    \vspace{-0.27in}
    \caption{Illustration of Edge of Convexity (EoC).}
    \label{fig:fig1-EoC}
    \end{subfigure}\hfill
    \begin{subfigure}[t]{0.49\textwidth}
    \centering
    \includegraphics[width=\linewidth]{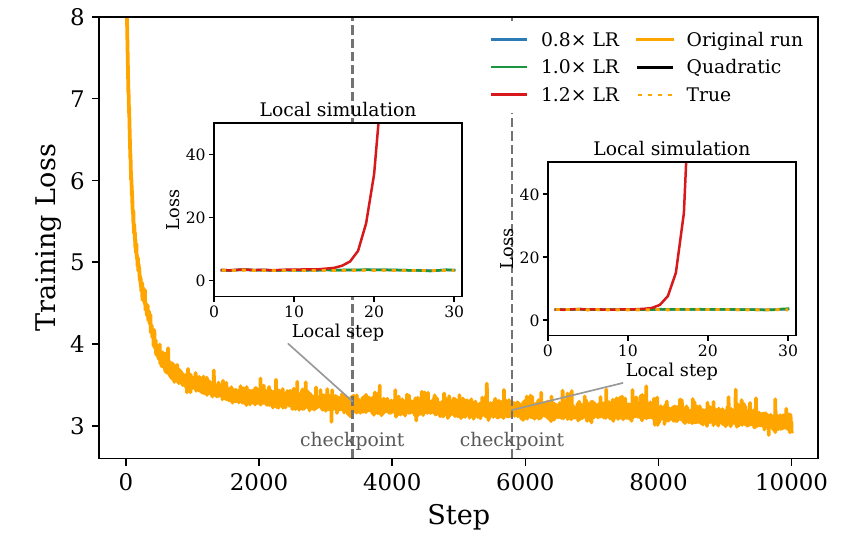}
    \vspace{-0.27in}
    \caption{Illustration of Edge of Stability (EoS).}
    \label{fig:fig2-EoS}
    \end{subfigure}
    \vspace{-0.1in}
    \caption{%
    LLM pretraining on FineWeb-10B exhibits two distinct failure modes of the quadratic approximation.
    (a) \textbf{EoC phenomenon}: regardless of the LR choice, the training dynamics under quadratic approximation diverge due to persistent negative curvature, while the true dynamics remain stable.
     (b) \textbf{EoS phenomenon}: after compensating for negative curvature, the quadratic loss remains as stable as the true loss for $0.8 \times$ and $1.0 \times$ the training LR, but increasing the LR slightly to $1.2 \times$ causes the quadratic loss to diverge.
     This shows that the current training LR is close to a critical stability threshold. See \Cref{sec:eoc-main,subsec:stochastic-eos-definition} for details.}
    \label{fig:eoc_eos_illustrations}
    \vspace{-0.14in}
\end{figure}

The EoS phenomenon therefore exposes a fundamental limitation of existing optimization analyses for deep learning, as many such analyses critically rely on the descent lemma and the underlying quadratic approximation.
This discrepancy has motivated a line of work that seeks to develop tools to understand training dynamics without critically relying on the quadratic approximation~\citep{arora2022understanding,damian2022self,cohen2024understanding}.

Despite these advances, existing empirical observations and theoretical analyses remain limited in several important aspects:
(i) Most theoretical definitions of EoS focus on full-batch training, whereas stochastic training is more common in practice;
(ii) Most studies focus on failures of the quadratic approximation caused by large positive Hessian eigenvalues, even though negative Hessian eigenvalues can also induce a similar gap between theory and practice, as we will elaborate in this work;
(iii)~Due to the prohibitive computational cost of computing and analyzing the Hessian,
direct measurements of EoS are largely limited to relatively small-scale settings such as CIFAR-10 and ImageNet, leaving open whether the same failure mode persists in large-scale training.

\paragraph{Our Contributions.} We present a systematic study of the failure modes of the quadratic approximation in stochastic, adaptive, and large-scale training, with a particular focus on LLM pretraining.
\begin{itemize}
    \item We first highlight an LR-independent failure mode of the quadratic approximation, where persistent negative curvature during training causes the quadratic dynamics to diverge regardless of the LR, even though the real dynamics remain stable.
    We term this phenomenon the \textit{Edge of Convexity}~(EoC). See~\Cref{subsec:eoc-definition}.

    \item We isolate and remove the effect of EoC, and then study the \textit{Edge of Stability}~(EoS) as an LR-dependent failure mode.
    We propose a generalized definition of EoS for stochastic training with adaptive optimizers, based on whether increasing the LR induces significant divergence in the quadratic dynamics while leaving the real dynamics largely unchanged.
    For full-batch GD, our definition recovers the original EoS definition. See~\Cref{subsec:stochastic-eos-definition,sec:agm}.

    \item We design efficient methods to detect both EoC and EoS at scale with theoretical guarantees (\Cref{sec:method-eoc,sec:method-eos}). Our results on LLM pretraining suggest that (1) EoC consistently arises; (2) EoS is prevalent but not universal, and disappears when the LR becomes sufficiently small (e.g., after decay) or when the batch size falls below a critical threshold that scales linearly with the critical batch size (\Cref{subsec:exp-scales,subsec:exp-bsz-threshold,subsec:exp-lr-eos}).
\end{itemize}

\begin{figure*}[t]
\vspace{-0.2in}
	\centering
	\begin{subfigure}[t]{0.32\textwidth}
		\centering
		\includegraphics[width=\linewidth]{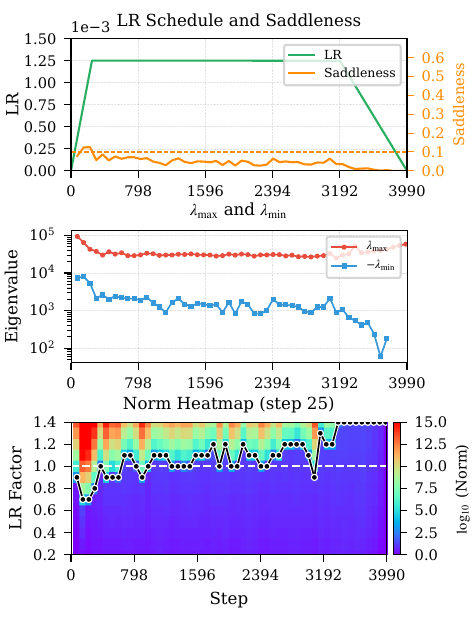}
		\caption{GPT-small (104M).}
		\label{fig-model:small-main}
	\end{subfigure}\hfill
	\begin{subfigure}[t]{0.32\textwidth}
		\centering
		\includegraphics[width=\linewidth]{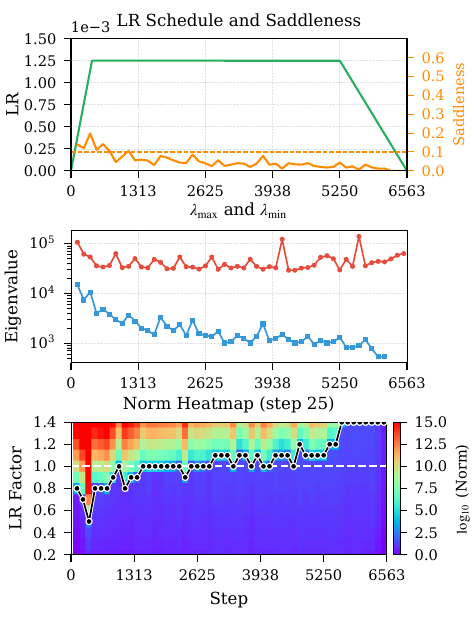}
		\caption{GPT-medium (344M).}
		\label{fig-model:medium-main}
	\end{subfigure}\hfill
	\begin{subfigure}[t]{0.32\textwidth}
		\centering
		\includegraphics[width=\linewidth]{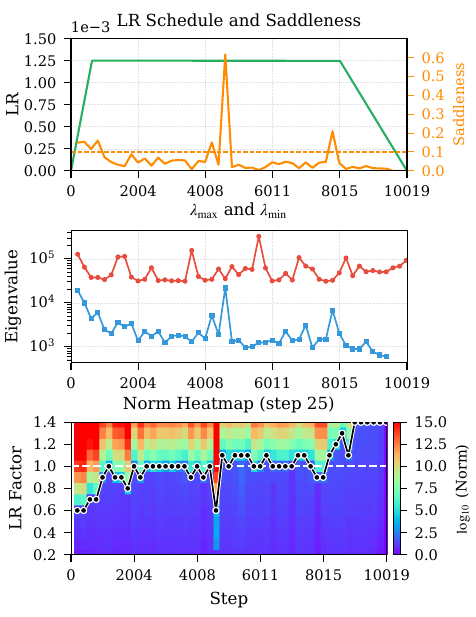}
		\caption{GPT-large (787M).}
		\label{fig-model:large}
	\end{subfigure}
\caption{\textbf{EoS and EoC appear across model sizes.} At most checkpoints, the bottom heatmaps show a sharp transition from bounded displacement to explosion near LR factor $1.0$, indicating that the training LR is at the critical threshold of EoS. In the middle row, $-\lambda_{\min}$ has a smaller magnitude than $\lambda_{\max}$, corresponding to saddleness below the $0.1$ reference line in the top plots and consistent with EoC. Each heatmap column is a checkpoint, and the vertical axis gives the simulated LR divided by the training LR. See \Cref{subsec:exp-scales} for the results and \Cref{app:detection-settings} for training and detection settings.}
	\label{fig-small-medium-large}
	\vspace{-0.15in}
\end{figure*}

\section{Related Work}
\label{sec:related-work}
\paragraph{Edge of Stability.}
Early studies connected sharpness and stability to optimizer behavior in neural network training: \citet{wu2018sgdselects} analyzed SGD minima selection through dynamical stability using sharpness and non-uniformity, and \citet{xing2018walk} observed SGD ``bouncing'' between valley walls during training. \citet{jastrzkebski2018relation} then showed that SGD visits increasingly sharp regions up to a maximum set by LR and batch size, and that the step is often too large along the sharpest directions. \citet{cohen2021gradient} later demonstrated that full-batch gradient descent typically operates at an Edge of Stability (EoS), where the largest Hessian eigenvalue oscillates near $2/\eta$ while loss decreases non-monotonically. \citet{cohen2022adaptive} extended this phenomenon to adaptive methods via saturation of a preconditioned stability threshold. Finally, \citet{arora2022understanding,li2022analyzing,damian2022self,lyu2022understanding} provide theoretical explanations for EoS in different training regimes, but a unified understanding of this phenomenon is still lacking.

\paragraph{Training Stability of Neural Networks} Series works studied the training stability of constant LR SGD in quadratic losses via a linear system analysis of the second moment of the parameters~\citep{wu2018sgdselects,ma2021linear,granziol2022learning,wuwangsu2022alignment,mulayoff2024exact}. They obtained tight conditions in the form of $\|\E[(\mI - \eta \mH_t)^{\otimes 2}]\| \le 1$, which are valid instability criteria for a particular Lyapunov
function. However, in the training
neural networks with (S)GD, the unstable convergence phenomenon was observed~\citep{wu2018sgdselects,xing2018walk,jastrzkebski2018relation}, formally identified and termed \textit{Edge of Stability} (EoS)~\citep{cohen2021gradient}. Subsequently, several studies sought to explain this unstable convergence~\citep{ma2022beyond,ahn2022understanding,wu2023implicit,arora2022understanding,damian2022self,cohen2024understanding}. Among them, \citet{cohen2024understanding} gave a general characterization of the oscillatory dynamics of deep learning in the EoS regime via a time-averaged differential equation termed \textit{central flow}. However, the effectiveness of the quadratic approximation was never carefully studied in the large-scale neural network training, hence it is not clear when and how the quadratic approximation fails or not in the realistic regime.

\paragraph{Hessian Spectrum of Neural Networks.} Most studies on the Hessian spectrum of neural networks reported that the Hessian spectra of neural networks consist of a ``bulk'' together with a few large ``outliers'' in magnitude~\citep{lecun2002efficient,dauphin2014identifying,sagun2016eigenvalues,sagun2017empirical,ghorbani2019investigation,chaudhari2019entropy,yao2020pyhessian}. Among them, \citet{sagun2016eigenvalues,sagun2017empirical} mentioned that there are still negative eigenvalues during training, but their magnitude is much smaller than the large outliers. More recent works reported and explained a near-block diagonal structure in the Hessian spectrum of CNNs and Transformers~\citep{zhang2024transformers,dong2025towards}. However, one thing that has long been overlooked is the characterization of the negative eigenvalues in the population Hessian during large-scale training. In this work, we present detailed empirical findings of the negative eigenvalues across different setups in large-scale neural network training.

\paragraph{Comparison with \citet{andreyev2024edge}.} The recent work~\citep{andreyev2024edge} is among the most relevant to this work, which defines Edge of Stochastic Stability for SGD via the batch sharpness, where Batch Sharpness hovers around $2/\eta$. Compared to \citet{andreyev2024edge}, our work
(1) In the full-batch setting, our definition for EoS can be reduced to the original definition of EoS by \citet{cohen2021gradient,cohen2022adaptive}, while theirs cannot;
(2)  We also identify another failure mode of the quadratic approximation called Edge of Convexity (EoC); (3) We provide scalable EoS and EoC detection methods with theoretical guarantees; (4) We conduct experiments on large-scale LLM pretraining with AdamW, while most of their experiments are limited to ResNets on CIFAR-10 and SVHN with vanilla SGD.

\section{Preliminary} \label{subsec:full-batch-eos}
\paragraph{Notation.}
Throughout, $\|\cdot\|$ denotes the spectral norm for matrices and the $\ell_2$ norm for vectors.
We write $\lambda_{\max}(\mA)$ and $\lambda_{\min}(\mA)$ for the largest and smallest eigenvalues of a symmetric matrix $\mA$,
and $\rho(\mA)$ for the spectral radius of a square matrix $\mA$.
We write $\vectorize(\cdot)$ for vectorization and $\otimes$ for the Kronecker product.
We use $\mathbb{S}_d$ for the set of $d\times d$ symmetric matrices and $\mathbb{S}_d^+$ for the subset of positive semidefinite matrices.
For $\mS\succeq 0$, we write $\|\vu\|_{\mS}=\sqrt{\vu^\top\mS\vu}$.
We use $[n]$ to denote $\{1,2,\ldots,n\}$.
We write $\grad f(\param)$ and $\hess f(\param)$ for the gradient and Hessian of $f$ at $\param$.

\paragraph{Problem Setup.}%
We study a neural network $f(\vx;\param):\R^{d+s}\to\R^q$, where $\vx\in\R^d$ is the input and $\param\in\R^s$ is the trainable parameter.
Given a dataset $\{(\vx_i,y_i)\}_{i=1}^n$ and a per sample loss function $\ell$, define the population loss and batch loss as $\loss(\param) \coloneqq
    \frac{1}{n}\sum_{i=1}^n\ell\bigl(f(\vx_i;\param),y_i\bigr)$ and $\loss_{\batch}(\param) \coloneqq
    \frac{1}{B}\sum_{i\in\batch}
    \ell\bigl(f(\vx_i;\param),y_i\bigr),$
where $\batch\subseteq[n]$ is a batch of size $B$.
\paragraph{Full-Batch Edge of Stability.}
We recall the full-batch Edge of Stability (EoS) introduced in \citet{cohen2021gradient}. %
Consider full-batch gradient descent (GD) with LR $\eta$ as $\param_{t+1} = \param_t-\eta \nabla \loss(\param_t).$
Let $\mH_t \coloneqq\nabla^2\loss(\param_t)$, and write $\lambda_{\max}(\mH_t)$ for the maximal eigenvalue of $\mH_t$.
Classical local analysis says that, if the loss around $\param_t$ is well described by a positive semidefinite quadratic function, then GD is stable only when $\eta \lambda_{\max}(\mH_t) \le 2.$
Equivalently, the threshold $2/\eta$ is the largest curvature without causing instability.

\citet{cohen2021gradient} found that neural network training does not usually stay well below this threshold.
During full batch training, $\lambda_{\max}(\mH_t)$ often increases until it reaches a value near $2/\eta$.
It then remains near this value for a long part of training.
The loss is not monotone at every step, but it continues to decrease over longer time scales.
They call this regime the \textit{Edge of Stability} (EoS).

To relate this observation to our discussion on the fidelity of quadratic approximation, we fix a checkpoint $\param_{t_0}$ and consider
\begin{align}
    Q_{t_0}(\param)
    =
    \loss(\param_{t_0})
    +(\param-\param_{t_0})^\top\grad\loss(\param_{t_0})
    +\frac{1}{2}(\param-\param_{t_0})^\top
    \mH_{t_0}(\param-\param_{t_0}).
    \label{eq:full-batch-quadratic-surrogate}
\end{align}
Suppose that this quadratic has a stationary point $\param_\star$ such that $ \grad\loss(\param_{t_0})
    +\mH_{t_0}(\param_\star-\param_{t_0})
    =0.$ Gradient descent on $Q_{t_0}$ then satisfies
\begin{align*}
    \param_{t+1}-\param_\star
    =
    (\mI-\eta\mH_{t_0})(\param_t-\param_\star).
\end{align*}
If $\mH_{t_0}\succeq 0$, this system is stable when $\rho(\mI-\eta\mH_{t_0})\leq 1,$
or equivalently, $\eta\lambda_{\max}(\mH_{t_0})\leq 2$.
Once $\eta\lambda_{\max}(\mH_{t_0})>2$, $\param_{t}$ oscillates with growing magnitude along a leading eigenvector of $\mH_{t_0}$.

Based on this perspective, the observation that $\lambda_{\max}$ concentrates near $2/\lr$ throughout training is striking.
It places training dynamics precisely at the boundary of stability predicted by the quadratic model, rather than safely within it. The EoS phenomenon therefore exposes a fundamental mismatch between the predictions of quadratic approximation and the actual behavior of neural network training, which is the starting point for our study of EoS in stochastic training.

\section{Stability of Stochastic Gradient Descent on Quadratics}
\label{sec:stochastic-quadratic-stability}

The full-batch EoS reviewed above is insightful but incomplete for modern neural networks since it relies on two simplifying conditions. First, the Hessian is assumed to be positive semidefinite. If it has a
negative eigenvalue, then the local quadratic dynamics contain an exploding direction for every positive LR, so the usual positive curvature stability threshold is no longer the only source of instability. Second, the analysis is deterministic and applies to full-batch gradient descent. It does not capture the additional randomness introduced by mini-batch sampling.

We then study these two effects on a stochastic quadratic objective. We consider a stochastic variant of \Cref{eq:full-batch-quadratic-surrogate}  to approximate the training loss. Specifically, at step $t$, let $f_t(\vx)=\frac{1}{2}\vx^\top\mA_t\vx+\vb_t^\top\vx+c_t,$ where $\mA_t$ is symmetric, $\Ebb[\mA_t]=\mA$, and $\Ebb[\vb_t]=\vb$.
We assume that $\{(\mA_t,\vb_t)\}_{t\geq 0}$ are sampled independently from a fixed distribution, have finite second moments, and are independent of the initialization.
SGD with LR $\eta$ gives
\begin{align}
    \vx_{t+1}
    =
    \vx_t-\eta\grad f_t(\vx_t)
    =
    (\mI-\eta\mA_t)\vx_t-\eta\vb_t.
    \label{eq:stochastic-quadratic-update}
\end{align}
Let $\mM_t \coloneqq \mI-\eta\mA_t.$
Suppose that the population quadratic has a stationary point $\vx_\star$ satisfying $\mA\vx_\star+\vb=0.$ And we define $\vvareps_t
    \coloneqq
    \vx_t -\vx_\star,
    \vr_t
    \coloneqq
    \mA_t\vx_\star+\vb_t.$
Then $\Ebb[\vr_t]=0$, the error $\vvareps_t$ follows
\begin{align}
    \vvareps_{t+1} =
    (\mI-\eta\mA_t)(\vx_t - \vx_\star) - \eta \mA_t \vx_\star -\eta\vb_t =
    \mM_t\vvareps_t-\eta\vr_t.
    \label{eq:error-dynamics-stochastic-quadratic}
\end{align}
We study the mean-square stability of the above stochastic quadratic dynamics.
\begin{definition}
\label{def:mean-square-stability}
For the stochastic dynamics in
\cref{eq:error-dynamics-stochastic-quadratic},
we say that the system is mean-square stable if, for every initialization
$\vvareps_0\in\R^d$ the resulting trajectory satisfies
\begin{align}
    \sup_{t\geq 0}
    \Ebb\left[\|\vvareps_t\|^2\right]
    <\infty.
    \label{eq:mean-square-stability}
\end{align}
\end{definition}

\paragraph{Why Mean-Square Stability.}
For a stochastic system, the mean trajectory does not describe the size of a typical perturbation.
Positive and negative deviations can cancel in expectation while their squared norm grows.
We therefore study $\Ebb[\|\vvareps_t\|^2]$.
This is a standard stability measure in many analyses of stochastic linear systems induced by SGD~\citep{wu2018sgdselects,ma2021linear,wuwangsu2022alignment,mulayoff2024exact}.
Our definition requires this quantity to remain bounded over time,
allowing persistent fluctuations without requiring individual trajectories
to converge. We discuss the relation to almost-sure convergence
in~\Cref{app:mean-square-vs-almost-sure}.

The next proposition gives the moment dynamics and a sufficient condition for mean-square stability.

\begin{proposition}
\label{prop:stochastic-quadratic-moment-dynamics}
For each fixed $\eta$, there exist a matrix $\mS$ and a vector $\vb_{\mathrm{quad}}$ that do not
depend on $t$ such that
\begin{align}
    \Ebb[\vvareps_{t+1}]
    &=
    (\mI-\eta\mA)\Ebb[\vvareps_t],
    \label{eq:first-moment-dynamics}
    \\
    \vectorize\!\left(
        \Ebb[\vvareps_{t+1}\vvareps_{t+1}^\top]
    \right)
    &=
    \Ebb[\mM_t\otimes\mM_t]
    \vectorize\!\left(
        \Ebb[\vvareps_t\vvareps_t^\top]
    \right)
    +\mS\Ebb[\vvareps_t]
    +\vb_{\mathrm{quad}}.
    \label{eq:second-moment-dynamics}
\end{align}
For every deterministic initialization, the affine system \eqref{eq:error-dynamics-stochastic-quadratic} has bounded second moments if
\begin{align*}
    \rho\!\left(\Ebb[\mM_t\otimes\mM_t]\right) < 1.
\end{align*}
\end{proposition}

The proof is given in \Cref{app:proof-moments}.
The above proposition implies that stochastic quadratics have two different instability modes.
The first is caused by the negative curvature of the population quadratic and is independent of the LR threshold.
The second is the usual positive-curvature instability, but in the stochastic setting its threshold is controlled by a second-moment operator rather than by $\lambda_{\max}(\mA)$ alone. The following arguments make the above ideas clear.

\subsection{Negative Curvature Divergence}\label{subsec:neg-divergence}
Suppose that $\lambda_{\min}(\mA)<0$, and let $\vu_{\min}$ be a corresponding unit eigenvector. Then the negative curvature makes the stochastic quadratic dynamics
unstable for every positive LR. Define the mean error in
this direction by $a_t    \coloneqq\vu_{\min}^{\top}\Ebb[\vvareps_t].$
Using the mean dynamics
$\Ebb[\vvareps_{t+1}]=(\mI-\eta\mA)\Ebb[\vvareps_t]$, we have
\begin{align*}
    a_{t+1}
    &=
    \vu_{\min}^{\top}
    (\mI-\eta\mA)
    \Ebb[\vvareps_t]
    =
    \left(
        1-\eta\lambda_{\min}(\mA)
    \right)a_t
    =
    \left(
        1+\eta|\lambda_{\min}(\mA)|
    \right)a_t.
\end{align*}
Therefore, we have that $a_t=
    \left(
        1+\eta|\lambda_{\min}(\mA)|
    \right)^t a_0.$
For any initialization with $a_0\neq 0$, this component grows without
bound for every $\eta>0$. Moreover,
\begin{align*}
    \Ebb\left[\|\vvareps_t\|^2\right]
    \geq
    \left\|\Ebb[\vvareps_t]\right\|^2
    \geq
    \left|
        \vu_{\min}^{\top}\Ebb[\vvareps_t]
    \right|^2
    =
    |a_t|^2.
\end{align*}
Hence the second moment is also unbounded.
This instability does not disappear when the LR is reduced. A smaller LR only slows down the escape along the negative curvature direction, while
the quadratic model remains unstable for every positive LR.

\subsection{Positive Curvature Divergence and Critical LR}
Next, we show that even if we factor out those directions causing negative divergence and focus on the locally convex subspace, the instability can also arise from mini-batch noise and admits a critical LR threshold.
\WCupdate{The most direct construction is to keep only the directions of positive}
\WCupdate{curvature, discarding the eigenvectors of $\mA$ with nonpositive eigenvalues.}
\WCupdate{On the remaining subspace the population quadratic is strongly convex.}

Let $\Pi^+$ be the orthogonal projection onto the eigenspace of $\mA$ with
positive eigenvalues. We apply this fixed projection at every step and
define $\mM_t^+\coloneqq\Pi^+\mM_t\Pi^+.$
The projected error dynamics are
\begin{align*}
    \vvareps_{t+1}^+
    =\mM_t^+\vvareps_t^+-\eta\Pi^+\vr_t.
\end{align*}
These dynamics keep the positive curvature part of the original stochastic quadratic while removing the negative curvature mode. We then define the critical LR for SGD using the mean-square stability of this projected system.

\begin{definition}
\label{def:stochastic-quadratic-critical-lr-general}
For the stochastic quadratic dynamics in \cref{eq:stochastic-quadratic-update}, let $\vvareps_t^+$ follow the projected error dynamics above from $\vvareps_0^+\coloneqq\Pi^+\vvareps_0$.
Define the critical LR on the positive curvature subspace by
\begin{align}
    \lr_\star
    \coloneqq
    \sup\left\{
        \eta>0:
        \sup_{t\geq 0}
        \Ebb\left[\|\vvareps_t^+\|^2\right]
        <\infty
        \text{ for every }\vvareps_0\in\R^d
    \right\}.
    \label{eq:stochastic-critical-lr-general}
\end{align}
\end{definition}
In the full batch case, $\mA_t\equiv\mA$ and $\vb_t\equiv\vb$, so
$\vr_t=\mA\vx_\star+\vb=0$ by stationarity of $\vx_\star$.
Since $\mA$ is symmetric, we can write the error in the eigenvector basis of $\mA$.
Projecting onto the positive curvature subspace simply means ignoring all
\WCupdate{coordinates associated with nonpositive eigenvalues. Along a remaining direction}
\WCupdate{with eigenvalue $\lambda>0$, the update multiplies the error by}
\WCupdate{$1-\eta\lambda$, so the error stays bounded when $|1-\eta\lambda|\leq 1$ and}
\WCupdate{decays when $|1-\eta\lambda|<1$.}
\WCupdate{Therefore, when $\lambda_{\max}(\mA)>0$, all remaining directions have bounded}
error exactly when
\begin{align*}
    0<\eta\leq\frac{2}{\lambda_{\max}(\mA)},
    \qquad
    \lr_\star=\frac{2}{\lambda_{\max}(\mA)}.
\end{align*}
The definition therefore recovers the full batch threshold in \Cref{subsec:full-batch-eos}.

For the stochastic case, the stability depends on more than the eigenvalues of $\mA$. For the projected transition,
\begin{align}
    \Ebb[\mM_t^+\otimes\mM_t^+]
    =
    \Ebb[\mM_t^+]\otimes\Ebb[\mM_t^+]
    +\operatorname{Cov}(\mM_t^+),
    \label{eq:second-moment-covariance-decomp}
\end{align}
where
\begin{align*}
    \operatorname{Cov}(\mM_t^+)
    \coloneqq
    \Ebb[\mM_t^+\otimes\mM_t^+]
    -\Ebb[\mM_t^+]\otimes\Ebb[\mM_t^+].
\end{align*}
The covariance term records variation in the Hessian across batches.
It can make the second moment unstable even when the mean transition is stable, and it can move the stability boundary away from $2/\lambda_{\max}(\mA)$.
This is what is missing from the full-batch EoS picture: stochastic EoS is not determined by $\lambda_{\max}$ alone, and before it can even be defined cleanly, the negative-curvature instability must be separated out.
Further stability analysis, including the treatment of zero curvature directions, is given in \Cref{app:critical}.

\section{Stability of Stochastic Gradient Descent on Neural Networks}
\label{sec:eoc-eos}\label{sec:main-results}\label{sec:methods}
We now apply the stochastic quadratic analysis in
\Cref{sec:stochastic-quadratic-stability} to neural network training with SGD. Fix a checkpoint $\param_{t_0}$. For a batch $\batch_t$,
we consider the quadratic surrogate
\begin{align}
    Q_{\batch_t}(\param)
    =
    \loss_{\batch_t}(\param_{t_0})
    +(\param-\param_{t_0})^\top g_{\batch_t}
    +\frac{1}{2}(\param-\param_{t_0})^\top
    \mH_{\batch_t}(\param-\param_{t_0}),
    \label{eq:stochastic-quadratic-surrogate}
\end{align}
where $g_{\batch_t}\coloneqq \nabla \loss_{\batch_t}(\param_{t_0})$  and $\mH_{\batch_t} \coloneqq \nabla^2 \loss_{\batch_t}(\param_{t_0})$. SGD with LR $\eta$ on this quadratic is given by
\begin{align}
    \param_{t+1}
    =
    \param_t-\eta\grad Q_{\batch_t}(\param_t).
    \label{eq:stochastic-quadratic-sgd-update}
\end{align}
As shown in \Cref{sec:stochastic-quadratic-stability}, applying SGD to the local quadratic surrogate gives a stochastic linear system. Therefore, the two instability mechanisms identified for stochastic quadratics also arise in the local quadratic surrogate of neural network (NN) training in \Cref{eq:stochastic-quadratic-sgd-update}.

First, we recall that $\mH_{t_0}
    \coloneqq
    \hess\loss(\param_{t_0})
    =
    \Ebb_{\batch_t}[\mH_{\batch_t}]$
is the population Hessian at step $t_0$. If $\lambda_{\min}(\mH_{t_0})<0$, the population quadratic $Q_t(\vw)=\Ebb[Q_{\batch_t}(\vw)]$ is unbounded below as $t \to \infty$, and its mean dynamics have an expanding direction for
every positive LR. This is the negative curvature failure studied in \Cref{subsec:neg-divergence}.

\WCupdate{Second, we remove the directions of nonpositive curvature by projecting every local update. Let $\Pi^+_{t_0}$ project onto the positive eigenspace of $\mH_{t_0}$, and let $\vvareps_t^+$ denote the displacement of this projected run from the checkpoint. The update is}
\[
 \vvareps_{t+1}^+
 =\Pi^+_{t_0}\bigl[(\mI-\eta\mH_{\batch_t})\vvareps_t^+
                   -\eta\grad\loss_{\batch_t}(\param_{t_0})\bigr],
 \qquad \vvareps_0^+\in\operatorname{range}(\Pi^+_{t_0}).
\]
As in \Cref{def:stochastic-quadratic-critical-lr-general}, $\lr_\star(\param_{t_0})$ is the supremum of positive LRs for which $\sup_t\Ebb\|\vvareps_t^+\|^2<\infty$ for every deterministic initial displacement in this subspace.

For NN training, once we replace the actual loss with $Q_{\batch_t}(\param)$, it is typically observed that the linear system formed by quadratic approximation diverges quickly, while the true dynamics of the NN remain stable (\Cref{fig:fig1-EoC},~\Cref{fig:fig2-EoS}). So a natural question is: which regime leads to the failure of quadratic approximation in NN training, negative or positive curvature divergence? We identify: NN training sits on an \emph{edge} where either divergence mechanism is \emph{just at the threshold} of breaking the local quadratic approximation.

\subsection{Edge of Convexity}
\label{subsec:eoc-definition}\label{sec:eoc-main}
\paragraph{Negative curvature makes the quadratic model unstable.}
Negative Hessian eigenvalues in NNs have been widely reported in
prior works~\citep{sagun2016eigenvalues,sagun2017empirical,alain2019negative}.
In this paper, we show that they can make the quadratic model diverge while NN training remains stable (\Cref{fig:fig1-EoC}). If $\lambda_{\min}(\mH_{t_0})<0$, SGD on the quadratic model in \cref{eq:stochastic-quadratic-surrogate} has an exploding direction for every $\eta>0$ (\Cref{subsec:neg-divergence}). A smaller LR slows this explosion but cannot prevent it.
We first measure the size of the negative curvature relative to the positive curvature at the same checkpoint.
\begin{definition}[Saddleness]
\label{def:relative-negative-curvature}
Suppose that $\lambda_{\max}(\mH_{t_0})>0$.
We define the saddleness at $\param_{t_0}$ by
\begin{align}
    \nu(\param_{t_0})
    \coloneqq
    \frac{\max\{0,-\lambda_{\min}(\mH_{t_0})\}}
    {\lambda_{\max}(\mH_{t_0})}.
    \label{eq:relative-negative-curvature}
\end{align}
\end{definition}

This ratio measures the negative curvature relative to the largest positive curvature. For any symmetric matrix $\mH$, $\inf_{\mH_+\succeq 0}\|\mH-\mH_+\|=\max\{0,-\lambda_{\min}(\mH)\}.$
Thus the numerator in \cref{eq:relative-negative-curvature} is the spectral norm distance from the Hessian to the positive semidefinite cone.
Dividing by $\lambda_{\max}$ removes the overall scale of the Hessian. When $\nu$ is small, the Hessian is close to positive semidefinite
relative to its largest positive curvature.
It is still nonconvex whenever $\lambda_{\min}<0$.

For the saddleness defined above, one observation is that for LLM pre-training, after the warm-up phase, $\nu \le 0.1$, which leads to the first edge area of LLM pre-training.
\begin{defbox}
\textbf{Edge of Convexity (EoC) at $\vw_{t_0}$.}
At a checkpoint $\param_{t_0}$, training is at EoC when $\lambda_{\min}(\mH_{t_0})<0$, but saddleness $\nu(\param_{t_0})$ itself is small.
\end{defbox}

EoC therefore describes a particular failure of the quadratic approximation. The Hessian is close to positive semidefinite in relative scale, but its negative curvature is enough to make the quadratic model unstable for every positive LR.

\paragraph{Measuring population negative curvature during LLM pretraining.}
Early studies on small image and synthetic datasets reported negative Hessian eigenvalues much smaller in magnitude than the largest positive eigenvalue~\citep{sagun2016eigenvalues,sagun2017empirical,alain2019negative}.
In this paper, we extend these measurements to the population Hessian throughout LLM pretraining (\Cref{fig-small-medium-large}).
Doing so requires $\lambda_{\min}(\mH_{t_0})$.
The usual approach at scale is to run Lanczos on one mini-batch Hessian~\citep{ghorbani2019investigation,yao2020pyhessian,zhang2024transformers}.
This gives $\lambda_{\min}(\mH_{\batch})$. And
since $\lambda_{\min}(\cdot)$ is concave, Jensen's inequality gives $\Ebb[\lambda_{\min}(\mH_{\batch})]\le\lambda_{\min}(\mH_{t_0})$, so the mini-batch value is biased below the population value. This downward bias can be substantial in practice. To this end, we propose a stochastic power iteration for the estimate of $\lambda_{\min}(\mH_{t_0})$, which further gives the systematic detection of EoC.  The complete EoC detection procedure is given in \Cref{sec:method-eoc}, along with the theoretical guarantee in \Cref{app:proof-of-min-eigen}.

\subsection{Edge of Stability under Stochastic Training}
\label{subsec:stochastic-eos-definition}
Even after removing the effect of EoC (i.e., restricting to the positive-curvature subspace in \Cref{def:stochastic-quadratic-critical-lr-general}), the quadratic approximation can still diverge. This divergence is caused by positive curvature when the LR exceeds the critical threshold $\lr_\star(\param_{t_0})$. More subtly, neural network training often operates near this boundary: a tiny positive increase in the LR can already trigger divergence.

\begin{defbox}
\textbf{Stochastic Edge of Stability (EoS) at $\param_{t_0}$.}
Training is at EoS when the current LR is close to the stochastic critical LR on the positive-curvature subspace, $\eta\approx \lr_\star(\param_{t_0}).$
\end{defbox}
\WCupdate{The mean curvature is positive definite on this subspace, so the LRs with bounded second moments have a single critical endpoint (\Cref{app:critical-lr-agm}). In the full-batch case at a stationary point, this critical LR is $2/\lambda_{\max}(\mH_{t_0})$ whenever $\lambda_{\max}(\mH_{t_0})>0$, as computed after \Cref{def:stochastic-quadratic-critical-lr-general}, which recovers the full-batch EoS condition of \citet{cohen2021gradient}. See \Cref{sec:related-work} for comparison with previous work on stochastic EoS.}

As discussed in \Cref{sec:stochastic-quadratic-stability}, for the stochastic case, even after projection, mean-square stability is governed by the second-moment operator $\Ebb[\mM_t^+\otimes\mM_t^+]
    =
    \Ebb[\mM_t^+]\otimes\Ebb[\mM_t^+]
    +
    \mathrm{Cov}(\mM_t^+).$
The covariance term can shift the critical LR away from $2/\lambda_{\max}$. Moreover, directly solving $\rho\!\left(\Ebb[\mM_t^+\otimes\mM_t^+]\right)<1$ is infeasible at neural network scale, since $\mM_t^+\otimes\mM_t^+$ is a fourth-order object and its entries are affected by noisy mini-batch Hessian samples.

Thus, the definitions above reveal two practical obstacles.
First, EoC and EoS live on different curvature subspaces, but explicitly separating these subspaces is infeasible for large NNs, especially for LLM pretraining.
Second, the stochastic EoS threshold is not only determined by $\lambda_{\max}$ alone but also correlated with the stochastic noise in training, so we cannot detect it by simply checking whether $\eta\lambda_{\max}\approx 2$.

\paragraph{EoS detection.}
\WCupdate{We use a \textit{local simulation} to detect EoS in pretraining. First, to compensate for negative curvature, we add $\tfrac{\extraWD}{2}\|\param-\param_{t_0}\|^2$ to the local quadratic surrogate \eqref{eq:stochastic-quadratic-surrogate}, with $\extraWD=\max\{0,-\widehat\lambda_{\min}\}$, where $\widehat\lambda_{\min}$ is the estimate of the minimum eigenvalue from \Cref{sec:method-eoc}. We then simulate SGD on this loss from the checkpoint for a fixed number of steps over a grid of LRs around the training LR $\lr$, using the same batch sequence for every run. The estimate $\widehat\lr_\star$ is the smallest LR in the grid for which the displacement norm $\|\vw -\vw_{t_0}\|$ reaches a prescribed cutoff. We declare EoS when $\widehat\lr_\star$ is close to $\lr$. The complete procedure is given in \Cref{sec:method-eos} and theoretical analysis is in \Cref{app:statements}.}

\paragraph{Illustration on CIFAR-10.}
\Cref{fig:cifar10-sgd-eos} applies both detectors to ResNet-20 and ResNet-32 trained with SGD.
Each column is a checkpoint, each row an LR factor, and color encodes the final displacement norm of the local simulation.
Across most checkpoints in the middle phase of training, the bounded-to-exploding boundary lies at LR factor $1.0$ (white dashed line).
The training LR is close to $\lr_\star$, and SGD training operates at the EoS for a large part of training.
\Cref{sec:experiments} applies the same protocol to LLM pretraining with AdamW, with definitions and detectors extended to adaptive methods below.
\subsection{Adaptive Gradient Methods at the Edge of Convexity and Edge of Stability}
\label{sec:agm}

So far, the discussion has focused on SGD. Modern LLMs are usually trained with adaptive methods such as AdamW~\citep{loshchilov2017decoupled}.
We now consider a general class of adaptive gradient methods~\citep{duchi2011adaptive,li2025adam} as
\begingroup
\renewcommand{\theHequation}{AGM}
\begin{equation}
\tag{AGM}
\begin{aligned}
    \vm_{t+1}
    &=
    \beta_1\vm_t
    +(1-\beta_1)\vg_{\batch_t}(\param_t),
    \\
    \mP_{t+1}
    &=
    \Psi(
        \mP_t,
        \vg_{\batch_t}(\param_t),
        \param_t
    ),
    \\
    \param_{t+1}
    &=
    \param_t
    -\eta\mP_{t+1}^{-1}\vm_{t+1}
    -\eta\lambda_{\rm WD}\param_t,
\end{aligned}
\label{eq:agm-main}
\end{equation}
\endgroup
where $0\leq\beta_1<1$, $\vm_t$ is the momentum,
$\mP_t\succ0$ is the preconditioner matrix, $\Psi(\cdot)$ is the preconditioner update rule, and $\lambda_{\rm WD}$ is the decoupled weight decay coefficient. AdamW, momentum SGD, and SGD are special cases of this update.

At checkpoint $\param_{t_0}$, we freeze $\mP\coloneqq\mP_{t_0}$ and use the local quadratic surrogate \eqref{eq:stochastic-quadratic-surrogate}.
The effective curvature and its saddleness are $\mA\coloneqq\mP^{-1/2}\mH_{t_0}\mP^{-1/2}+\wdc\mI,\nu\coloneqq \max\{0,-\bar\mu_{\min}\}/\bar\mu_{\max},$
where $\bar\mu_{\min}=\lambda_{\min}(\mA)$ and $\bar\mu_{\max}=\lambda_{\max}(\mA)>0$.
Negative effective curvature gives an expanding mean direction for every positive LR, even with momentum (\Cref{app:agm-EoC}).

For EoS, we project both momentum and displacement onto the strictly positive eigenspace of $\mA$ in preconditioned coordinates at every local update.
As for SGD, the critical LR is the supremum of LRs for which these projected dynamics are mean-square stable.
\Cref{prop:agm-critical-lr-equivalence} shows that this threshold is positive and finite, and gives an equivalent spectral characterization.
\begin{defbox}
\textbf{Edge of Convexity (EoC) for AGMs.}
At checkpoint $\param_{t_0}$, AGM training is at EoC when $\bar\mu_{\min}<0$ and $\nu$ is small.
\end{defbox}
\begin{defbox}
\textbf{Edge of Stability (EoS) for AGMs.}
At checkpoint $\param_{t_0}$, AGM training is at EoS when its current LR is close to the unique critical LR of the projected local dynamics, $\lr\approx\lr_\star$.
\end{defbox} The local model and the definitions of AGM saddleness and critical LR $\lr_\star$ are given in \Cref{app:agm-derivation}. The detailed discussion of detecting EoC and EoS for AGMs can be found in \Cref{sec:methods-detailed}.

\section{Experiments}
\label{sec:experiments}
We apply our detection methods to study EoS and EoC during language model pretraining. Our experiments examine how broadly these phenomena occur and how EoS depends on the training configurations like batch size and learning rate schedule.

\begin{figure*}[t]
\vspace{-0.3in}
	\centering
	\begin{subfigure}[t]{0.33\textwidth}
		\centering
		\includegraphics[width=\linewidth]{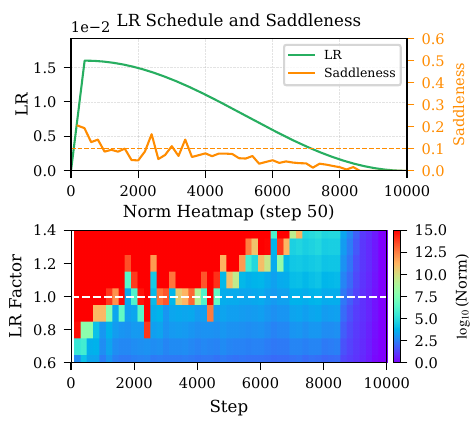}
		\caption{Cosine schedule.}
		\label{fig-sc:cos}
	\end{subfigure}\hfill
	\begin{subfigure}[t]{0.33\textwidth}
		\centering
		\includegraphics[width=\linewidth]{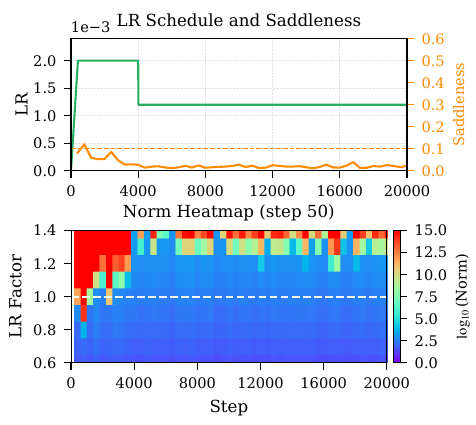}
		\caption{Multistep schedule I.}
		\label{fig-sc:ms}
    \end{subfigure}\hfill
    \begin{subfigure}[t]{0.33\textwidth}
		\centering
		\includegraphics[width=\linewidth]{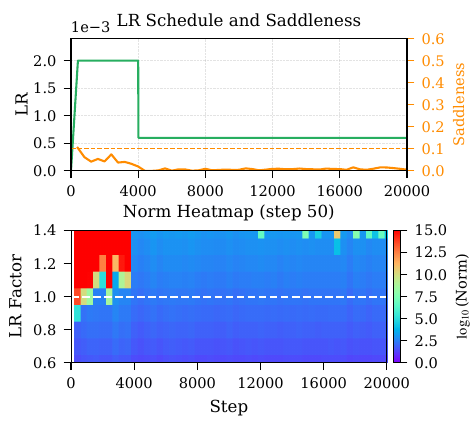}
		\caption{Multistep schedule II.}
		\label{fig-sc:ms-II}
    \end{subfigure}
    \vspace{-0.1in}
    \caption{\textbf{Training may leave the EoS regime after LR decay.} In (a) and (b), the transition from bounded displacement to explosion moves to larger LR factors during cosine decay and after the abrupt multistep LR drop, respectively. After the larger LR drop in (c), the local simulations remain bounded at the training LR for the rest of the observed run, where training enters the stable regime and does not come back to EoS. See \Cref{subsec:exp-lr-eos,appendix:exp:LRdrop,appendix:exp:LRdecay} for further discussion.}
	\label{fig-sc}
\end{figure*}

\begin{figure*}[ht]
	\centering
	\begin{subfigure}[t]{0.45\textwidth}
		\centering
		\includegraphics[width=\linewidth]{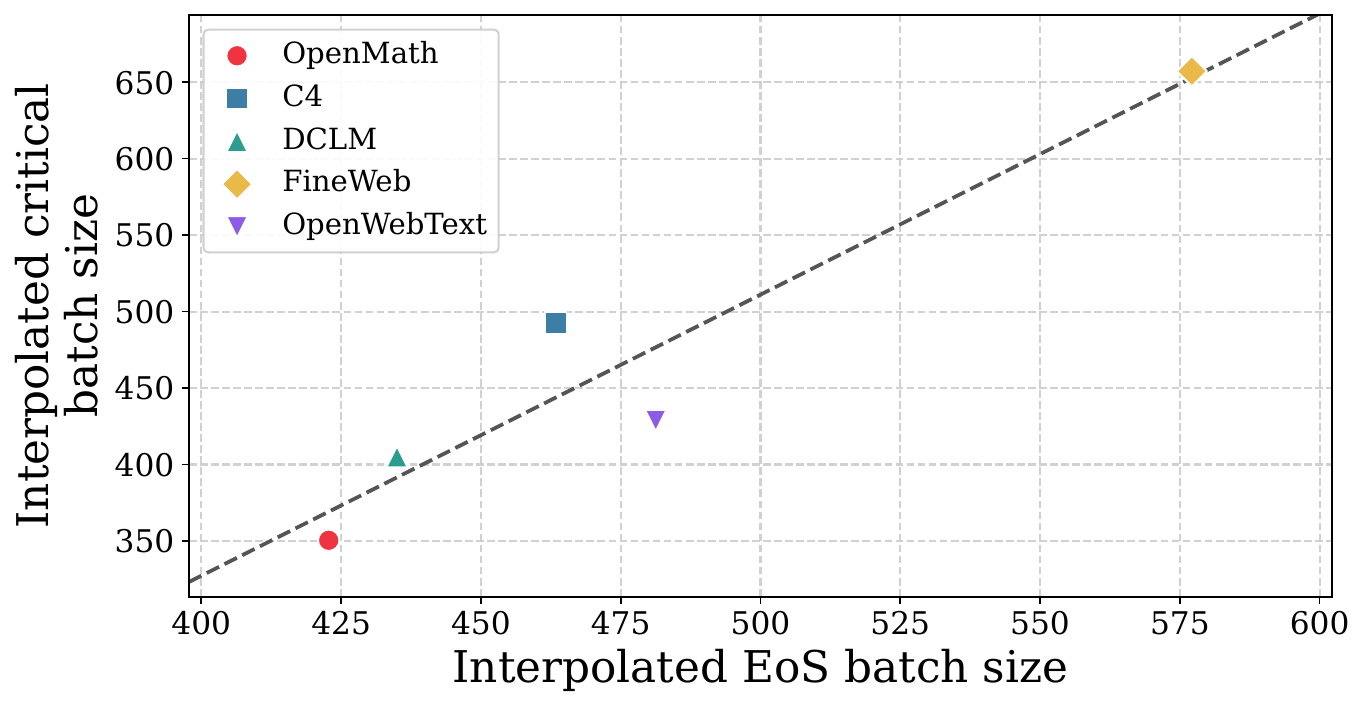}
		\caption{EoS batch size threshold vs.\ critical batch size.}
		\label{fig-ebs:comp}
	\end{subfigure}\hfill
	\begin{subfigure}[t]{0.45\textwidth}
		\centering
		\includegraphics[width=\linewidth]{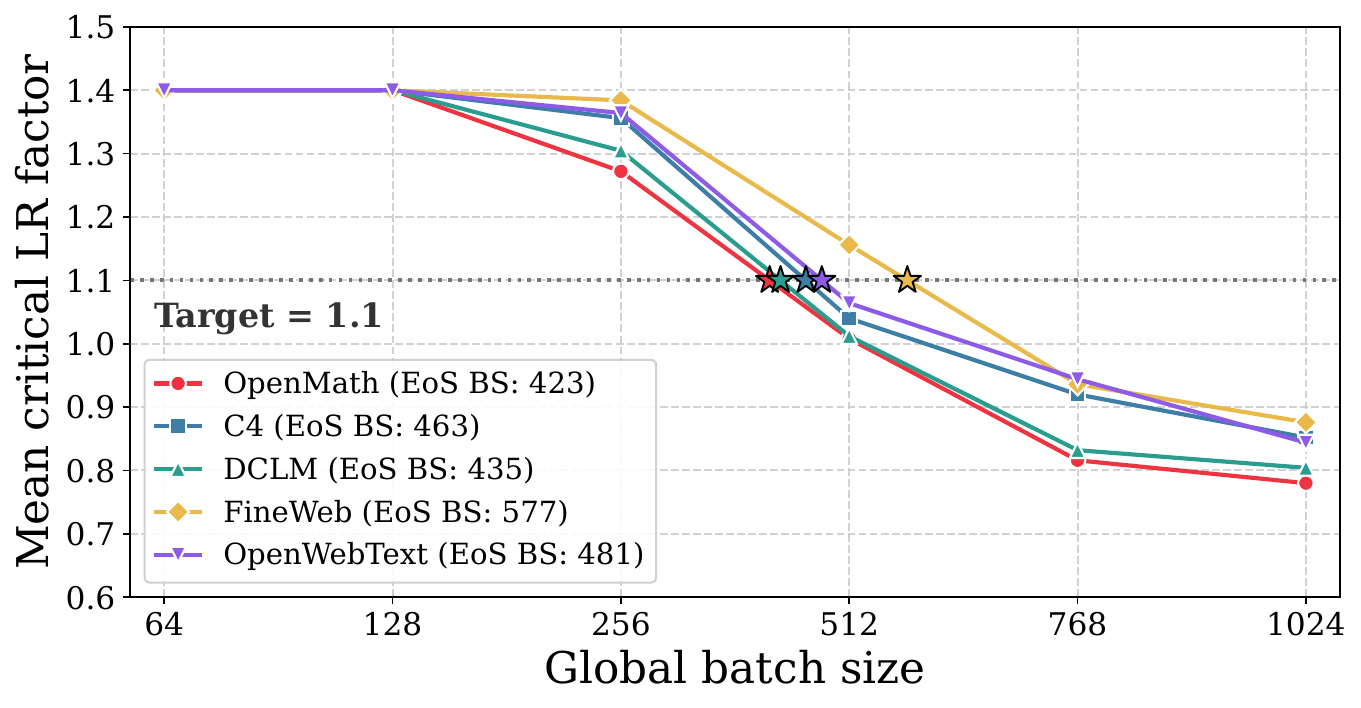}
		\caption{Critical LR factor $q$ vs.\ batch size.}
		\label{fig-ebs:ebs}
	\end{subfigure}\hfill
    \vspace{-0.1in}
    \caption{\textbf{The EoS batch size (EBZ) and the critical batch size (CBZ) show a positive correlation across datasets.} (a) Larger EBZ generally corresponds to larger CBZ, and the dashed line is a linear fit. (b) The critical LR factor $q=\widehat\lr_\star/\lr$, averaged over checkpoints, decreases as batch size increases. Stars mark the interpolated crossings of $q=1.1$, which define EBZ. See \Cref{subsec:exp-bsz-threshold,appendix:exp:ebs} for the experimental protocol and threshold estimation.}
	\label{fig-ebs}
\end{figure*}

\subsection{Experimental Setup}
We pretrain GPT-style models at four scales ranging from 43M to 787M. Unless otherwise specified, runs use FineWeb~\citep{penedo2024fineweb}, AdamW ($\mathrm{WD}=0.1$, $\beta_1=0.9$, $\beta_2=0.95$), sequence length 1024, and a Warm--Stable--Decay (WSD) LR schedule~\citep{hu2024minicpmunveilingpotentialsmall} with 4\% warmup steps and 20\% decay steps. Full details are provided in \Cref{app:add-exp-setup}.

\subsection{EoS and EoC Are Widely Observed}
\label{subsec:exp-scales}

We evaluate EoS and EoC at over $3{,}000$ checkpoints from more than $100$ pretraining runs spanning different model sizes, peak LRs, batch sizes, and shapes of LR schedules. Both phenomena are widely observed across our experiments (\Cref{fig-small-medium-large,fig:lr-sweep-summary,fig-lr,fig-bs,fig-sc}). At many checkpoints, the estimated critical LR is close to the training LR, while the estimated minimum eigenvalue is negative and saddleness is below $0.1$, consistent with EoS and EoC, respectively. Additional results and discussion are given in \Cref{app:additional-results,appendix:exp:ebs,appendix:exp:LRdrop,appendix:exp:LRdecay}.
We next examine how these phenomena depend on the training configuration.
Specifically, we focus on two key training hyperparameters: batch size and the LR schedule. Note that EoC remains widely observed as these hyperparameters vary, so the following analysis focuses primarily on how they affect EoS.

\begin{figure*}[t]
    \centering
    \begin{subfigure}[t]{0.48\textwidth}
        \centering
        \includegraphics[width=\linewidth]{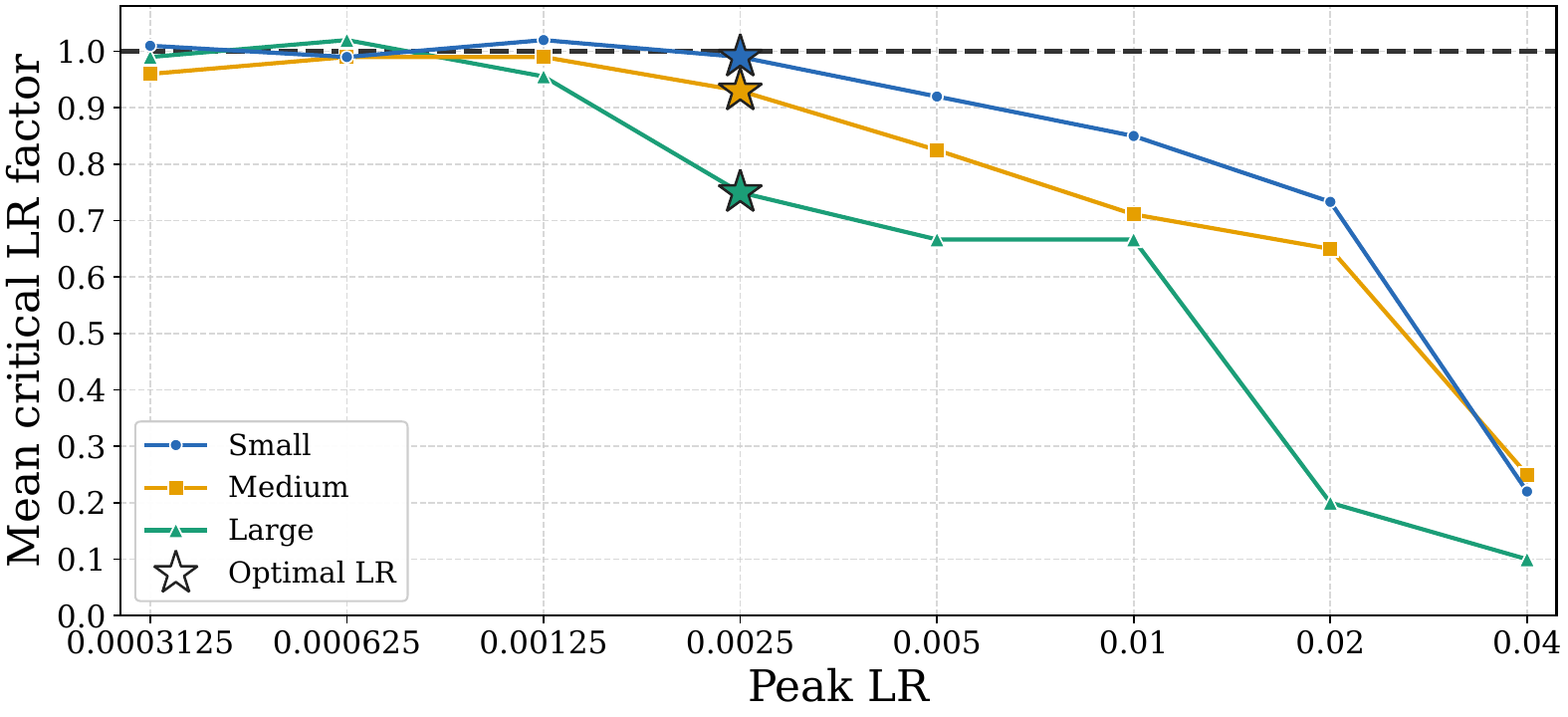}
        \vspace{-0.15in}
        \caption{Mean critical LR factor.}
        \label{fig:lr-sweep-critical-factor}
    \end{subfigure}\hfill
    \begin{subfigure}[t]{0.48\textwidth}
        \centering
        \includegraphics[width=\linewidth]{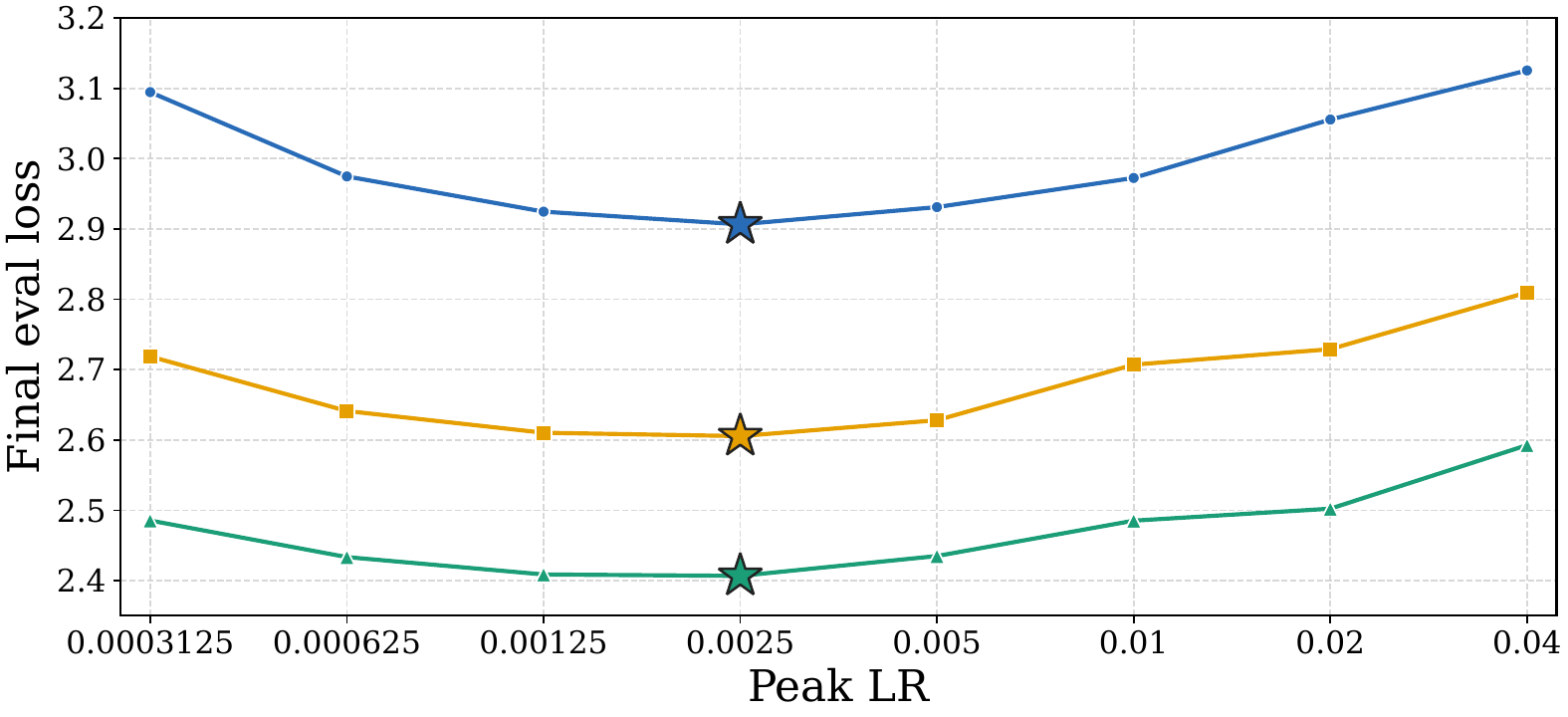}
        \vspace{-0.15in}
        \caption{Final eval loss.}
        \label{fig:lr-sweep-mu}
    \end{subfigure}
\vspace{-0.1in}
\caption{\textbf{Training may go beyond EoS and enter an extremely unstable regime with large LR.} (a) As the peak LR increases, the critical LR factor $q=\widehat\lr_\star/\lr$, averaged over the tested checkpoints, falls from near the reference line at $1.0$ to well below it. Values near $1.0$ indicate EoS, while smaller values mean that the training LR exceeds the estimated critical LR. (b) These large peak LRs also yield higher final evaluation loss.
See \Cref{subsec:exp-scales,subsec:exp-lr-eos} for further discussion.}
    \label{fig:lr-sweep-summary}
    \vspace{-0.2in}
\end{figure*}

\subsection{Batch size dependence of EoS}\label{subsec:exp-bsz-threshold}
For the dependence of batch size, we first show that EoS emerges above a threshold of batch size, which is referred to as the \textbf{EoS Batch Size (EBZ)}. We then demonstrate an interesting correlation between EBZ and the critical batch size (CBZ) across different datasets.
\paragraph{EoS emerges above the EBZ threshold.}
\Cref{fig-bs} shows that EoS is observed over a larger fraction of training at $B=512$ and $B=1024$ than at $B=256$.
For the two larger batch sizes, the transition from bounded displacement to explosion stays near LR factor $1.0$ over much of the stable phase, indicating that the training LR is close to the estimated critical LR.
Below this batch size threshold, the compensated local simulations at checkpoints after the early phase of training remain bounded at every tested LR factor.
These runs use GPT-extra-small on FineWeb with WSD schedules and adjust the LR, weight decay, and $\beta_2$ according to the batch size scaling rule in \Cref{appendix:exp:ebs}.
Further, across all five datasets, \Cref{fig-ebs:ebs} shows that the critical LR factor, averaged over the tested checkpoints of each run, decreases as batch size increases. We estimate EBZ as the batch size at which the fitted mean critical LR factor first decreases to $1.1$ as batch size increases. This corresponds to the crossing at $1.1$ marked in the figure.
\paragraph{EBZ and CBZ show a positive correlation across datasets.}
In \Cref{fig-ebs:comp}, each point represents one dataset, and the points follow an overall upward linear trend. FineWeb has the largest EBZ and CBZ, while OpenWebMath has the smallest values of both. Here, CBZ is estimated from the same training runs as EBZ and denotes the largest batch size whose final evaluation loss is within $0.3\%$ of the best loss in that dataset's sweep. Thus, datasets with a larger threshold for EoS also tend to maintain near optimal final evaluation loss at larger batch sizes. Full sweep and estimation details are given in \Cref{appendix:exp:ebs}.

\subsection{LR Dependence of EoS}
\label{subsec:exp-lr-eos}
We next show that LR decay can cause training to leave EoS, and that training may not return to the EoS regime after a sufficiently large LR decay. We also show that training at excessively large peak LRs can enter an extremely unstable regime beyond EoS.
\paragraph{Training can enter the stable regime after LR decay.}
Before LR decay, training often operates at EoS, with the training LR close to the estimated critical LR (\Cref{fig-sc,fig:ms-ablation,fig:decay-ablation}).
After the abrupt LR drop in \Cref{fig-sc:ms}, the transition from bounded displacement to explosion moves to larger LR factors. The estimated critical LR is now higher than the reduced training LR, and the compensated local simulation stays bounded at the training LR over the tested horizon.
As training continues, the estimated critical LR can fall back toward the current training LR and EoS can return, as in \Cref{fig:ms-lr0008}. With a sufficiently large LR reduction, however, the estimated critical LR stays well above the training LR for the rest of the observed run (\Cref{fig-sc:ms-II,fig:ms-tau03,fig:decay-y-02}). The compensated local simulations then remain bounded at the training LR, and EoS does not return.
The effects of the drop magnitude and decay duration are discussed in \Cref{appendix:exp:LRdrop,appendix:exp:LRdecay}.

\paragraph{Training can enter an extremely unstable regime with large LR.}
Training with peak LRs substantially above the empirically optimal range can yield higher loss and enter an \textit{extremely unstable regime} beyond EoS (\Cref{fig:lr-sweep-summary}). In this regime, the compensated local quadratic dynamics still allow a critical LR, but it lies far below the training LR. For example, at peak LR $0.04$, the estimated critical LR is only about one tenth of the training LR on average, and the compensated local simulations explode at the training LR. See \Cref{app:additional-results} for further discussion and \Cref{fig:tpp20-small,fig:tpp20-medium,fig:tpp20-large} for detailed diagnostics across model sizes and peak LRs.

\section{Conclusion}
This paper extends the definition of EoS to the LLM training under a unified framework, together with another clearly clarified phenomenon EoC, accounting for the two failure modes of the quadratic approximation. %
Several directions remain open for future study. (i) Our experiments mainly focus on SGD and AdamW. It is interesting to see how EoS would behave when training with other optimizers like Muon. (ii) Fully understanding EoS and EoC also calls for solid theoretical analyses of the realistic neural network training dynamics.

\bibliography{ref}
\bibliographystyle{iclr2027_conference}

\newpage

\tableofcontents

\newpage

\appendix
\crefalias{section}{appendix}
\crefalias{subsection}{subappendix}
\crefalias{subsubsection}{subsubappendix}

\section{Detecting EoC and EoS for Adaptive Gradient Methods}\label{sec:methods-detailed}
Here we provide detailed detection methods of EoC and EoS for AGMs. We first define the local AGM model, its saddleness, and its critical LR in \Cref{app:agm-derivation}.
\Cref{sec:method-eoc,sec:method-eos} then describe the EoC and EoS detection procedures and their guarantees.
Both procedures use Hessian-vector products without forming the Hessian.
\Cref{app:hvp-cost} gives their computational cost.
The complete theoretical statements and proofs are in \Cref{app:statements,app:proofs-moments,app:proof-power-iteration,fh:sec:appendix}.

\subsection{Adaptive Gradient Methods}\label{app:agm-derivation}
We consider adaptive gradient methods (AGMs) with a momentum vector $\vm_t$ and a positive definite preconditioner $\mP_t$.
For the mini-batch loss $\loss_{\batch_t}$, the updates are
\begin{align*}
 \vm_{t+1}
 &=\beta_1\vm_t+(1-\beta_1)\grad\loss_{\batch_t}(\param_t),\\
 \mP_{t+1}
 &=\Psi\bigl(\mP_t,\grad\loss_{\batch_t}(\param_t),\param_t\bigr),\\
 \param_{t+1}
 &=\param_t-\lr\mP_{t+1}^{-1}\vm_{t+1}-\lr\wdc\param_t.
\end{align*}
Here $\lr>0$ is the LR, $0\le\beta_1<1$ is the momentum coefficient, and $\wdc\ge0$ is the decoupled weight decay coefficient.
The preconditioner update $\Psi$ preserves positive definiteness.

\paragraph{The local quadratic model.}
Fix a checkpoint $\param_{t_0}$.
For each mini-batch $\batch_t$, we approximate its loss $\loss_{\batch_t}$ by the second-order Taylor expansion around this checkpoint.
Write $\vvareps=\param-\param_{t_0}$ for the displacement, and define the gradient and Hessian at the checkpoint by
\[
 \vg_{\batch_t}\coloneqq\grad\loss_{\batch_t}(\param_{t_0}),
 \qquad
 \mH_{\batch_t}\coloneqq\hess\loss_{\batch_t}(\param_{t_0}).
\]
The resulting \emph{quadratic surrogate} is
\begin{equation}\label{eq:agm-quadratic-surrogate}
 Q_{\batch_t}(\param_{t_0}+\vvareps)
 \coloneqq\loss_{\batch_t}(\param_{t_0})
 +\vg_{\batch_t}^{\top}\vvareps
 +\tfrac12\vvareps^{\top}\mH_{\batch_t}\vvareps.
\end{equation}
This surrogate matches the loss, gradient, and Hessian at the checkpoint and omits terms of higher order in the displacement.
At each local step, we sample a fresh batch and use its surrogate in place of the neural network loss.
The coefficients $\vg_{\batch_t}$ and $\mH_{\batch_t}$ are always evaluated at $\param_{t_0}$.

\paragraph{Local AGM dynamics.}
We freeze the preconditioner at $\mP\coloneqq\mP_{t_0}\succ0$ and let the momentum evolve under the surrogate gradients.
Let $\vvareps_t$ be the displacement after $t$ local steps.
The surrogate gradient at this displacement is $\vg_{\batch_t}+\mH_{\batch_t}\vvareps_t$.
Keeping $\lr$, $\beta_1$ and $\wdc$ fixed, the local updates become
\begin{equation}\label{eq:agm-checkpoint-dynamics}
\begin{aligned}
 \vm_{t+1}
 &=\beta_1\vm_t+(1-\beta_1)(\vg_{\batch_t}+\mH_{\batch_t}\vvareps_t),\\
 \vvareps_{t+1}
 &=\vvareps_t-\lr\mP^{-1}\vm_{t+1}
   -\lr\wdc(\param_{t_0}+\vvareps_t).
\end{aligned}
\end{equation}
The simulation starts at $\vvareps_0=\vzero$ with the checkpoint momentum.
Weight decay acts on the full parameter $\param_{t_0}+\vvareps_t$.
The choice $\beta_1=0$, $\mP=\mI$, and $\wdc=0$ recovers SGD.

In preconditioned coordinates, we stack the momentum and displacement as
\[
 \vz_t\coloneqq(\lr\mP^{-1/2}\vm_t,\mP^{1/2}\vvareps_t)^\top.
\]
Then \eqref{eq:agm-checkpoint-dynamics} has the affine form
\begin{equation}\label{eq:agm-local-linear-system}
 \vz_{t+1}=\mM_t(\lr)\vz_t+\vq_t.
\end{equation}
Write $\mD_{\batch_t}=\mP^{-1/2}\mH_{\batch_t}\mP^{-1/2}$ for the preconditioned mini-batch Hessian.
The transition matrix and forcing are
\[
\begin{aligned}
 \mM_t(\lr)
    &=
    \begin{bmatrix}
   \beta_1\mI & (1-\beta_1)\lr\mD_{\batch_t}\\
   -\beta_1\mI & (1-\lr\wdc)\mI-(1-\beta_1)\lr\mD_{\batch_t}
 \end{bmatrix},\\[3pt]
 \vq_t
    &=
    \begin{bmatrix}
   (1-\beta_1)\lr\mP^{-1/2}\vg_{\batch_t}\\
   -(1-\beta_1)\lr\mP^{-1/2}\vg_{\batch_t}
     -\lr\wdc\mP^{1/2}\param_{t_0}
 \end{bmatrix}.
\end{aligned}
\]
Two trajectories driven by the same batches have the same forcing $\vq_t$.
Their difference therefore evolves by multiplication with $\mM_t(\lr)$.
We use this homogeneous evolution to define stability of local perturbations.

\paragraph{Effective curvature and saddleness.}
Let $\mH_{t_0}\coloneqq\hess\loss(\param_{t_0})$ be the population Hessian.
The preconditioned Hessians and effective curvatures are
\begin{equation}\label{eq:agm-effective-curvature}
\begin{aligned}
 \mD_{t_0}&\coloneqq\mP^{-1/2}\mH_{t_0}\mP^{-1/2},
 &\mA&\coloneqq\mD_{t_0}+\wdc\mI,\\
 \mD_{\batch_t}&\coloneqq\mP^{-1/2}\mH_{\batch_t}\mP^{-1/2},
 &\mA_{\batch_t}&\coloneqq\mD_{\batch_t}+\wdc\mI.
\end{aligned}
\end{equation}
If $\lambda_{\min}(\mA)<0$, some local perturbations grow in mean at every positive LR, even with momentum (\Cref{app:agm-EoC}).
We measure the size of this negative effective curvature relative to the largest positive effective curvature.

\begin{definition}[Saddleness for AGMs]\label{def:agm-saddleness}
Suppose $\lambda_{\max}(\mA)>0$.
The AGM saddleness at the checkpoint is
\begin{equation}\label{eq:agm-saddleness}
 \nu
 \coloneqq\frac{\max\{0,-\lambda_{\min}(\mA)\}}{\lambda_{\max}(\mA)}.
\end{equation}
\end{definition}
The preconditioner and weight decay are included through $\mA$.
AGM training is at EoC when $\lambda_{\min}(\mA)<0$ and this ratio is small.
For SGD without weight decay, $\mA=\mH_{t_0}$, so this definition reduces to \Cref{def:relative-negative-curvature}.

\paragraph{Critical LR.}
To isolate instability caused by the LR, we remove the directions of nonpositive effective curvature.
Let $\Pi^+$ project onto the strictly positive eigenspace of $\mA$, assumed to be nonempty.
We project both momentum and displacement in preconditioned coordinates at every update.
\HWupdate{The effective curvature after the projection is therefore positive definite on this subspace.}
Write $\mM_t^+(\lr)$ and $\vq_t^+$ for the transition and forcing in coordinates on this subspace.
The projected local system is
\begin{equation}\label{eq:agm-projected-local-system}
 \vz_{t+1}^+=\mM_t^+(\lr)\vz_t^++\vq_t^+.
\end{equation}
Following \Cref{def:mean-square-stability}, we call this system mean-square stable if
\[
 \sup_{t\ge0}\Ebb\|\vz_t^+\|^2<\infty
 \qquad\text{for every deterministic initial state }\vz_0^+.
\]
The initial state includes both momentum and displacement.
This definition allows persistent fluctuations and does not require the second moment to tend to zero.

\begin{definition}[Critical LR for AGMs]\label{def:agm-critical-lr}
Keep the checkpoint, $\mP$, $\beta_1$, $\wdc$, and the batch distribution fixed as the LR varies.
The AGM critical LR is
\begin{equation}\label{eq:agm-critical-lr-definition}
 \lr_\star\coloneqq
 \sup\bigl\{\lr>0:\text{the projected system \eqref{eq:agm-projected-local-system} is mean-square stable}\bigr\}.
\end{equation}
\end{definition}

The next proposition gives an equivalent way to determine this critical LR from the homogeneous transition.
Its second-moment matrix is
\begin{equation}\label{eq:agm-local-second-moment-operator}
 \mK^+(\lr)\coloneqq\Ebb[\mM_t^+(\lr)\otimes\mM_t^+(\lr)].
\end{equation}
We assume that the pairs $(\mH_{\batch_t},\vg_{\batch_t})$ are i.i.d.\ and independent of the initial state.
Both entries have finite second moments, each Hessian is symmetric, and their means are the population derivatives at the checkpoint.
These are the mini-batch assumptions of \Cref{ass:minibatch}.

\begin{proposition}[Spectral characterization of the AGM critical LR]\label{prop:agm-critical-lr-equivalence}
Under \Cref{ass:minibatch}, fix $\mP\succ0$, $0\le\beta_1<1$ and $\wdc\ge0$.
Suppose the strictly positive eigenspace of $\mA$ is nonempty.
The critical LR in \Cref{def:agm-critical-lr} satisfies
\begin{equation}\label{eq:agm-critical-lr-spectral}
 0<\lr_\star=\sup\bigl\{\lr>0:\rho(\mK^+(\lr))<1\bigr\}<\infty.
\end{equation}
Moreover,
\begin{equation}\label{eq:agm-critical-lr}
 \rho\bigl(\mK^+(\lr)\bigr)
 \begin{cases}
  <1,&0<\lr<\lr_\star,\\
  =1,&\lr=\lr_\star,\\
  >1,&\lr>\lr_\star.
 \end{cases}
\end{equation}
\end{proposition}
Below $\lr_\star$, the projected affine system is mean-square stable.
Above $\lr_\star$, some deterministic initialization has an unbounded second moment.
At $\lr_\star$, boundedness can depend on the forcing.
Thus the proposition equates the two critical LRs without identifying boundedness with strict decay at the boundary.
The proof is in \Cref{app:proof-agm-critical-lr-equivalence}.
AGM training is at EoS when its training LR is close to $\lr_\star$.

\WCupdate{The stability notion and the projection are those of the main text.}
\WCupdate{For SGD without weight decay, this critical LR is the one of \Cref{def:stochastic-quadratic-critical-lr-general} applied to the local quadratic.}
The next two subsections describe how we estimate the effective curvature and the critical LR of this local model.

\subsection{EoC Detection}\label{sec:method-eoc}
EoC at $\param_{t_0}$ means that $\mA$ has a negative eigenvalue while the saddleness $\nu$ is small.
Detecting it requires $\lambda_{\max}(\mA)$ and $\lambda_{\min}(\mA)$.
We estimate the maximum eigenvalue by Lanczos on a large batch.
In our measurements, the top eigenpair varies little across mini-batches, whereas the bottom eigenvectors are nearly orthogonal (\Cref{fig:batch-lanczos-eigenpair-stability}).
Lanczos on one mini-batch estimates $\lambda_{\min}(\mA_{\batch_t})$, whose expectation is at most $\lambda_{\min}(\mA)$ by Jensen's inequality.
We therefore estimate $\lambda_{\min}(\mA)$ by \HWupdate{a stochastic power iteration method} that draws a fresh mini-batch at every step.

\begin{figure}[htbp]
    \centering
    \includegraphics[width=0.98\linewidth]{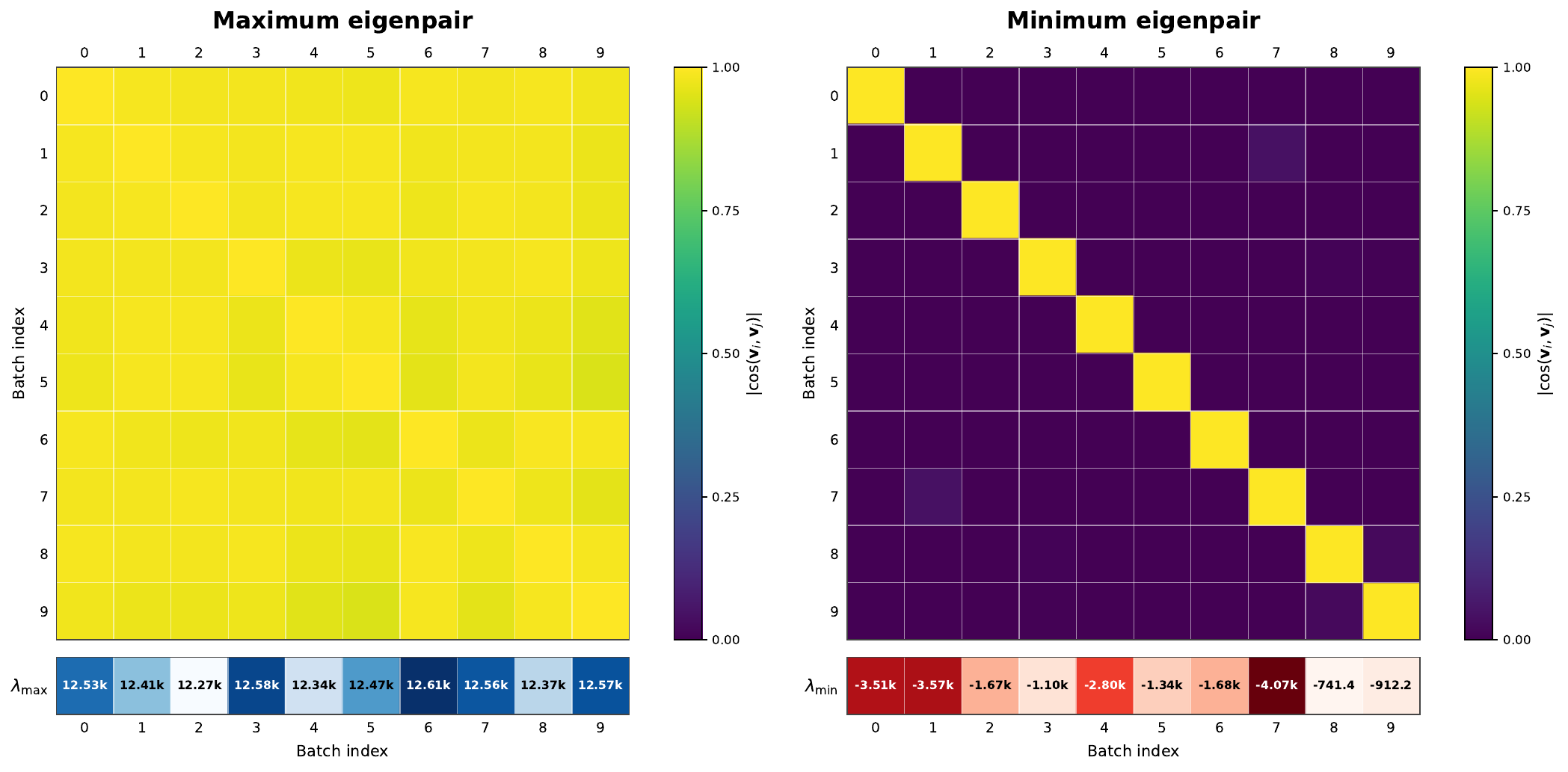}
    \caption{\textbf{Mini-batch Lanczos estimates are stable for the top eigenpair but unstable for the bottom eigenpair.}
    We run Lanczos on ten independently sampled mini-batches at the same checkpoint for the Adam-preconditioned Hessian.
    The maximum eigenvectors have high pairwise absolute cosine and similar eigenvalues across batches.
    In contrast, the minimum eigenvectors are nearly orthogonal across batches and the corresponding $\lambda_{\min}$ values vary substantially, showing that per-batch Lanczos estimates are dominated
    by batch-specific negative-curvature directions.}
    \label{fig:batch-lanczos-eigenpair-stability}
\end{figure}

\paragraph{Stochastic power iteration with bisection.}
\Cref{alg:min_eig_stochastic_bisect} runs the shifted iteration on the mini-batch effective curvatures,
\[
  \vx_{t+1}=\bigl(\mI-\etapi(\mA_{\batch_t}+\lambda\mI)\bigr)\vx_t,\qquad t=0,\dots,T_{\mathrm p}-1,
\]
with a fresh mini-batch at every step, each $\mA_{\batch_t}\vx_t$ formed from one Hessian-vector product.
Since the batches are independent, the mean iterate follows the population matrix, $\Ebb[\vx_{t+1}]=(\mI-\etapi(\mA+\lambda\mI))\Ebb[\vx_t]$.
If $\lambda_{\min}(\mA)+\lambda<0$, this matrix has an eigenvalue larger than one. When the initial vector has nonzero overlap with a bottom eigenvector, the component of $\Ebb[\vx_t]$ along that eigenvector grows geometrically.
If $\lambda_{\min}(\mA)+\lambda\ge0$ and $\etapi$ is small enough, $\Ebb[\vx_t]$ stays bounded.
A query at shift $\lambda$ is therefore declared to explode when $\|\vx_t\|$ exceeds the cutoff $B$ at some step $t\le T_{\mathrm p}$, a relative cutoff since $\|\vx_0\|=1$.
\WCupdate{Bisection starts from a given search interval $[0,\Lambda]$ and moves its lower endpoint to an exploding shift and its upper endpoint to a non-exploding shift.}
\WCupdate{After $N$ rounds the output $\widehat\lambda$ is the final upper endpoint, and the estimate of the minimum eigenvalue is $\widehat\lambda_{\min}=-\widehat\lambda$.}

\WCupdate{Three choices matter in practice: the initial vector, the stepsize, and the search interval.}
The initial vector $\vx_0$ is the same unit vector for every query. For nonzero checkpoint momentum, we use $\vm_{t_0}/\|\vm_{t_0}\|$.
\WCupdate{The stepsize $\etapi$ must satisfy $\etapi(\lambda_{\max}(\mA)+\Lambda)\le1$ over the whole search interval. This makes the mean transition positive semidefinite for every queried shift and keeps $\|\mI-\etapi(\mA+\lambda\mI)\|\le1$ for every $\lambda\in[\lambda_\star,\Lambda]$, where $\lambda_\star\coloneqq\max\{0,-\lambda_{\min}(\mA)\}$.}
We set it from the Lanczos estimate $\widehat\lambda_{\max}$, which \Cref{alg:eoc_local_Detection_AGM} computes first on $\bar\mA$, the effective curvature averaged over a fixed set of mini-batches.
This estimate does not certify the population stepsize condition.
\WCupdate{The search interval $[0,\Lambda]$ is an input to the algorithm and must contain $\lambda_\star$. The guarantee takes it as given, assumes $\lambda_{\min}(\mA)<0$, and includes the bisection resolution $\Lambda2^{-N}$. We use $B\ge2$. The algorithm also runs when $\lambda_{\min}(\mA)\ge0$, but the guarantee below does not cover that case.}
Each query uses at most $T_{\mathrm p}$ Hessian-vector products.
\Cref{alg:eoc_local_Detection_AGM,alg:eos_local_Detection_AGM} pass $T_{\mathrm p}$ to \Cref{alg:min_eig_stochastic_bisect}, and \Cref{alg:eos_local_Detection_AGM} writes $T$ for its own simulation horizon. The horizons used in the experiments are listed in \Cref{app:detection-settings}.

\paragraph{EoC criterion.}
\Cref{alg:eoc_local_Detection_AGM} combines the two eigenvalue estimates into the saddleness estimate
$\widehat\nu=\max\{0,-\widehat\lambda_{\min}\}/\widehat\lambda_{\max}$, which estimates the AGM saddleness $\nu$ of \Cref{def:agm-saddleness}. We declare EoC when $\widehat\lambda_{\min}<0$ and $\widehat\nu\le\nu_{\rm tol}$. This ratio requires $\widehat\lambda_{\max}>0$. The experiments use $\nu_{\rm tol}=0.1$.
The same estimate sets the compensation coefficient $\extraWD=\max\{0,-\widehat\lambda_{\min}\}$ for EoS detection.

\begin{algorithm}[t]
\caption{Minimum eigenvalue estimation by stochastic power iteration and bisection}
\label{alg:min_eig_stochastic_bisect}
\begin{algorithmic}[1]
    \REQUIRE \WCupdate{Symmetric samples $\mA_{\batch_t}$, i.i.d.\ with mean $\mA$; inner stepsize $\etapi$; horizon $T_{\mathrm p}$; bisection rounds $N$; cutoff $B\ge2$; a given unit vector $\vx_0$; and a search interval $[0,\Lambda]$ containing $\lambda_\star=\max\{0,-\lambda_{\min}(\mA)\}$ with $\etapi(\lambda_{\max}(\mA)+\Lambda)\le1$.}
    \ENSURE \WCupdate{Estimated shift $\widehat\lambda$ and the estimate $\widehat{\lambda}_{\min}=-\widehat\lambda$ of the minimum eigenvalue.}

\STATE \textbf{Subroutine} \textsc{Stable}$(\lambda)$:
\STATE \hspace{0.6cm} Set $\vx\gets \vx_0$; use fresh independent samples for this query.
\FOR{$t=0,1,\dots,T_{\mathrm p}-1$}
    \STATE Draw a fresh mini-batch $\batch_t$ and form $\mA_{\batch_t}$.
    \STATE $\vx \gets \vx - \etapi (\mA_{\batch_t} + \lambda\mI)\vx$.
    \IF{$\|\vx\| > B$}
        \STATE \textbf{return} \texttt{false} \COMMENT{explode}
    \ENDIF
\ENDFOR
\STATE \textbf{return} \texttt{true} \COMMENT{stable}

\STATE \textbf{Main procedure:}
    \STATE $\lambda_{\mathrm{lo}} \gets 0$, \quad $\lambda_{\mathrm{hi}} \gets \Lambda$. \COMMENT{\WCupdate{start from the given search interval}}

\FOR{$k=1,2,\dots,N$}
    \STATE $\lambda_{\mathrm{mid}} \gets (\lambda_{\mathrm{lo}}+\lambda_{\mathrm{hi}})/2$.
    \IF{\textsc{Stable}$(\lambda_{\mathrm{mid}})=\texttt{true}$}
        \STATE $\lambda_{\mathrm{hi}} \gets \lambda_{\mathrm{mid}}$ \COMMENT{stable}
    \ELSE
        \STATE $\lambda_{\mathrm{lo}} \gets \lambda_{\mathrm{mid}}$ \COMMENT{explode}
    \ENDIF
\ENDFOR

\STATE $\widehat\lambda \gets \lambda_{\mathrm{hi}}$, \quad $\widehat{\lambda}_{\min} \gets -\widehat\lambda$.
\STATE \textbf{return} $\widehat\lambda,\ \widehat{\lambda}_{\min}$.
\end{algorithmic}
\end{algorithm}

\begin{algorithm}[t]
\caption{EoC detection for AGMs}
\label{alg:eoc_local_Detection_AGM}
\begin{algorithmic}[1]
\REQUIRE Checkpoint $\param_{t_0}$; frozen preconditioner $\mP=\mP_{t_0}\succ 0$; weight decay $\wdc\ge 0$;
tolerance $\nu_{\rm tol}>0$;
momentum state $\vm_{t_0}$; \Cref{alg:min_eig_stochastic_bisect} parameters $(\etapi,T_{\mathrm p},N,B,\Lambda)$;
Lanczos iterations $T_{\rm LZ}$; averaged operator $\bar\mA$, the mean of $\mA_{\batch}$ over a fixed set of mini-batches.
\ENSURE $\textsc{EoC}\in\{\texttt{true},\texttt{false}\}$; estimates $\widehat\lambda_{\min}$, $\widehat\lambda_{\max}$, and $\widehat\nu$.

\STATE Define symmetric stochastic samples
$\mA_{\batch_t} := \mP^{-1/2}\mH_{\batch_t}\mP^{-1/2} + \wdc \mI$.

\STATE $\widehat{\lambda}_{\max} \gets \textsc{LanczosMaxEig}(\bar\mA;\,T_{\rm LZ})$. \COMMENT{$\etapi$ is chosen from $\widehat{\lambda}_{\max}$, see the text}
\STATE Set $\vx_0\gets\vm_{t_0}/\|\vm_{t_0}\|$.
\STATE $(\widehat{\lambda},\, \widehat{\lambda}_{\min}) \gets
\text{\Cref{alg:min_eig_stochastic_bisect}}(\mA_{\batch_t};\etapi,T_{\mathrm p},N,B,\Lambda,\vx_0)$.

\STATE $\widehat{\nu} \gets \dfrac{\max\{0,\,-\widehat{\lambda}_{\min}\}}{\widehat{\lambda}_{\max}}$.

\STATE $\textsc{EoC} \gets \big(\widehat{\lambda}_{\min} < 0\big)\ \wedge\ \big(\widehat{\nu} \le \nu_{\rm tol}\big)$.

\STATE \textbf{return} $\textsc{EoC},\,\widehat{\lambda}_{\min},\,\widehat{\lambda}_{\max},\,\widehat{\nu}$.
\end{algorithmic}
\end{algorithm}

\paragraph{Guarantee.}
The following theorem separates the error due to curvature noise, the finite horizon, and the bisection resolution.
\begin{theorem}[Informal; accuracy of the stochastic power iteration]\label{thm:power-iteration-informal}
Let $\lambda_{\min}(\mA)<0$, and let the mini-batch effective curvatures $\mA_{\batch_t}=\mA+\mDelta_t$ be i.i.d.\ symmetric with $\Ebb[\mDelta_t]=\zeroB$ and $\|\mDelta_t\|\le\sigma$ almost surely.
  \WCupdate{Let the unit initialization $\vx_0$ satisfy $|\vu_{\min}^\top\vx_0|\ge\varepsilon>0$ for a unit eigenvector $\vu_{\min}$ of $\mA$ with eigenvalue $\lambda_{\min}(\mA)$, and let the search-interval and stepsize conditions of \Cref{thm:lambda_hat_bound} hold.}
For fixed confidence $1-\delta$, \Cref{alg:min_eig_stochastic_bisect} with horizon $T_{\mathrm p}$, cutoff $B$ and $N$ bisection rounds outputs, with probability at least $1-\delta$,
\[
  \bigl|\widehat\lambda_{\min}-\lambda_{\min}(\mA)\bigr|
  \le O\!\bigl(\etapi\log(2N)\bigr)
     +O\!\Bigl(\frac{1}{\etapi T_{\mathrm p}}\Bigr)
     +\Lambda2^{-N},
\]
where the hidden constants depend on the noise bound $\sigma$, cutoff $B$, initial overlap $\varepsilon$, and failure probability $\delta$.
\end{theorem}
The first term comes from curvature noise, the second from running each query for only $T_{\mathrm p}$ steps, and the third from the number of bisection rounds.
For large $T_{\mathrm p}$, choosing $\etapi$ of order $[T_{\mathrm p}\log(2N)]^{-1/2}$ balances the first two terms and gives $O(\sqrt{\log(2N)/T_{\mathrm p}})$, plus the bisection error.
Increasing the horizon at a fixed stepsize reduces the finite-horizon term but does not remove the noise term from this bound.
\WCupdate{The theorem does not verify the search-interval condition or the overlap. Both are inputs, and the theorem is silent when $\lambda_{\min}(\mA)\ge0$.}
The formal statement is \Cref{thm:lambda_hat_bound} in \Cref{app:proof-of-min-eigen}, and the proof is in \Cref{app:proof-power-iteration}.

\clearpage
\subsection{EoS Detection}\label{sec:method-eos}
\WCupdate{The AGM critical LR $\lr_\star$ is the threshold for bounded second moments on the strictly positive eigenspace of $\mA$ (\Cref{def:agm-critical-lr}).}
Forming this projection is impractical at LLM scale, and running the local dynamics without it is misleading, because negative effective curvature produces an expanding direction at every positive LR (\Cref{app:agm-EoC}).
\WCupdate{We therefore replace the projection by \emph{compensation}: the penalty $\tfrac{\extraWD}{2}\|\vvareps\|_{\mP}^2$, with compensation coefficient $\extraWD=\max\{0,-\widehat\lambda_{\min}\}$ from \Cref{sec:method-eoc}, shifts the effective curvature from $\mA$ to $\mA+\extraWD\mI$.}
\WCupdate{We then locate the stability boundary of this compensated model by simulation.}

\begin{algorithm}[t]
\caption{EoS detection for AGMs}
\label{alg:eos_local_Detection_AGM}
\begin{algorithmic}[1]
\REQUIRE Checkpoint $\param_{t_0}$; momentum state $\vm_{t_0}$; frozen preconditioner $\mP=\mP_{t_0}\succ 0$;
weight decay $\wdc \ge 0$; training LR $\lr$; momentum parameter $\beta_1\in[0,1)$;
horizon $T$; cutoff $M$; random seed $s$; increasing LR factors $\{\alpha_k\}_{k=1}^K$; EoS tolerance $\tau_{\rm eos}>0$;
parameters $(\etapi,T_{\mathrm p},N,B,\Lambda,\vx_0)$ for \Cref{alg:min_eig_stochastic_bisect}, with $\vx_0=\vm_{t_0}/\|\vm_{t_0}\|$.
\ENSURE Estimated critical LR $\widehat{\lr}_\star$, critical LR ratio $q(\param_{t_0})=\widehat{\lr}_\star/\lr$, and $\textsc{EoS}$ flag.

\STATE \textbf{(Step 1: estimate the compensation coefficient.)}
Define symmetric stochastic samples
$\mA_{\batch_t} :=\mP^{-1/2}\mH_{\batch_t}\mP^{-1/2} + \wdc \mI$.
Run \Cref{alg:min_eig_stochastic_bisect} on $\mA_{\batch_t}$ to obtain $\widehat{\lambda}_{\min}$.
Set $\extraWD \gets \max\{0,-\widehat{\lambda}_{\min}\}$.

\STATE \textbf{Subroutine} \textsc{ExplodeDetect}$(\alpha)$:
\STATE \hspace{0.6cm} Set the random seed to $s$; $\vvareps\gets \vzero$; $\vm\gets\vm_{t_0}$; $R\gets 0$.
\FOR{$t=0,1,\dots,T-1$}
    \STATE \hspace{0.6cm} Sample mini-batch $\batch_t$.
    \STATE \hspace{0.6cm} $\vg \gets \vg_{\batch_t}
        + \mH_{\batch_t}\,\vvareps
        + \extraWD\,\mP\vvareps$.
    \STATE \hspace{0.6cm} $\vm \gets \beta_1\, \vm + (1-\beta_1)\, \vg$.
    \STATE \hspace{0.6cm} $\vvareps \gets \vvareps - \alpha\lr\,\mP^{-1}\vm - \alpha\lr\,\wdc(\vvareps+\param_{t_0})$.
    \STATE \hspace{0.6cm} $R \gets \max\{R,\|\vvareps\|_{\mP}\}$.
    \IF{$R \ge M$}
        \STATE \hspace{0.6cm} \textbf{return} \texttt{true} \COMMENT{explode}
    \ENDIF
\ENDFOR
\STATE \hspace{0.6cm} \textbf{return} \texttt{false} \COMMENT{stable}

\STATE \textbf{(Step 2: search over $\{\alpha_k\}$.)}
Find $k^\star := \min\{k:\textsc{ExplodeDetect}(\alpha_k)=\texttt{true}\}$.
\IF{no such $k^\star$ exists}
    \STATE Record the censored bounds $\widehat{\lr}_\star>\alpha_K\lr$ and $q>\alpha_K$; set $\textsc{EoS}\gets\texttt{false}$.
    \STATE \textbf{return} $\widehat{\lr}_\star, q, \textsc{EoS}$.
\ENDIF

\STATE $\widehat{\lr}_\star \gets \alpha_{k^\star}\lr$.
\STATE $q(\param_{t_0}) \gets \alpha_{k^\star}$.
\STATE $\textsc{EoS} \gets \big(q(\param_{t_0})\in[1-\tau_{\rm eos},\,1+\tau_{\rm eos}]\big)$.
\STATE \textbf{return} $\widehat{\lr}_\star, q, \textsc{EoS}$.
\end{algorithmic}
\end{algorithm}

\paragraph{Compensated local loss.}
Let $\vvareps\coloneqq\param-\param_{t_0}$ be the displacement from the checkpoint. With $\extraWD$ from \Cref{sec:method-eoc}, the compensated local loss is the quadratic surrogate \eqref{eq:agm-quadratic-surrogate} with the penalty measured in the norm of the frozen preconditioner,
\[
  G_t(\vvareps)
  \coloneqq
  \loss_{\batch_t}(\param_{t_0})
  +\vvareps^\top\vg_{\batch_t}
  +\tfrac{1}{2}\vvareps^\top\mH_{\batch_t}\vvareps
  +\tfrac{\extraWD}{2}\|\vvareps\|_{\mP}^2,
  \qquad \|\vvareps\|_{\mP}^2=\vvareps^\top\mP\vvareps .
\]
\WCupdate{The penalty shifts the mini-batch effective curvature from $\mA_{\batch_t}$ to $\mA_{\batch_t}+\extraWD\mI$ in preconditioned coordinates. It enters the gradient used in the momentum update, while the decoupled weight decay keeps its coefficient $\wdc$. \Cref{app:setting} writes the resulting dynamics as a stochastic linear system.}

\paragraph{Local simulation.}
\Cref{alg:eos_local_Detection_AGM} first sets $\extraWD$ from the power-iteration estimate and then runs the local AGM dynamics \eqref{eq:agm-checkpoint-dynamics} with the compensated gradient $\vg_{\batch_t}+(\mH_{\batch_t}+\extraWD\mP)\vvareps_t$.
The preconditioner is frozen at $\mP$. The momentum starts from the checkpoint value $\vm_{t_0}$ and is updated inside the simulation. The displacement starts from $\vvareps_0=\vzero$.
One trajectory is run for each LR factor $\alpha_k$ in an increasing grid $\alpha_1<\cdots<\alpha_K$ around the training LR $\lr$, and all trajectories use the same mini-batch sequence, fixed through the random seed $s$.
A trajectory is marked as exploding when $\|\vvareps_t\|_{\mP}$ reaches the cutoff $M$ at some step $t\le T$.
The estimate $\widehat\lr_\star$ is the smallest LR in the grid at which the trajectory explodes, and $q=\widehat\lr_\star/\lr$ is the critical LR ratio reported in the experiments.
We declare EoS when $q\in[1-\tau_{\rm eos},1+\tau_{\rm eos}]$, with $\tau_{\rm eos}=0.1$ in the experiments (\Cref{fig-ebs}).
If no factor in the grid explodes, the run is recorded as censored, $q>\alpha_K$, and the checkpoint is not counted as being at EoS.
We hold the training LR $\lr$ fixed during the simulation because the simulation horizon $T$ is short relative to the training schedule. This approximates schedules such as WSD or cosine whenever their LR changes little over the simulation window.

\paragraph{\WCupdate{The expected detector.}}
\WCupdate{We analyze an idealized version, the expected detector. It runs the same stochastic dynamics but replaces the single trajectory by $\Ebb\|\vvareps_t\|^2$, since $\lr_\star$ is a mean-square threshold (\Cref{app:mean-square-vs-almost-sure}).}
\WCupdate{Fix a search window $\mathcal I$ around $\lr$, the horizon $T$ and a cutoff $M>1$. Run the same compensated dynamics from every unit displacement $\|\vvareps_0\|=1$ with the checkpoint momentum, and take the expectation over the mini-batches. The expected detector returns}
\WCupdate{\[
  \lrhat(M,T)
  =\sup\Bigl\{\lr'\in\mathcal I:\ \max_{0\le t\le T}\Ebb\|\vvareps_t\|^2\le M^2\ \text{for every }\|\vvareps_0\|=1\Bigr\}.
\]}
\WCupdate{The detector accepts $\lr'$ if the condition in the braces holds, so $\lrhat(M,T)$ is the supremum of the accepted LRs (\Cref{fh:sec:main}). We keep $\widehat\lr_\star$ for the output of the algorithm.}
The implemented detector differs from it in three ways.
\WCupdate{First, the expected detector takes the supremum of accepted LRs in a continuous window, while the algorithm returns the first exploding LR in an increasing grid.}
Finite-horizon acceptance need not be monotone in the LR, so these thresholds can differ.
\WCupdate{Second, the expected detector starts from unit displacements, so $M>1$ is a relative cutoff, while the algorithm starts from zero and uses $M$ as an absolute bound on $\|\vvareps_t\|_{\mP}$.}
\WCupdate{Third, the expected detector takes the expectation of $\|\vvareps_t\|^2$ over the mini-batch sequence, while the algorithm runs one trajectory.}
\WCupdate{The algorithm uses the preconditioner norm, in which the compensation is measured; the analysis of the expected detector is unchanged under this norm (\Cref{fh:rem:weighted}).}

\begin{samepage}
\paragraph{Guarantee.}
  \WCupdate{Let $\lrP$, the compensated critical LR, be the critical LR of the compensated loss without projection (\Cref{fh:def:stable}).}
  \WCupdate{The output of the expected detector differs from $\lr_\star$ because $\lrP\neq\lr_\star$ and because the simulation stops after $T$ steps.}
  \WCupdate{The following theorem bounds these two errors separately in inverse LR.}
  \WCupdate{It assumes a fixed model with $\mA+\extraWD\mI\succ0$ (\Cref{fh:ass:comparison}), $\lr'\wdc\le1$ on the search window, and that the expected detector accepts $\lrP$ (\Cref{fh:ass:transient}).}
\begin{theorem}[Informal; compensated finite-horizon threshold]\label{thm:clr-shift-informal}
    \WCupdate{Let $\lrhat(M,T)$ be the output of the expected detector, and let $v\ge0$ satisfy $\Ebb(\mDelta_t^+)^2\preceq v\mA^+$ and $\Ebb\mDelta_t^2\preceq v(\mA+\extraWD\mI)$, where $\mDelta_t=\mA_{\batch_t}-\mA$ and the superscript $+$ denotes the block on the positive eigenspace of $\mA$.}
Under the assumptions in \Cref{fh:sec:main},
\[
      \WCupdate{\left|\frac{1}{\lrhat(M,T)}-\frac{1}{\lr_\star}\right|}
 \le
 \underbrace{\frac{1-\beta_1}{2(1+\beta_1)}\extraWD}_{\text{compensation}}+\underbrace{\frac v2}_{\text{Hessian noise}}
 +O\!\Bigl(\frac{1+\log M}{T}\Bigr).
\]
The $O(\cdot)$ is for a fixed model and search window as $(1+\log M)/T\to0$, over choices of $M$ and $T$ that satisfy \Cref{fh:ass:transient}.
\end{theorem}
\end{samepage}
The compensation term is linear in $\extraWD$, which equals $\widehat\nu\,\widehat\lambda_{\max}$ for the detector. The noise level $v$ need not be small. For a fixed cutoff $M$, the finite-horizon term decreases as $1/T$.
\WCupdate{\Cref{fh:cor:saddleness} proves this simplified bound. \Cref{fh:thm:main} gives a bound of a different form, which compares the noise of the two systems directly instead of bounding each by $v$. Small saddleness alone need not make the total error small.}
\WCupdate{The theorem applies to the expectation over mini-batches from unit initial displacements, not to the single trajectory from zero displacement that the experiments run.}
\WCupdate{Since the compensation coefficient $\extraWD$ is estimated, it does not by itself guarantee $\mA+\extraWD\mI\succ0$.}
\Cref{fh:prop:cutoff} gives a sufficient cutoff for accepting the compensated critical LR.
The complete assumptions and the full theorem are in \Cref{fh:sec:main}. The simplified bound and its proof are in \Cref{fh:sec:comparison}.

\subsection{Cost of a Hessian-Vector Product}\label{app:hvp-cost}
Curvature enters both procedures only through Hessian-vector products (HVPs), which we compute with PyTorch's autograd by differentiating a gradient-vector inner product,
\begin{equation}
    \vg(\param) \;=\; \nabla_{\param}\loss(\param),
    \qquad
    \mathrm{HVP}(\vv) \;=\; \nabla_{\param}\big(\vg(\param)^\top \vv\big)
    \;=\; \nabla^2_{\param}\loss(\param)\,\vv.
\end{equation}
We estimate the cost for a linear layer, since matrix multiplications dominate the attention and MLP projections in our models.

\paragraph{Matrix-multiplication count for a linear layer.}
Consider a linear layer
\begin{equation}
    \mY = \mX \mW,
\end{equation}
where $\mX\in\R^{n\times d_{\mathrm{in}}}$ holds a batch of $n$ inputs, $\mW\in\R^{d_{\mathrm{in}}\times d_{\mathrm{out}}}$, and $\mY\in\R^{n\times d_{\mathrm{out}}}$.
Let $\mG=\partial \loss/\partial \mY$ be the upstream gradient.
The forward pass uses one general matrix multiplication (GEMM),
\begin{equation}
    \mY=\mX\mW \quad \Rightarrow \quad 1~\text{GEMM}.
\end{equation}
The first-order backward pass computes gradients with respect to $\mX$ and $\mW$,
\begin{equation}
    \frac{\partial \loss}{\partial \mX} = \mG \mW^\top,
    \qquad
    \frac{\partial \loss}{\partial \mW} = \mX^\top \mG,
\end{equation}
which costs two GEMMs. Thus, a forward and backward pass through a linear layer costs approximately three GEMMs, ignoring elementwise operations.

The autograd procedure above computes an HVP by backpropagating through the gradient computation in a second reverse-mode pass. For layers whose cost is dominated by matrix multiplication, this incurs overhead comparable to an additional forward and backward pass through the same linear operators, plus extra GEMMs from differentiating the backward equations. Including the required second-order terms gives approximately
\begin{equation}
    \underbrace{1}_{\text{forward}} \;+\;
    \underbrace{2}_{\text{backward}} \;+\;
    \underbrace{2}_{\text{differentiate backward (w.r.t.\ }\mW\text{)}} \;+\;
    \underbrace{4}_{\text{propagate directional adjoints}}
    \;\approx\; 9~\text{GEMMs}.
\end{equation}
This is about $3\times$ the cost of one gradient evaluation ($9$ versus $3$ GEMMs per linear layer), or $9\times$ a forward pass when GEMMs dominate. Since transformer blocks contain many such linear layers (QKV projections, the attention output projection, and the two MLP projections), this constant-factor overhead provides a practical estimate of HVP cost for our models.

\section{Experimental Details}\label{app:add-exp-setup}
This appendix gives the models, training and detection settings behind \Cref{sec:experiments}, the protocol of the batch size threshold experiment, and the figures referenced from the main text but not shown there.

\subsection{Models, Training, and Detection Settings}
\paragraph{Model architectures.}
\Cref{tab:model-arch} lists the models used in our experiments.

\begin{table}[h]
\centering
\caption{Model architectures used in our experiments.}
\label{tab:model-arch}
\begin{tabular}{lccccc}
\toprule
Model & \#Layers & Hidden Dim & \#Heads & Head Dim & Params \\
\midrule
GPT-extra-small & 4  & 768  & 12 & 64 & 43M \\
GPT-small       & 12 & 768  & 12 & 64 & 104M \\
GPT-medium      & 24 & 1{,}024 & 16 & 64 & 344M \\
GPT-large       & 36 & 1{,}280 & 20 & 64 & 787M \\
\bottomrule
\end{tabular}
\end{table}

For each dataset, we train a byte-level BPE tokenizer with a vocabulary of $8{,}192$ tokens and use it for all model sizes. We use a small vocabulary instead of the standard $50{,}257$-token GPT-2 tokenizer because rarely trained tokens distort the preconditioned curvature. Under Adam, the second-moment estimate of an embedding row grows with how often its token is trained on, so a token that is seldom seen has $\sqrt{\hat v}$ close to the floor set by $\epsilon$. When such a token appears in the batch used for the Hessian-vector products, its row is amplified by roughly $1/\sqrt{\hat v}$ and can dominate the top eigenvalues. A smaller vocabulary leaves fewer such tokens. With $8{,}192$ tokens, at most $72$ tokens are rare by the criterion below at the checkpoints we used to build the filter.

Even with the smaller vocabulary, we remove the remaining rare tokens from the data used for EoS and EoC detection. All detection batches come from a single training shard that no run is trained on. For a run with batch size $B$ at checkpoint step $K$, we call a token rare if it has not appeared in the first $K$ training steps, if it has appeared less than once per step on average, or if it has not appeared in the last $20$ steps. The padding token is always treated as rare. We take the union of the rare sets over the stable-phase checkpoints of all three model sizes ($72$ tokens) and drop every sequence of the shard whose inputs or targets contain a rare token. We also drop every sequence in which a token that appears fewer than $100$ times per training step occurs more than $32$ times, because such a burst over-represents that token's row in a single batch. The filtered shard keeps $92{,}060$ of $97{,}656$ sequences ($94.3\%$), and the same shard is used for every model size, learning rate and checkpoint.

For model architectures, we apply QK-normalization~\citep{henry-etal-2020-query} and attention gating~\citep{qiu2025gated} for stability.
All models are trained with a sequence length of 1024.

\label{app:detection-settings}We run most experiments on GPT-extra-small. Unless otherwise specified, we use 100 stochastic power iterations and 50 compensated local simulation steps for GPT-extra-small. For the larger models (GPT-small/medium/large), we use 50 power iterations and 25 local simulation steps. In all local simulations, we mark a run as exploding when the displacement norm reaches the cutoff $M=10^5$.

\subsection{Additional Results}\label{app:additional-results}

\begin{figure*}[t]
      \centering
      \begin{subfigure}[t]{0.36\textwidth}
          \centering
          \includegraphics[width=\linewidth]{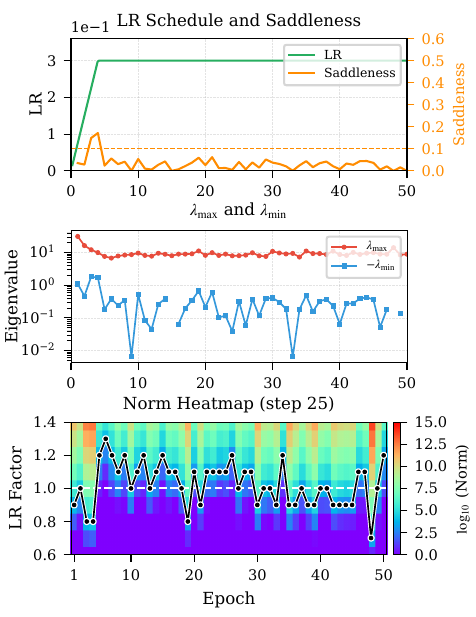}
          \caption{ResNet20 Model}
          \label{fig:cifar10-resnet20}
      \end{subfigure}\hspace{0.04\textwidth}
      \begin{subfigure}[t]{0.36\textwidth}
          \centering
          \IfFileExists{figs/fig-cifar10-resnet32-lr0.4-bs1024.pdf}{%
            \includegraphics[width=\linewidth]{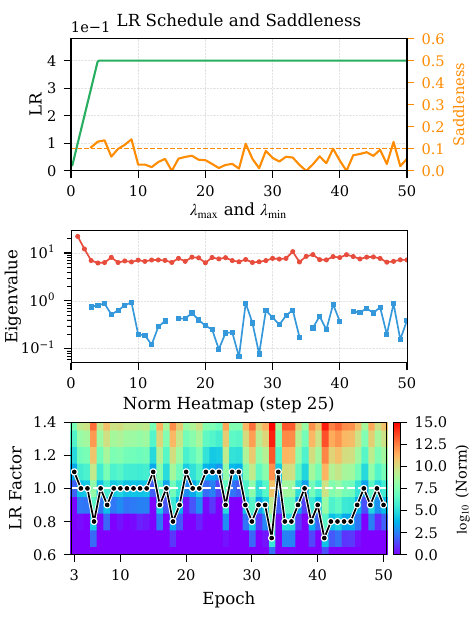}%
          }{%
            \fbox{\parbox[c][1.5in][c]{0.90\linewidth}{\centering
              Figure unavailable in the inherited source.\\[0.5em]
              ResNet32, CIFAR-10\\LR $0.4$, batch size $1024$.}}%
          }
          \caption{ResNet32 Model}
          \label{fig:cifar10-resnet32}
      \end{subfigure}

      \caption{\textbf{EoS and EoC appears on model trained with SGD.}
      We train ResNet20 and ResNet32 with MSE loss and SGD, using batch size $1024$ and weight decay $5\times10^{-4}$ for $50$ epochs.
      Panels follow \Cref{fig-small-medium-large}, using the raw Hessian and epoch on the horizontal axis.}
      \label{fig:cifar10-sgd-eos}
\end{figure*}

\paragraph{EoS across the pretraining runs.} The bounded-to-exploding transition of the compensated local simulation appears throughout our AdamW pretraining runs.
Across models from 43M to 787M parameters, \Cref{fig-small-medium-large,fig-lr} show many checkpoints at which the local simulation changes sharply from bounded trajectories to explosion near LR factor $1$.
\Cref{fig-sc,fig-lr} show that the location and duration of this behavior depend on the LR schedule and the peak LR.
\Cref{fig-ebs,fig-bs} show that it also depends on the batch size.
Taken together, these experiments show that EoS is common in our AdamW pretraining runs, and absent at the smallest batch size and extreme peak LRs.
Within the LR range tested around the training LR, the simulations show one transition at most checkpoints.
Under the assumptions of \Cref{fh:thm:decay}, strict homogeneous stability holds on a single LR interval whose endpoint is the AGM critical LR (\Cref{prop:agm-critical-lr-equivalence}).

\paragraph{Robustness to the compensation coefficient.}
\Cref{fig:convex-comp} shows one checkpoint with three compensation coefficients. The sharp transition is absent without compensation, appears at the estimated value, and remains visible at the larger value $600$.

\subsection{Batch Size Threshold Experiments}\label{appendix:exp:ebs}
This subsection gives the protocol behind \Cref{subsec:exp-bsz-threshold}.
We study five datasets (FineWeb~\citep{penedo2024fineweb}, OpenWebMath~\citep{paster2024openwebmath}, OpenWebText~\citep{gokaslan2019openwebtext}, C4~\citep{raffel2020t5}, and DCLM~\citep{li2024datacomplm}) using GPT-extra-small.
We use AdamW and the WSD schedule throughout, with baseline weight decay $\wdc=0.1$.

To vary the batch size, we follow the square-root scaling rule~\citep{malladi2022sdes}.
When increasing the batch size from $B$ to $kB$, we scale the LR and the weight decay, keeping the remaining hyperparameters fixed,
\[
\lr \leftarrow \sqrt{k}\,\lr, \qquad \wdc \leftarrow \sqrt{k}\,\wdc,
\]
so that the timescale $B/(\lr\wdc)$ is the same for all batch sizes. We also adjust the second-moment decay coefficient $\beta_2$ following \citet{marek2025smallbatchsizetraining}:
\[
\beta_2 \leftarrow \beta_2^{k}.
\]

\paragraph{CBZ estimation.}
For each dataset, we sweep the batch size under this scaling rule and record the final evaluation loss of each run.
We estimate the CBZ discussed in \Cref{subsec:exp-bsz-threshold} as the largest batch size whose final evaluation loss stays within $0.3\%$ of the best loss over the sweep.

\paragraph{EBZ estimation.}
For each batch size $B$, we estimate the critical LR ratio $q=\widehat\lr_\star/\lr$ defined in \Cref{subsec:exp-bsz-threshold}, averaged over checkpoints. Runs whose local simulations stay bounded over the whole LR-factor grid are recorded as censored at the top of the grid.
We then fit the averaged ratio as a function of $B$. The EBZ of \Cref{subsec:exp-bsz-threshold} is the smallest $B$ at which the fitted curve reaches $1.1$.

\Cref{fig-ebs:comp} compares EBZ to CBZ.
Across the five datasets, EBZ and CBZ are strongly positively correlated (Pearson correlation $0.95$). EBZ is below CBZ for C4 and FineWeb and above it for OpenWebMath, DCLM, and OpenWebText, so neither threshold is uniformly smaller.

\subsection{Multistep LR Drop Ablations}\label{appendix:exp:LRdrop}
\paragraph{Abrupt LR drops.}
We evaluate multistep schedule variants in \Cref{fig:ms-ablation}.
Before the drop, all six runs are at EoS ($0.9\le q\le1.1$) at most checkpoints, $8$ of the $9$ checkpoints before the drop in panels a--c and $12$ to $13$ of the $14$ in panels d--f.
After the drop, the transition moves to larger LR factors, and EoS returns only when the LR after the drop stays large.
With the peak LR fixed at $0.002$, no checkpoint after the drop is at EoS for any of the drop ratios $0.6$, $0.45$, and $0.3$ (panels a--c). The median $q$ after the drop is $1.3$ for the ratio $0.6$ and $1.4$ for $0.45$. For the ratio $0.3$, no LR factor in the grid explodes at $35$ of the $41$ checkpoints after the drop.
With the drop ratio fixed at $0.7$ (panels d--f), EoS returns at $26$ of the $36$ checkpoints after the drop for the peak LR $0.008$, at $7$ for $0.004$, and at none for $0.0005$.

\subsection{Linear LR Decay Ablations}\label{appendix:exp:LRdecay}
We evaluate different linear-decay configurations in \Cref{fig:decay-ablation}.
As in the multistep ablations, the decay ratio affects whether and when EoS re-emerges during the decay phase.
The transition time also depends on the decay length. Longer decay schedules give a slower transition, with the transition time roughly proportional to the number of decay steps in our settings.

\Cref{fig:decay-ablation} shows how EoC changes during LR decay.
As the LR decays, $-\widehat\lambda_{\min}$ falls from about $850$ in the stable phase to between $56$ and $540$ once the LR is below half its peak, and the saddleness falls from about $0.09$ to $0.03$ or less. EoC therefore weakens in the small-LR regime of these runs.

\begin{figure*}[p]
  \centering
  \begin{subfigure}[p]{0.33\textwidth}
    \centering
    \includegraphics[width=\linewidth]{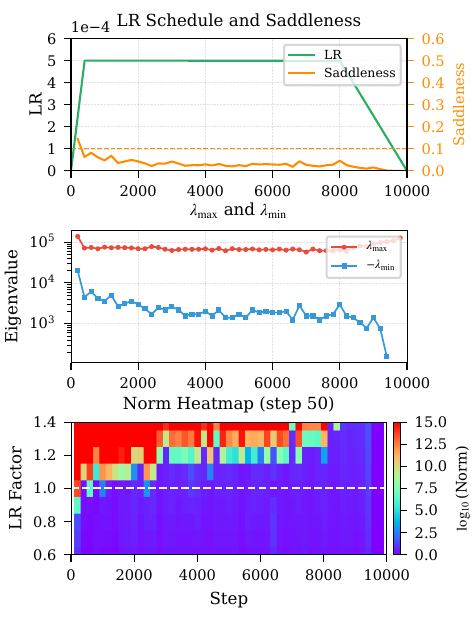}
    \caption{WSD, LR $=0.0005$.}
    \label{fig-lr:0.0005}
  \end{subfigure}\hfill
  \begin{subfigure}[p]{0.33\textwidth}
    \centering
    \includegraphics[width=\linewidth]{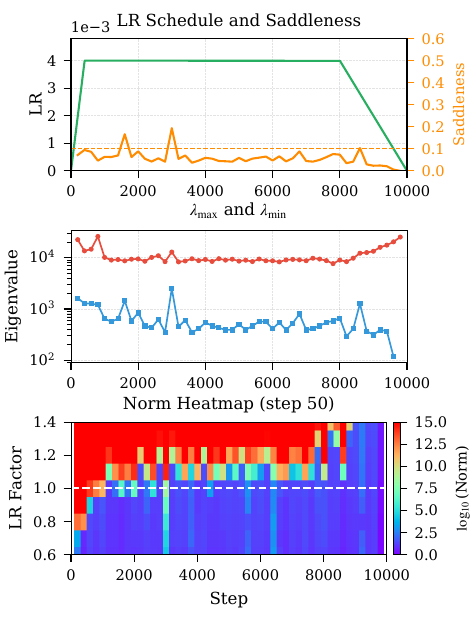}
    \caption{WSD, LR $=0.004$.}
    \label{fig-lr:0.004}
  \end{subfigure}\hfill
  \begin{subfigure}[p]{0.33\textwidth}
    \centering
    \includegraphics[width=\linewidth]{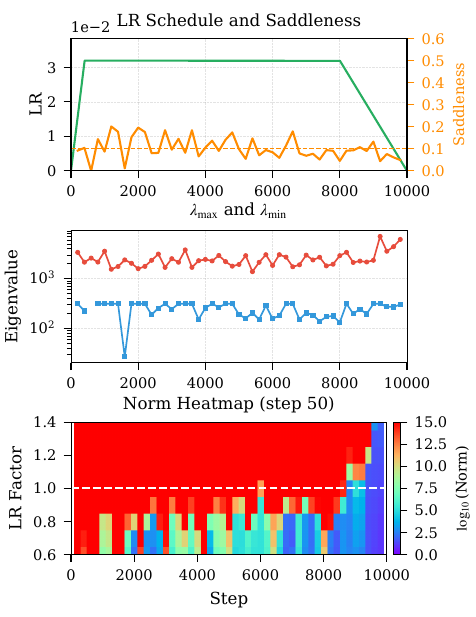}
    \caption{WSD, LR $=0.032$.}
    \label{fig-lr:0.032}
  \end{subfigure}

  \caption{\textbf{EoS and EoC measurements under different LRs.}
  We train GPT-extra-small (43M) on FineWeb-10B using AdamW with the WSD schedule and vary the LR.
  All runs use 400 linear warmup steps, 7600 stable steps, and 2000 linear decay steps. We identify $0.004$ as the optimal peak LR via a separate sweep.}
  \label{fig-lr}
\end{figure*}

\begin{figure*}[p]
  \centering
  \begin{subfigure}[p]{0.33\textwidth}
    \centering
    \includegraphics[width=\linewidth]{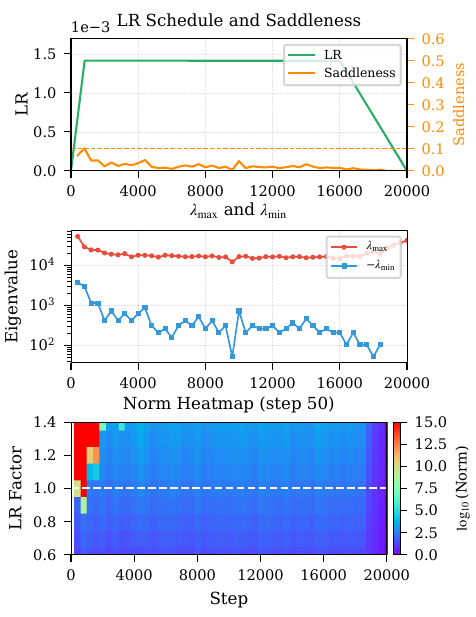}
    \caption{WSD, batch size $=256$.}
    \label{fig-bs:256}
  \end{subfigure}\hfill
  \begin{subfigure}[p]{0.33\textwidth}
    \centering
    \includegraphics[width=\linewidth]{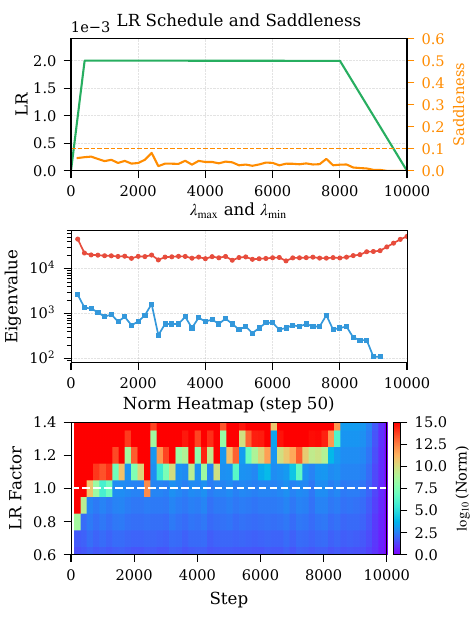}
    \caption{WSD, batch size $=512$.}
    \label{fig-bs:512}
  \end{subfigure}\hfill
  \begin{subfigure}[p]{0.33\textwidth}
    \centering
    \includegraphics[width=\linewidth]{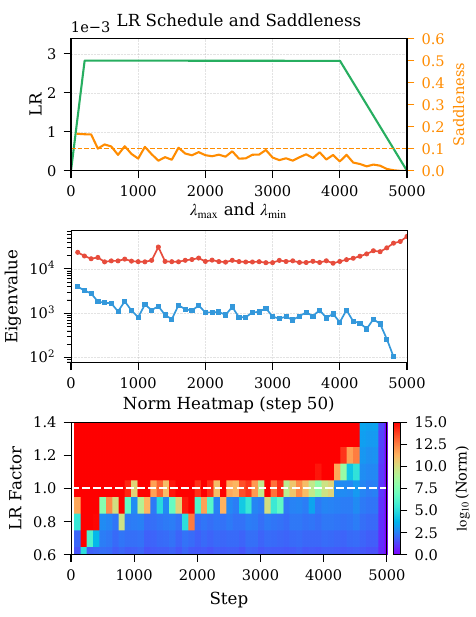}
    \caption{WSD, batch size $=1024$.}
    \label{fig-bs:1024}
  \end{subfigure}

  \caption{\textbf{EoS emerges only above a batch size threshold.}
  We train GPT-extra-small (43M) on FineWeb using AdamW with the WSD schedule at global batch sizes $256$, $512$, and $1{,}024$ with the same number of training tokens.
  Starting from LR $0.002$, weight decay $0.1$, and $\beta_2=0.95$ at batch size $512$, we scale the LR, the weight decay, and $\beta_2$ with the square-root rule of \Cref{appendix:exp:ebs}.
		Each panel follows the reporting protocol of \Cref{fig-small-medium-large}.
  (a) At batch size $256$, the transition is near factor $1.0$ only at the first checkpoints, and from step $3{,}600$ on no LR factor in the grid explodes.
  (b, c) At $512$ and $1{,}024$, the transition stays close to factor $1.0$ for most of training.}
  \label{fig-bs}
\end{figure*}

\begin{figure}[p]
  \centering
  \begin{subfigure}[p]{0.33\textwidth}
    \centering
    \includegraphics[width=\linewidth]{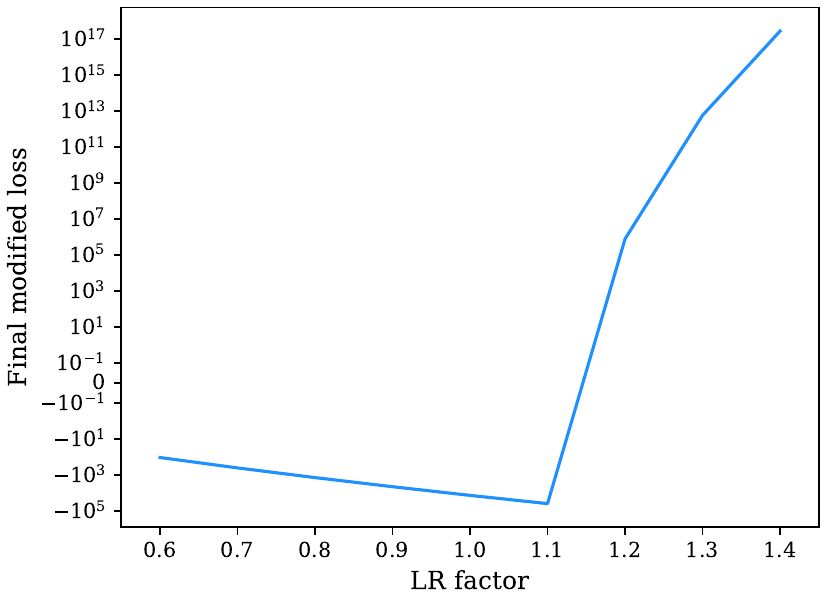}
    \caption{No compensation ($\extraWD=0$).}
    \label{fig:cc-0}
  \end{subfigure}\hfill
  \begin{subfigure}[p]{0.33\textwidth}
    \centering
    \includegraphics[width=\linewidth]{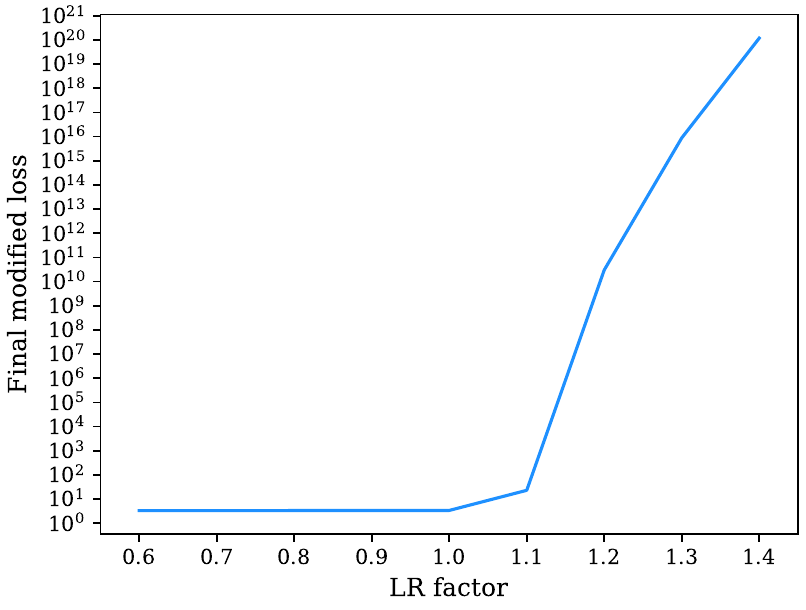}
    \caption{Estimated compensation ($\extraWD\approx 408$).}
    \label{fig:cc-408}
  \end{subfigure}\hfill
  \begin{subfigure}[p]{0.33\textwidth}
    \centering
    \includegraphics[width=\linewidth]{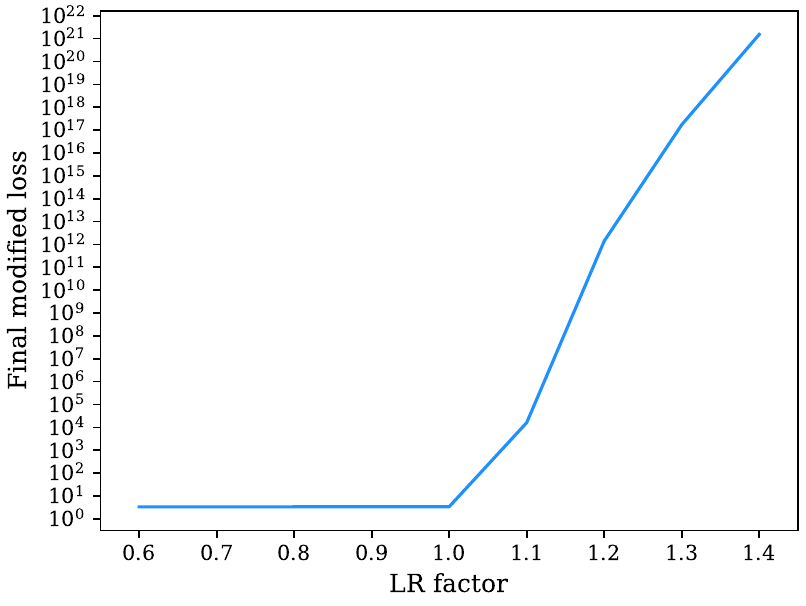}
    \caption{Compensation $\extraWD=600$.}
    \label{fig:cc-600}
  \end{subfigure}
  \caption{\textbf{Effects of different compensation coefficients.}
  We train GPT-extra-small (43M) with AdamW on FineWeb-10B, run EoS detection with different compensation coefficients $\extraWD$, and report the loss at local step 25.
  (a) $\extraWD=0$ (no compensation). (b) $\extraWD\approx 408$ (estimated). (c) $\extraWD=600$.}
  \label{fig:convex-comp}
\end{figure}

\begin{figure}[p]
  \centering
  \begin{subfigure}[p]{0.33\textwidth}
    \centering
    \includegraphics[width=\linewidth]{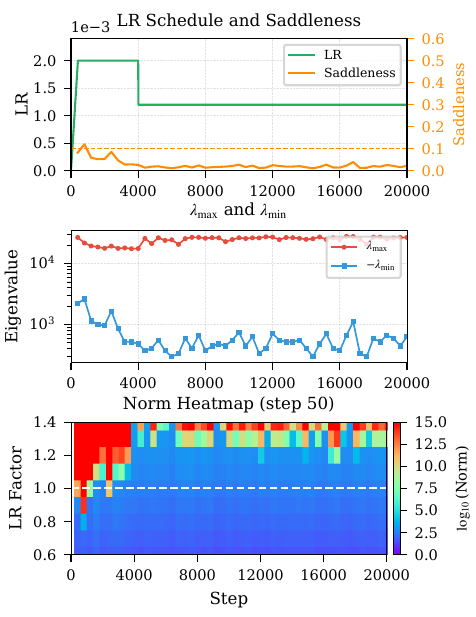}
    \caption{Multistep, drop ratio $=0.6$.}
    \label{fig:ms-tau06}
  \end{subfigure}\hfill
  \begin{subfigure}[p]{0.33\textwidth}
    \centering
    \includegraphics[width=\linewidth]{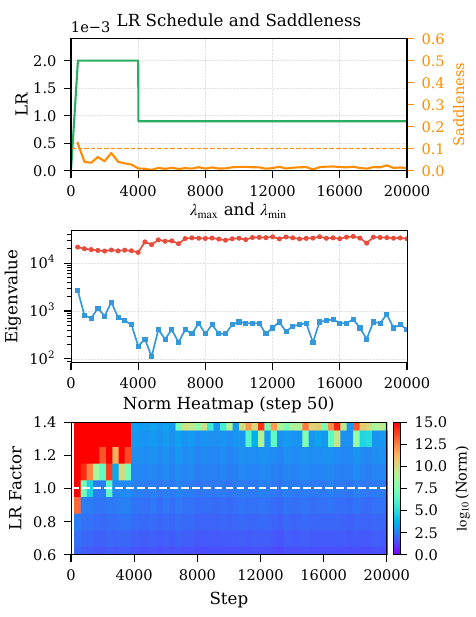}
    \caption{Multistep, drop ratio $=0.45$.}
    \label{fig:ms-tau045}
  \end{subfigure}\hfill
  \begin{subfigure}[p]{0.33\textwidth}
    \centering
    \includegraphics[width=\linewidth]{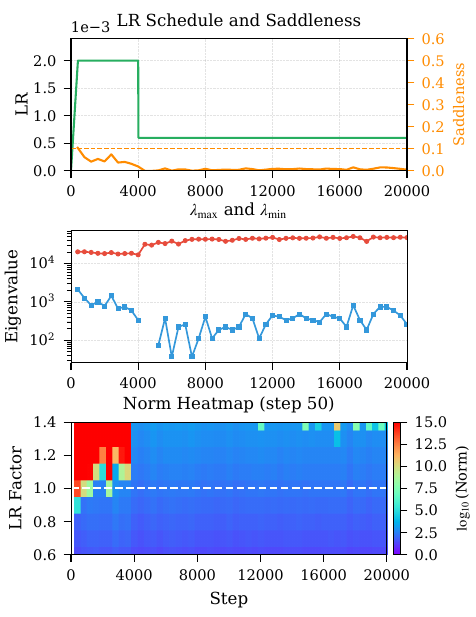}
    \caption{Multistep, drop ratio $=0.3$.}
    \label{fig:ms-tau03}
  \end{subfigure}

  \vspace{2pt}

  \begin{subfigure}[p]{0.33\textwidth}
    \centering
    \includegraphics[width=\linewidth]{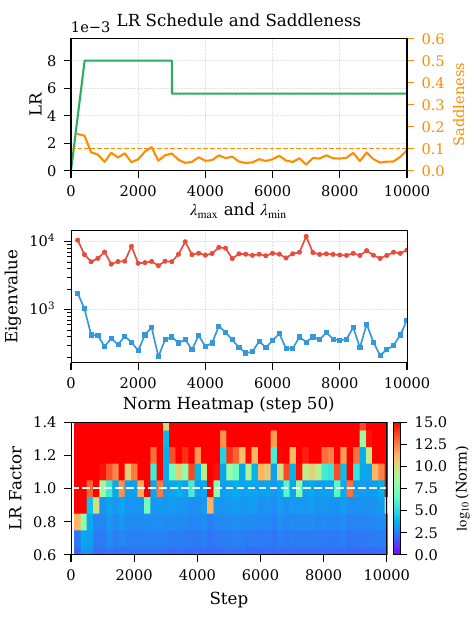}
    \caption{Multistep, peak LR $=0.008$.}
    \label{fig:ms-lr0008}
  \end{subfigure}\hfill
  \begin{subfigure}[p]{0.33\textwidth}
    \centering
    \includegraphics[width=\linewidth]{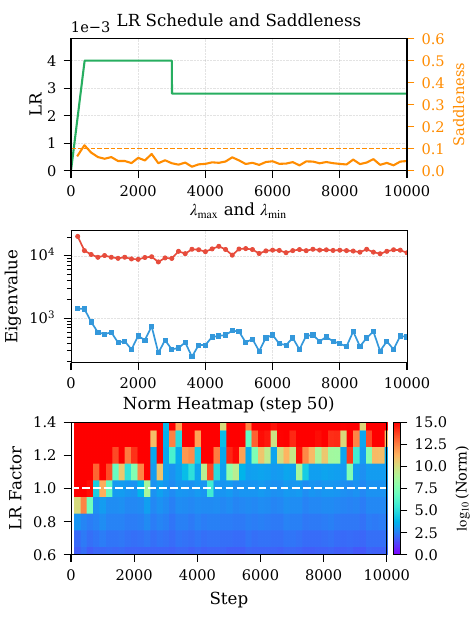}
    \caption{Multistep, peak LR $=0.004$.}
    \label{fig:ms-lr0004}
  \end{subfigure}\hfill
  \begin{subfigure}[p]{0.33\textwidth}
    \centering
    \includegraphics[width=\linewidth]{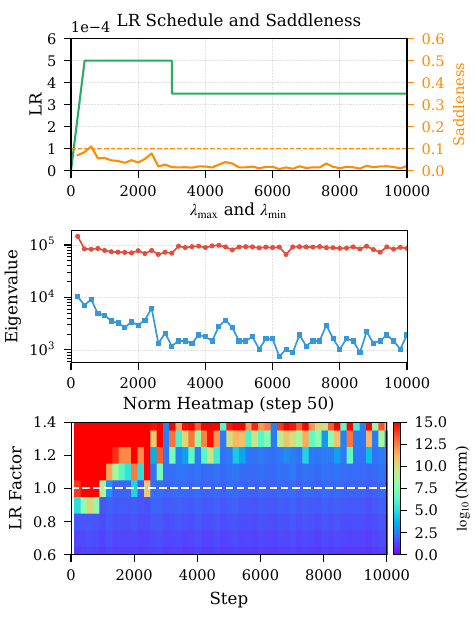}
    \caption{Multistep, peak LR $=0.0005$.}
    \label{fig:ms-lr00005}
  \end{subfigure}

  \caption{\textbf{Ablations on the multistep LR drop.}
  We train GPT-extra-small (43M) and vary the peak LR and the drop ratio of the multistep schedule.
  (a)--(c) The peak LR is fixed at $0.002$ and the drop ratio varies over $0.6$, $0.45$, and $0.3$, with the drop at step 4000.
  (d)--(f) The drop ratio is fixed at $0.7$ and the peak LR varies over $0.008$, $0.004$, and $0.0005$, with the drop at step 3000.}
  \label{fig:ms-ablation}
\end{figure}

\begin{figure}[p]
  \centering
  \begin{subfigure}[p]{0.33\textwidth}
    \centering
    \includegraphics[width=\linewidth]{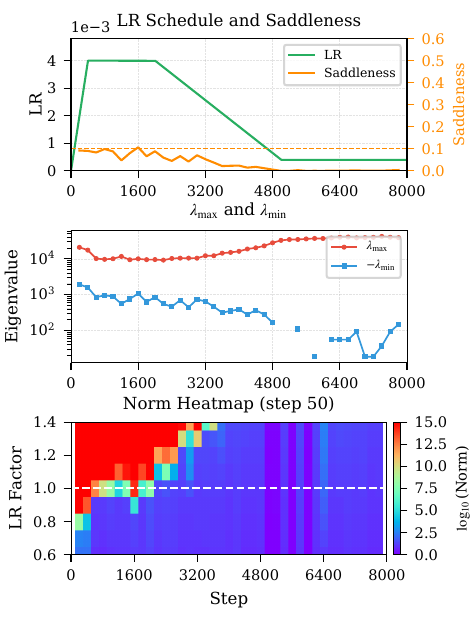}
    \caption{Decay length $=3000$.}
    \label{fig:decay-x-3000}
  \end{subfigure}\hfill
  \begin{subfigure}[p]{0.33\textwidth}
    \centering
    \includegraphics[width=\linewidth]{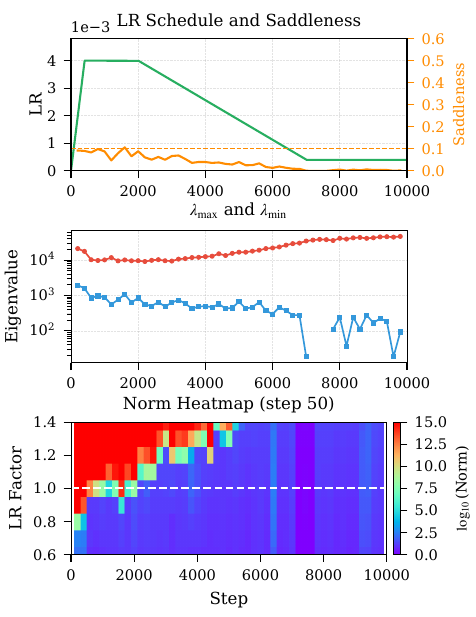}
    \caption{Decay length $=5000$.}
    \label{fig:decay-x-5000}
  \end{subfigure}\hfill
  \begin{subfigure}[p]{0.33\textwidth}
    \centering
    \includegraphics[width=\linewidth]{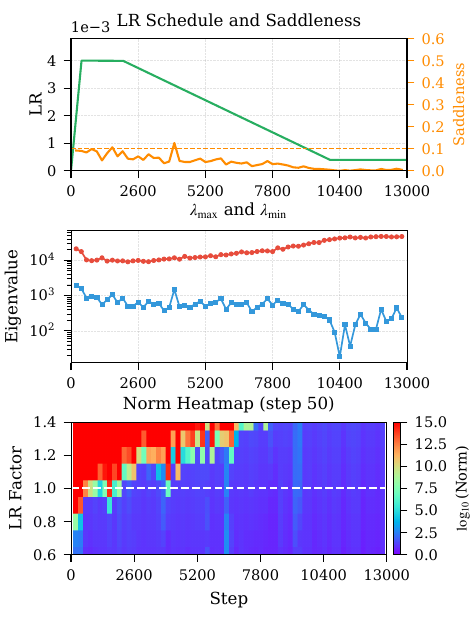}
    \caption{Decay length $=8000$.}
    \label{fig:decay-x-8000}
  \end{subfigure}

  \vspace{2pt}

  \begin{subfigure}[p]{0.33\textwidth}
    \centering
    \includegraphics[width=\linewidth]{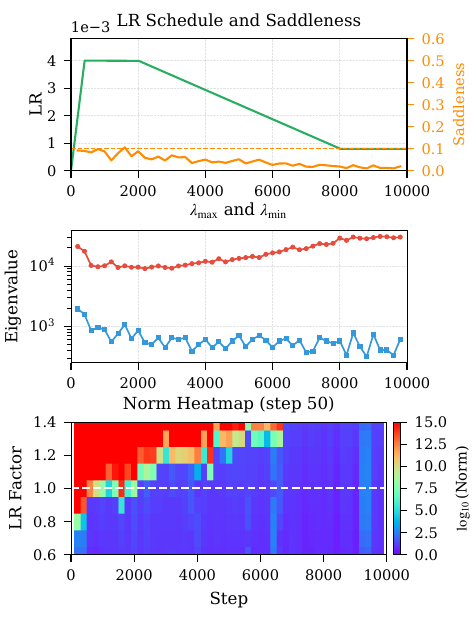}
    \caption{Decay ratio $=0.2$.}
    \label{fig:decay-y-02}
  \end{subfigure}\hfill
  \begin{subfigure}[p]{0.33\textwidth}
    \centering
    \includegraphics[width=\linewidth]{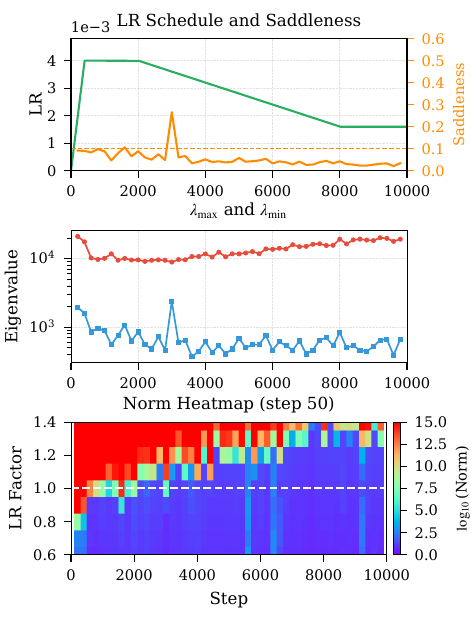}
    \caption{Decay ratio $=0.4$.}
    \label{fig:decay-y-04}
  \end{subfigure}\hfill
  \begin{subfigure}[p]{0.33\textwidth}
    \centering
    \includegraphics[width=\linewidth]{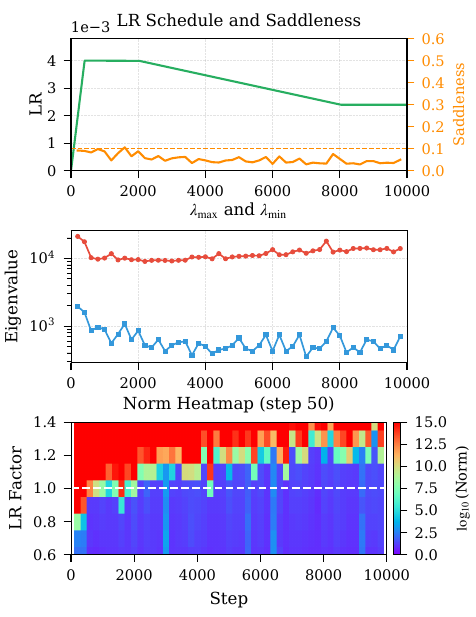}
    \caption{Decay ratio $=0.6$.}
    \label{fig:decay-y-06}
  \end{subfigure}

  \caption{\textbf{Ablations for the linear decay phase.}
  We train GPT-extra-small (43M) and vary the decay length and decay ratio of the linear decay phase.
  (a)--(c) The decay ratio is fixed at $0.1$, with decay lengths of 3000, 5000, and 8000 steps.
  (d)--(f) The decay length is fixed at 6000 steps, with decay ratios of $0.2$, $0.4$, and $0.6$.}
  \label{fig:decay-ablation}
\end{figure}

\subsection{Experiments Across Model Scales}\label{appendix:exp:scales}
We train GPT-small, GPT-medium, and GPT-large (\Cref{tab:model-arch}) on FineWeb for $20$ tokens per parameter with batch sizes $512$, $1{,}024$, and $1{,}536$, AdamW ($\beta_1=0.9$, $\beta_2=0.95$, weight decay $0.1$), and a WSD schedule with $4\%$ warmup and a linear decay to zero over the last $20\%$ of steps.
For each size, we sweep eight peak LRs from $3.125\times10^{-4}$ to $0.04$, doubling at each step.
We run EoS and EoC detection at ten checkpoints in the stable phase of each run, with compensated local simulations at LR factors from $0.6$ to $1.4$.
For the runs with large peak LRs, where no LR factor from $0.6$ stays bounded at many checkpoints, we re-estimate $\extraWD$ at LR factor $0.1$ and add LR factors $0.2$ to $0.5$.
\Cref{fig:tpp20-small,fig:tpp20-medium,fig:tpp20-large} show all $24$ runs, and \Cref{fig:lr-sweep-summary} summarizes them.
\Cref{fig-small-medium-large} uses one additional run per size at peak LR $1.25\times10^{-3}$, with checkpoints saved throughout training.

At the three smallest peak LRs, $17$, $18$, and $19$ of the $30$ checkpoints of each size are at EoS ($0.9\le q\le1.1$), and $26$, $27$, and $26$ satisfy the EoC criterion.
From peak LR $0.01$ on, only $12$ of the $90$ checkpoints are at EoS, and at $38$ of them no LR factor in the grid stays bounded.

\begin{figure}
    \centering
    \includegraphics[width=0.6\linewidth]{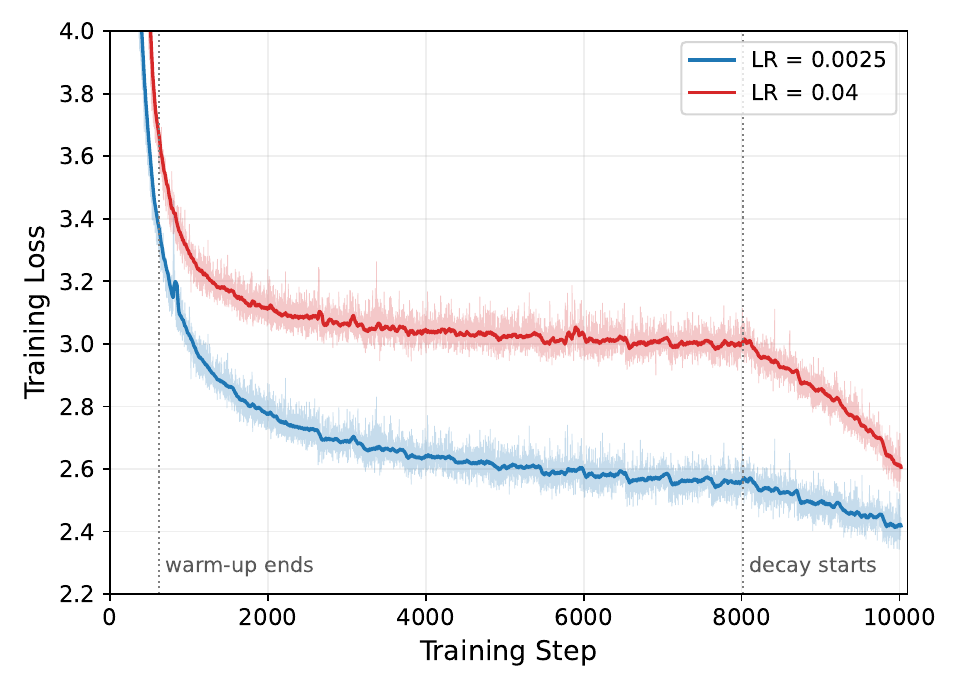}
\caption{\textbf{Training loss of GPT-large.}
We compare training loss versus steps at the
empirically optimal peak LR ($0.0025$) and the largest
tested peak LR ($0.04$).
Both runs are trained for 20 tokens per parameter.}
\label{fig:GPT2-large-training-loss}
\end{figure}

\begin{figure}[p]
    \centering
    \includegraphics[width=0.9\linewidth]{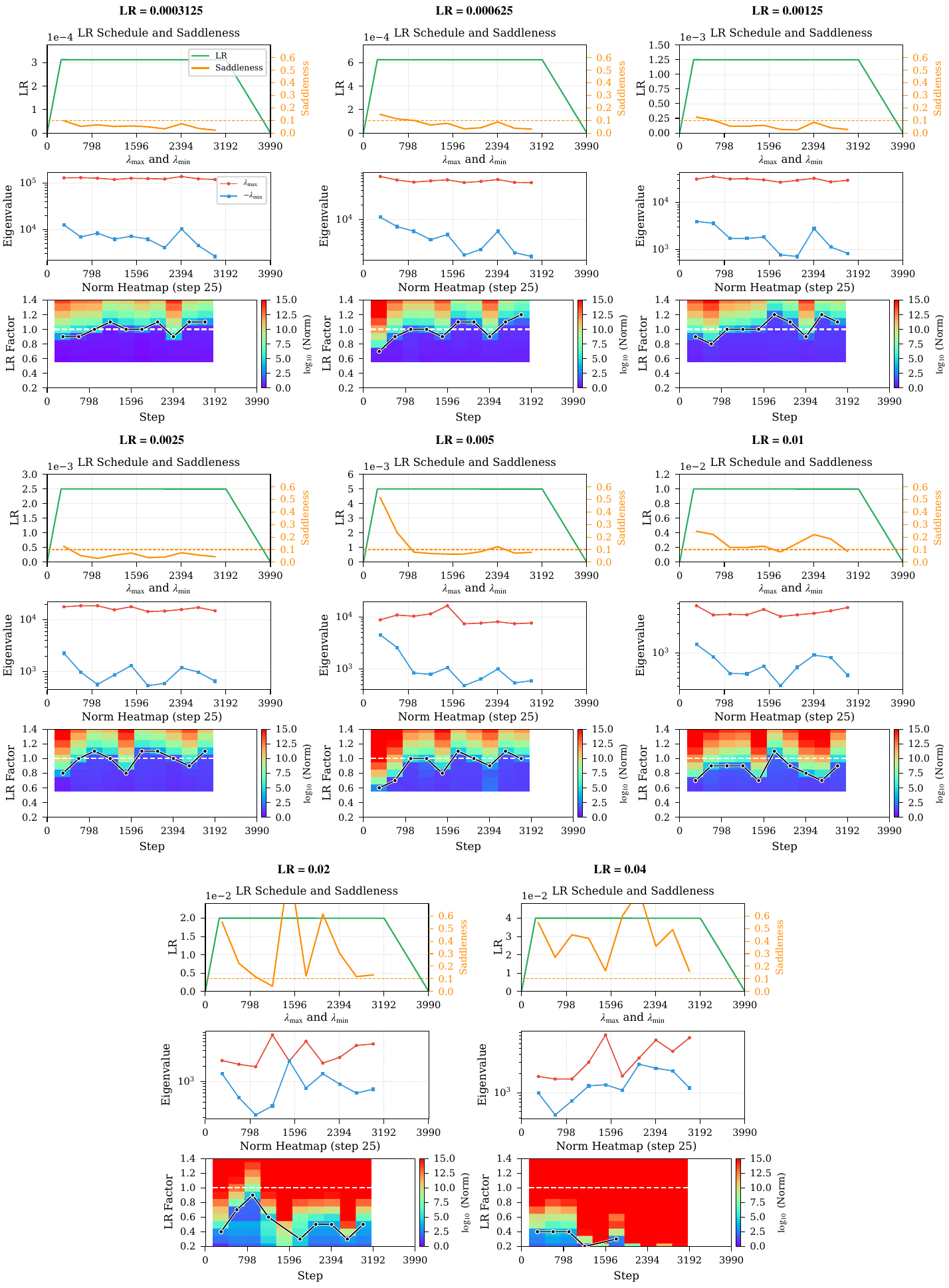}
    \caption{\textbf{GPT-small training stays near EoS up to peak LR $0.01$ and goes beyond EoS at larger peak LRs.}
    Up to peak LR $0.01$, the critical LR factor $q=\widehat\lr_\star/\lr$, the smallest LR factor whose simulation explodes, lies between $0.9$ and $1.2$ at $51$ of the $60$ checkpoints, and $35$ of these are at EoS ($0.9\le q\le1.1$).
    Over the same peak LRs, $44$ of the $60$ checkpoints satisfy the EoC criterion.
    For peak LRs from $0.02$ on, the training LR exceeds the estimated critical LR ($q<0.9$) at $19$ of the $20$ checkpoints, and at $5$ of the $10$ checkpoints of peak LR $0.04$ no LR factor in the grid stays bounded.
    In these two runs, the rows below LR factor $0.6$, the $-\lambda_{\min}$ curve, and the saddleness use the compensation coefficient $\extraWD$ estimated at LR factor $0.1$, while the rows from $0.6$ up use the one estimated at $0.6$. Because the two coefficients differ, a bounded factor $0.6$ sits above exploding lower factors at one checkpoint of peak LR $0.02$.
    Each panel is one peak LR, each heatmap column is one of ten checkpoints in the stable phase, and the black line marks the largest LR factor that stays bounded.
    Saddleness above $0.7$ runs off the top of the axis, and blank heatmap cells were not simulated.
    See \Cref{fig-small-medium-large} for the panel format and \Cref{appendix:exp:scales} for the setup.}
    \label{fig:tpp20-small}
\end{figure}

\begin{figure}[p]
    \centering
    \includegraphics[width=0.9\linewidth]{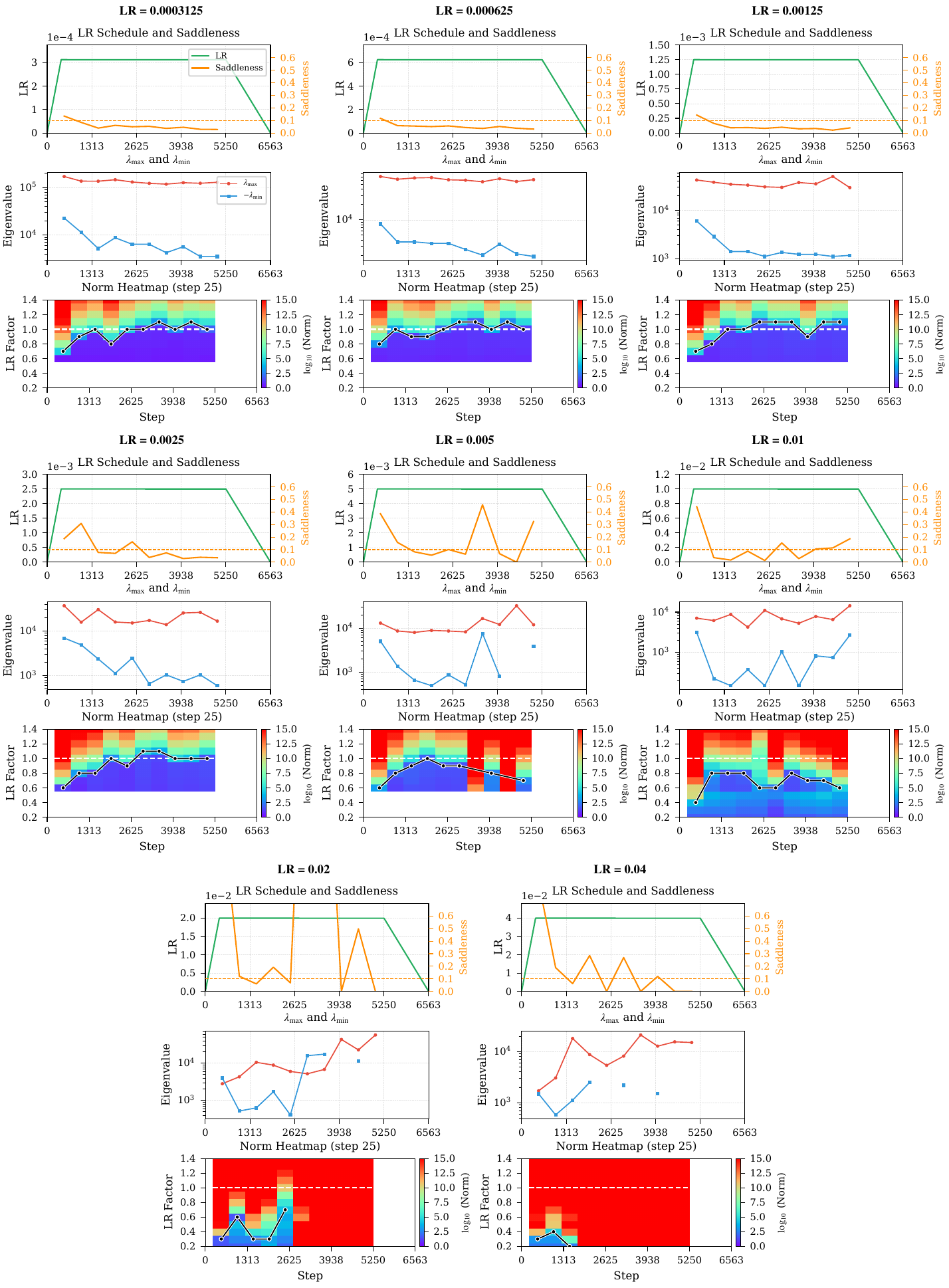}
    \caption{\textbf{GPT-medium training stays near EoS up to peak LR $5\times10^{-3}$ and goes beyond EoS at larger peak LRs.}
    Up to peak LR $5\times10^{-3}$, the critical LR factor $q=\widehat\lr_\star/\lr$, the smallest LR factor whose simulation explodes, lies between $0.9$ and $1.2$ at $43$ of the $50$ checkpoints, and $31$ of these are at EoS ($0.9\le q\le1.1$).
    Over the same peak LRs, $38$ of the $50$ checkpoints satisfy the EoC criterion.
    For peak LRs from $0.01$ on, the training LR exceeds the estimated critical LR ($q<0.9$) at $26$ of the $30$ checkpoints, and at $7$ of the $10$ checkpoints of peak LR $0.04$ no LR factor in the grid stays bounded.
    In these three runs, the rows below LR factor $0.6$, the $-\lambda_{\min}$ curve, and the saddleness use the compensation coefficient $\extraWD$ estimated at LR factor $0.1$, while the rows from $0.6$ up use the one estimated at $0.6$.
    Each panel is one peak LR, each heatmap column is one of ten checkpoints in the stable phase, and the black line marks the largest LR factor that stays bounded.
    Saddleness above $0.7$ runs off the top of the axis, and blank heatmap cells were not simulated.
    See \Cref{fig-small-medium-large} for the panel format and \Cref{appendix:exp:scales} for the setup.}
    \label{fig:tpp20-medium}
\end{figure}

\begin{figure}[p]
    \centering
    \includegraphics[width=0.9\linewidth]{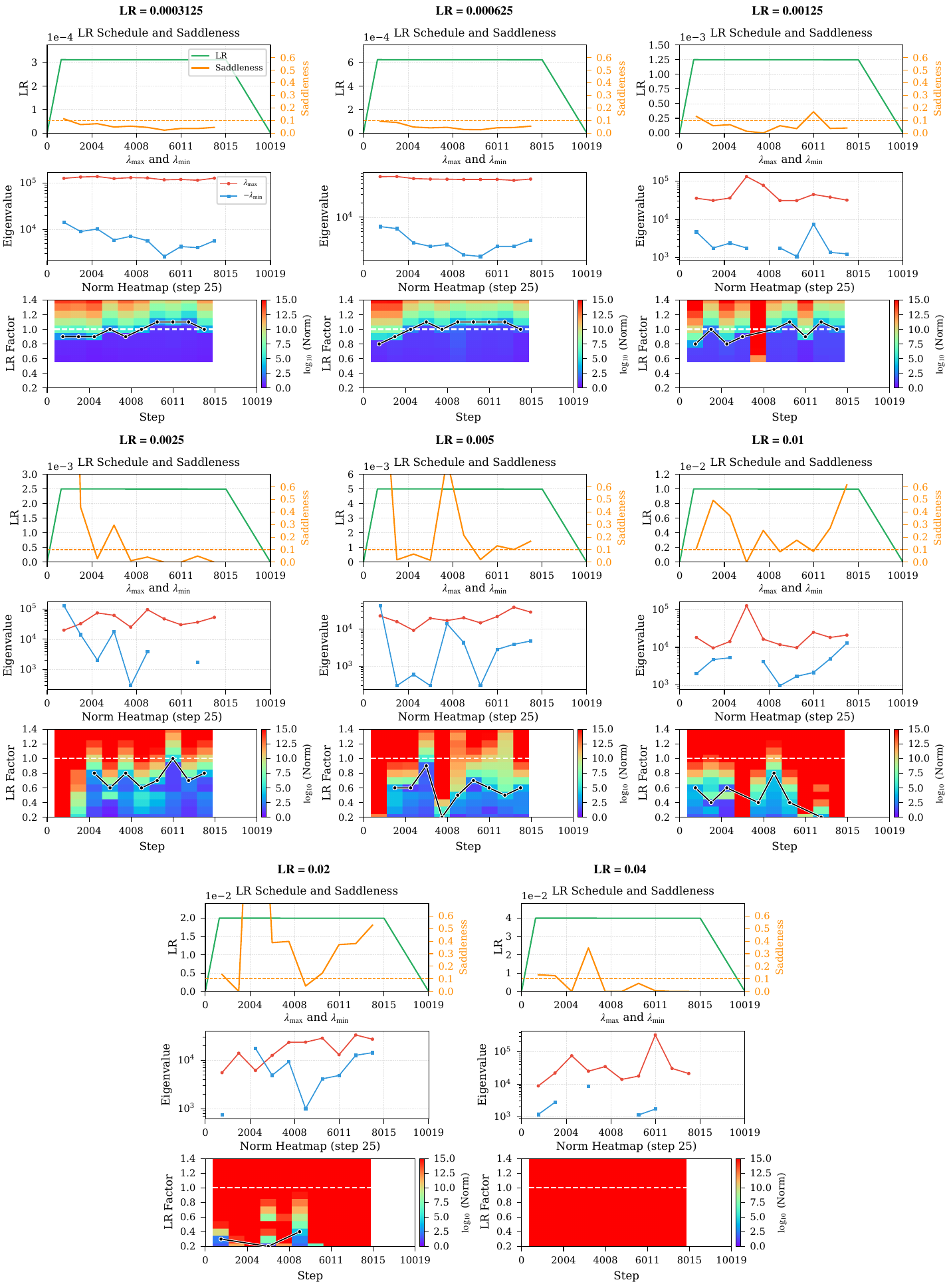}
    \caption{\textbf{GPT-large training stays near EoS up to peak LR $1.25\times10^{-3}$ and goes beyond EoS at larger peak LRs.}
    Up to peak LR $1.25\times10^{-3}$, the critical LR factor $q=\widehat\lr_\star/\lr$, the smallest LR factor whose simulation explodes, lies between $0.9$ and $1.2$ at $29$ of the $30$ checkpoints, and $19$ of these are at EoS ($0.9\le q\le1.1$).
    Over the same peak LRs, $26$ of the $30$ checkpoints satisfy the EoC criterion.
    For peak LRs from $2.5\times10^{-3}$ on, the training LR exceeds the estimated critical LR ($q<0.9$) at $44$ of the $50$ checkpoints, and at every checkpoint of peak LR $0.04$ no LR factor in the grid stays bounded.
    In these five runs, the rows below LR factor $0.6$, the $-\lambda_{\min}$ curve, and the saddleness use the compensation coefficient $\extraWD$ estimated at LR factor $0.1$, while the rows from $0.6$ up use the one estimated at $0.6$. Because the two coefficients differ, a bounded factor $0.6$ sits above exploding lower factors at two checkpoints of peak LR $2.5\times10^{-3}$ and one of $0.01$.
    Each panel is one peak LR, each heatmap column is one of ten checkpoints in the stable phase, and the black line marks the largest LR factor that stays bounded.
    Saddleness above $0.7$ runs off the top of the axis, and blank heatmap cells were not simulated.
    See \Cref{fig-small-medium-large} for the panel format and \Cref{appendix:exp:scales} for the setup.}
    \label{fig:tpp20-large}
\end{figure}

\begin{figure}[p]
    \centering
    \includegraphics[width=0.9\linewidth]{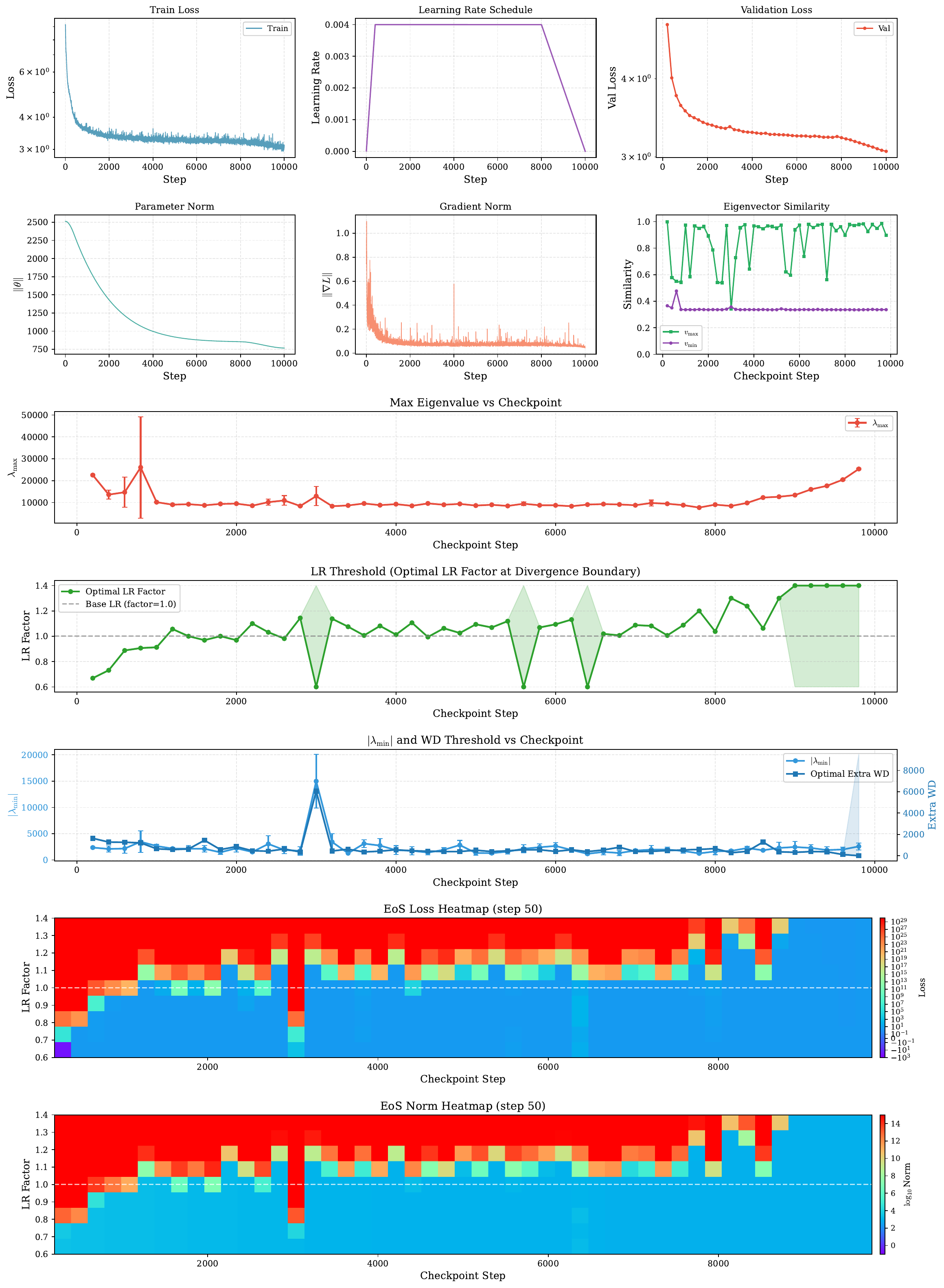}
    \caption{\textbf{Summary of a representative training run.}
    We show training dynamics for GPT-extra-small trained on FineWeb-10B using AdamW with LR $0.004$, weight decay $0.1$, and the WSD LR schedule.}
    \label{fig:training-summary}
\end{figure}
\clearpage

\section{Theoretical Analysis}\label{app:statements}\label{sec:Adaptive-Gradient-Methods}
\HWupdate{This appendix studies optimization dynamics on quadratic approximations of the loss around a fixed
training checkpoint, with the preconditioner held fixed. We describe how the mean and second
moment of the state (momentum and parameter displacement) evolve from step to step, and present
stability thresholds and error bounds for the EoC and EoS detectors.}
Proofs are in \Cref{app:proofs-moments,app:proof-power-iteration,fh:sec:appendix}.

\subsection{Setting and the Local Linear System}\label{app:setting}
\paragraph{The local system.}
We consider adaptive gradient methods (AGMs), including Adam~\citep{kingma2014adam} and AdamW~\citep{loshchilov2017decoupled}, with the updates of \eqref{eq:agm-main}, recalled here,
\begin{align}
    \vm_{t+1} &=  \beta_1 \cdot  \vm_t + (1-\beta_1) \cdot \grad \loss_{\batch_t}(\param_t) \label{eq:AGM1} \\
    \precond_{t+1} &= \Psi (\precond_t, \grad \loss_{\batch_t}(\param_t), \param_t) \label{eq:AGM2} \\
    \param_{t+1} &= \param_t - \lr \precond_{t+1}^{-1} \vm_{t+1} - \lr \wdc \param_t, \label{eq:AGM3}
\end{align}
where $\vm_t\in\R^d$ is the momentum and $\Psi$ maps a positive definite preconditioner and two vectors in $\R^d$ to a positive definite preconditioner.
This class includes SGD with and without momentum.
We assume that the preconditioner varies slowly along the training trajectory and freeze it over the short window of a local simulation at a checkpoint $\param_{t_0}$, writing $\mP\coloneqq\mP_{t_0}$.
This is an idealization. Adam's second-moment decay coefficient $\beta_2$ close to one motivates this approximation, but freezing $\mP_t$ still requires its change over the simulation window to be small.

Recall the mini-batch quadratic surrogate \eqref{eq:agm-quadratic-surrogate} from \Cref{app:agm-derivation}:
\begin{align}
    Q_{\batch_t}(\param) &\coloneqq \loss_{\batch_t}(\param_{t_0}) + (\param - \param_{t_0})^\top \grad \loss_{\batch_t}(\param_{t_0}) \notag \\
    &\quad + \tfrac{1}{2} (\param - \param_{t_0})^\top \hess \loss_{\batch_t}(\param_{t_0}) (\param - \param_{t_0}),\label{eq:Q_b_t}
\end{align}
where $\batch_t$ is the mini-batch sampled at local step $t$.
SGD with LR $\eta$ on this surrogate gives
\begin{align}
    \param_{t_0+t+1} = \param_{t_0+t} -  \eta \nabla Q_{\batch_t}(\param_{t_0+t}).\label{eq:update-SGD}
\end{align}
Write $\vvareps_t\coloneqq\param_{t_0+t}-\param_{t_0}$ for the displacement after $t$ local steps, as in the local AGM model of \Cref{app:agm-derivation}.
For the local simulation of an AGM, the momentum and parameter updates \eqref{eq:AGM1} and \eqref{eq:AGM3} are run on the surrogate \eqref{eq:Q_b_t}, with the gradient $\grad\loss_{\batch_t}(\param_t)$ replaced by $\grad Q_{\batch_t}(\param_{t_0+t})=\vg_{\batch_t}+\mH_{\batch_t}\vvareps_t$ and the preconditioner fixed at $\mP$. Here $\vg_{\batch_t}=\grad\loss_{\batch_t}(\param_{t_0})$ and $\mH_{\batch_t}=\hess\loss_{\batch_t}(\param_{t_0})$ are the mini-batch gradient and Hessian at the checkpoint, as in \Cref{sec:eoc-eos}. The local dynamics are
\begin{align}
  \vm_{t+1} &= \beta_1\vm_t+(1-\beta_1)
       (\vg_{\batch_t}+\mH_{\batch_t}\vvareps_t),\notag\\
  \vvareps_{t+1} &= \vvareps_t-\eta\mP^{-1}\vm_{t+1}
       -\eta\wdc(\vvareps_t+\param_{t_0}).\label{eq:local-agm-dynamics}
\end{align}

\paragraph{Compensation.}
Compensation adds a quadratic penalty to the local surrogate, measured in the norm of the frozen preconditioner. As in \Cref{sec:method-eos}, we write $G_t$ for this loss:
\begin{align*}
    G_t(\vvareps) \coloneqq \loss_{\batch_t}(\param_{t_0}) + \vvareps^\top \grad \loss_{\batch_t}(\param_{t_0}) + \tfrac{1}{2} \vvareps^\top \hess \loss_{\batch_t}(\param_{t_0}) \vvareps + \tfrac{\extraWD}{2} \|\vvareps\|^2_{\mP}.
\end{align*}
Here $\|\vvareps\|_{\mP}^2=\vvareps^\top\mP\vvareps$, with $\mP$ fixed throughout the local simulation. For SGD, $\mP=\mI$, so the penalty is $\tfrac{\extraWD}{2}\|\vvareps\|^2$.
Its gradient is $\grad G_t(\vvareps)=\vg_{\batch_t}+(\mH_{\batch_t}+\extraWD\mP)\vvareps$, so the penalty replaces the mini-batch Hessian $\mH_{\batch_t}$ by $\mH_{\batch_t}+\extraWD\mP$ in \eqref{eq:local-agm-dynamics}, which in preconditioned coordinates is a shift of the curvature by $\extraWD\mI$ (\Cref{prop:augmented-system}).
This is the same shift that the Euclidean penalty in \Cref{subsec:stochastic-eos-definition} applies to the raw Hessian.
The decoupled weight decay keeps its coefficient $\wdc$, and $\extraWD=0$ recovers the uncompensated surrogate.

\begin{assumption}\label{ass:minibatch}
The pairs $(\mH_{\batch_t},\vg_{\batch_t})$, $t\ge0$, are i.i.d., independent of the initial state $(\vm_0,\vvareps_0)$, each $\mH_{\batch_t}$ is symmetric, and both have finite second moments.
Their means are the population derivatives at the checkpoint,
\[
 \Ebb[\vg_{\batch_t}]=\grad\loss(\param_{t_0}),\qquad
 \Ebb[\mH_{\batch_t}]=\hess\loss(\param_{t_0}).
\]
\end{assumption}
\HWupdate{At each local step, we sample a fresh mini-batch independently from the same distribution. The expected batch Hessian and gradient equal their population counterparts at the fixed checkpoint. Together with the finite second moment assumption, this allows us to derive recursions for the mean and second moment of the local iterates.}

\paragraph{The augmented linear system.}\label{app:linear-system-AGM}
Centering the displacement at a reference point and changing to preconditioned coordinates turns \eqref{eq:local-agm-dynamics} into the stochastic linear system \eqref{eq:agm-local-linear-system} of \Cref{app:agm-derivation}. The following proposition gives its matrices in terms of the preconditioned mini-batch Hessian $\mD_{\batch_t}\coloneqq\mP^{-1/2}\mH_{\batch_t}\mP^{-1/2}$ defined there.
\begin{proposition}\label{prop:augmented-system}
Let $\vx_\star\in\R^d$ be any fixed vector.
We express the local iterates relative to
$\param_{t_0}+\vx_\star$ and define
\begin{align*}
    \vz_t&=
    \begin{bmatrix}\eta\mP^{-1/2}\vm_t\\\mP^{1/2}(\vvareps_t-\vx_\star)
    \end{bmatrix},
 &\vh_t&=(1-\beta_1)\eta\mP^{-1/2}
           (\vg_{\batch_t}+\mH_{\batch_t}\vx_\star).
\end{align*}
Under the local dynamics \eqref{eq:local-agm-dynamics}, $\vz_{t+1}=\mM_t(\eta)\vz_t+\vq_t$, where
\begin{align}
    \mM_t(\eta)&=
    \begin{bmatrix}
 \beta_1\mI & (1-\beta_1)\eta\mD_{\batch_t}\\
 -\beta_1\mI & (1-\eta\wdc)\mI-(1-\beta_1)\eta\mD_{\batch_t}
 \end{bmatrix},
    &\vq_t&=
    \begin{bmatrix}
 \vh_t\\-\vh_t-\eta\wdc\mP^{1/2}(\vx_\star+\param_{t_0})
 \end{bmatrix}.
 \label{eq:agm-augmented-matrix}
\end{align}
With compensation, the same holds with $\mD_{\batch_t}$ replaced by $\mD_{\batch_t}+\extraWD\mI$ in $\mM_t(\eta)$ and $\mH_{\batch_t}$ replaced by $\mH_{\batch_t}+\extraWD\mP$ in $\vh_t$. The transition matrix becomes
\begin{align}
  \widetilde\mM_t(\eta)
  =
  \begin{bmatrix}
    \beta_1\mI
    &
    (1-\beta_1)\eta(\mD_{\batch_t}+\extraWD\mI)
    \\
    -\beta_1\mI
    &
    (1-\eta\wdc)\mI
    -(1-\beta_1)\eta(\mD_{\batch_t}+\extraWD\mI)
  \end{bmatrix}.
  \label{eq:comp-agm-augmented-matrix}
\end{align}
Under \Cref{ass:minibatch}, $(\mM_t(\eta),\vq_t)$ is a function of $\batch_t$ alone and is independent of $\vz_t$.
\end{proposition}
\HWupdate{The choice of $\vx_\star$ changes the additive term $\vq_t$ but leaves $\mM_t(\eta)$ unchanged. For two trajectories using the same batches, $\vq_t$ cancels when we subtract their updates, so $\mM_t(\eta)$ determines how their difference evolves. Setting $\vx_\star=\vzero$ recovers the local system in \eqref{eq:agm-local-linear-system}. Compensation adds $\extraWD\mI$ to the preconditioned batch Hessian in the momentum update, while the weight decay update remains unchanged. The proof is given in \Cref{app:proof-linear-system}.

\paragraph{Moment recursions and mean square stability.}
\Cref{prop:augmented-system} expresses the local AGM updates as
$\vz_{t+1}=\mM_t(\eta)\vz_t+\vq_t$.
We now use this representation to study whether
$\Ebb[\|\vz_t\|^2]$ remains bounded.
For a fixed LR $\eta$, we write $\mM_t=\mM_t(\eta)$.
The following proposition gives the recursions for the mean
and second moment, together with a sufficient condition
for mean square stability.}

\begin{proposition}[Moment recursion of a stochastic linear system]\label{prop:moment-recursion}
Let $\vz_{t+1}=\mM_t\vz_t+\vq_t$, where $\vz_0$ has finite second moment and the pairs $(\mM_t,\vq_t)$, $t\ge0$, are i.i.d., have finite second moments, and are independent of $\vz_0$. Let $\mM\coloneqq\Ebb[\mM_t]$, $\vz_t^{(1)}\coloneqq\Ebb[\vz_t]$ and $\vz^{(2)}_t\coloneqq\vectorize(\Ebb[\vz_t\vz_t^\top])$, and set
\begin{align*}
    \mS\coloneqq\Ebb[\vq_t \otimes \mM_t + \mM_t \otimes \vq_t],\quad\vb_{\mathrm{quad}}\coloneqq\vectorize(\Ebb[\vq_t\vq_t^\top]).
\end{align*}
Then
\begin{align}
    \begin{bmatrix}
        \vz_{t+1}^{(1)}\\
        \vz_{t+1}^{(2)}
    \end{bmatrix} =
    \begin{bmatrix}
        \mM&0\\
        \mS&\Ebb[\mM_t \otimes \mM_t]
    \end{bmatrix}
    \begin{bmatrix}
        \vz_{t}^{(1)}\\
        \vz_{t}^{(2)}
    \end{bmatrix}+
    \begin{bmatrix}
        \Ebb[\vq_t]\\
        \vb_{\mathrm{quad}}
    \end{bmatrix},\label{eq:generic-moment-recursion}
\end{align}
and the spectral radius of the block matrix is $\max\{\rho(\mM),\rho(\Ebb[\mM_t\otimes\mM_t])\}$.
If $\rho(\Ebb[\mM_t\otimes\mM_t])<1$, both moments converge to finite limits, so $\sup_{t\ge0}\Ebb\|\vz_t\|^2<\infty$. The transition $\mM_t$ and the forcing $\vq_t$ may be correlated.
\end{proposition}
Under \Cref{ass:minibatch}, the augmented system of \Cref{prop:augmented-system} is of this form with $\mM_t=\mM_t(\eta)$, and the compensated system with $\mM_t=\widetilde\mM_t(\eta)$.
The two diagonal blocks govern the mean and the second moment. The second-moment transition splits as
\[
 \Ebb[\mM_t \otimes \mM_t]
 = \mM \otimes \mM
 + \Ebb[(\mM_t-\mM)\otimes(\mM_t-\mM)],
\]
since $\otimes$ is bilinear and $\Ebb[\mM_t-\mM]=\zeroB$. For the augmented system this is the AGM analogue of the decomposition
\begin{align*}
    \Ebb[\mM_t^+\otimes\mM_t^+]
    =
    \Ebb[\mM_t^+]\otimes\Ebb[\mM_t^+]
    +
    \mathrm{Cov}(\mM_t^+)
\end{align*}
in \Cref{sec:main-results}. The mean term is the full-batch transition, and the centered term records the variation of the mini-batch Hessian.
The limiting moments of the forced system need not be zero.
The proof is in \Cref{app:proof-moment-recursion}.

We call the homogeneous system \emph{strictly mean-square stable} if $\Ebb\|\vz_t\|^2\to0$ for every deterministic $\vz_0$. By the second-moment characterization lemma (\Cref{fh:lem:covariance}) this holds if and only if $\rho(\Ebb[\mM_t(\eta)\otimes\mM_t(\eta)])<1$. For the same state space, strict homogeneous stability implies bounded moments for the affine system under \Cref{prop:moment-recursion}. Comparisons of their critical LRs must also project onto the same subspace, as we explain in \Cref{app:critical}.

For SGD the displacement block of the augmented system evolves on its own, because the momentum is the current gradient. Applying \Cref{prop:moment-recursion} to that block gives the moment recursion of \Cref{prop:stochastic-quadratic-moment-dynamics} for the surrogate of a checkpoint. The stability statements follow in \Cref{app:critical}.
\begin{corollary}[Moment recursion for SGD]
\label{thm:linear-system-sgd}
Let \Cref{ass:minibatch} hold, let $\beta_1=0$, $\mP=\mI$ and $\wdc=0$, and suppose there exists a $\param_\star$ such that
\begin{align}
    \nabla \loss (\param_{t_0}) + \hess \loss (\param_{t_0})(\param_\star-\param_{t_0}) =0. \label{eq:existence-of-optima}
\end{align}
Under the SGD update \eqref{eq:update-SGD}, the first and second moments
\begin{align*}
\vvareps_t^{(1)} &\coloneqq \Ebb \big[\param_{t_0+t} - \param_\star \big], \\
\vvareps_t^{(2)} &\coloneqq \vectorize \Big(\Ebb [(\param_{t_0+t} - \param_\star)(\param_{t_0+t} -\param_\star)^\top] \Big)
\end{align*}
satisfy
\begin{align}
    \begin{bmatrix}
        \vvareps_{t+1}^{(1)}\\
        \vvareps_{t+1}^{(2)}
    \end{bmatrix} =
    \begin{bmatrix}
        \Ebb[\mM_t]&\zeroB\\
        \mS&\Ebb[\mM_t \otimes \mM_t]
    \end{bmatrix}
    \begin{bmatrix}
        \vvareps_{t}^{(1)}\\
        \vvareps_{t}^{(2)}
    \end{bmatrix}+
    \begin{bmatrix}
        \zeroB\\
        \vb_{\mathrm{quad}}
    \end{bmatrix},\label{eq:linear-system-SGD}
\end{align}
where
\begin{align*}
\mM_t &\coloneqq \mI - \lr \mH_{\batch_t}, &
\vr_t&\coloneqq\vg_{\batch_t}+\mH_{\batch_t}(\param_\star-\param_{t_0}),\\
\mS&\coloneqq-\eta\,\Ebb[\mM_t\otimes\vr_t+\vr_t\otimes\mM_t], &
\vb_{\mathrm{quad}}&\coloneqq\eta^2\vectorize\bigl(\Ebb[\vr_t\vr_t^\top]\bigr),
\end{align*}
and $\Ebb[\vr_t]=\vzero$.
\end{corollary}
The corollary gives closed recursions for the mean and the second moment, driven by the constant terms $\vb_{\mathrm{quad}}$ and $\mS$. It makes no stability claim by itself.
Applying the strict spectral condition of \Cref{prop:moment-recursion} gives the affine boundedness assertion in \Cref{prop:stochastic-quadratic-moment-dynamics}.
The stationarity condition \eqref{eq:existence-of-optima} removes the forcing from the mean recursion. It holds for some $\param_\star$ whenever $\grad\loss(\param_{t_0})$ lies in the range of $\hess\loss(\param_{t_0})$, in particular whenever the Hessian is invertible.
The proof is in \Cref{app:proof-moments}.

\subsection{Stability of the Local Dynamics}\label{app:critical}
Negative effective curvature makes the local dynamics unstable at every positive LR.
On the strictly positive subspace, the LRs for strict homogeneous stability form a single interval.

\paragraph{Negative curvature.}
The first instability mode does not depend on the LR. Write $\mD=\Ebb[\mD_{\batch_t}]$ for the population preconditioned Hessian ($\mD_{t_0}$ in \Cref{app:agm-derivation}).
\begin{lemma}[Negative effective curvature]\label{app:agm-EoC}
Let \Cref{ass:minibatch} hold. If the effective curvature $\mA=\mD+\wdc\mI$ has a negative eigenvalue, then the homogeneous AGM dynamics $\vz_{t+1}=\mM_t(\eta)\vz_t$ of \Cref{prop:augmented-system} are not mean-square stable for any $\eta>0$.
\end{lemma}
Along an eigendirection of $\mD$ with negative effective curvature, the mean dynamics reduce to a $2\times2$ system with a real eigenvalue above one for every $\eta>0$. Momentum changes the rate of growth, not its existence. A smaller LR slows the escape and does not prevent it.
For SGD the rate is explicit. Write $\mH_{t_0}=\hess\loss(\param_{t_0})$ for the population Hessian at the checkpoint.
\begin{corollary}[Negative-curvature divergence]\label{cor:neg-curvature-divergence}
Under the assumptions of \Cref{thm:linear-system-sgd}, suppose
$\mH_{t_0}$ has a unit eigenvector $\vu$ with eigenvalue $\lambda<0$.
The population quadratic is unbounded below. For every $\eta>0$,
\[
\vu^\top\Ebb[\param_{t_0+t}-\param_\star]
=(1-\eta\lambda)^{t}\,\vu^\top\Ebb[\param_{t_0}-\param_\star].
\]
Thus the mean diverges whenever the initial mean has a nonzero component along $\vu$. The homogeneous system is not mean-square stable.
\end{corollary}
This is the negative-curvature failure of \Cref{subsec:neg-divergence}, now for the surrogate of a checkpoint. The proofs of the lemma and the corollary are in \Cref{app:proof-corollaries}.

\paragraph{Critical LR.}\label{app:critical-lr-agm}
To isolate the instability induced by the LR, we restrict the dynamics to the subspace of positive effective curvature, as in \Cref{def:agm-critical-lr}.
Let $\Pi^+$
be the orthogonal projection onto the strictly positive eigenspace of the
symmetric effective mean $\mD+\wdc\mI$.
All projected matrices below act on $\operatorname{range}(\Pi^+)$.
Both the momentum and the displacement are restricted to this subspace, which we assume is nonempty.
Directions of zero effective curvature are dropped because their mean transition has an eigenvalue equal to one, so no LR would be strictly mean-square stable if they were kept.
Let $\mM_t^+(\eta)$ denote the restricted transition and define
\begin{align}
    \mK^+(\eta)
    \coloneqq
    \Ebb\left[\mM_t^+(\eta)\otimes\mM_t^+(\eta)\right].
    \label{eq:agm-projected-second-moment-operator}
\end{align}
The strictly mean-square stable LR set of the projected homogeneous system is
\begin{align}
    \mathcal S_{\rm AGM}(\param_{t_0})
    \coloneqq \{\eta>0:\rho(\mK^+(\eta))<1\}.
    \label{eq:agm-stable-lr-set-app}
\end{align}
\WCupdate{On the range of $\Pi^+$, the effective mean is positive definite by construction.}
The next theorem shows that under \Cref{ass:minibatch} this set is the interval $(0,\eta_\star)$ for a unique finite $\eta_\star$, even when $\beta_1>0$. Momentum makes the transition nonsymmetric but does not create a second stable interval.
\HWupdate{
\begin{theorem}[Unique stability threshold, restated from \Cref{fh:thm:decay-exact}]
\label{fh:thm:decay}
Let \Cref{ass:minibatch} hold.
Fix the preconditioner $\mP\succ0$, $0\le\beta_1<1$, $\wdc\ge0$, and the batch distribution.
Suppose $\Pi^+\ne0$ and
\[
    \vu^\top
    \bigl(\mP^{-1/2}\mH_{t_0}\mP^{-1/2}+\wdc\mI\bigr)
    \vu>0
    \qquad
    \text{for every }
    \vu\in\operatorname{range}(\Pi^+)\setminus\{\vzero\}.
\]
Then there exists a unique $\eta_\star\in(0,\infty)$ such that
\[
    \rho\bigl(\mK^+(\eta)\bigr)
    \begin{cases}
        <1, & 0<\eta<\eta_\star,\\
        =1, & \eta=\eta_\star,\\
        >1, & \eta>\eta_\star.
    \end{cases}
\]
Consequently, the projected homogeneous system
$\vz_{t+1}^+=\mM_t^+(\eta)\vz_t^+$
satisfies
\[
    \lim_{t\to\infty}\Ebb\bigl[\|\vz_t^+\|^2\bigr]=0
    \quad\text{for every deterministic initial state }\vz_0^+
\]
if and only if $0<\eta<\eta_\star$.
\end{theorem}

}
By \Cref{prop:agm-critical-lr-equivalence}, this threshold equals the AGM critical LR defined by bounded affine moments in \Cref{def:agm-critical-lr}.
\WCupdate{\Cref{fh:sec:main} bounds the finite-horizon error of the expected detector.}
The same statement holds for the compensated system of \Cref{prop:augmented-system} whenever $\mD+(\wdc+\extraWD)\mI\succ0$, since it is the same system with a shifted effective mean.
The exact stability criterion is stated in \Cref{fh:sec:criterion}. The proof of the AGM threshold statement is in \Cref{fh:sec:decay}.

Above the critical LR, the second moment of some initial state grows without bound.
\begin{corollary}[Divergence above the critical LR]\label{cor:pos-curvature-divergence}
Under the assumptions of \Cref{fh:thm:decay}, for every $\eta>\eta_\star$ there is a deterministic initial state $\vz_0$ of the projected homogeneous system for which $\Ebb\|\vz_t\|^2\to+\infty$.
\end{corollary}
The proof is in \Cref{app:proof-pos-curvature}. It uses only that the second-moment spectral radius exceeds one for $\eta>\eta_\star$.

\WCupdate{For SGD, $\beta_1=0$, $\mP=\mI$ and $\wdc=0$, the projected system reduces to the displacement restricted to $\operatorname{range}(\Pi^+_{t_0})$, where $\Pi^+_{t_0}$ is the spectral projector of $\mH_{t_0}$ onto its strictly positive eigenspace,}
\begin{align*}
  \mM_t^+(\eta)&=\WCupdate{\mI-\eta\Pi^+_{t_0}\mH_{\batch_t}\Pi^+_{t_0}},
\qquad \vx_{t+1}=\mM_t^+(\eta)\vx_t,\\
\lr_\star&=\sup\{\eta>0:\rho(\Ebb[\mM_t^+(\eta)\otimes\mM_t^+(\eta)])<1\}.
\end{align*}
In words, $\lr_\star$ is the largest LR below which every perturbation in the convex subspace decays in mean square, and \Cref{cor:pos-curvature-divergence} says that above it some perturbation grows without bound. This is the positive-curvature failure of \Cref{sec:stochastic-quadratic-stability}.
The stability notion in \Cref{def:agm-critical-lr} is the bounded affine notion of the main text.
\WCupdate{We project after every update, so a batch Hessian cannot move mass from the discarded directions back into the range of $\Pi^+_{t_0}$.}
\WCupdate{This is the projection of \Cref{def:stochastic-quadratic-critical-lr-general}, so for SGD the two critical LRs coincide.}
\Cref{prop:agm-critical-lr-equivalence} justifies computing this threshold from strict homogeneous stability.
Below the threshold, strict decay implies bounded affine moments.
Above it, differences of affine trajectories driven by the same batches rule out boundedness for every start.
The proof is in \Cref{app:proof-agm-critical-lr-equivalence}.
The two stable sets can differ at the endpoint, but their suprema agree.

\WCupdate{Zero-curvature directions must be dropped. For example, take SGD with $\mH_{\batch_t}=\operatorname{diag}(1,0)$ and residual gradient $\vr_t=(0,\xi_t)^\top$, where the $\xi_t$ are independent uniform signs. From zero initial error, the second coordinate satisfies}
\[
 \varepsilon_{t,2}=-\eta\sum_{s=0}^{t-1}\xi_s,
 \qquad \Ebb[\varepsilon_{t,2}^2]=t\eta^2.
\]
\WCupdate{If the zero-curvature coordinate were kept, no positive LR would give bounded second moments. The projection onto the positive eigenspace keeps only the first coordinate, and its critical LR is $2$.}

\paragraph{Mean-square and almost-sure stability.}\label{app:mean-square-vs-almost-sure}
Consider the homogeneous recursion $\vvareps_{t+1}=\mM_t\vvareps_t$ with i.i.d.\ $\mM_t$ of finite second moment. Recall the strict notion of \Cref{app:setting}. The recursion is strictly mean-square stable if $\Ebb[\|\vvareps_t\|^2]\to0$ for every deterministic initial state, and this holds if and only if the second-moment map below has spectral radius less than one.
The main-text definition (\Cref{def:mean-square-stability}) asks only for a bounded second moment.
The AGM critical LR in \Cref{def:agm-critical-lr} uses this same notion.
\Cref{prop:agm-critical-lr-equivalence} shows that its threshold agrees with the strict homogeneous threshold on the same positive-mean subspace.
Boundedness at the endpoint may depend on affine forcing.
The spectral analysis below and in \Cref{fh:sec:appendix} uses strict homogeneous stability.
\WCupdate{Exponential mean-square decay requires constants $K<\infty$ and $0<\varrho<1$ such that}
\[
    \Ebb[\|\vvareps_t\|^2]
    \leq K \varrho^t \|\vvareps_0\|^2.
\]
Since $\mM_t$ is independent of $\vvareps_t$, the second moment $\Ebb[\vvareps_t\vvareps_t^\top]$ evolves as
$\Ebb[\vvareps_{t+1}\vvareps_{t+1}^\top]=\mathcal{T}(\Ebb[\vvareps_t\vvareps_t^\top])$, where
\[
    \mathcal{T}(\mX)
    =
    \Ebb[\mM_t \mX \mM_t^\top].
\]
Its matrix representation is the $\Ebb[\mM_t\otimes\mM_t]$ of \Cref{prop:moment-recursion}, and strict mean-square stability is the spectral condition $\rho(\mathcal{T})<1$.
This characterization makes mean-square stability useful for analyzing SGD
\citep{wu2018sgdselects,ma2021linear,wuwangsu2022alignment,mulayoff2024exact}.
Almost-sure stability requires
$\|\vvareps_t\|\to 0$ with probability one.
Its exponential version is governed by the top Lyapunov exponent of the random matrix product,
\[
    \limsup_{t\to\infty}
    \frac{1}{t}
    \log \|\mM_{t-1}\cdots\mM_0\|
    <0
    \qquad \text{almost surely},
\]
under the usual integrability conditions
\citep{furstenberg1960products,bougerol2012products}.

The two notions are not equivalent.
\WCupdate{Exponential mean-square decay implies almost-sure exponential stability by Markov's inequality and the Borel--Cantelli lemma.}
The converse is false because rare expanding trajectories can dominate the second moment without changing the behavior of a typical sample path.
For example, consider the scalar recursion
$\varepsilon_{t+1}=a_t\varepsilon_t$, where the $a_t$ are i.i.d.\ and
\[
    a_t =
    \begin{cases}
        1/2, & \text{with probability } 0.99,\\
        10,  & \text{with probability } 0.01.
    \end{cases}
\]
Here
\begin{align*}
 \Ebb[\log a_t]&=0.99\log\tfrac12+0.01\log 10\approx-0.66<0,\\
 \Ebb[a_t^2]&=0.99\cdot\tfrac14+0.01\cdot100=1.2475>1.
\end{align*}
Since $\log|\varepsilon_t|=\log|\varepsilon_0|+\sum_{s=0}^{t-1}\log a_s$, the strong law of large numbers gives $\tfrac1t\log|\varepsilon_t|\to\Ebb[\log a_t]<0$ almost surely, so $|\varepsilon_t|\to 0$ almost surely, whereas $\Ebb[|\varepsilon_t|^2]=\Ebb[a_t^2]^{\,t}\,|\varepsilon_0|^2$ by independence, so
$\Ebb[|\varepsilon_t|^2]$ grows exponentially.
The second moment can therefore reveal growth that an almost-sure criterion does not detect.
We use mean-square stability because it controls the expected perturbation energy and admits a closed second-moment recursion.
The experiments use a single-trajectory diagnostic, whose threshold need not coincide with this population threshold.

\subsection{Stochastic Power Iteration for EoC}\label{app:proof-of-min-eigen}
We state the accuracy of the stochastic power iteration of
\Cref{sec:method-eoc}. The theorem separates the error due to the finite
horizon, the curvature noise, and the bisection resolution.

\paragraph{The system.}
We use the effective curvature of \Cref{app:critical}, writing the population preconditioned Hessian as $\mD_{t_0}$, the $\mD$ of that section:
\begin{align*}
 \mA&\coloneqq\mD_{t_0}+\wdc\mI,
 &\mA_{\batch_t}&\coloneqq
 \mP^{-1/2}\mH_{\batch_t}\mP^{-1/2}+\wdc\mI,\\
 \mDelta_t&\coloneqq\mA_{\batch_t}-\mA.
\end{align*}
Thus the deterministic weight-decay term is included in both curvatures and
cancels from the noise $\mDelta_t$. For SGD without weight decay,
$\mA=\mH_{t_0}$, so the same procedure estimates the minimum eigenvalue of the raw population Hessian.
Let $\vu_{\min}$ be a unit eigenvector of $\mA$ for its smallest eigenvalue.
The shift to estimate is
\begin{equation}\label{eq:pi-target}
 \lambda_\star\coloneqq\max\{0,-\lambda_{\min}(\mA)\},
\end{equation}
the smallest nonnegative shift making $\mA+\lambda\mI$ positive semidefinite.

A query at shift $\lambda\in[0,\Lambda]$ starts from the same fixed unit
vector $\vx_0$ and uses the inner stepsize $\etapi$ for $T_{\mathrm p}$ steps:
\begin{align}
 \vx_{t+1}
 &=\bigl[\mI-\etapi(\mA_{\batch_t}+\lambda\mI)\bigr]\vx_t\notag\\
 &=\bigl[\mM(\lambda)-\etapi\mDelta_t\bigr]\vx_t,
 \qquad t=0,\ldots,T_{\mathrm p}-1,
 \label{eq:pi-iteration}\\
 \mM(\lambda)&\coloneqq\mI-\etapi(\mA+\lambda\mI).
 \notag
\end{align}
Here $\etapi$ is the power-iteration stepsize, distinct from the training LR
$\eta$, and $T_{\mathrm p}$ is its horizon, distinct from the local-simulation
horizon $T$. A query explodes if
$\max_{0\le t\le T_{\mathrm p}}\|\vx_t\|>B$, where $B>1$ is the cutoff.
Starting from $[0,\Lambda]$, \Cref{alg:min_eig_stochastic_bisect} queries the
\WCupdate{midpoint of the current search interval. An exploding query replaces the lower}
endpoint and a non-exploding query replaces the upper endpoint. After $N$
queries the output $\widehat\lambda$ is the upper endpoint, and
$\widehat\lambda_{\min}=-\widehat\lambda$.

Within each query, let $\mathcal F_t$ contain all observations of earlier
queries and the current observations $\mDelta_0,\ldots,\mDelta_{t-1}$.
The queried shift and $\vx_t$ are therefore $\mathcal F_t$-measurable.
All vector norms below are Euclidean and matrix norms are spectral.

\begin{assumption}[Bounded curvature noise]\label{ass:pi-noise}
For each query and each step $t$, the symmetric noise $\mDelta_t$ is
$\mathcal F_{t+1}$-measurable and satisfies
\begin{equation}\label{eq:pi-noise}
 \Ebb[\mDelta_t\mid\mathcal F_t]=\zeroB,
 \qquad \|\mDelta_t\|\le\sigma\quad\text{almost surely},
\end{equation}
for a constant $\sigma>0$.
\end{assumption}
Fresh independent mini-batches, as in \Cref{app:setting}, give the conditional
mean-zero property. The almost-sure norm bound is an additional assumption
used by the concentration inequalities in the proof.

\begin{assumption}[\WCupdate{Search interval and initial overlap}]\label{ass:pi-bracket}
The population curvature satisfies $\lambda_{\min}(\mA)<0$ and
$\lambda_\star\le\Lambda$. The fixed initialization satisfies
\[
 \|\vx_0\|=1,
 \qquad |\vu_{\min}^{\top}\vx_0|\ge\varepsilon,
 \qquad 0<\varepsilon\le1.
\]
\end{assumption}
\WCupdate{The search interval must contain the target. The overlap supplies a nonzero initial}
component in the direction that grows below $\lambda_\star$.
Initializing with the checkpoint momentum does not by itself ensure this
condition.

\begin{assumption}[Power-iteration stepsize]\label{ass:pi-step}
The inner stepsize satisfies
\begin{equation}\label{eq:pi-lr-condition}
 \etapi>0,
 \qquad \etapi\bigl(\lambda_{\max}(\mA)+\Lambda\bigr)\le1.
\end{equation}
\end{assumption}
This makes the mean transition $\mM(\lambda)$ positive semidefinite for
every queried shift. Above $\lambda_\star$ its norm is
$1-\etapi(\lambda-\lambda_\star)$, which controls the probability of a
false explosion. The explosion argument below $\lambda_\star$ does not
need this stepsize condition.

The following result applies to any fixed finite horizon. It imposes no
upper bound on that horizon.
\begin{theorem}[Accuracy of the stochastic power iteration]
\label{thm:lambda_hat_bound}
Let \Cref{ass:pi-noise,ass:pi-bracket,ass:pi-step} hold. Fix a cutoff $B>1$,
positive integers $N$ and $T_{\mathrm p}$, and $0<\delta<1$. Define
\begin{align}
 g_-&\coloneqq
 \frac{4\etapi\sigma^2B^2}{\varepsilon^2}\log\frac N\delta,
 &g_+&\coloneqq
 2\etapi\sigma^2\max\left\{1,\frac{\log(N/\delta)}{\log B}\right\},
 \notag\\
 e_{\mathrm p}&\coloneqq
 \frac{(2B/\varepsilon)^{1/T_{\mathrm p}}-1}{\etapi}.
 \label{eq:pi-gaps}
\end{align}
Then the bisection output satisfies, with probability at least $1-\delta$,
\begin{equation}\label{eq:final_lambda_bound}
 \lambda_\star-\max\{g_-,e_{\mathrm p}\}
 \le\widehat\lambda
 \le\lambda_\star+g_++\frac{\Lambda}{2^N}.
\end{equation}
The finite-horizon term $e_{\mathrm{p}}$ obeys $e_{\mathrm p}\le g_-$ if and only if
\begin{equation}\label{eq:pi-Tmin}
 T_{\mathrm p}\ge
 \left\lceil\frac{\log(2B/\varepsilon)}
 {\log(1+\etapi g_-)}\right\rceil.
\end{equation}
\end{theorem}

\paragraph{Interpretation of the bound.}
The lower error is the larger of the noise term $g_-$ and the finite-horizon
term $e_{\mathrm p}$. For fixed $B$, $\varepsilon$, and $\etapi$,
\[
 e_{\mathrm p}
 =\frac{\log(2B/\varepsilon)}{\etapi T_{\mathrm p}}
   +O(T_{\mathrm p}^{-2})
 \qquad\text{as }T_{\mathrm p}\to\infty.
\]
This is the margin needed for the initial component of size $\varepsilon$
to grow past $B$ within the allotted steps. Above $\lambda_\star$, the
error consists of the noise term $g_+$ and the bisection resolution
$\Lambda2^{-N}$. The displayed noise terms do not decrease with the horizon
when the stepsize, cutoff, overlap, confidence, and number of queries are
fixed. A longer horizon therefore reduces the finite-horizon term but does
not make this bound vanish at a fixed stepsize.
\WCupdate{The theorem assumes the search-interval condition and the overlap and does not verify them.}
It also assumes that the population minimum is negative. It bounds the
estimation error under that assumption and does not certify the sign of the
population minimum from the sign of the returned estimate.
Its probability guarantee is for each fixed choice of $T_{\mathrm p}$.
The proof is in \Cref{app:proof-power-iteration}.

\clearpage
\subsection{Compensated Local Simulation for EoS}\label{fh:sec:main}
\HWupdate{Our goal is to estimate the critical LR $\lrstar$ used to define EoS. The detector runs a compensated local model for only $T$ steps. Compensation changes the system being tested, whose critical LR we denote by $\lrP$, and a finite simulation can accept LRs above this threshold. \Cref{fh:thm:main} bounds the resulting error of the expected detector by accounting for these two effects. We first specify the systems and the detector, then introduce the quantities needed to bound each error.}

\paragraph{Curvature at the checkpoint.}

We continue with the checkpoint, frozen preconditioner and local model of \Cref{app:setting}. Recall the effective curvature and its mini-batch version from \Cref{app:proof-of-min-eigen},
\begin{equation}\label{fh:eq:effective}
 \mA\coloneqq\mD_{t_0}+\wdc\mI,\qquad
 \mA_{\batch_t}\coloneqq\mP^{-1/2}\mH_{\batch_t}\mP^{-1/2}+\wdc\mI,\qquad
 \mD_{t_0}\coloneqq\mP^{-1/2}\mH_{t_0}\mP^{-1/2},
\end{equation}
and $\mDelta_t\coloneqq\mA_{\batch_t}-\mA$ is the deviation of the mini-batch effective curvature from its mean. It is symmetric with $\Ebb\mDelta_t=0$. Let $\Pi^+$ be the orthogonal projector onto the strictly positive eigenspace of $\mA$, of dimension $p\ge1$, and choose orthogonal coordinates adapted to it, so that
\[
 \mA=\operatorname{diag}(\mA^+,\mA^-),\qquad\mA^+\succ0,\qquad\mA^-\preceq0,
\]
\WCupdate{where $\mA^+$ is the $p\times p$ block on the positive eigenspace and $\mA^-$ the block on the discarded directions, absent when $p=d$. For a symmetric $d\times d$ matrix $\mM$ written in these coordinates, $\mM^+$ denotes its $p\times p$ block on the positive eigenspace and $\mM^-$ its discarded block, so that $\Pi^+\mM\Pi^+=\operatorname{diag}(\mM^+,0)$ is $\mM^+$ padded with zeros. In particular $\mDelta_t^+$ is the block of the noise on the positive eigenspace, of the same size as $\mA^+$, and $\mDelta_t^\perp\coloneqq\mDelta_t-\Pi^+\mDelta_t\Pi^+$ collects the discarded and cross blocks. We use $L\coloneqq\lambda_{\max}(\mA)=\lambda_{\max}(\mA^+)$ throughout. Thus the compensated mean has largest eigenvalue $L+\extraWD$. The compensation shifts the effective curvature to $\mA_{\mathsf P}\coloneqq\mA+\extraWD\mI$.}

\paragraph{The compensated local simulation.}
The detector runs the local AGM update of \Cref{app:setting} on the compensated loss $G_t$. We retain $\vvareps_t$ for the displacement from the checkpoint, so the simulated parameter is $\param_{t_0}+\vvareps_t$. A run at LR $\lr>0$ from a prescribed displacement $\vvareps_0$ and the checkpoint momentum $\vm_0=\vm_{t_0}$ follows
\begin{equation}\label{fh:eq:compensated-model}
\begin{aligned}
 \vm_{t+1}&=\beta_1\vm_t+(1-\beta_1)\bigl[\vg_{\batch_t}+(\mH_{\batch_t}+\extraWD\mP)\vvareps_t\bigr],\\
 \vvareps_{t+1}&=(1-\lr\wdc)\vvareps_t-\lr\mP^{-1}\vm_{t+1}-\lr\wdc\param_{t_0}.
\end{aligned}
\end{equation}

\paragraph{The two critical LRs.}
The homogeneous differences of two runs driven by the same mini-batches satisfy a momentum recurrence. In preconditioned position coordinates this recurrence is
\[
 \vx_{t+1}=\bigl[(1+b)\mI-\alpha(\mC+\mDelta_t)\bigr]\vx_t-b\vx_{t-1},
 \qquad b\coloneqq\beta_1(1-\lr\wdc),\quad \alpha\coloneqq\lr(1-\beta_1).
\]
\WCupdate{For the compensated system, the mean and noise are $(\mC,\mDelta_t)=(\mA_{\mathsf P},\mDelta_t)$. For the projected system, they are $(\mA^+,\mDelta_t^+)$, on the $p$-dimensional positive eigenspace. The derivation and relation to the augmented state of \Cref{app:setting} are given in \Cref{fh:sec:criterion}.}
\begin{definition}[Critical LRs]\label{fh:def:stable}
For either homogeneous system, strict mean-square stability means that $\Ebb\|\vx_t\|^2+\Ebb\|\vx_{t-1}\|^2\to0$ for every deterministic initial pair. The projected critical LR $\lrstar$ and the compensated critical LR $\lrP$ are the respective suprema of the strictly stable LRs.
\end{definition}
By \Cref{prop:agm-critical-lr-equivalence}, $\lrstar$ is the projected critical LR defined by bounded affine moments in \Cref{def:agm-critical-lr}.
The same argument identifies $\lrP$ with the bounded affine threshold of the compensated system when $\mA_{\mathsf P}\succ0$.
By \Cref{fh:thm:decay-exact}, the strictly stable rates form the intervals $(0,\lrstar)$ and $(0,\lrP)$, respectively.
Their second-moment maps are denoted $\Tcal_\lr$ and $\Tcal_\lr^{\mathsf P}$, as defined in \eqref{fh:eq:F}.

\paragraph{The detector.}

\WCupdate{Fix a search window $\mathcal I=[\lr-w,\lr+w]\subset(0,\infty)$ around the training LR $\lr$, an integer horizon $T\ge1$ and a cutoff $M>1$. When $\wdc>0$, assume $\mathcal I\subset(0,1/\wdc]$, so each candidate $\lr'\in\mathcal I$ has effective momentum $b=\beta_1(1-\lr'\wdc)\in[0,\beta_1]$. The expected detector, the idealized version of \Cref{alg:eos_local_Detection_AGM} analyzed here, returns}
\begin{equation}\label{fh:eq:compensated-detector}
 \lrhat(M,T)\coloneqq\sup\Bigl\{\lr'\in\mathcal I:\ \max_{0\le t\le T}\Ebb\norm{\vvareps_t(\lr')}^2\le M^2\ \text{for every }\norm{\vvareps_0}=1\Bigr\},
\end{equation}
\WCupdate{where $\vvareps_t(\lr')$ follows \eqref{fh:eq:compensated-model} with prescribed momentum $\vm_0=\vm_{t_0}$, and the expectation is over mini-batches. Differences between this expected detector and its implementation are listed in \Cref{sec:method-eos}.}
\WCupdate{The detector accepts $\lr'\in\mathcal I$ if $\lr'$ satisfies the condition inside the braces in \eqref{fh:eq:compensated-detector}, so $\lrhat(M,T)$ is the supremum of the accepted LRs.}

\begin{assumption}[Frozen compensated model]\label{fh:ass:comparison}
\Cref{ass:minibatch} holds. Fix $0\le\beta_1<1$, $\wdc\ge0$ and $\extraWD>0$. The compensation satisfies $\mA_{\mathsf P}=\mA+\extraWD\mI\succ0$. All parameters and the batch law remain fixed while $\lr$ varies.
\end{assumption}
The preconditioner and batch law are frozen at the checkpoint. The Hessian and gradient from the same mini-batch may be correlated, and no small-noise assumption is imposed. Positive definiteness of $\mA_{\mathsf P}$ requires compensation to remove all negative and zero curvature directions.
\WCupdate{If the compensation is estimated from data, the theorem applies after fixing that estimate only when the resulting mean is positive definite and the simulation batches still satisfy \Cref{ass:minibatch}. In particular, after the checkpoint and compensation are fixed, those batches must remain i.i.d.\ and independent of the prescribed initial state. The estimated compensation coefficient alone does not verify positive definiteness.}

\begin{assumption}[Acceptance at the compensated critical LR]
\label{fh:ass:transient}
The compensated critical LR $\lrP$ belongs to the window $\mathcal I$ and is accepted by the detector \eqref{fh:eq:compensated-detector}: $\max_{0\le t\le T}\Ebb\norm{\vvareps_t(\lrP)}^2\le M^2$ for every $\norm{\vvareps_0}=1$, with the prescribed initial momentum.
\end{assumption}
The boundary condition $\rho(\Tcal_{\lrP}^{\mathsf P})=1$ does not imply acceptance.
Transient growth, checkpoint momentum, and affine forcing can exceed the cutoff.
\Cref{fh:prop:cutoff} gives an explicit sufficient condition on $M$.
Acceptance of $\lrP$ is used for the upper bound below. The lower bound only needs $\lrP\in\mathcal I$ and a nonempty accepted set.

\HWupdate{\paragraph{The two sources of error.}
The detector output $\lrhat(M,T)$ estimates the threshold $\lrP$ of the compensated system. To compare it with the target $\lrstar$, we separate the error in inverse LR as
\[
 \frac1{\lrhat(M,T)}-\frac1{\lrstar}
 =\underbrace{\left(\frac1{\lrhat(M,T)}-\frac1{\lrP}\right)}_{\text{finite simulation}}
 +\underbrace{\left(\frac1{\lrP}-\frac1{\lrstar}\right)}_{\text{compensation}}.
\]
The first term comes from observing only $T$ steps. The second comes from replacing projection onto the positive eigenspace with compensation, and remains even if the compensated threshold is known exactly. We compare inverse LRs because $(1-b)/\alpha=1/\lr+\beta_1\wdc/(1-\beta_1)$ in the stability criterion below.}

\paragraph{The quantity that determines stability.}
Unqualified vector norms are Euclidean, matrix norms are spectral, and $\|\cdot\|_F$ is the Frobenius norm. For a positive definite mean $\mC$ with its corresponding centered noise, and $-1<b<1$, let
\[
 c_b\coloneqq\frac{1-b}{1+b},\qquad k_b\coloneqq\frac{b}{1-b^2}.
\]
For a symmetric test matrix, $[\mC,\mX]\coloneqq\mC\mX-\mX\mC$ is the commutator. Define
\begin{align}
 q_{\mC,b}(\mX)
 &\coloneqq\frac{\Ebb\tr(\mX\mDelta_t\mX\mDelta_t)+c_b\tr(\mX\mC\mX\mC)-k_b\norm{[\mC,\mX]}_F^2}{2\tr(\mC\mX^2)},
 \qquad 0\ne\mX\succeq0,
 \label{fh:eq:quotient}\\*
 r_{\mC}(b)&\coloneqq\max_{\mX\succeq0,\ \norm{\mX}_F=1}q_{\mC,b}(\mX).
 \label{fh:eq:variational}
\end{align}
The noise is $\mDelta_t^+$ when $\mC=\mA^+$ and $\mDelta_t$ when $\mC=\mA_{\mathsf P}$. The maximum is attained on the compact set of positive semidefinite matrices of unit Frobenius norm. The exact criterion in \Cref{fh:thm:exact} shows that the corresponding recurrence is strictly mean-square stable precisely when $(1-b)/\alpha>r_{\mC}(b)$.

\HWupdate{\paragraph{Curvatures in the error bound.}
The maximizers of the stability quotient determine which curvatures enter the comparison. Set $\bstar\coloneqq\beta_1(1-\lrstar\wdc)$ and $\bstarP\coloneqq\beta_1(1-\lrP\wdc)$, and choose
\begin{equation}\label{fh:eq:xstar}
\begin{aligned}
 \mX_\star&\in\argmax_{\mX\succeq0,\ \norm{\mX}_F=1}q_{\mA^+,\bstar}(\mX),
 &\mustar&\coloneqq\tr(\mA^+\mX_\star^2),\\
 \XstarP&\in\argmax_{\mX\succeq0,\ \norm{\mX}_F=1}q_{\mA_{\mathsf P},\bstarP}(\mX),
 &\mustarP&\coloneqq\tr\bigl(\mA_{\mathsf P}(\XstarP)^2\bigr).
\end{aligned}
\end{equation}
The first maximization is over $p\times p$ matrices and the second over $d\times d$ matrices. Since $\tr(\mX_\star^2)=\tr((\XstarP)^2)=1$, these are weighted averages of curvature, with
\[
 \lambda_{\min}(\mA^+)\le\mustar\le L,\qquad
 \lambda_{\min}(\mA_{\mathsf P})\le\mustarP\le L+\extraWD.
\]
In particular, $\mustar$ need not equal the largest curvature $L$. Noise in directions with smaller mean curvature can also limit stability.

\paragraph{Compensation error.}
Compensation shifts the curvature on the positive eigenspace and retains the directions removed by projection. We account for these effects separately. The shift is controlled by the function
\begin{equation}\label{fh:eq:Phi}
 \Phi(s)\coloneqq\frac{\cstar\extraWD}{2}
 -\frac{\extraWD}{s+\extraWD}\Bigl(r_{\mA^+}(\bstar)-\frac{\cstar s}{2}\Bigr),
 \qquad 0\le s\le L,
\end{equation}
where $\cstar\coloneqq(1-\bstar)/(1+\bstar)$. This function is increasing and satisfies $\Phi(L)\le\cstar\extraWD/2$. The second term can exceed the first, so $\Phi(s)$ need not be positive.

The compensated system also retains noise in the discarded directions and between the two subspaces. We compare its noise contribution with that of the projected system through
\begin{equation}\label{fh:eq:delta}
 \dcov\coloneqq\max_{\mX\succeq0,\ \norm{\mX}_F=1}
 \frac{\Ebb\tr(\mX\mDelta_t\mX\mDelta_t)-\Ebb\tr(\mX^+\mDelta_t^+\mX^+\mDelta_t^+)}{2\tr(\mA_{\mathsf P}\mX^2)}\ge0,
\end{equation}
where $\mX^+$ is the block on the positive eigenspace. This quantity is finite and vanishes when the discarded and cross blocks of the noise vanish. \Cref{fh:lem:deltabound} bounds it using their second moments. Together, $\Phi$ and $\dcov$ bound the difference between $1/\lrP$ and $1/\lrstar$ (\Cref{fh:prop:compensation}).

\paragraph{Finite simulation error.}
An LR above $\lrP$ can satisfy the expected detector's displacement bound over only $T$ steps. Under \Cref{fh:ass:comparison}, every accepted LR $\lr'$ satisfies $\rho(\Tcal_{\lr'}^{\mathsf P})\le\rhoT$ by \Cref{fh:lem:detector-radius}, where
\begin{equation}\label{fh:eq:radius}
 \rhoT\coloneqq\left[\frac{M(1+\beta_1)}{1-\beta_1}\right]^{2/T}.
\end{equation}
This bound approaches one as $(1+\log M)/T\to0$, limiting the instability that can go undetected. \Cref{fh:prop:horizon-true} converts it into a bound on inverse LR. With $\alpha_\star^{\mathsf P}\coloneqq(1-\beta_1)\lrP$, the resulting error term is
\begin{equation}\label{fh:eq:HT}
 H_T\coloneqq(\rhoT-1)\left[
 \frac{2\bstarP}{(\rhoT+\bstarP)\alpha_\star^{\mathsf P}}
 +\frac{(1-\bstarP/\rhoT)^2}{2\mustarP(\alpha_\star^{\mathsf P})^2}
 \right].
\end{equation}
Combining the compensation and finite simulation bounds gives the following result.

\begin{theorem}[Accuracy of the expected EoS detector]\label{fh:thm:main}\label{fh:thm:complete}
Under \Cref{fh:ass:comparison,fh:ass:transient}, the expected detector satisfies
\begin{equation}\label{fh:eq:main}
 \min\{\Phi(\mustar),0\}-H_T
 \;\le\;\frac1{\lrhat(M,T)}-\frac1{\lrstar}
 \;\le\;\max\{\Phi(L)+\dcov,0\}.
\end{equation}
For fixed model parameters, $H_T=O((1+\log M)/T)$ as $(1+\log M)/T\to0$. If $\beta_1\wdc=0$, the minimum and maximum with zero can be omitted from \eqref{fh:eq:main}.
\end{theorem}

Acceptance of $\lrP$ ensures $\lrhat(M,T)\ge\lrP$, so the finite simulation error lies in $[-H_T,0]$. This explains why $H_T$ appears only in the lower bound. Increasing $T$ reduces this contribution when $(1+\log M)/T\to0$ and \Cref{fh:ass:transient} continues to hold, but does not remove the compensation error. \Cref{fh:prop:cutoff} shows how a cutoff growing polynomially with $T$ can ensure acceptance of $\lrP$.

Weight decay makes the effective momentum depend on the LR. The minimum and maximum with zero account for this dependence when comparing the thresholds. Without momentum, \eqref{fh:eq:HT} reduces to $H_T=(\rhoT-1)/(2\mustarP(\lrP)^2)$.
The proof of \Cref{fh:thm:main} is in \Cref{fh:sec:appendix}. \Cref{fh:cor:saddleness} gives the simpler compensation and noise bound used in \Cref{thm:clr-shift-informal}.}

\begingroup
\allowdisplaybreaks[0]
\raggedbottom
\section{Proofs for the Local Linear System and the Stability Results}\label{app:proofs-moments}
This appendix proves the statements of \Cref{app:setting,app:critical} in the order they are stated.
\Cref{app:proof-linear-system} derives the augmented linear system (\Cref{prop:augmented-system}) and its compensated form.
\Cref{app:proof-moment-recursion} derives the moment recursion of a stochastic linear system (\Cref{prop:moment-recursion}).
\Cref{app:proof-moments} derives the moment recursion for SGD (\Cref{thm:linear-system-sgd}) from the two propositions.
\Cref{app:proof-corollaries} proves the two negative-curvature results (\Cref{app:agm-EoC,cor:neg-curvature-divergence}) and \Cref{app:proof-pos-curvature} the divergence above the critical LR (\Cref{cor:pos-curvature-divergence}).
\Cref{fh:thm:decay} is restated from \Cref{fh:thm:decay-exact} and proved in \Cref{fh:sec:decay}.
\Cref{app:proof-agm-critical-lr-equivalence} proves that bounded affine stability gives the same critical LR as the spectral criterion (\Cref{prop:agm-critical-lr-equivalence}).
Throughout, $\rho$ denotes the spectral radius and $\|\cdot\|$ the spectral norm. \HWupdate{We use gray text to indicate the identity, definition, or assumption used in each step of the derivation.}

\subsection{Proof of \texorpdfstring{\Cref{prop:augmented-system}}{the augmented linear system}}\label{app:proof-linear-system}
\begin{proof}[Proof of \Cref{prop:augmented-system}]
The idea is to center the displacement at $\vx_\star$, multiply each of the two updates in \eqref{eq:local-agm-dynamics} by the power of $\mP$ that appears in the corresponding block of $\vz_t$, and identify the blocks of $\mM_t(\eta)$ and $\vq_t$.
Here $\vx_\star$ is the fixed reference displacement in the proposition and need not be a stationary point.
Write $\bar\vvareps_t\coloneqq\vvareps_t-\vx_\star$ for the centered displacement, so that $\vz_t=(\eta\mP^{-1/2}\vm_t,\mP^{1/2}\bar\vvareps_t)$.
Recall the forcing in the momentum block,
\[
 \vh_t=(1-\beta_1)\eta\mP^{-1/2}(\vg_{\batch_t}+\mH_{\batch_t}\vx_\star).
\]
Recall that $\mD_{\batch_t}=\mP^{-1/2}\mH_{\batch_t}\mP^{-1/2}$ is the preconditioned mini-batch Hessian, as defined before \Cref{prop:augmented-system}.

\textbf{Step 1. The momentum block.}
Multiplying the momentum update in \eqref{eq:local-agm-dynamics} by $\eta\mP^{-1/2}$ and substituting $\vvareps_t=\bar\vvareps_t+\vx_\star$ gives the first block of $\vz_{t+1}$,
{\allowdisplaybreaks[0]
\begin{align*}
 \eta\mP^{-1/2}\vm_{t+1}
 &=\beta_1\,\eta\mP^{-1/2}\vm_t
   +(1-\beta_1)\eta\mP^{-1/2}(\vg_{\batch_t}+\mH_{\batch_t}\vvareps_t)
 && \explain{momentum update}\\
 &=\beta_1\,\eta\mP^{-1/2}\vm_t
   +(1-\beta_1)\eta\mP^{-1/2}\mH_{\batch_t}\bar\vvareps_t\\
 &\qquad+(1-\beta_1)\eta\mP^{-1/2}(\vg_{\batch_t}+\mH_{\batch_t}\vx_\star)
 && \explain{$\vvareps_t=\bar\vvareps_t+\vx_\star$}\\
 &=\beta_1\,\eta\mP^{-1/2}\vm_t
   +(1-\beta_1)\eta\,\mP^{-1/2}\mH_{\batch_t}\mP^{-1/2}\,\mP^{1/2}\bar\vvareps_t\\
 &\qquad+(1-\beta_1)\eta\mP^{-1/2}(\vg_{\batch_t}+\mH_{\batch_t}\vx_\star)
 && \explain{$\mP^{-1/2}\mP^{1/2}=\mI$}\\
 &=\beta_1\,\eta\mP^{-1/2}\vm_t
   +(1-\beta_1)\eta\mD_{\batch_t}\,\mP^{1/2}\bar\vvareps_t+\vh_t.
 && \explain{definitions of $\mD_{\batch_t}$, $\vh_t$}
\end{align*}}

\textbf{Step 2. The displacement block.}
Subtracting $\vx_\star$ from both sides of the displacement update and multiplying by $\mP^{1/2}$ gives the second block,
\begin{align*}
 \mP^{1/2}\bar\vvareps_{t+1}
 &=\mP^{1/2}\bar\vvareps_t-\eta\mP^{-1/2}\vm_{t+1}
   -\eta\wdc\mP^{1/2}(\bar\vvareps_t+\vx_\star+\param_{t_0})
 && \explain{$\vvareps_t=\bar\vvareps_t+\vx_\star$}\\
 &=(1-\eta\wdc)\mP^{1/2}\bar\vvareps_t-\eta\mP^{-1/2}\vm_{t+1}
   -\eta\wdc\mP^{1/2}(\vx_\star+\param_{t_0})
 && \explain{collect $\bar\vvareps_t$ terms}\\
 &=-\beta_1\,\eta\mP^{-1/2}\vm_t
   +\bigl((1-\eta\wdc)\mI-(1-\beta_1)\eta\mD_{\batch_t}\bigr)\mP^{1/2}\bar\vvareps_t\\
 &\qquad-\vh_t-\eta\wdc\mP^{1/2}(\vx_\star+\param_{t_0}).
 && \explain{insert Step 1}
\end{align*}
In the state $\vz_t$, the two identities read $\vz_{t+1}=\mM_t(\eta)\vz_t+\vq_t$ with $\mM_t(\eta)$ and $\vq_t$ as in \eqref{eq:agm-augmented-matrix}.

\textbf{Step 3. Compensation.}
The gradient of the compensated loss $G_t$ of \Cref{app:setting} is
\begin{align*}
 \grad G_t(\vvareps)
 &=\grad\loss_{\batch_t}(\param_{t_0})+\hess\loss_{\batch_t}(\param_{t_0})\vvareps+\extraWD\mP\vvareps
 && \explain{$\grad\tfrac{\extraWD}{2}\vvareps^\top\mP\vvareps=\extraWD\mP\vvareps$, as $\mP$ is symmetric}\\
 &=\vg_{\batch_t}+(\mH_{\batch_t}+\extraWD\mP)\vvareps,
 && \explain{definitions of $\vg_{\batch_t}$ and $\mH_{\batch_t}$}
\end{align*}
so the penalty replaces the mini-batch Hessian $\mH_{\batch_t}$ by $\mH_{\batch_t}+\extraWD\mP$ in the momentum update. In preconditioned coordinates,
\begin{align*}
 \mP^{-1/2}(\mH_{\batch_t}+\extraWD\mP)\mP^{-1/2}
 &=\mP^{-1/2}\mH_{\batch_t}\mP^{-1/2}+\extraWD\,\mP^{-1/2}\mP\mP^{-1/2}\\
 &=\mD_{\batch_t}+\extraWD\mI.
 && \explain{$\mP^{-1/2}\mP\mP^{-1/2}=\mI$}
\end{align*}
Repeating Steps 1 and 2 with $\mH_{\batch_t}+\extraWD\mP$ in place of $\mH_{\batch_t}$, and hence $\mD_{\batch_t}+\extraWD\mI$ in place of $\mD_{\batch_t}$, gives the augmented system with the transition matrix $\widetilde\mM_t(\eta)$ of \eqref{eq:comp-agm-augmented-matrix} and with $\vh_t$ formed from $\mH_{\batch_t}+\extraWD\mP$. The weight decay term is not touched by the penalty, so its coefficient remains $\wdc$.

\textbf{Step 4. Independence.}
Under \Cref{ass:minibatch}, $(\mM_t(\eta),\vq_t)$ is a function of $\batch_t$ alone, while $\vz_t$ depends only on the initial state and on $\batch_0,\ldots,\batch_{t-1}$, so $\mM_t(\eta)$ and $\vq_t$ are independent of $\vz_t$.
This completes the proof of \Cref{prop:augmented-system}.
\end{proof}

\subsection{Proof of \texorpdfstring{\Cref{prop:moment-recursion}}{the moment recursion}}\label{app:proof-moment-recursion}
\begin{proof}[Proof of \Cref{prop:moment-recursion}]
We first derive the closed recursions for the mean and the second moment. A strict spectral bound then makes this finite-dimensional affine recursion converge, including when the current transition and forcing are correlated.
Since $\vz_t$ is a function of $\vz_0$ and of $(\mM_s,\vq_s)_{s<t}$, the pair $(\mM_t,\vq_t)$ is independent of $\vz_t$. Finite second moments propagate by induction:
\begin{align*}
 \Ebb\|\vz_{t+1}\|^2
 &\le2\Ebb\|\mM_t\vz_t\|^2+2\Ebb\|\vq_t\|^2
 &&\explain{$\|\va+\vb\|^2\le2\|\va\|^2+2\|\vb\|^2$}\\
 &\le2\Ebb\|\mM_t\|^2\,\Ebb\|\vz_t\|^2+2\Ebb\|\vq_t\|^2<\infty.
 &&\explain{independence and the induction hypothesis}
\end{align*}
The initial moment is finite by assumption. The Cauchy--Schwarz inequality also gives $\Ebb[\|\mM_t\|\,\|\vq_t\|]<\infty$, so the cross terms below are well defined without independence between $\mM_t$ and $\vq_t$.

\textbf{Step 1. The mean.}
Recall that $\mM=\Ebb[\mM_t]$ and $\vz_t^{(1)}=\Ebb[\vz_t]$. Then
\begin{align*}
 \vz_{t+1}^{(1)}
 &=\Ebb[\mM_t\vz_t]+\Ebb[\vq_t]
 && \explain{expectation of $\vz_{t+1}=\mM_t\vz_t+\vq_t$}\\
 &=\mM\vz_t^{(1)}+\Ebb[\vq_t].
 && \explain{$\mM_t$ is independent of $\vz_t$}
\end{align*}

\textbf{Step 2. The second moment.}
Recall that $\vz^{(2)}_t = \vectorize(\Ebb[\vz_t\vz_t^\top])$, $\mS = \Ebb[\vq_t \otimes \mM_t + \mM_t \otimes \vq_t]$ and $\vb_{\mathrm{quad}} =\vectorize(\Ebb[\vq_t\vq_t^\top])$.
Expanding the outer product of $\vz_{t+1}=\mM_t\vz_t+\vq_t$ with itself gives
\[
 \vz_{t+1}\vz_{t+1}^\top
 =\mM_t\vz_t\vz_t^\top\mM_t^\top+\mM_t\vz_t\vq_t^\top+\vq_t\vz_t^\top\mM_t^\top+\vq_t\vq_t^\top.
\]
Vectorize term by term using $\vectorize(\mX\mY\mZ)=(\mZ^\top\otimes\mX)\vectorize(\mY)$. In particular,
\begin{align*}
 \vectorize(\mM_t\vz_t\vq_t^\top)&=(\vq_t\otimes\mM_t)\vz_t,\\
 \vectorize(\vq_t\vz_t^\top\mM_t^\top)&=(\mM_t\otimes\vq_t)\vz_t.
\end{align*}
Taking expectations gives
\begin{align*}
 \vz_{t+1}^{(2)}
 &=\Ebb\bigl[(\mM_t\otimes\mM_t)\vectorize(\vz_t\vz_t^\top)\bigr]\\
 &\qquad+\Ebb\bigl[(\vq_t\otimes\mM_t+\mM_t\otimes\vq_t)\vz_t\bigr]\\
 &\qquad+\vectorize(\Ebb[\vq_t\vq_t^\top])
 && \explain{vectorize term by term}\\
 &= \Ebb[\mM_t \otimes \mM_t] \vz_{t}^{(2)} + \mS \vz^{(1)}_t + \vb_{\mathrm{quad}}.
 && \explain{$\vz_t$ independent of $(\mM_t,\vq_t)$}
\end{align*}
Combining the two moment recursions gives the block form \eqref{eq:generic-moment-recursion}.

\textbf{Step 3. The spectral radius.}
For the block lower triangular matrix of the proposition and every scalar $z$,
\begin{align*}
    &\det
    \begin{bmatrix}
   z\mI-\mM&\zeroB\\-\mS&z\mI-\Ebb[\mM_t\otimes\mM_t]
  \end{bmatrix}\\
 &\qquad=\det(z\mI-\mM)\,\det(z\mI-\Ebb[\mM_t\otimes\mM_t]),
 && \explain{determinant of a block lower triangular matrix}
\end{align*}
so the eigenvalues of the block matrix are those of $\mM$ together with those of $\Ebb[\mM_t\otimes\mM_t]$, and its spectral radius is
$\max\{\rho(\mM),\rho(\Ebb[\mM_t\otimes\mM_t])\}$.

\textbf{Step 4. Bounded affine moments under the strict condition.}
Suppose $\rho(\Ebb[\mM_t\otimes\mM_t])<1$. By the mean-radius inequality in the second-moment characterization (\Cref{fh:lem:covariance}),
\[
 \rho(\mM)^2\le\rho(\Ebb[\mM_t\otimes\mM_t])<1.
\]
Let $\mathbf B$ be the block matrix in \eqref{eq:generic-moment-recursion}, let $\mathbf c\coloneqq(\Ebb[\vq_t],\vb_{\mathrm{quad}})^\top$ be its constant forcing, and write $\mathbf u_t\coloneqq(\vz_t^{(1)},\vz_t^{(2)})^\top$. Step 3 gives $\rho(\mathbf B)<1$. Iterating the moment recursion yields
\begin{align*}
 \mathbf u_t
 &=\mathbf B^t\mathbf u_0+\sum_{j=0}^{t-1}\mathbf B^j\mathbf c
 &&\explain{iterate $\mathbf u_{t+1}=\mathbf B\mathbf u_t+\mathbf c$}\\
 &\longrightarrow(\mI-\mathbf B)^{-1}\mathbf c.
 &&\explain{$\mathbf B^t\to0$ and the Neumann series, $\rho(\mathbf B)<1$}
\end{align*}
The initial vector $\mathbf u_0$ is finite by the second-moment assumption. The trace of the second-moment matrix is $\Ebb\|\vz_t\|^2$, so convergence of its entries implies convergence, and hence boundedness, of this trace. This proves the affine stability assertion for every initial state of finite second moment independent of the batch sequence. The cross block $\mS$ already includes any correlation between the current transition and forcing.
This completes the proof of \Cref{prop:moment-recursion}.
\end{proof}

\subsection{Proof of \texorpdfstring{\Cref{thm:linear-system-sgd}}{the moment recursion for SGD}}\label{app:proof-moments}
\begin{proof}[Proof of \Cref{thm:linear-system-sgd}]
We center the SGD recurrence at the population stationary point, which removes its mean forcing.

\textbf{Step 1. The displacement block closes.}
Take $\beta_1=0$, $\mP=\mI$ and $\wdc=0$. Write $\vx_\star\coloneqq\param_\star-\param_{t_0}$ and $\bar\vvareps_t\coloneqq\vvareps_t-\vx_\star=\param_{t_0+t}-\param_\star$. The gradient of the surrogate is $\vg_{\batch_t}+\mH_{\batch_t}\vvareps_t$, so centering its SGD update gives
\begin{align*}
 \bar\vvareps_{t+1}
 &=\bar\vvareps_t-\eta\bigl(\vg_{\batch_t}+\mH_{\batch_t}(\bar\vvareps_t+\vx_\star)\bigr)
 && \explain{$\vvareps_t=\bar\vvareps_t+\vx_\star$}\\
 &=(\mI-\eta\mH_{\batch_t})\bar\vvareps_t-\eta(\vg_{\batch_t}+\mH_{\batch_t}\vx_\star)
 && \explain{collect the centered displacement}\\
 &=\mM_t\bar\vvareps_t-\eta\vr_t.
 && \explain{definitions of $\mM_t$ and $\vr_t$}
\end{align*}
So the centered displacement is itself a stochastic linear system with transition $\mM_t$ and forcing $-\eta\vr_t$. Under \Cref{ass:minibatch} the pairs $(\mM_t,-\eta\vr_t)$ are i.i.d.\ functions of $(\mH_{\batch_t},\vg_{\batch_t})$, have finite second moments, and are independent of $\bar\vvareps_0=-\vx_\star$.

\textbf{Step 2. Apply \Cref{prop:moment-recursion}.}
Apply the proposition to $\bar\vvareps_{t+1}=\mM_t\bar\vvareps_t+(-\eta\vr_t)$. Its mean forcing vanishes,
\begin{align*}
 \Ebb[-\eta\vr_t]
 &=-\eta\bigl(\Ebb[\vg_{\batch_t}]+\Ebb[\mH_{\batch_t}]\vx_\star\bigr)
 && \explain{$\vx_\star$ is deterministic}\\
 &=-\eta\bigl(\grad\loss(\param_{t_0})+\hess\loss(\param_{t_0})(\param_\star-\param_{t_0})\bigr)
 && \explain{unbiased mini-batch gradient and Hessian}\\
 &=\vzero,
 && \explain{by \eqref{eq:existence-of-optima}}
\end{align*}
and its two constants are those of the corollary,
\begin{align*}
 \Ebb[(-\eta\vr_t)\otimes\mM_t+\mM_t\otimes(-\eta\vr_t)]
 &=-\eta\,\Ebb[\vr_t\otimes\mM_t+\mM_t\otimes\vr_t]=\mS,
 && \explain{$\otimes$ is bilinear}\\
 \vectorize\bigl(\Ebb[(\eta\vr_t)(\eta\vr_t)^\top]\bigr)
 &=\eta^2\vectorize\bigl(\Ebb[\vr_t\vr_t^\top]\bigr)=\vb_{\mathrm{quad}}.
 && \explain{definition of $\vb_{\mathrm{quad}}$}
\end{align*}
The two mean identities used above are part of \Cref{ass:minibatch}. The moments $\Ebb[\bar\vvareps_t]$ and $\vectorize(\Ebb[\bar\vvareps_t\bar\vvareps_t^\top])$ are the $\vvareps_t^{(1)}$ and $\vvareps_t^{(2)}$ of the corollary. Thus \eqref{eq:generic-moment-recursion} is \eqref{eq:linear-system-SGD}, with zero forcing in the mean block.
For the generic quadratic of \Cref{sec:stochastic-quadratic-stability}, apply the same proposition with $\mM_t=\mI-\eta\mA_t$ and $\vq_t=-\eta(\mA_t\vx_\star+\vb_t)$. Its assumptions follow from the stated iid finite-moment model, and $\Ebb[\vq_t]=-\eta(\mA\vx_\star+\vb)=0$. This proves the moment equations of \Cref{prop:stochastic-quadratic-moment-dynamics}. If its second-moment spectral radius is less than one, Step 4 of the generic proof gives finite limiting moments for every deterministic initialization, completing that proposition's affine boundedness claim.
This completes the proof of \Cref{thm:linear-system-sgd}.
\end{proof}

\subsection{Proofs of \texorpdfstring{\Cref{app:agm-EoC,cor:neg-curvature-divergence}}{the negative-curvature results}}\label{app:proof-corollaries}
\begin{proof}[Proof of \Cref{app:agm-EoC}]
The idea is that the mean AGM dynamics decouple along the eigenvectors of $\mD$ into $2\times2$ systems, and that the one along a negative effective curvature has a real eigenvalue larger than one.
Let $\mu_1,\ldots,\mu_d$ be the eigenvalues of the population preconditioned Hessian $\mD$.
The corresponding eigenvalues of the effective curvature $\mA=\mD+\wdc\mI$ are $\bar\mu_i\coloneqq\mu_i+\wdc$.
Since \eqref{eq:agm-augmented-matrix} is affine in $\mD_{\batch_t}$, the mean transition $\Ebb[\mM_t(\eta)]$ is obtained by replacing $\mD_{\batch_t}$ with its mean $\mD$, so its blocks are $\beta_1\mI$, $(1-\beta_1)\eta\mD$, $-\beta_1\mI$ and $(1-\eta\wdc)\mI-(1-\beta_1)\eta\mD$. Every block is a polynomial in $\mD$, so this transition is block diagonal in an eigenbasis of $\mD$.
Along an eigendirection with eigenvalue $\mu_i$, the mean AGM dynamics
reduce to a two-dimensional linear system with transition matrix
$\mM_{\mu_i}(\eta)$, the $2\times2$ block of $\Ebb[\mM_t(\eta)]$ along that eigendirection, obtained by replacing $\mD$ with $\mu_i$ in each block:
\begin{align*}
 \mM_{\mu_i}(\eta)
    =
    \begin{bmatrix}
   \beta_1 & (1-\beta_1)\eta\mu_i\\
   -\beta_1 & 1-\eta\wdc-(1-\beta_1)\eta\mu_i
  \end{bmatrix}.
\end{align*}
Its characteristic polynomial
evaluated at $1$ is
\begin{align*}
    \det\!\left(
        \mI-\mM_{\mu_i}(\eta)
    \right)
    &=\det
    \begin{bmatrix}
   1-\beta_1 & -(1-\beta_1)\eta\mu_i\\
   \beta_1 & \eta\wdc+(1-\beta_1)\eta\mu_i
  \end{bmatrix}\\
    &=(1-\beta_1)\bigl(\eta\wdc+(1-\beta_1)\eta\mu_i\bigr)+\beta_1(1-\beta_1)\eta\mu_i
    && \explain{$2\times2$ determinant}\\
    &=(1-\beta_1)\eta\wdc+(1-\beta_1)\eta\mu_i\bigl((1-\beta_1)+\beta_1\bigr)
    && \explain{collect the $\mu_i$ terms}\\
    &=
    (1-\beta_1)\eta\bar\mu_i.
    && \explain{$\bar\mu_i=\mu_i+\wdc$}
\end{align*}
If $\bar\mu_i<0$, this quantity is negative for every $\eta>0$. The characteristic polynomial $\lambda\mapsto\det(\lambda\mI-\mM_{\mu_i}(\eta))$ is monic of degree two, so it is positive for all large $\lambda$. Since it is negative at $\lambda=1$, the intermediate value theorem gives a real root larger than one, that is, an eigenvalue of $\mM_{\mu_i}(\eta)$ larger than one.
The homogeneous mean dynamics therefore have an expanding mode.
Choose $\vz_0$ to be a nonzero right eigenvector of the mean transition for this real eigenvalue $r>1$. Independence across steps gives $\Ebb[\vz_t]=r^t\vz_0$, and Jensen's inequality yields
\[
 \Ebb\|\vz_t\|^2\ge\|\Ebb[\vz_t]\|^2
 =r^{2t}\|\vz_0\|^2\longrightarrow\infty.
\]
Thus the homogeneous system is not mean-square stable.
This completes the proof of \Cref{app:agm-EoC}.
\end{proof}

\begin{proof}[Proof of \Cref{cor:neg-curvature-divergence}]
The idea is that the mean recursion is the full-batch recursion, along whose negative eigendirection the mean grows geometrically, and that the second moment dominates the squared mean.
Recall from \Cref{thm:linear-system-sgd} that
$\vvareps_t^{(1)}=\Ebb[\param_{t_0+t}-\param_\star]$ is the mean error and
$\mM_t=\mI-\eta\mH_{\batch_t}$ is the mini-batch
transition matrix, so that the mean recursion in \eqref{eq:linear-system-SGD} reads
$\vvareps_{t+1}^{(1)}=\Ebb[\mM_t]\vvareps_t^{(1)}$. Recall also that
$\mH_{t_0}=\nabla^2\loss(\param_{t_0})$ is the population Hessian at the
checkpoint (\Cref{app:critical}), so that
$\Ebb[\mM_t]=\mI-\eta\mH_{t_0}$ for every $t$. If
$\mH_{t_0}\vu=\lambda\vu$ with $\lambda<0$ and $\|\vu\|=1$, then
\begin{align*}
 \vu^\top\Ebb[\param_{t_0+t}-\param_\star]
 &=\vu^\top\Ebb[\mM_t]^{t}\,\Ebb[\param_{t_0}-\param_\star]
 && \explain{iterate the mean recursion for $t$ steps}\\
 &=\vu^\top(\mI-\eta\mH_{t_0})^{t}\,\Ebb[\param_{t_0}-\param_\star]
 && \explain{$\Ebb[\mM_t]=\mI-\eta\mH_{t_0}$}\\
 &=(1-\eta\lambda)^{t}
   \vu^\top\Ebb[\param_{t_0}-\param_\star].
 && \explain{$\vu^\top\mH_{t_0}=\lambda\vu^\top$, as $\mH_{t_0}$ is symmetric}
\end{align*}
Because $1-\eta\lambda>1$, a nonzero initial component along $\vu$
grows geometrically. The second moment then diverges for this
initialization:
\begin{align*}
 \Ebb\|\param_{t_0+t}-\param_\star\|^2
 &\ge\|\Ebb[\param_{t_0+t}-\param_\star]\|^2
 && \explain{by Jensen's inequality}\\
 &\ge\bigl(\vu^\top\Ebb[\param_{t_0+t}-\param_\star]\bigr)^2
 && \explain{by the Cauchy--Schwarz inequality}\\
 &=(1-\eta\lambda)^{2t}
   \bigl(\vu^\top\Ebb[\param_{t_0}-\param_\star]\bigr)^2
 && \explain{by the display above}\\
 &\to\infty.
 && \explain{$1-\eta\lambda>1$}
\end{align*}
The same computation with $\vr_t\equiv\vzero$ shows that the homogeneous system $\bar\vvareps_{t+1}=\mM_t\bar\vvareps_t$ is not mean-square stable.
The population quadratic evaluated at $\param_\star+a\vu$ is a quadratic
polynomial in $a$ whose $a^2$ coefficient is
$\tfrac12\vu^\top\mH_{t_0}\vu=\lambda/2<0$, so it is unbounded below.
This completes the proof of \Cref{cor:neg-curvature-divergence}.
\end{proof}

\subsection{Proof of \texorpdfstring{\Cref{cor:pos-curvature-divergence}}{the divergence above the critical LR}}\label{app:proof-pos-curvature}
\begin{proof}[Proof of \Cref{cor:pos-curvature-divergence}]
We use a PSD eigenmatrix of the adjoint second-moment map to identify an initialization with geometric second-moment growth.
  \WCupdate{Recall from \Cref{app:critical-lr-agm} that $\mM_t^+(\eta)$ is the transition of the projected homogeneous system, acting on the momentum and displacement restricted to the positive eigenspace, and that $\mK^+(\eta)=\Ebb[\mM_t^+(\eta)\otimes\mM_t^+(\eta)]$.}
By the unique stability threshold (\Cref{fh:thm:decay}), $\rho(\mK^+(\eta))>1$ for every $\eta>\eta_\star$.
Write this radius as $r>1$ and define the second-moment map on symmetric matrices by
\[
 \mathcal T(\mX)\coloneqq\Ebb[\mM_t^+(\eta)\mX\mM_t^+(\eta)^\top].
\]
Restriction to symmetric matrices preserves the spectral radius, so $\rho(\mathcal T)=r$ (see the discussion after \Cref{fh:lem:covariance} in \Cref{fh:sec:tools}).
The adjoint map
$\mathcal T^*(\mX)=\Ebb[\mM_t^+(\eta)^\top\mX\,\mM_t^+(\eta)]$ preserves
PSD matrices and has the same spectral radius as $\mathcal T$, so by the
Perron--Frobenius theorem for cone-preserving maps (\Cref{fh:lem:perron}) it has a nonzero PSD eigenmatrix $\mW$ with
$\mathcal T^*(\mW)=r\mW$.
Choose a deterministic initial state $\vz_0$ of the projected system such that
$\vz_0^\top\mW\vz_0>0$, for instance an eigenvector of $\mW$ with positive
eigenvalue. Its second moment $\mX_t=\Ebb[\vz_t\vz_t^\top]$ satisfies
$\mX_t=\mathcal T^t(\vz_0\vz_0^\top)$, by the recursion
$\Ebb[\vz_{t+1}\vz_{t+1}^\top]=\mathcal T(\Ebb[\vz_t\vz_t^\top])$ of the
second-moment characterization (\Cref{fh:lem:covariance}). Hence
\begin{align*}
 \Ebb\|\vz_t\|^2
 &=\operatorname{tr}(\mX_t)
 && \explain{$\|\vz_t\|^2=\operatorname{tr}(\vz_t\vz_t^\top)$; $\Ebb$ and $\operatorname{tr}$ commute}\\
 &\ge\frac{\operatorname{tr}(\mW\mX_t)}{\|\mW\|}
 && \explain{$\mX_t\succeq0$ and $\zeroB\preceq\mW\preceq\|\mW\|\mI$}\\
 &=\frac{\operatorname{tr}\bigl(\mW\,\mathcal T^t(\vz_0\vz_0^\top)\bigr)}{\|\mW\|}
 && \explain{$\mX_t=\mathcal T^t(\vz_0\vz_0^\top)$}\\
 &=\frac{\operatorname{tr}\bigl((\mathcal T^*)^t(\mW)\,\vz_0\vz_0^\top\bigr)}{\|\mW\|}
 && \explain{definition of the adjoint, applied $t$ times}\\
 &=\frac{r^t\operatorname{tr}(\mW\vz_0\vz_0^\top)}{\|\mW\|}
 && \explain{$\mathcal T^*(\mW)=r\mW$}\\
 &=\frac{r^t\vz_0^\top\mW\vz_0}{\|\mW\|}\longrightarrow\infty.
 && \explain{cyclicity of the trace; $r>1$ and $\vz_0^\top\mW\vz_0>0$}
\end{align*}
  \WCupdate{For SGD, $\beta_1=0$, $\mP=\mI$ and $\wdc=0$, the projected state reduces to the restricted displacement $\vx_t\in\operatorname{range}(\Pi^+_{t_0})$ with transition $\mM_t^+(\eta)=\mI-\eta\Pi^+_{t_0}\mH_{\batch_t}\Pi^+_{t_0}$, and the same argument gives the exploding initial displacement of the SGD case in \Cref{app:critical-lr-agm}.}
This completes the proof of \Cref{cor:pos-curvature-divergence}.
\end{proof}

\clearpage
\subsection{Proof of the AGM Critical LR Characterization}\label{app:proof-agm-critical-lr-equivalence}
\begin{proof}[Proof of \Cref{prop:agm-critical-lr-equivalence}]
We show that every LR below the spectral threshold has bounded affine moments, while no LR above it does so for every initialization.
This identifies the supremum without making an assumption about the endpoint.
Recall the projected affine system \eqref{eq:agm-projected-local-system},
\[
 \vz_{t+1}^+=\mM_t^+(\eta)\vz_t^++\vq_t^+,
 \qquad
 \mK^+(\eta)=\Ebb[\mM_t^+(\eta)\otimes\mM_t^+(\eta)].
\]
The pairs $(\mM_t^+(\eta),\vq_t^+)$ are i.i.d.\ with finite second moments by \Cref{ass:minibatch}.
The transition and forcing within a pair may be correlated.
The effective curvature $\mA$ is positive definite on the subspace onto which $\Pi^+$ projects.
By \Cref{fh:thm:decay-exact} and the change of state in \Cref{fh:sec:tools}, $\rho(\mK^+(\eta))$ is less than, equal to, or greater than one below, at, or above a threshold $\eta_{\mathrm{crit}}\in(0,\infty)$.

\paragraph{Step 1. LRs below the spectral threshold.}
If $0<\eta<\eta_{\mathrm{crit}}$, then $\rho(\mK^+(\eta))<1$.
The affine moment recursion in \Cref{prop:moment-recursion} therefore gives
\[
 \sup_{t\ge0}\Ebb\|\vz_t^+\|^2<\infty
 \qquad\text{for every deterministic }\vz_0^+.
\]
Thus every such LR belongs to the stable set in \Cref{def:agm-critical-lr}.

\paragraph{Step 2. LRs above the spectral threshold.}
Let $\eta>\eta_{\mathrm{crit}}$, so that $\rho(\mK^+(\eta))>1$.
The Perron--Frobenius argument in \Cref{app:proof-pos-curvature} then gives a deterministic initial vector $\vu$ whose homogeneous second moment diverges.
Run two affine trajectories, $\vz_t^+$ and $\widehat\vz_t^+$, with the same pairs $(\mM_t^+(\eta),\vq_t^+)$ and initial states $\vz_0^+=\vu$ and $\widehat\vz_0^+=0$.
Their difference $\vd_t\coloneqq\vz_t^+-\widehat\vz_t^+$ satisfies
\begin{align*}
 \vd_{t+1}
 &=\bigl(\mM_t^+(\eta)\vz_t^++\vq_t^+\bigr)
   -\bigl(\mM_t^+(\eta)\widehat\vz_t^++\vq_t^+\bigr)
 &&\explain{use the same batches and forcing}\\
 &=\mM_t^+(\eta)\vd_t,\qquad \vd_0=\vu.
 &&\explain{the forcing cancels pathwise}
\end{align*}
If the affine system were bounded in mean square for every deterministic start, then
\begin{align*}
 \sup_{t\ge0}\Ebb\|\vd_t\|^2
 &\le2\sup_{t\ge0}\Ebb\|\vz_t^+\|^2
     +2\sup_{t\ge0}\Ebb\|\widehat\vz_t^+\|^2
 &&\explain{$\|\va-\vb\|^2\le2\|\va\|^2+2\|\vb\|^2$}\\
 &<\infty,
 &&\explain{boundedness for the two chosen starts}
\end{align*}
contradicting the choice of $\vu$.
Thus no $\eta>\eta_{\mathrm{crit}}$ belongs to the stable set.
Together, the two steps give
\[
 (0,\eta_{\mathrm{crit}})
 \subseteq\{\eta>0:\text{\eqref{eq:agm-projected-local-system} is mean-square stable}\}
 \subseteq(0,\eta_{\mathrm{crit}}].
\]
The middle set has supremum $\eta_\star$ by \Cref{def:agm-critical-lr}, so $\eta_\star=\eta_{\mathrm{crit}}$.
The spectral characterization and the three cases in \eqref{eq:agm-critical-lr} now follow from \Cref{fh:thm:decay-exact}.
At the endpoint, forcing can change boundedness without changing this supremum.
For example, $x_{t+1}=-x_t$ has bounded squared iterates, whereas $x_{t+1}=-x_t+\xi_t$ with independent uniform signs $\xi_t$ satisfies
\begin{align*}
 \Ebb[x_{t+1}^2]
 &=\Ebb[x_t^2]-2\Ebb[x_t\xi_t]+\Ebb[\xi_t^2]\\
 &=\Ebb[x_t^2]+1
 =x_0^2+t+1.
\end{align*}
The second line uses independence, $\Ebb[\xi_t]=0$ and $\xi_t^2=1$, then iterates the recursion.
Both have second-moment spectral radius one.
This completes the proof of \Cref{prop:agm-critical-lr-equivalence}.
\end{proof}

\clearpage
\section{Proof of the Power-Iteration Bound}\label{app:proof-power-iteration}

\subsection{The Bisection Bound}\label{sec:pi-proof}
We prove \Cref{thm:lambda_hat_bound} by bounding the two possible query errors,
\WCupdate{then tracking the search interval during bisection. Below the target shift, we show that the}
bottom-eigenvector component grows past the cutoff with high probability.
Above the target, we bound the probability of ever crossing the cutoff.

Recall that $\mA=\mD_{t_0}+\wdc\mI$ is the effective curvature and $\mDelta_t=\mA_{\batch_t}-\mA$ is its centered mini-batch noise.
A query at shift $\lambda$, with inner stepsize $\etapi$ and norm cutoff $B$, follows
\[
 \vx_{t+1}=\bigl[\mM(\lambda)-\etapi\mDelta_t\bigr]\vx_t,
 \qquad \mM(\lambda)=\mI-\etapi(\mA+\lambda\mI),
 \qquad \|\vx_0\|=1.
\]
Under \Cref{ass:pi-bracket}, $\lambda_\star=-\lambda_{\min}(\mA)$.
For this proof, define the largest eigenvalue of the mean transition and the
first exit time by
\begin{align*}
 \rho(\lambda)&\coloneqq1+\etapi(\lambda_\star-\lambda),\\
 \tau&\coloneqq\inf\{t\ge0:\|\vx_t\|>B\},\qquad\inf\varnothing=\infty.
\end{align*}
A query of horizon $T_{\mathrm p}$ explodes precisely when
$\tau\le T_{\mathrm p}$. The filtration $\mathcal F_t$ includes all previous
queries and the first $t$ observations of the current query, as defined in
\Cref{app:proof-of-min-eigen}. The query bounds below hold conditionally on
the previous queries, because \Cref{ass:pi-noise} gives conditional centering
and the same noise bound under this conditioning.

The next two lemmas control the query errors below and above
$\lambda_\star$. Their proofs are given in
\Cref{sec:pi-below,sec:pi-above}.
In the first lemma, write $\rho=\rho(\lambda)$.

\begin{lemma}[Explosion below $\lambda_\star$]\label{lem:pi-explode}
Let \Cref{ass:pi-noise,ass:pi-bracket} hold, let $\etapi>0$, $B>1$, $\lambda\in[0,\lambda_\star)$, and let $T_{\mathrm p}\ge1$ be an integer with $B\rho^{-T_{\mathrm p}}<\varepsilon$. Then
\begin{equation}\label{eq:pi-explode-bound}
 \Pbb\{\tau>T_{\mathrm p}\}\le\exp\biggl(-\frac{(\varepsilon-B\rho^{-T_{\mathrm p}})^2\,(\rho^2-1)}{2\etapi^2\sigma^2B^2}\biggr).
\end{equation}
In particular, let $g>0$ and let $e_{\mathrm p}$ be the finite-horizon term of \eqref{eq:pi-gaps}. If $\lambda\le\lambda_\star-\max\{g,e_{\mathrm p}\}$, then
\begin{equation}\label{eq:pi-explode-gap}
 \Pbb\{\tau>T_{\mathrm p}\}\le\exp\Bigl(-\frac{\varepsilon^2g}{4\etapi\sigma^2B^2}\Bigr),
\end{equation}
and for $g=g_-$ the right-hand side equals $\delta/N$.
\end{lemma}

\begin{lemma}[No explosion above $\lambda_\star$]\label{lem:pi-noexplode}
Let \Cref{ass:pi-noise,ass:pi-bracket,ass:pi-step} hold, let $B>1$, and fix
$\lambda\in[\lambda_\star,\Lambda]$. With
$a\coloneqq\etapi(\lambda-\lambda_\star)$ and $h\coloneqq\etapi\sigma$, suppose $a>h^2$.
Then
\[
 \Pbb\bigl\{\exists\,t\ge0:\|\vx_t\|>B\bigr\}
 \le B^{-(a-h^2)/h^2}.
\]
If $a\ge2h^2$, the right-hand side is at most
$B^{-(\lambda-\lambda_\star)/(2\etapi\sigma^2)}$.
In particular, when $\lambda\ge\lambda_\star+g_+$, with $g_+$ from
\eqref{eq:pi-gaps}, the probability is at most $\delta/N$.
\end{lemma}

\begin{proof}[Proof of \Cref{thm:lambda_hat_bound}]
Recall that $g_-$ and $g_+$ in \eqref{eq:pi-gaps} are the noise margins below and above $\lambda_\star$, respectively.
The finite horizon contributes $e_{\mathrm p}=((2B/\varepsilon)^{1/T_{\mathrm p}}-1)/\etapi$.
Fix an integer $T_{\mathrm p}\ge1$.

\medskip\noindent\textbf{Step 1. Control the queries outside the error margins.}\par\nobreak\smallskip
By the gap form \eqref{eq:pi-explode-gap} of \Cref{lem:pi-explode} with
$g=g_-$, a query at
$\lambda\le\lambda_\star-\max\{g_-,e_{\mathrm p}\}$ explodes with
probability at least $1-\delta/N$. By \Cref{lem:pi-noexplode}, a query at
$\lambda\ge\lambda_\star+g_+$ does not explode with probability at least
$1-\delta/N$.

For the $k$th query, let $\mathcal G_{k-1}$ contain all observations of
the preceding queries, and let $E_k$ be the event that this query returns
the wrong outcome in either of these two regions. Conditionally on
$\mathcal G_{k-1}$, the queried shift is fixed and \Cref{ass:pi-noise}
still holds. The two regions are disjoint, so the query can incur only
one of the two errors. Therefore
\[
 \Pbb(E_k\mid\mathcal G_{k-1})\le\frac\delta N
 \qquad\text{almost surely}.
\]
Taking expectations and then a union bound gives
\begin{align*}
 \Pbb\Bigl(\bigcup_{k=1}^N E_k\Bigr)
 &\le\sum_{k=1}^N\Pbb(E_k)
       &&\explain{union bound}\\
 &=\sum_{k=1}^N\Ebb\bigl[\Pbb(E_k\mid\mathcal G_{k-1})\bigr]
       &&\explain{tower property}\\
 &\le\sum_{k=1}^N\frac\delta N=\delta.
       &&\explain{the conditional bound}
\end{align*}
Thus, with probability at least $1-\delta$, every query in either region
has the stated outcome. Queries strictly between the two regions may
return either outcome.

  \medskip\noindent\textbf{\WCupdate{Step 2. The search interval.}}\par\nobreak\smallskip
  \WCupdate{On this event, an exploding query has $\lambda<\lambda_\star+g_+$ and a non-exploding query has $\lambda>\lambda_\star-\max\{g_-,e_{\mathrm p}\}$. Every lower endpoint of the search interval is either $0$ or an exploding query, hence at most $\lambda_\star+g_+$. Every upper endpoint is either $\Lambda$, which is at least $\lambda_\star$ by \Cref{ass:pi-bracket}, or a non-exploding query, hence at least $\lambda_\star-\max\{g_-,e_{\mathrm p}\}$. After $N$ halvings the search interval has length $\Lambda/2^N$, so}
\[
 \lambda_\star-\max\{g_-,e_{\mathrm p}\}\le\widehat\lambda=\lambda_{\mathrm{lo}}+\frac{\Lambda}{2^N}\le\lambda_\star+g_++\frac{\Lambda}{2^N},
\]
which is \eqref{eq:final_lambda_bound}.

\medskip\noindent\textbf{Step 3. The crossover.}\par\nobreak\smallskip
By the definition of $e_{\mathrm p}$,
\begin{align*}
 e_{\mathrm p}\le g_-
 &\iff (2B/\varepsilon)^{1/T_{\mathrm p}}\le1+\etapi g_-\\
 &\iff \frac{\log(2B/\varepsilon)}{T_{\mathrm p}}\le\log(1+\etapi g_-)\\
 &\iff T_{\mathrm p}\ge\frac{\log(2B/\varepsilon)}{\log(1+\etapi g_-)}.
\end{align*}
The second equivalence takes logarithms of positive quantities.
Since $T_{\mathrm p}$ is an integer, this is equivalent to the ceiling condition in \eqref{eq:pi-Tmin}. This completes the proof of \Cref{thm:lambda_hat_bound}.
\end{proof}

\subsection{Proof of the Explosion Bound}\label{sec:pi-below}
Fix a query at a shift $\lambda\in[0,\lambda_\star)$. Recall that $\rho=\rho(\lambda)=1+\etapi(\lambda_\star-\lambda)>1$ is the largest eigenvalue of $\mM(\lambda)$, attained at $\vu_{\min}$, so that $\vu_{\min}^\top\mM(\lambda)=\rho\,\vu_{\min}^\top$, and that $\tau$ is the first time at which $\|\vx_t\|$ exceeds the cutoff $B$. Choose the sign of $\vu_{\min}$ so that $\vu_{\min}^\top\vx_0\ge\varepsilon$, which \Cref{ass:pi-bracket} allows.

\begin{proof}[Proof of \Cref{lem:pi-explode}]
The idea is to divide the bottom-eigenvector component by its mean growth
factor $\rho^t$ and stop it at $\tau$. The resulting martingale has increment
bounds whose squares form a summable geometric sequence. The
Azuma--Hoeffding inequality then bounds the probability that this component
falls as far as the event $\tau>T_{\mathrm p}$ would require.

\medskip\noindent\textbf{Step 1. A stopped martingale with summable increment bounds.}\par\nobreak\smallskip
Set $Z_t\coloneqq\rho^{-t}\,\vu_{\min}^\top\vx_t$. By the iteration \eqref{eq:pi-iteration},
\begin{align*}
 Z_{t+1}-Z_t
 &=\rho^{-(t+1)}\,\vu_{\min}^\top\bigl(\mM(\lambda)-\etapi\mDelta_t\bigr)\vx_t-\rho^{-t}\,\vu_{\min}^\top\vx_t\\
 &=\bigl(\rho^{-(t+1)}\rho-\rho^{-t}\bigr)\vu_{\min}^\top\vx_t
   -\etapi\rho^{-(t+1)}\vu_{\min}^\top\mDelta_t\vx_t\\
 &=-\etapi\rho^{-(t+1)}\,\vu_{\min}^\top\mDelta_t\vx_t.
\end{align*}
The second equality uses
$\vu_{\min}^\top\mM(\lambda)=\rho\vu_{\min}^\top$.
Since $\vx_t$ is $\mathcal F_t$-measurable, the conditional increment is
\begin{align*}
 \Ebb[Z_{t+1}-Z_t\mid\mathcal F_t]
 &=-\etapi\rho^{-(t+1)}\vu_{\min}^\top
       \Ebb[\mDelta_t\mid\mathcal F_t]\vx_t
       &&\explain{conditional measurability}\\
 &=0.
       &&\explain{\Cref{ass:pi-noise}}
\end{align*}
The bounded noise makes the state, and hence $Z_t$, integrable at each
finite time. Thus $(Z_t)_{t\ge0}$ is a martingale with
$Z_0=\vu_{\min}^\top\vx_0\ge\varepsilon$.

We stop it at $\tau$. The event
$\{t<\tau\}=\{\max_{s\le t}\|\vx_s\|\le B\}$ is
$\mathcal F_t$-measurable. The stopped increments therefore satisfy
\begin{align*}
 Z_{(t+1)\wedge\tau}-Z_{t\wedge\tau}
 &=\mathbf 1_{\{t<\tau\}}(Z_{t+1}-Z_t),\\
 \Ebb[Z_{(t+1)\wedge\tau}-Z_{t\wedge\tau}\mid\mathcal F_t]
 &=\mathbf 1_{\{t<\tau\}}\Ebb[Z_{t+1}-Z_t\mid\mathcal F_t]\\
 &=0.
\end{align*}
Thus the stopped process is again a martingale. On $\{t<\tau\}$ we have
$\|\vx_t\|\le B$, so
\begin{align*}
 \bigl|Z_{(t+1)\wedge\tau}-Z_{t\wedge\tau}\bigr|
 &=\mathbf 1_{\{t<\tau\}}\etapi\rho^{-(t+1)}
       |\vu_{\min}^\top\mDelta_t\vx_t|\\
 &\le\mathbf 1_{\{t<\tau\}}\etapi\rho^{-(t+1)}
       \|\vu_{\min}\|\,\|\mDelta_t\|\,\|\vx_t\|
       &&\explain{Cauchy--Schwarz}\\
 &\le\etapi\sigma B\,\rho^{-(t+1)}.
       &&\explain{$\|\vx_t\|\le B$ before $\tau$}
\end{align*}
Since $\rho>1$, extending the series of positive terms and summing it gives
\begin{align*}
 \sum_{t=0}^{T_{\mathrm p}-1}\bigl(\etapi\sigma B\rho^{-(t+1)}\bigr)^2
 &=\etapi^2\sigma^2B^2\sum_{t=1}^{T_{\mathrm p}}\rho^{-2t}
       &&\explain{reindex}\\
 &\le\etapi^2\sigma^2B^2\sum_{t=1}^{\infty}\rho^{-2t}
       &&\explain{nonnegative summands}\\
 &=\etapi^2\sigma^2B^2\cdot\frac{\rho^{-2}}{1-\rho^{-2}}
       &&\explain{geometric series, $\rho>1$}\\
 &=\frac{\etapi^2\sigma^2B^2}{\rho^2-1}.
       &&\explain{multiply by $\rho^2/\rho^2$}
\end{align*}

\medskip\noindent\textbf{Step 2. The Azuma--Hoeffding bound.}\par\nobreak\smallskip
Suppose $\tau>T_{\mathrm p}$. Then $T_{\mathrm p}\wedge\tau=T_{\mathrm p}$ and $\|\vx_{T_{\mathrm p}}\|\le B$, so by the Cauchy--Schwarz inequality
\begin{align*}
 Z_{T_{\mathrm p}\wedge\tau}
 &=Z_{T_{\mathrm p}}
  =\rho^{-T_{\mathrm p}}\vu_{\min}^\top\vx_{T_{\mathrm p}}\\
 &\le\rho^{-T_{\mathrm p}}\|\vx_{T_{\mathrm p}}\|\\
 &\le B\rho^{-T_{\mathrm p}}.
\end{align*}
Since $Z_0\ge\varepsilon$, it follows that
\[
 Z_{T_{\mathrm p}\wedge\tau}-Z_0
 \le B\rho^{-T_{\mathrm p}}-\varepsilon<0.
\]
By the Azuma--Hoeffding inequality (\Cref{lem:pi-azuma}), applied to the martingale $(Z_{t\wedge\tau})$ with the increment bounds above and with $r=\varepsilon-B\rho^{-T_{\mathrm p}}>0$,
\begin{align*}
 \Pbb\{\tau>T_{\mathrm p}\}
 &\le\Pbb\bigl\{Z_{T_{\mathrm p}\wedge\tau}-Z_0
      \le-(\varepsilon-B\rho^{-T_{\mathrm p}})\bigr\}\\
 &\le\exp\biggl(-\frac{(\varepsilon-B\rho^{-T_{\mathrm p}})^2}
      {2\etapi^2\sigma^2B^2/(\rho^2-1)}\biggr),
\end{align*}
which is \eqref{eq:pi-explode-bound}.

\medskip\noindent\textbf{Step 3. The gap form.}\par\nobreak\smallskip
Suppose $\lambda\le\lambda_\star-\max\{g,e_{\mathrm p}\}$. Recall that $e_{\mathrm p}=((2B/\varepsilon)^{1/T_{\mathrm p}}-1)/\etapi$, so $1+\etapi e_{\mathrm p}=(2B/\varepsilon)^{1/T_{\mathrm p}}$. Then
\begin{align*}
 \rho&=1+\etapi(\lambda_\star-\lambda)\ge1+\etapi\max\{g,e_{\mathrm p}\},\\
 B\rho^{-T_{\mathrm p}}&\le B\,(1+\etapi e_{\mathrm p})^{-T_{\mathrm p}}\\
 &=B\cdot\frac{\varepsilon}{2B}=\frac{\varepsilon}{2},\\
 \rho^2-1&=(\rho-1)(\rho+1)\ge2(\rho-1)\ge2\etapi g.
\end{align*}
The last line uses $\rho\ge1$ and $\rho\ge1+\etapi g$.
Substituting these two bounds into the exponent of \eqref{eq:pi-explode-bound},
\begin{align*}
 \frac{(\varepsilon-B\rho^{-T_{\mathrm p}})^2(\rho^2-1)}{2\etapi^2\sigma^2B^2}
 &\ge\frac{(\varepsilon/2)^2\cdot2\etapi g}{2\etapi^2\sigma^2B^2}\\
 &=\frac{\varepsilon^2g}{4\etapi\sigma^2B^2},
\end{align*}
which gives \eqref{eq:pi-explode-gap}.
Recall $g_-=4\etapi\sigma^2B^2\log(N/\delta)/\varepsilon^2$.
Substituting this value gives
\begin{align*}
 \frac{\varepsilon^2g_-}{4\etapi\sigma^2B^2}
 &=\frac{\varepsilon^2}{4\etapi\sigma^2B^2}
   \frac{4\etapi\sigma^2B^2}{\varepsilon^2}\log\frac N\delta\\
 &=\log\frac N\delta,\\
 \Pbb\{\tau>T_{\mathrm p}\}
 &\le\exp\left(-\log\frac N\delta\right)=\frac\delta N.
\end{align*}
This completes the proof of \Cref{lem:pi-explode}.
\end{proof}
The horizon enters \eqref{eq:pi-explode-bound} through
$B\rho^{-T_{\mathrm p}}$. The definition of $e_{\mathrm p}$ ensures that
this quantity is at most $\varepsilon/2$.

\subsection{Proof of the No-Explosion Bound}\label{sec:pi-above}
Fix a query at $\lambda\in[\lambda_\star,\Lambda]$.
In the adaptive application, we condition on $\mathcal F_0$, which contains
the previous queries, so this shift is fixed throughout the proof.
Recall that the update is
$\vx_{t+1}=[\mM(\lambda)-\etapi\mDelta_t]\vx_t$, where
$\mM(\lambda)=\mI-\etapi(\mA+\lambda\mI)$ and
$\mDelta_t=\mA_{\batch_t}-\mA$ has conditional mean zero and norm at most
$\sigma$. Define the contraction
margin and the noise bound per step by
\[
 a\coloneqq\etapi(\lambda-\lambda_\star),
 \qquad h\coloneqq\etapi\sigma.
\]
The margin lies in $[0,1]$. Indeed, since
$\lambda_\star=-\lambda_{\min}(\mA)$,
\begin{align*}
 0\le a
 &=\etapi\bigl(\lambda+\lambda_{\min}(\mA)\bigr)\\
 &\le\etapi\bigl(\Lambda+\lambda_{\max}(\mA)\bigr)\\
 &\le1. &&\explain{\Cref{ass:pi-step}}
\end{align*}
The same assumption bounds the eigenvalues of the mean transition:
\[
 0\le1-\etapi\bigl(\lambda_{\max}(\mA)+\Lambda\bigr)
 \le\lambda_i\bigl(\mM(\lambda)\bigr)
 \le1-\etapi\bigl(\lambda_{\min}(\mA)+\lambda\bigr)=1-a.
\]
The upper endpoint is attained at $\vu_{\min}$, so
\begin{equation}\label{eq:pi-kappa}
 \|\mM(\lambda)\|=1-a.
\end{equation}
We use this contraction to control the chance that curvature noise pushes
the trajectory past the cutoff.

\begin{proof}[Proof of \Cref{lem:pi-noexplode}]
We choose a positive power of the norm whose conditional expectation cannot
increase. Hoeffding's lemma controls the curvature noise in one step, and
Ville's maximal inequality then controls the probability of ever crossing
the cutoff.

\medskip\noindent\textbf{Step 1. Expand the one-step change in the norm.}\par\nobreak\smallskip
For $\vx_t\ne\zeroB$, set $\hat\vx_t\coloneqq\vx_t/\|\vx_t\|$.
Define
\[
 \xi_t\coloneqq-2\etapi
 \langle\mM(\lambda)\hat\vx_t,\mDelta_t\hat\vx_t\rangle.
\]
If $\vx_t=\zeroB$, set both $\hat\vx_t$ and $\xi_t$ to zero.
Then $\vx_{t+1}=\zeroB$ and the norm bounds below remain valid.
For $\vx_t\ne\zeroB$, expanding the update gives
\begin{align*}
 \frac{\|\vx_{t+1}\|^2}{\|\vx_t\|^2}
 &=\|\mM(\lambda)\hat\vx_t-\etapi\mDelta_t\hat\vx_t\|^2
       &&\explain{the update}\\
 &=\|\mM(\lambda)\hat\vx_t\|^2
       -2\etapi\langle\mM(\lambda)\hat\vx_t,
                           \mDelta_t\hat\vx_t\rangle\notag\\
 &\qquad+\etapi^2\|\mDelta_t\hat\vx_t\|^2
       &&\explain{expand the square}\\
 &=\|\mM(\lambda)\hat\vx_t\|^2+\xi_t
       +\etapi^2\|\mDelta_t\hat\vx_t\|^2
       &&\explain{the definition of $\xi_t$}\\
 &\le(1-a)^2+\xi_t+h^2
       &&\explain{\eqref{eq:pi-kappa} and $\|\mDelta_t\|\le\sigma$}\\
 &=1-2a+a^2+\xi_t+h^2
       &&\explain{expand $(1-a)^2$}\\
 &\le1+\xi_t-(a-h^2)
       &&\explain{$0\le a\le1$}\\
 &\le\exp\bigl(\xi_t-(a-h^2)\bigr)
 &&\explain{$1+y\le e^y$}
\end{align*}

The conditional mean of $\xi_t$ is zero. The queried shift and
$\hat\vx_t$ are $\mathcal F_t$-measurable, so \Cref{ass:pi-noise} gives
\begin{align*}
 \Ebb[\xi_t\mid\mathcal F_t]
 &=-2\etapi\Ebb\bigl[
   \langle\mM(\lambda)\hat\vx_t,\mDelta_t\hat\vx_t\rangle
   \mid\mathcal F_t\bigr]\\
 &=-2\etapi\left\langle
   \mM(\lambda)\hat\vx_t,
   \Ebb[\mDelta_t\mid\mathcal F_t]\hat\vx_t\right\rangle
       &&\explain{conditional measurability}\\
 &=0.
       &&\explain{\Cref{ass:pi-noise}}
\end{align*}
Its magnitude is bounded by $2h$:
\begin{align*}
 |\xi_t|
 &\le2\etapi\|\mM(\lambda)\hat\vx_t\|
                 \|\mDelta_t\hat\vx_t\|
       &&\explain{Cauchy--Schwarz}\\
 &\le2\etapi(1-a)\sigma
       &&\explain{$\|\hat\vx_t\|\le1$}\\
 &\le2\etapi\sigma=2h.
       &&\explain{$0\le a\le1$}
\end{align*}

\medskip\noindent\textbf{Step 2. Construct a nonnegative supermartingale.}\par\nobreak\smallskip
Since $a>h^2$ and $h=\etapi\sigma>0$, define
\[
 p\coloneqq\frac{a-h^2}{2h^2}>0,
 \qquad Y_t\coloneqq\|\vx_t\|^{2p}.
\]
The preceding norm-ratio bound implies
\[
 Y_{t+1}\le Y_t\exp\bigl(p\xi_t-p(a-h^2)\bigr).
\]
When $\vx_t=\zeroB$, the next state is also zero, so this inequality holds
with both sides zero without dividing by $\|\vx_t\|$.
By Hoeffding's lemma (\Cref{lem:pi-hoeffding}), applied conditionally to
$\xi_t$ with $|\xi_t|\le2h$,
\begin{align}
 \Ebb[Y_{t+1}\mid\mathcal F_t]
 &\le\Ebb\bigl[Y_t e^{p\xi_t-p(a-h^2)}\mid\mathcal F_t\bigr]
       &&\explain{the one-step bound}\notag\\
 &=Y_t e^{-p(a-h^2)}\Ebb[e^{p\xi_t}\mid\mathcal F_t]
       &&\explain{$Y_t$ is $\mathcal F_t$-measurable}\notag\\
 &\le Y_t\exp\left(-p(a-h^2)+\frac{p^2(2h)^2}{2}\right)
       &&\explain{Hoeffding's lemma}\notag\\
 &=Y_t\exp\left(-\frac{(a-h^2)^2}{2h^2}
                    +\frac{(a-h^2)^2}{2h^2}\right)
       &&\explain{the definition of $p$}\notag\\
 &=Y_t.\label{eq:pi-lognorm}
 \end{align}
Because $Y_0=1$, induction gives $\Ebb[Y_t]\le1$ for every finite $t$.
Thus $(Y_t)_{t\ge0}$ is an integrable nonnegative supermartingale.

\medskip\noindent\textbf{Step 3. Bound the probability of crossing the cutoff.}\par\nobreak\smallskip
Since $p>0$, crossing $B$ in norm is the same as crossing $B^{2p}$ for
$Y_t$. Ville's maximal inequality (\Cref{lem:pi-ville}) gives
\begin{align*}
 \Pbb\bigl\{\exists\,t\ge0:\|\vx_t\|>B\bigr\}
 &=\Pbb\bigl\{\exists\,t\ge0:Y_t>B^{2p}\bigr\}\\
 &\le\Pbb\Bigl\{\sup_{t\ge0}Y_t\ge B^{2p}\Bigr\}\\
 &\le\frac{\Ebb[Y_0]}{B^{2p}}
       &&\explain{Ville's maximal inequality}\\
 &=B^{-2p}
       &&\explain{$Y_0=\|\vx_0\|^{2p}=1$}\\
 &=B^{-(a-h^2)/h^2}.
       &&\explain{the definition of $p$}
\end{align*}
If $a\ge2h^2$, then
\begin{align*}
 \frac{a-h^2}{h^2}
 &\ge\frac{a}{2h^2}\\
 &=\frac{\etapi(\lambda-\lambda_\star)}{2\etapi^2\sigma^2}
  =\frac{\lambda-\lambda_\star}{2\etapi\sigma^2}.
\end{align*}
Since $B>1$, the probability bound is therefore at most
$B^{-(\lambda-\lambda_\star)/(2\etapi\sigma^2)}$.
Finally, recall
$g_+=2\etapi\sigma^2\max\{1,\log(N/\delta)/\log B\}$.
If $\lambda-\lambda_\star\ge g_+$, its two defining lower bounds give
\begin{align*}
 a&=\etapi(\lambda-\lambda_\star)
       \ge2\etapi^2\sigma^2=2h^2,\\
 \frac{\lambda-\lambda_\star}{2\etapi\sigma^2}
   &\ge\frac{\log(N/\delta)}{\log B}.
\end{align*}
Consequently,
\begin{align*}
 B^{-(\lambda-\lambda_\star)/(2\etapi\sigma^2)}
 &\le B^{-\log(N/\delta)/\log B}\\
 &=\exp\bigl(-\log(N/\delta)\bigr)=\frac\delta N.
\end{align*}
This completes the proof of \Cref{lem:pi-noexplode}.
\end{proof}

\subsection{Auxiliary lemmas}\label{sec:pi-tools}
We state the three standard concentration results used above. Throughout,
$(\mathcal F_t)_{t\ge0}$ is a filtration.

\begin{lemma}[Hoeffding's lemma, conditional form; \citet{hoeffding1963probability}, see also Lemma~2.2 of \citet{boucheron2013concentration}]\label{lem:pi-hoeffding}
Let $D$ be a random variable and $\Gcal$ a sigma-algebra such that $\Ebb[D\mid\Gcal]=0$ and $|D|\le a$ almost surely for a constant $a\ge0$. Then for every $\theta\in\Rbb$,
\[
 \Ebb\bigl[e^{\theta D}\mid\Gcal\bigr]\le e^{\theta^2a^2/2}.
\]
\end{lemma}

\begin{lemma}[Ville's maximal inequality for nonnegative supermartingales; \citet{ville1939etude}, Theorem~4.8.4 of \citet{durrett2019probability}]\label{lem:pi-ville}
Let $(Y_t)_{t\ge0}$ be a nonnegative supermartingale with respect to $(\mathcal F_t)_{t\ge0}$. Then for every $c>0$,
\[
 \Pbb\Bigl\{\sup_{t\ge0}Y_t\ge c\Bigr\}\le\frac{\Ebb[Y_0]}{c}.
\]
\end{lemma}

\noindent
\begin{minipage}{\linewidth}
\begin{lemma}[Azuma--Hoeffding inequality; \citet{azuma1967weighted,hoeffding1963probability}]\label{lem:pi-azuma}
Let $(D_t)_{t\ge0}$ be random variables such that $D_t$ is $\mathcal F_{t+1}$-measurable, $\Ebb[D_t\mid\mathcal F_t]=0$, and $|D_t|\le a_t$ almost surely for constants $a_t\ge0$. Let $S_t\coloneqq\sum_{s<t}D_s$.
For every integer $T_{\mathrm p}\ge1$ with $\sum_{t<T_{\mathrm p}}a_t^2>0$ and every $r>0$,
 \[
 \Pbb\{S_{T_{\mathrm p}}\le-r\}\le\exp\Bigl(-\frac{r^2}{2\sum_{t<T_{\mathrm p}}a_t^2}\Bigr).
 \]
\end{lemma}
\end{minipage}\par

\clearpage
\section{Proof of the Compensated Finite-Horizon Threshold Bound}\label{fh:sec:appendix}
\HWupdate{In this appendix, we give the detailed proof of \Cref{fh:thm:main}.} The proof compares the projected critical LR, the compensated critical LR, and the finite-horizon detector output. \Cref{fh:sec:criterion} gives the stability results used in this comparison. \Cref{fh:sec:assembly} then proves the main bound from one compensation bound and one horizon bound. We prove these two bounds in \Cref{fh:sec:comparison,fh:sec:horizon}. The proofs of the stability results and the auxiliary matrix facts follow in \Cref{fh:sec:criterion-proofs,fh:sec:tools}.

\subsection{Stability Criterion}\label{fh:sec:criterion}

We first analyze a generic momentum recurrence. Let the mean be a symmetric positive definite $\mC\in\Sbb_n$, let the parameters satisfy $-1<b<1$ and $\alpha>0$, and let the noise matrices be i.i.d.\ and symmetric $\mDelta_t$ with $\Ebb\mDelta_t=0$, $\Ebb\norm{\mDelta_t}_F^2<\infty$, independent of the deterministic initial pair $(\vx_0,\vx_{-1})$:
\begin{equation}\label{fh:eq:genericrecurrence}
 \vx_{t+1}=\bigl[(1+b)\mI-\alpha(\mC+\mDelta_t)\bigr]\vx_t-b\,\vx_{t-1}.
\end{equation}
The history state $\vy_t\coloneqq(\vx_t,\vx_{t-1})^\top$ evolves by $\vy_{t+1}=\mF_t\vy_t$ with
\begin{equation}\label{fh:eq:F}
 \mF_t\coloneqq
 \begin{pmatrix}
 (1+b)\mI-\alpha(\mC+\mDelta_t)&-b\,\mI\\
 \mI&0
 \end{pmatrix},\qquad
 \Tcal(\mS)\coloneqq\Ebb[\mF_t\mS\mF_t^\top],
\end{equation}
where $\Tcal$, the \emph{second-moment map}, sends the second moment of $\vy_t$ to that of $\vy_{t+1}$.
Strict mean-square stability means that $\Ebb\|\vy_t\|^2\to0$ for every deterministic initial pair.

We write $\Sbb_n$ and $\Sbb_n^+$ for the symmetric and positive semidefinite $n\times n$ matrices, respectively, and $\langle\mX,\mY\rangle=\tr(\mX\mY)$ for their Frobenius inner product. The largest eigenvalue of the generic mean is $L_{\mC}\coloneqq\lambda_{\max}(\mC)$. This distinguishes it from the paper's $L=\lambda_{\max}(\mA)$.

The stability quotient $q$ gives both an exact test for stability and a quantitative bound when the moment spectral radius may exceed one. The proofs are in \Cref{fh:sec:criterion-proofs}.

Recall the quotient and its maximum from \eqref{fh:eq:quotient}--\eqref{fh:eq:variational}:
\begin{align*}
 q_{\mC,b}(\mX)
 &=\frac{\Ebb\tr(\mX\mDelta_t\mX\mDelta_t)
       +c_b\tr(\mX\mC\mX\mC)-k_b\norm{[\mC,\mX]}_F^2}
      {2\tr(\mC\mX^2)},\\
 r_{\mC}(b)&=\max_{\mX\succeq0,\ \norm{\mX}_F=1}q_{\mC,b}(\mX),
 \qquad c_b=\frac{1-b}{1+b},\quad k_b=\frac{b}{1-b^2}.
\end{align*}
The first result compares the normalized inverse coefficient $(1-b)/\alpha$ with the largest value of this quotient.

\begin{theorem}[Exact criterion]\label{fh:thm:exact}
For the momentum recurrence \eqref{fh:eq:genericrecurrence} with $\mC\succ0$, $-1<b<1$, $\alpha>0$, and i.i.d.\ symmetric centered noise of finite second moment, strict mean-square stability holds if and only if
\begin{equation}\label{fh:eq:pdtest}
 \frac{1-b}{\alpha}>r_{\mC}(b).
\end{equation}
Moreover, $r_{\mC}(b)\ge c_bL_{\mC}/2>0$. For fixed $b$, the stable coefficients form the interval $0<\alpha<(1-b)/r_{\mC}(b)$.
\end{theorem}

\begin{proposition}[A spectral-radius bound controls the stability quotient]
\label{fh:prop:radius-quotient}
For the same recurrence, suppose additionally that $0\le b<1$. Let $R\ge1$ satisfy $\rho(\Tcal)\le R$. For every $\mX\succeq0$ with $\norm{\mX}_F=1$, set $\mu\coloneqq\tr(\mC\mX^2)>0$. Then
\begin{equation}\label{fh:eq:radius-quotient}
 q_{\mC,b}(\mX)-\frac{1-b}{\alpha}
 \le(R-1)\left[
   \frac{2b}{(R+b)\alpha}
   +\frac{(1-b/R)^2}{2\mu\alpha^2}
 \right].
\end{equation}
\end{proposition}

The test matrix can be chosen freely and need not maximize $q$. At $R=1$, the result gives $q_{\mC,b}(\mX)\le(1-b)/\alpha$. For $R>1$, it controls the amount by which this inequality can fail. When $b=0$, the bound reduces to $(R-1)/(2\mu\alpha^2)$. In \Cref{fh:sec:horizon}, the detector supplies $R=\rhoT$, and the critical test matrix converts the left side into a bound on the inverse LR.

\paragraph{Changing the momentum.}
First compare two momentum values while keeping the mean, the noise and the test matrix fixed. Recall $c_b=(1-b)/(1+b)$ and $k_b=b/(1-b^2)$. For $-1<b'\le b<1$, we have $c_{b'}\ge c_b$ and $k_{b'}\le k_b$. The noise term cancels when we subtract the two quotients, so for every nonzero PSD $\mX$,
\begin{equation}\label{fh:eq:momentum-monotonicity}
\begin{aligned}
 q_{\mC,b'}(\mX)-q_{\mC,b}(\mX)
 &=\frac{(c_{b'}-c_b)\tr(\mX\mC\mX\mC)
 +(k_b-k_{b'})\norm{[\mC,\mX]}_F^2}{2\tr(\mC\mX^2)}\\
 &\ge0.
\end{aligned}
\end{equation}
Both coefficients and both quadratic terms in the numerator are nonnegative. Thus $q_{\mC,b'}(\mX)\ge q_{\mC,b}(\mX)$ for every admissible test matrix $\mX$.

\paragraph{Shifting the mean.}
A scalar shift changes only the mean-dependent terms of the quotient. The following identity will be used in both the quantitative stability bound and the compensation comparison. Its proof is in \Cref{fh:sec:quantitative-proof}.

\begin{lemma}[The shift identity]\label{fh:lem:shift}
Let $\mC\succ0$ on $\Rbb^n$ with noise $\mDelta_t$ and momentum $b\in(-1,1)$, and let $h\ge0$. The shifted mean $\mC+h\mI$ carries the same noise. For every $\mX\in\Sbb_n^+$ with $\norm{\mX}_F=1$ and $\mu\coloneqq\tr(\mC\mX^2)$,
\begin{equation}\label{fh:eq:shiftidentity}
 q_{\mC+h\mI,b}(\mX)=\frac{\mu\,q_{\mC,b}(\mX)+c_bh(\mu+h/2)}{\mu+h}.
\end{equation}
\end{lemma}

\paragraph{Application to AGMs.}

Consider two affine runs $(\vvareps_t^{(1)},\vm_t^{(1)})$ and $(\vvareps_t^{(2)},\vm_t^{(2)})$ driven by the same mini-batches. Define their preconditioned differences by
\[
 \vx_t\coloneqq\mP^{1/2}(\vvareps_t^{(1)}-\vvareps_t^{(2)}),\qquad
 \vn_t\coloneqq\mP^{-1/2}(\vm_t^{(1)}-\vm_t^{(2)}).
\]
Subtracting the two copies of \eqref{fh:eq:compensated-model} cancels the common batch gradient and checkpoint drift. Thus
\begin{equation}\label{fh:eq:model}
\begin{aligned}
 \vn_{t+1}&=\beta_1\vn_t+(1-\beta_1)(\mA_{\mathsf P}-\wdc\mI+\mDelta_t)\vx_t,\\
 \vx_{t+1}&=(1-\lr\wdc)\vx_t-\lr\vn_{t+1}.
\end{aligned}
\end{equation}
The curvature term follows from $\mP^{-1/2}(\mH_{\batch_t}+\extraWD\mP)\mP^{-1/2}=\mA_{\batch_t}-\wdc\mI+\extraWD\mI$.
These are the homogeneous difference dynamics of the compensated augmented system in \eqref{eq:comp-agm-augmented-matrix}.
Its state $(\lr\vn_t,\vx_t)$ is written here in the coordinates $(\vx_t,\vn_t)$.
To obtain a recurrence for the positions, set
\[
 \alpha\coloneqq\lr(1-\beta_1),\qquad b\coloneqq\beta_1(1-\lr\wdc).
\]
For $t\ge1$, the preceding position update gives $-\lr\vn_t=\vx_t-(1-\lr\wdc)\vx_{t-1}$. Substituting this into \eqref{fh:eq:model},
\begin{align*}
 \vx_{t+1}
 &=(1-\lr\wdc)\vx_t-\lr\beta_1\vn_t
   -\alpha(\mA_{\mathsf P}-\wdc\mI+\mDelta_t)\vx_t\\
 &=(1-\lr\wdc)\vx_t
   +\beta_1\bigl[\vx_t-(1-\lr\wdc)\vx_{t-1}\bigr]
   -\alpha(\mA_{\mathsf P}+\mDelta_t)\vx_t+\alpha\wdc\vx_t\\
 &=\bigl[(1+\beta_1-\lr\wdc+\alpha\wdc)\mI
          -\alpha(\mA_{\mathsf P}+\mDelta_t)\bigr]\vx_t-b\vx_{t-1}.
\end{align*}
The scalar coefficient simplifies because
\[
 1+\beta_1-\lr\wdc+\alpha\wdc
 =1+\beta_1-\lr\wdc+\lr(1-\beta_1)\wdc
 =1+\beta_1(1-\lr\wdc)=1+b.
\]
Thus the positions satisfy
\begin{equation}\label{fh:eq:recurrence}
 \vx_{t+1}=\bigl[(1+b)\mI-\alpha(\mA_{\mathsf P}+\mDelta_t)\bigr]\vx_t-b\,\vx_{t-1}.
\end{equation}
The effective curvature already includes weight decay. Its remaining contribution is absorbed into the \emph{effective momentum} $b$. The LR enters through $\alpha$ and $b$. The elimination is valid for every $\lr$, including $\lr\wdc=1$, where $b=0$ and the recurrence has first order.

The generic criterion applies to two positive definite means.
The compensated system \eqref{fh:eq:recurrence} uses $\mC=\mA_{\mathsf P}$ and noise $\mDelta_t$ in dimension $d$.
\WCupdate{The projected system uses $\mC=\mA^+$ and noise $\mDelta_t^+$ in dimension $p$, the dimension of the positive eigenspace.}
It is the homogeneous difference recurrence associated with the projected affine system in \Cref{def:agm-critical-lr}.
In both systems, $(b,\alpha)=(\beta_1(1-\lr\wdc),\lr(1-\beta_1))$, so that
\begin{equation}\label{fh:eq:normalized}
 \frac{1-b}{\alpha}=\frac{1-\beta_1+\beta_1\lr\wdc}{\lr(1-\beta_1)}=\frac1\lr+\frac{\beta_1\wdc}{1-\beta_1}.
\end{equation}
The normalized inverse LR $(1-b)/\alpha$ and $1/\lr$ differ by a constant.

For the two AGM instances, we write $\Tcal_\lr$ for the projected system and $\Tcal_\lr^{\mathsf P}$ for the compensated system. For either instance, this map has the same spectral radius as the corresponding second-moment operator for the augmented position--momentum state in \Cref{app:setting}: the two states are related by an invertible linear change when $\lr\wdc\ne1$, and both spectral radii depend continuously on $\lr$ (\Cref{fh:sec:tools}).

We now vary the LR along the AGM parameter curve, keeping $\mC$, the noise law, $\beta_1$ and $\wdc$ fixed.
Here $0\le\beta_1<1$ and $\wdc\ge0$. Let $\lr_{\max}\coloneqq(1+\beta_1)/(\beta_1\wdc)$ when $\beta_1\wdc>0$, and $\lr_{\max}\coloneqq\infty$ otherwise. This is the upper endpoint of the range where $b>-1$. Define
\begin{equation}\label{fh:eq:rootfunction}
 G_{\mC}(\lr)\coloneqq\frac1\lr+\frac{\beta_1\wdc}{1-\beta_1}
 -r_{\mC}\bigl(\beta_1(1-\lr\wdc)\bigr),
 \qquad 0<\lr<\lr_{\max}.
\end{equation}
The next corollary shows that this function crosses zero once, so the fixed-parameter criterion also gives a unique critical LR along the AGM parameter curve.

\begin{corollary}[AGM critical LR]\label{fh:thm:decay-exact}
For an AGM instance of \eqref{fh:eq:genericrecurrence} with $\mC\succ0$, the centered i.i.d.\ symmetric noise of finite second moment, and the parameter curve above, $G_{\mC}$ has a unique root $\lr_{\mathrm{crit}}$. The strictly mean-square stable LRs are $(0,\lr_{\mathrm{crit}})$. The second-moment spectral radius equals one at $\lr_{\mathrm{crit}}$ and is strictly greater than one for every $\lr>\lr_{\mathrm{crit}}$. For $0<\lr_2<\lr_1<\lr_{\max}$,
\begin{equation}\label{fh:eq:increment}
 G_{\mC}(\lr_2)-G_{\mC}(\lr_1)\ge\frac1{\lr_2}-\frac1{\lr_1}.
\end{equation}
\end{corollary}

The corollary gives the projected and compensated thresholds $\lrstar$ and $\lrP$ from \Cref{fh:def:stable}, with critical effective momenta in $(-1,1)$.
These agree with the corresponding thresholds defined by bounded affine moments by \Cref{prop:agm-critical-lr-equivalence} and its argument for $\mA_{\mathsf P}\succ0$.

At the unique root, write
$b_{\mathrm{crit}}\coloneqq\beta_1(1-\lr_{\mathrm{crit}}\wdc)$.
The normalized inverse rate equals the quotient maximum:
\begin{equation}\label{fh:eq:criticalidentity}
 r_{\mC}(b_{\mathrm{crit}})
 =\frac{1-b_{\mathrm{crit}}}{(1-\beta_1)\lr_{\mathrm{crit}}}
 =\frac1{\lr_{\mathrm{crit}}}+\frac{\beta_1\wdc}{1-\beta_1}.
\end{equation}
The offset $\beta_1\wdc/(1-\beta_1)$ is the same for the projected and compensated systems.

\paragraph{Relation to the exact SGD threshold.}
When $b=0$, we have $c_0=1$ and $k_0=0$. Centering of the noise gives
\[
 q_{\mC,0}(\mX)
 =\frac{\Ebb\tr\!\bigl(\mX(\mC+\mDelta_t)\mX(\mC+\mDelta_t)\bigr)}
        {2\tr(\mC\mX^2)}.
\]
Under the sampling and positive-semidefinite sample-Hessian assumptions of \citet[Theorem~5]{mulayoff2024exact}, maximizing this quotient recovers their exact SGD threshold on the positive-curvature subspace. Our criterion extends it to momentum and allows indefinite curvature samples with a positive definite mean. For the normalized momentum update used here, $\alpha=\lr(1-\beta_1)$. Thus, without weight decay,
\[
 \lr_{\mathrm{crit}}=\frac{1}{r_{\mC}(0)}\quad\text{for SGD},
 \qquad
 \lr_{\mathrm{crit}}=\frac{1}{r_{\mC}(\beta_1)}\quad\text{with momentum}.
\]
Momentum changes the deterministic-curvature coefficient and adds a commutator term, so it does not universally rescale the SGD threshold.
When $\beta_1\wdc>0$, the threshold is the unique root of $G_{\mC}$ in \eqref{fh:eq:rootfunction}.

\paragraph{The noise-free threshold.}
Set $\mDelta_t=0$ while retaining the same mean $\mC$, whose largest eigenvalue is $L_{\mC}$. The noise term in $q_{\mC,b}$ then vanishes. The remaining numerator is at most $c_bL_{\mC}\tr(\mC\mX^2)$ by \eqref{fh:eq:detquadraticbound}, so division by its denominator gives $q_{\mC,b}(\mX)\le c_bL_{\mC}/2$ for every admissible $\mX$. Equality holds at $\mX=\vu\vu^\top$ for a unit top eigenvector of $\mC$. Together with the rank-one lower bound,
\begin{equation}\label{fh:eq:noisefree}
 r_{\mC}(b)=\frac{c_bL_{\mC}}{2}\ \text{without noise},\qquad
 r_{\mC}(b)\ge\frac{c_bL_{\mC}}{2}\ \text{for arbitrary noise}.
\end{equation}
Thus noise cannot increase the critical LR. For the noise-free AGM system, the equality in the stability criterion becomes
\begin{align*}
 \alpha L_{\mC}&=2(1+b),\\
 \lr(1-\beta_1)L_{\mC}&=2(1+\beta_1)-2\beta_1\wdc\lr.
\end{align*}
Solving for the LR gives
\begin{equation}\label{fh:eq:detthreshold-generic}
 \lr_{\mathrm{det}}=\frac{2(1+\beta_1)}{(1-\beta_1)L_{\mC}+2\beta_1\wdc},
\end{equation}
the noise-free critical LR of this AGM system.
\subsection{Proof of the Threshold Bound}\label{fh:sec:assembly}
The detector runs the compensated system, whereas EoS is defined by the projected system. We therefore compare three LRs: the detector output $\lrhat(M,T)$, the compensated critical rate $\lrP$, and the projected critical rate $\lrstar$. Their inverse difference splits into two terms:
\begin{equation}\label{fh:eq:split}
 \frac1{\lrhat(M,T)}-\frac1{\lrstar}
 =\underbrace{\Bigl(\frac1{\lrhat(M,T)}-\frac1{\lrP}\Bigr)}_{\text{finite-horizon error}}
 +\underbrace{\Bigl(\frac1{\lrP}-\frac1{\lrstar}\Bigr)}_{\text{compensation error}}.
\end{equation}
The first comes from observing a run for only $T$ steps. The second remains even with an unlimited horizon, because compensation changes the system being tested. The theorem bounds both terms without assuming that the Hessian noise is small.

\paragraph{Proof from the two error bounds.}
The following propositions control the two terms in \eqref{fh:eq:split}. They are proved in \Cref{fh:sec:comparison,fh:sec:horizon}, respectively.
Recall that $\mustar$ is the curvature at a test matrix maximizing the projected stability quotient.
The function $\Phi$ gives the scalar-shift correction, and $\dcov$ is the maximum normalized noise mismatch between the projected and compensated systems.
Their formulas are given in \eqref{fh:eq:xstar}, \eqref{fh:eq:Phi}, and \eqref{fh:eq:delta}, respectively.

\begin{proposition}[Compensation error]\label{fh:prop:compensation}
Under \Cref{fh:ass:comparison},
\begin{equation}\label{fh:eq:compensation}
 \min\{\Phi(\mustar),0\}
 \;\le\;\frac1{\lrP}-\frac1{\lrstar}
 \;\le\;\max\{\Phi(L)+\dcov,0\}.
\end{equation}
If $\beta_1\wdc=0$, the clipping can be omitted:
$\Phi(\mustar)\le1/\lrP-1/\lrstar\le\Phi(L)+\dcov$.
\end{proposition}

\begin{proposition}[Finite-horizon error]\label{fh:prop:horizon-true}
Under \Cref{fh:ass:comparison}, let $\lr\in\mathcal I$ be accepted by the detector. If $\lr\le\lrP$, then $1/\lr\ge1/\lrP$. Otherwise $\bstarP\ge0$ and
\begin{equation}\label{fh:eq:horizontrue}
 \frac1\lr\ge\frac1{\lrP}-H_T,
\end{equation}
with $H_T$ given by \eqref{fh:eq:HT}. In particular, \eqref{fh:eq:horizontrue} holds for every accepted LR whenever $\lrP\in\mathcal I$.
\end{proposition}

\begin{proof}[Proof of \Cref{fh:thm:main}]
We bound the two inverse-rate differences in \eqref{fh:eq:split} separately and then add them. Recall that $\lrhat(M,T)$ is the supremum of the accepted rates in the search window. By \Cref{fh:ass:transient}, $\lrP$ belongs to that window and is accepted, so $\lrhat(M,T)\ge\lrP$.

By the finite-horizon bound (\Cref{fh:prop:horizon-true}), every accepted rate satisfies $1/\lr\ge1/\lrP-H_T$. We now pass this inequality to the supremum without assuming that the accepted rates form an interval. Every accepted $\lr$ is at most $\lrhat(M,T)$, so its reciprocal is at least $1/\lrhat(M,T)$. Conversely, by the definition of a supremum, there is a sequence of accepted rates $\lr_j\to\lrhat(M,T)$. The window is bounded and bounded away from zero, so
\begin{align*}
 \inf_{\substack{\lr\in\mathcal I\\\lr\text{ accepted}}}\frac1\lr
 &\ge\frac1{\lrhat(M,T)},\\
 \inf_{\substack{\lr\in\mathcal I\\\lr\text{ accepted}}}\frac1\lr
 &\le\lim_{j\to\infty}\frac1{\lr_j}
 =\frac1{\lrhat(M,T)}.
\end{align*}
The two inequalities identify the infimum. Applying the bound for each accepted rate and then using $\lrhat(M,T)\ge\lrP$ gives
\begin{align*}
 \frac1{\lrhat(M,T)}
 &=\inf_{\substack{\lr\in\mathcal I\\\lr\text{ accepted}}}\frac1\lr
 \ge\frac1{\lrP}-H_T,\\
 -H_T&\le\frac1{\lrhat(M,T)}-\frac1{\lrP}\le0.
\end{align*}
The compensation bound (\Cref{fh:prop:compensation}) controls the remaining difference. Thus
\begin{align*}
 \frac1{\lrhat(M,T)}-\frac1{\lrstar}
 &=\left(\frac1{\lrhat(M,T)}-\frac1{\lrP}\right)
   +\left(\frac1{\lrP}-\frac1{\lrstar}\right),\\
 \min\{\Phi(\mustar),0\}-H_T
 &\le\frac1{\lrhat(M,T)}-\frac1{\lrstar}
 \le\max\{\Phi(L)+\dcov,0\}.
\end{align*}
When $\beta_1\wdc=0$, the compensation bound has no clipping, so neither does the sum. When $\beta_1=0$, the definitions give $\bstarP=0$ and $\alpha_\star^{\mathsf P}=\lrP$. The first term in the bracket defining $H_T$ vanishes, and the second gives $H_T=(\rhoT-1)/(2\mustarP(\lrP)^2)$. Finally, with $x\coloneqq(2/T)\log(M(1+\beta_1)/(1-\beta_1))$,
\[
 \rhoT-1=e^x-1\le x e^x=O\bigl((1+\log M)/T\bigr)
 \quad\text{as }(1+\log M)/T\to0.
\]
The search window gives $0\le\bstarP<1$, and $\rhoT\ge1$. Thus the bracket in \eqref{fh:eq:HT} has the fixed upper bound
\begin{align*}
 &\frac{2\bstarP}{(\rhoT+\bstarP)\alpha_\star^{\mathsf P}}
  +\frac{(1-\bstarP/\rhoT)^2}{2\mustarP(\alpha_\star^{\mathsf P})^2}\\
 &\quad\le\frac{2\bstarP}{(1+\bstarP)\alpha_\star^{\mathsf P}}
  +\frac1{2\mustarP(\alpha_\star^{\mathsf P})^2}.
\end{align*}
The denominator in each term is positive and fixed with the model. Multiplying by the preceding bound on $\rhoT-1$ gives the stated order of $H_T$ and completes the proof of \Cref{fh:thm:main}.
\end{proof}
\subsection{Compensation Error}\label{fh:sec:comparison}

This section proves \Cref{fh:prop:compensation} by comparing the two systems at the same momentum and the same test matrix. A scalar shift of the mean gives an exact identity. Adding the discarded directions does not change the threshold if their noise is removed, and restoring that noise costs at most $\dcov$.

Recall that $\mA=\mD_{t_0}+\wdc\mI$ is the effective curvature and $\mDelta_t=\mA_{\batch_t}-\mA$ is its centered noise.
\WCupdate{Its block on the positive eigenspace is $\mA^+\succ0$, its compensated mean is $\mA_{\mathsf P}=\mA+\extraWD\mI\succ0$, and $L=\lambda_{\max}(\mA^+)$.}
The quotient and its maximum from \eqref{fh:eq:quotient}--\eqref{fh:eq:variational} are
\begin{align*}
 q_{\mC,b}(\mX)
 &=\frac{\Ebb\tr(\mX\mDelta_t\mX\mDelta_t)
       +c_b\tr(\mX\mC\mX\mC)-k_b\norm{[\mC,\mX]}_F^2}
      {2\tr(\mC\mX^2)},\\
 r_{\mC}(b)&=\max_{\mX\succeq0,\ \norm{\mX}_F=1}q_{\mC,b}(\mX),
 \qquad c_b=\frac{1-b}{1+b},\quad k_b=\frac{b}{1-b^2}.
\end{align*}
The noise in this formula is $\mDelta_t^+$ for the projected mean $\mA^+$ and the full $\mDelta_t$ for the compensated mean $\mA_{\mathsf P}$. The shift identity (\Cref{fh:lem:shift}) was established in \Cref{fh:sec:criterion} and applies here with the corresponding noise law fixed.

\begin{proof}[Proof of \Cref{fh:prop:compensation}]
We compare first at $\bstar$, then at the systems' own critical momenta using \eqref{fh:eq:momentum-monotonicity} and its consequence for $r_{\mC}$ proved in \Cref{fh:sec:decay}.

Recall that $\bstar=\beta_1(1-\lrstar\wdc)$ is the projected critical momentum and $\cstar=(1-\bstar)/(1+\bstar)$. The matrix $\mX_\star$ maximizes $q_{\mA^+,\bstar}$, with $\norm{\mX_\star}_F=1$ and $\mustar=\tr(\mA^+\mX_\star^2)$. All threshold functions in Steps~1--2 are evaluated at the same momentum $\bstar$.

  \medskip\noindent\textbf{\WCupdate{Step 1. Shift the mean on the positive eigenspace.}}\par\nobreak\smallskip
  \WCupdate{Fix $\mX\succeq0$ with $\norm{\mX}_F=1$ on the positive eigenspace, and write $\mu\coloneqq\tr(\mA^+\mX^2)$, as in the shift identity. Since $0\prec\mA^+\preceq L\mI$ and $\mX^2\succeq0$,}
\[
 0<\mu=\tr(\mA^+\mX^2)\le L\tr(\mX^2)=L.
\]
Apply \eqref{fh:eq:shiftidentity} with $\mC=\mA^+$, $b=\bstar$ and $h=\extraWD$. Subtracting $r_{\mA^+}(\bstar)$ and collecting its coefficients gives
\begin{align*}
 &q_{\mA^++\extraWD\mI,\bstar}(\mX)-r_{\mA^+}(\bstar)\\
 &\quad=\frac{\mu\,q_{\mA^+,\bstar}(\mX)+\cstar\extraWD(\mu+\extraWD/2)}{\mu+\extraWD}
        -r_{\mA^+}(\bstar)\\
 &\quad=\frac{\mu\bigl(q_{\mA^+,\bstar}(\mX)-r_{\mA^+}(\bstar)\bigr)}{\mu+\extraWD}
        +\frac{\extraWD\bigl[\cstar(\mu+\extraWD/2)-r_{\mA^+}(\bstar)\bigr]}{\mu+\extraWD}.
\end{align*}
Recall that $\Phi$ is the compensation term defined in \eqref{fh:eq:Phi}. At the curvature $\mu=\tr(\mA^+\mX^2)$, its definition gives
\[
 \Phi(\mu)=\frac{\cstar\extraWD}{2}
 -\frac{\extraWD}{\mu+\extraWD}
  \left(r_{\mA^+}(\bstar)-\frac{\cstar\mu}{2}\right),
 \qquad \cstar=\frac{1-\bstar}{1+\bstar}.
\]
To identify the second term above with $\Phi(\mu)$, expand its numerator as follows:
\begin{align*}
 &\frac{\extraWD\bigl[\cstar(\mu+\extraWD/2)-r_{\mA^+}(\bstar)\bigr]}{\mu+\extraWD}\\
 &\quad=\frac{\extraWD\bigl[\cstar(\mu+\extraWD)/2+\cstar\mu/2-r_{\mA^+}(\bstar)\bigr]}{\mu+\extraWD}\\
 &\quad=\frac{\cstar\extraWD}{2}
        -\frac{\extraWD}{\mu+\extraWD}\left(r_{\mA^+}(\bstar)-\frac{\cstar\mu}{2}\right)\\
 &\quad=\Phi(\mu).
\end{align*}
Consequently,
\begin{equation}\label{fh:eq:shiftcompare}
\begin{aligned}
 &q_{\mA^++\extraWD\mI,\bstar}(\mX)-r_{\mA^+}(\bstar)\\
 &\quad=\frac{\mu}{\mu+\extraWD}
 \bigl(q_{\mA^+,\bstar}(\mX)-r_{\mA^+}(\bstar)\bigr)+\Phi(\mu).
\end{aligned}
\end{equation}
For $0\le s\le L$, the rank-one bound \eqref{fh:eq:rankonetest} yields
\[
 r_{\mA^+}(\bstar)-\frac{\cstar s}{2}
 \ge\frac{\cstar(L-s)}{2}\ge0,
 \qquad\text{hence}\qquad
 \Phi(s)\le\frac{\cstar\extraWD}{2}.
\]
To compare $\Phi(\mu)$ with $\Phi(L)$, differentiate its equivalent form
$\Phi(s)=\extraWD[\cstar(s+\extraWD/2)-r_{\mA^+}(\bstar)]/(s+\extraWD)$:
\begin{align*}
 \Phi'(s)
 &=\extraWD\frac{\cstar(s+\extraWD)-\bigl[\cstar(s+\extraWD/2)-r_{\mA^+}(\bstar)\bigr]}{(s+\extraWD)^2}\\
 &=\frac{\extraWD\bigl[r_{\mA^+}(\bstar)+\cstar\extraWD/2\bigr]}{(s+\extraWD)^2}>0.
\end{align*}
The numerator is positive because $\extraWD>0$, $\cstar>0$, and $r_{\mA^+}(\bstar)\ge\cstar L/2>0$.

Since $q_{\mA^+,\bstar}(\mX)\le r_{\mA^+}(\bstar)$, the first term in \eqref{fh:eq:shiftcompare} is nonpositive. Moreover, $\Phi$ is increasing and $\mu\le L$. Thus, for every admissible $\mX$,
\begin{align*}
 q_{\mA^++\extraWD\mI,\bstar}(\mX)-r_{\mA^+}(\bstar)
 &\le\Phi(\mu)\\
 &\le\Phi(L).
\end{align*}
Writing out the corresponding maximum gives
\begin{align*}
 &r_{\mA^++\extraWD\mI}(\bstar)-r_{\mA^+}(\bstar)\\
 &\quad=\max_{\mX\in\Sbb_p^+,\ \norm{\mX}_F=1}
       \bigl[q_{\mA^++\extraWD\mI,\bstar}(\mX)-r_{\mA^+}(\bstar)\bigr]\\
 &\quad\le\max_{\mX\in\Sbb_p^+,\ \norm{\mX}_F=1}\Phi\bigl(\tr(\mA^+\mX^2)\bigr)\\
 &\quad\le\Phi(L).
\end{align*}
The first equality moves the constant $r_{\mA^+}(\bstar)$ inside the maximum. Each subsequent inequality uses a bound valid at every matrix in this same set.

For the lower bound, the maximizing matrix $\mX_\star$ satisfies $q_{\mA^+,\bstar}(\mX_\star)=r_{\mA^+}(\bstar)$. Thus
\begin{align*}
 &r_{\mA^++\extraWD\mI}(\bstar)-r_{\mA^+}(\bstar)\\
 &\quad\ge q_{\mA^++\extraWD\mI,\bstar}(\mX_\star)-r_{\mA^+}(\bstar)\\
 &\quad=\frac{\mustar}{\mustar+\extraWD}
       \bigl[q_{\mA^+,\bstar}(\mX_\star)-r_{\mA^+}(\bstar)\bigr]+\Phi(\mustar)\\
 &\quad=\Phi(\mustar).
\end{align*}
The inequality tests the shifted maximum at $\mX_\star$, the next line uses \eqref{fh:eq:shiftcompare}, and the last line uses its maximizing property. We have proved
\[
 \Phi(\mustar)\le r_{\mA^++\extraWD\mI}(\bstar)-r_{\mA^+}(\bstar)\le\Phi(L).
\]

\medskip\noindent\textbf{Step 2. Restore the discarded directions and their noise.}\par\nobreak\smallskip
  \WCupdate{Fix a matrix $\mX\in\Sbb_p^+$ on the positive eigenspace with $\norm{\mX}_F=1$. Its embedding $\operatorname{diag}(\mX,0)$ is positive semidefinite and has the same Frobenius norm, since}
\begin{align*}
 \operatorname{diag}(\mX,0)^2
 &=\operatorname{diag}(\mX^2,0),\\
 \norm{\operatorname{diag}(\mX,0)}_F^2
 &=\tr\bigl(\operatorname{diag}(\mX^2,0)\bigr)
 =\tr(\mX^2)=1.
\end{align*}
We check the four terms of its quotient. For the noise term, block multiplication gives
\begin{align*}
 \operatorname{diag}(\mX,0)\mDelta_t\operatorname{diag}(\mX,0)
 &=\operatorname{diag}(\mX\mDelta_t^+\mX,0),\\
 \tr\bigl(\operatorname{diag}(\mX,0)\mDelta_t
             \operatorname{diag}(\mX,0)\mDelta_t\bigr)
 &=\tr(\mX\mDelta_t^+\mX\mDelta_t^+).
\end{align*}
The equality holds for each sample, and therefore also after taking expectations. For the mean term and the denominator, the block diagonal form of $\mA_{\mathsf P}$ gives
\begin{align*}
 \operatorname{diag}(\mX,0)\mA_{\mathsf P}
 &=\operatorname{diag}\bigl(\mX(\mA^++\extraWD\mI),0\bigr),\\
 \tr\bigl(\operatorname{diag}(\mX,0)\mA_{\mathsf P}
             \operatorname{diag}(\mX,0)\mA_{\mathsf P}\bigr)
 &=\tr\bigl(\mX(\mA^++\extraWD\mI)
                  \mX(\mA^++\extraWD\mI)\bigr),\\
 \tr\bigl(\mA_{\mathsf P}\operatorname{diag}(\mX,0)^2\bigr)
 &=\tr\bigl((\mA^++\extraWD\mI)\mX^2\bigr).
\end{align*}
  \WCupdate{The commutator also reduces to the block on the positive eigenspace,}
\begin{align*}
 [\mA_{\mathsf P},\operatorname{diag}(\mX,0)]
 &=\operatorname{diag}\bigl([\mA^++\extraWD\mI,\mX],0\bigr),\\
 \norm{[\mA_{\mathsf P},\operatorname{diag}(\mX,0)]}_F^2
 &=\norm{[\mA^++\extraWD\mI,\mX]}_F^2.
\end{align*}
Substituting these equalities into the quotient proves
\begin{equation}\label{fh:eq:embedding}
 q_{\mA_{\mathsf P},\bstar}\bigl(\operatorname{diag}(\mX,0)\bigr)
 =q_{\mA^++\extraWD\mI,\bstar}(\mX).
\end{equation}
The full maximization includes every embedded test matrix. In particular,
\begin{align*}
 r_{\mA_{\mathsf P}}(\bstar)
 &=\max_{\substack{\mX\in\Sbb_d^+\\\norm{\mX}_F=1}}
       q_{\mA_{\mathsf P},\bstar}(\mX)\\
 &\ge\max_{\substack{\mX\in\Sbb_p^+\\\norm{\mX}_F=1}}
       q_{\mA_{\mathsf P},\bstar}\bigl(\operatorname{diag}(\mX,0)\bigr)\\
 &=\max_{\substack{\mX\in\Sbb_p^+\\\norm{\mX}_F=1}}
       q_{\mA^++\extraWD\mI,\bstar}(\mX)\\
 &=r_{\mA^++\extraWD\mI}(\bstar).
\end{align*}
  \WCupdate{The inequality restricts the test matrices to those embedded from the positive eigenspace, and the next equality uses \eqref{fh:eq:embedding}.}

For the upper bound, first remove the discarded and cross blocks of the noise. The resulting samples are
\[
 \mA_{\mathsf P}+\Pi^+\mDelta_t\Pi^+
    =
    \begin{pmatrix}
   \mA^++\extraWD\mI+\mDelta_t^+&0\\
   0&\mA^-+\extraWD\mI
  \end{pmatrix}.
\]
When $p<d$, this recurrence is strictly mean-square stable if and only if both block recurrences are stable. Sufficiency follows by adding their expected squared state norms, and necessity follows by initializing only one block. Both block means are positive definite because $\mA_{\mathsf P}\succ0$.

At the fixed momentum $\bstar$, let the coefficient $\alpha>0$ vary. By the exact criterion (\Cref{fh:thm:exact}) and the noise-free threshold \eqref{fh:eq:noisefree}, both blocks are strictly stable if and only if
\[
 \begin{cases}
  (1-\bstar)/\alpha>r_{\mA^++\extraWD\mI}(\bstar),\\
  (1-\bstar)/\alpha>\cstar\bigl(\lambda_{\max}(\mA^-)+\extraWD\bigr)/2
 \end{cases}
\]
These two inequalities are equivalent to
\[
 \frac{1-\bstar}{\alpha}>
 \max\left\{r_{\mA^++\extraWD\mI}(\bstar),
 \frac{\cstar}{2}\bigl(\lambda_{\max}(\mA^-)+\extraWD\bigr)\right\}.
\]
Thus the maximum on the right is the threshold function of this block diagonal system. Since $\lambda_{\max}(\mA^-)\le0<L$, the rank-one bound \eqref{fh:eq:rankonetest} gives
\begin{align*}
 \frac{\cstar}{2}\bigl(\lambda_{\max}(\mA^-)+\extraWD\bigr)
 &<\frac{\cstar}{2}(L+\extraWD)\\
 &\le r_{\mA^++\extraWD\mI}(\bstar).
\end{align*}
The threshold is therefore $r_{\mA^++\extraWD\mI}(\bstar)$. When $p=d$, there is no discarded block and this equality holds directly.

Now fix a full-space test matrix $\mX\in\Sbb_d^+$ with $\norm{\mX}_F=1$. Before the omitted noise is restored, its noise term is
\begin{align*}
 &\Ebb\tr\bigl(\mX(\Pi^+\mDelta_t\Pi^+)
                    \mX(\Pi^+\mDelta_t\Pi^+)\bigr)\\
 &\qquad=\Ebb\tr\bigl((\Pi^+\mX\Pi^+)\mDelta_t
                         (\Pi^+\mX\Pi^+)\mDelta_t\bigr)\\
 &\qquad=\Ebb\tr(\mX^+\mDelta_t^+\mX^+\mDelta_t^+).
\end{align*}
The first equality uses cyclicity of the trace and $(\Pi^+)^2=\Pi^+$, and the second follows by block multiplication. The mean, commutator and denominator are unchanged. Subtracting the change in the noise term therefore recovers the quotient before restoration:
\begin{align*}
 &q_{\mA_{\mathsf P},\bstar}(\mX)
  -\frac{\Ebb\tr(\mX\mDelta_t\mX\mDelta_t)
          -\Ebb\tr(\mX^+\mDelta_t^+\mX^+\mDelta_t^+)}
         {2\tr(\mA_{\mathsf P}\mX^2)}\\
 &\qquad=\frac{\Ebb\tr(\mX^+\mDelta_t^+\mX^+\mDelta_t^+)
    +\cstar\tr(\mX\mA_{\mathsf P}\mX\mA_{\mathsf P})
    -k_{\bstar}\norm{[\mA_{\mathsf P},\mX]}_F^2}
   {2\tr(\mA_{\mathsf P}\mX^2)}\\
 &\qquad\le r_{\mA^++\extraWD\mI}(\bstar).
\end{align*}
The denominator is positive because $\mA_{\mathsf P}\succ0$ and $\mX\ne0$. Recall that $\dcov$ is the largest normalized change in the noise term over these full-space test matrices:
\[
 \dcov=\max_{\substack{\mX\in\Sbb_d^+\\\norm{\mX}_F=1}}
 \frac{\Ebb\tr(\mX\mDelta_t\mX\mDelta_t)
       -\Ebb\tr(\mX^+\mDelta_t^+\mX^+\mDelta_t^+)}
      {2\tr(\mA_{\mathsf P}\mX^2)}.
\]
Adding back the subtracted fraction and bounding it by this maximum, as defined in \eqref{fh:eq:delta}, gives
\begin{align*}
 q_{\mA_{\mathsf P},\bstar}(\mX)
 &\le r_{\mA^++\extraWD\mI}(\bstar)
   +\frac{\Ebb\tr(\mX\mDelta_t\mX\mDelta_t)
          -\Ebb\tr(\mX^+\mDelta_t^+\mX^+\mDelta_t^+)}
         {2\tr(\mA_{\mathsf P}\mX^2)}\\
 &\le r_{\mA^++\extraWD\mI}(\bstar)+\dcov.
\end{align*}
This holds for every full-space test matrix. Maximizing both sides,
\begin{equation}\label{fh:eq:commonmean}
\begin{aligned}
 r_{\mA_{\mathsf P}}(\bstar)
 &=\max_{\substack{\mX\in\Sbb_d^+\\\norm{\mX}_F=1}}
     q_{\mA_{\mathsf P},\bstar}(\mX)\\
 &\le\max_{\substack{\mX\in\Sbb_d^+\\\norm{\mX}_F=1}}
     \bigl[r_{\mA^++\extraWD\mI}(\bstar)+\dcov\bigr]\\
 &=r_{\mA^++\extraWD\mI}(\bstar)+\dcov.
\end{aligned}
\end{equation}
We can now combine the lower bound from the embedded test matrices and the upper bound \eqref{fh:eq:commonmean} with Step~1:
\begin{align*}
 r_{\mA_{\mathsf P}}(\bstar)-r_{\mA^+}(\bstar)
 &\ge r_{\mA^++\extraWD\mI}(\bstar)-r_{\mA^+}(\bstar)\\
 &\ge\Phi(\mustar),\\[2pt]
 r_{\mA_{\mathsf P}}(\bstar)-r_{\mA^+}(\bstar)
 &\le r_{\mA^++\extraWD\mI}(\bstar)-r_{\mA^+}(\bstar)+\dcov\\
 &\le\Phi(L)+\dcov.
\end{align*}
Equivalently,
\begin{equation}\label{fh:eq:comparison}
 \Phi(\mustar)
 \le r_{\mA_{\mathsf P}}(\bstar)-r_{\mA^+}(\bstar)
 \le\Phi(L)+\dcov.
\end{equation}

\medskip\noindent\textbf{Step 3. Compare the critical LRs.}\par\nobreak\smallskip
Recall that $\bstarP=\beta_1(1-\lrP\wdc)$ and $\bstar=\beta_1(1-\lrstar\wdc)$ are the momenta at the two critical LRs. The critical-rate identity \eqref{fh:eq:criticalidentity} expresses each quotient maximum as the inverse critical rate plus the same weight-decay offset:
\begin{align*}
 r_{\mA_{\mathsf P}}(\bstarP)
 &=\frac1{\lrP}+\frac{\beta_1\wdc}{1-\beta_1},\\
 r_{\mA^+}(\bstar)
 &=\frac1{\lrstar}+\frac{\beta_1\wdc}{1-\beta_1}.
\end{align*}
Subtracting the same offset from both thresholds,
\begin{align*}
 \frac1{\lrP}-\frac1{\lrstar}
 &=\left[r_{\mA_{\mathsf P}}(\bstarP)-\frac{\beta_1\wdc}{1-\beta_1}\right]
   -\left[r_{\mA^+}(\bstar)-\frac{\beta_1\wdc}{1-\beta_1}\right]\\
 &=r_{\mA_{\mathsf P}}(\bstarP)-r_{\mA^+}(\bstar).
\end{align*}
The two critical momenta differ by
\begin{align*}
 \bstarP-\bstar
 &=\beta_1(1-\lrP\wdc)-\beta_1(1-\lrstar\wdc)\\
 &=\beta_1\wdc(\lrstar-\lrP).
\end{align*}
Also, both critical LRs are positive, so
\[
 \frac1{\lrP}-\frac1{\lrstar}
 =\frac{\lrstar-\lrP}{\lrP\lrstar}.
\]
These two identities determine the signs in the following cases.

If $\lrP\le\lrstar$, then $\bstarP\ge\bstar$ and $1/\lrP-1/\lrstar\ge0$. Since $r_{\mA_{\mathsf P}}$ is nonincreasing in its momentum argument, $r_{\mA_{\mathsf P}}(\bstarP)\le r_{\mA_{\mathsf P}}(\bstar)$. Combining this with the upper bound in \eqref{fh:eq:comparison},
\begin{align*}
 \frac1{\lrP}-\frac1{\lrstar}
 &=r_{\mA_{\mathsf P}}(\bstarP)-r_{\mA^+}(\bstar)\\
 &\le r_{\mA_{\mathsf P}}(\bstar)-r_{\mA^+}(\bstar)\\
 &\le\Phi(L)+\dcov.
\end{align*}
In this case both clipped bounds follow from
\begin{align*}
 \min\{\Phi(\mustar),0\}&\le0\le\frac1{\lrP}-\frac1{\lrstar},\\
 \frac1{\lrP}-\frac1{\lrstar}&\le\Phi(L)+\dcov\le\max\{\Phi(L)+\dcov,0\}.
\end{align*}

If $\lrP\ge\lrstar$, then $\bstarP\le\bstar$ and $1/\lrP-1/\lrstar\le0$. Monotonicity now gives $r_{\mA_{\mathsf P}}(\bstarP)\ge r_{\mA_{\mathsf P}}(\bstar)$. Combining this with the lower bound in \eqref{fh:eq:comparison},
\begin{align*}
 \frac1{\lrP}-\frac1{\lrstar}
 &=r_{\mA_{\mathsf P}}(\bstarP)-r_{\mA^+}(\bstar)\\
 &\ge r_{\mA_{\mathsf P}}(\bstar)-r_{\mA^+}(\bstar)\\
 &\ge\Phi(\mustar).
\end{align*}
The clipped bounds now follow from
\begin{align*}
 \min\{\Phi(\mustar),0\}&\le\Phi(\mustar)\le\frac1{\lrP}-\frac1{\lrstar},\\
 \frac1{\lrP}-\frac1{\lrstar}&\le0\le\max\{\Phi(L)+\dcov,0\}.
\end{align*}
This proves \eqref{fh:eq:compensation} in both cases.

Finally, if $\beta_1\wdc=0$, then
\[
 \bstarP=\beta_1-\beta_1\wdc\lrP=\beta_1
 =\beta_1-\beta_1\wdc\lrstar=\bstar.
\]
The critical identities and \eqref{fh:eq:comparison} therefore give
\begin{align*}
 \Phi(\mustar)
 &\le r_{\mA_{\mathsf P}}(\bstar)-r_{\mA^+}(\bstar)\\
 &=r_{\mA_{\mathsf P}}(\bstarP)-r_{\mA^+}(\bstar)\\
 &=\frac1{\lrP}-\frac1{\lrstar}\\
 &\le\Phi(L)+\dcov.
\end{align*}
Hence no clipping is needed in this case. This completes the proof of \Cref{fh:prop:compensation}.
\end{proof}

\paragraph{A scalar example.}
Let $\beta_1=\wdc=0$, $d=p=1$, $\mA=1$, $\Ebb\mDelta_t^2=1$, and $\extraWD=0.01$.
For the homogeneous scalar SGD recurrence with mean $a>0$ and centered noise of variance $\sigma^2$, the factor by which one step multiplies the second moment is
\begin{align*}
 \Ebb\bigl(1-\lr(a+\Delta_t)\bigr)^2
 &=1-2\lr a+\lr^2\Ebb(a+\Delta_t)^2\\
 &=1-2\lr a+\lr^2(a^2+\sigma^2).
\end{align*}
For $\lr>0$, this factor is below one if and only if
\[
 -2\lr a+\lr^2(a^2+\sigma^2)<0
 \quad\Longleftrightarrow\quad
 \lr<\frac{2a}{a^2+\sigma^2}.
\]
The inverse critical rate is therefore $(a^2+\sigma^2)/(2a)$. In the projected system $a=1$ and $\sigma^2=1$, while compensation replaces $a$ by $1+\extraWD$ and leaves the variance unchanged. Hence
\begin{align*}
 \frac1{\lrstar}&=1,\\
 \frac1{\lrP}-\frac1{\lrstar}
 &=\frac{(1+\extraWD)^2+1}{2(1+\extraWD)}-1\\
 &=\frac{\extraWD^2}{2(1+\extraWD)}\approx4.95\cdot10^{-5}.
\end{align*}
Here $\mustar=L=1$ and $\dcov=0$. Since $b_\star=0$, $c_\star=1$ and $r_{\mA^+}(b_\star)=1$, the definition of $\Phi$ gives
\[
 \Phi(1)=\frac{\extraWD}{2}-\frac{\extraWD}{1+\extraWD}\left(1-\frac12\right)
 =\frac{\extraWD^2}{2(1+\extraWD)}.
\]
Both compensation endpoints therefore equal the inverse-rate difference. The comparison is exact because the shared noise is compared directly rather than bounded separately in each system.

\paragraph{The SGD case.}
For SGD without weight decay, $\beta_1=\wdc=0$ and $\mP=\mI$, so $\mA=\mH_{t_0}$. Here $\bstarP=0$, $\alpha_\star^{\mathsf P}=\lrP$ and $\cstar=1$. The critical identity \eqref{fh:eq:criticalidentity} gives $r_{\mA^+}(0)=1/\lrstar$. Substituting these values into the definitions of $\rhoT$, $H_T$ and $\Phi$ gives
\begin{align*}
 \rhoT&=M^{2/T},\qquad
 H_T=\frac{M^{2/T}-1}{2\mustarP(\lrP)^2},\\
 \Phi(s)&=\frac{\extraWD}{2}-\frac{\extraWD}{s+\extraWD}
       \left(\frac1{\lrstar}-\frac s2\right).
\end{align*}
The theorem applies without clipping. The direct LR bound derived in \Cref{fh:sec:horizon} is $\lrhat(M,T)\le\lrP+(M^{2/T}-1)/(2\mustarP)$.

The covariance mismatch can be bounded by second moments of the discarded and cross blocks of the noise. Recall that $\mDelta_t^\perp=\mDelta_t-\Pi^+\mDelta_t\Pi^+$ contains precisely those discarded and cross blocks.

\begin{lemma}[A bound on the covariance mismatch]\label{fh:lem:deltabound}
Under \Cref{fh:ass:comparison}, let $v_0,v_E\ge0$ satisfy $\Ebb(\mDelta_t^+)^2\preceq v_0(\mA^++\extraWD\mI)$ and $\Ebb(\mDelta_t^\perp)^2\preceq v_E\mA_{\mathsf P}$. Then
\[
 \dcov\le\sqrt{v_0v_E}+\frac{v_E}{2}.
\]
\end{lemma}
\begin{proof}[Proof of \Cref{fh:lem:deltabound}]
We bound the noise difference that defines $\dcov$, keeping its denominator $2\tr(\mA_{\mathsf P}\mX^2)$ unchanged. Within this proof write
\[
 \mDelta_t^{\circ}\coloneqq\Pi^+\mDelta_t\Pi^+
 =\operatorname{diag}(\mDelta_t^+,0),\qquad
 \mDelta_t=\mDelta_t^{\circ}+\mDelta_t^\perp.
\]
\medskip\noindent\textbf{Step 1. Bound the separate noise terms.}\par\nobreak\smallskip
  \WCupdate{The moment bound on $\mDelta_t^+$ in the hypothesis extends to the full space:}
\begin{align*}
 \Ebb(\mDelta_t^{\circ})^2
 &=\operatorname{diag}\bigl(\Ebb(\mDelta_t^+)^2,0\bigr)\\
 &\preceq v_0\operatorname{diag}(\mA^++\extraWD\mI,0)\\
 &\preceq v_0\mA_{\mathsf P}.
\end{align*}
The last inequality uses the positive definiteness of the discarded block when it is present, and is an equality when $p=d$.
Fix $0\ne\mX\in\Sbb_d^+$.
Since $\mDelta_t^\circ=\operatorname{diag}(\mDelta_t^+,0)$, block multiplication gives
\[
 \Ebb\tr(\mX^+\mDelta_t^+\mX^+\mDelta_t^+)
 =\Ebb\tr(\mX\mDelta_t^{\circ}\mX\mDelta_t^{\circ}).
\]
We first estimate the cross term and the discarded-noise term separately. For any symmetric $\mDelta$, symmetry and cyclicity of the trace give
\begin{align*}
 \norm{\mDelta\mX}_F^2
 &=\tr\bigl((\mDelta\mX)^\top\mDelta\mX\bigr)
       &&\explain{Frobenius norm}\\
 &=\tr(\mX\mDelta^2\mX)
       &&\explain{symmetry}\\
 &=\tr(\mX^2\mDelta^2),
       &&\explain{cyclicity of the trace}\\
 \norm{\mX\mDelta}_F^2
 &=\tr\bigl((\mX\mDelta)^\top\mX\mDelta\bigr)
       &&\explain{Frobenius norm}\\
 &=\tr(\mDelta\mX^2\mDelta)
       &&\explain{symmetry}\\
 &=\tr(\mX^2\mDelta^2).
       &&\explain{cyclicity of the trace}
\end{align*}
The two moment bounds therefore imply
\begin{align*}
 \Ebb\norm{\mDelta_t^\circ\mX}_F^2
 &=\tr\bigl(\mX^2\Ebb(\mDelta_t^\circ)^2\bigr)
       &&\explain{the norm identity above}\\
 &\le v_0\tr(\mA_{\mathsf P}\mX^2),
    &&\explain{\WCupdate{the moment bound on $\mDelta_t^+$}}\\
 \Ebb\norm{\mX\mDelta_t^\perp}_F^2
 &=\tr\bigl(\mX^2\Ebb(\mDelta_t^\perp)^2\bigr)
       &&\explain{the norm identity above}\\
 &\le v_E\tr(\mA_{\mathsf P}\mX^2).
       &&\explain{the discarded-noise moment bound}
\end{align*}
Both inequalities use the same trace-order fact. For $\mQ\succeq0$,
\[
 \tr(\mX^2\mQ)=\tr(\mX\mQ\mX)\ge0.
\]
Apply this to the differences between the right and left matrices in the two moment bounds.

For the cross term, the Cauchy--Schwarz inequality in the Frobenius inner product and then in $L^2$ gives
\begin{align*}
 &\Ebb\tr(\mX\mDelta_t^\circ\mX\mDelta_t^\perp)\\
 &\quad=\Ebb\langle\mDelta_t^\circ\mX,\mX\mDelta_t^\perp\rangle_F
       &&\explain{symmetry of $\mX$ and $\mDelta_t^\circ$}\\
 &\quad\le\Ebb\bigl[\norm{\mDelta_t^\circ\mX}_F\norm{\mX\mDelta_t^\perp}_F\bigr]
       &&\explain{Frobenius Cauchy--Schwarz}\\
 &\quad\le\sqrt{\Ebb\norm{\mDelta_t^\circ\mX}_F^2\,
                       \Ebb\norm{\mX\mDelta_t^\perp}_F^2}
       &&\explain{Cauchy--Schwarz in $L^2$}\\
 &\quad\le\sqrt{v_0v_E}\tr(\mA_{\mathsf P}\mX^2).
       &&\explain{the two moment bounds}
\end{align*}
For the discarded-noise term, the two Frobenius norms are equal, so
\begin{align*}
 &\Ebb\tr(\mX\mDelta_t^\perp\mX\mDelta_t^\perp)\\
 &\quad=\Ebb\langle\mDelta_t^\perp\mX,\mX\mDelta_t^\perp\rangle_F\\
 &\quad\le\Ebb\bigl[\norm{\mDelta_t^\perp\mX}_F\norm{\mX\mDelta_t^\perp}_F\bigr]
       &&\explain{Frobenius Cauchy--Schwarz}\\
 &\quad=\Ebb\norm{\mX\mDelta_t^\perp}_F^2
       &&\explain{the norm identity above}\\
 &\quad\le v_E\tr(\mA_{\mathsf P}\mX^2).
       &&\explain{the discarded-noise moment bound}
\end{align*}
\medskip\noindent\textbf{Step 2. Combine the terms and normalize.}\par\nobreak\smallskip
Expand $\mDelta_t=\mDelta_t^\circ+\mDelta_t^\perp$ in the numerator. The two cross terms have the same trace by cyclicity, and the preceding bounds give
\begin{align*}
 &\Ebb\tr(\mX\mDelta_t\mX\mDelta_t)-\Ebb\tr(\mX\mDelta_t^{\circ}\mX\mDelta_t^{\circ})\\
 &\qquad=2\,\Ebb\tr(\mX\mDelta_t^{\circ}\mX\mDelta_t^\perp)+\Ebb\tr(\mX\mDelta_t^\perp\mX\mDelta_t^\perp)\\
 &\qquad\le\bigl(2\sqrt{v_0v_E}+v_E\bigr)\tr(\mA_{\mathsf P}\mX^2).
\end{align*}
  \WCupdate{For every nonzero PSD $\mX$, the denominator $2\tr(\mA_{\mathsf P}\mX^2)$ is positive. Dividing by it and using the block trace identity above,}
\[
 \frac{\Ebb\tr(\mX\mDelta_t\mX\mDelta_t)
       -\Ebb\tr(\mX^+\mDelta_t^+\mX^+\mDelta_t^+)}
      {2\tr(\mA_{\mathsf P}\mX^2)}
 \le\frac{2\sqrt{v_0v_E}+v_E}{2}
 =\sqrt{v_0v_E}+\frac{v_E}{2}.
\]
The right side is independent of $\mX$, so it also bounds the maximum $\dcov$. This completes the proof of \Cref{fh:lem:deltabound}.
\end{proof}
Since $\mA^+\preceq\mA^++\extraWD\mI$, one may take $v_0$ to be the noise level of the projected system, the smallest $v$ with $\Ebb(\mDelta_t^+)^2\preceq v\mA^+$. The lemma controls the noise in the discarded and cross blocks relative to the compensated effective curvature. A bound on the compensation coefficient alone does not control these quantities, as the discarded-noise example below shows.

\paragraph{Dependence on saddleness.}
The next bound makes the dependence on the compensation coefficient explicit. It bounds the Hessian noise in each system by its second moment relative to the corresponding positive definite mean.

\begin{corollary}[Compensation and saddleness]\label{fh:cor:saddleness}
Under \Cref{fh:ass:comparison}, let $v\ge0$ satisfy
\begin{equation}\label{fh:eq:noise-level}
 \Ebb(\mDelta_t^+)^2\preceq v\mA^+,
 \qquad
 \Ebb\mDelta_t^2\preceq v\mA_{\mathsf P}.
\end{equation}
Then the projected and compensated critical LRs satisfy
\begin{equation}\label{fh:eq:saddleness-compensation}
 \left|\frac1{\lrP}-\frac1{\lrstar}\right|
 \le \frac{1-\beta_1}{2(1+\beta_1)}\extraWD+\frac v2.
\end{equation}
  \WCupdate{If \Cref{fh:ass:transient} also holds, with the search window specified in \Cref{fh:sec:main}, then the expected detector satisfies}
\begin{equation}\label{fh:eq:saddleness-detector}
 \left|\frac1{\lrhat(M,T)}-\frac1{\lrstar}\right|
 \le \frac{1-\beta_1}{2(1+\beta_1)}\extraWD+\frac v2+H_T,
\end{equation}
where $H_T$ is the finite-horizon term in \eqref{fh:eq:HT}.
\end{corollary}

A finite $v$ always exists because both means are positive definite and the Hessian noise has finite second moments. The corollary does not assume that $v$ is small. Its second condition uses $\mA_{\mathsf P}$, so $v$ can depend on the compensation coefficient and need not stay bounded as that coefficient decreases.

Recall that the effective saddleness in \eqref{eq:agm-saddleness} is
\[
 \nu
 =\frac{\max\{0,-\lambda_{\min}(\mA)\}}{L},
 \qquad L=\lambda_{\max}(\mA)>0.
\]
When this saddleness is positive, choosing $\extraWD=c\nu L$ with $c>1$ makes $\mA_{\mathsf P}$ positive definite.
More generally, whenever a valid compensation satisfies $\extraWD\le c\nu L$, \eqref{fh:eq:saddleness-compensation} gives
\[
 \left|\frac1{\lrP}-\frac1{\lrstar}\right|
 \le \frac{c(1-\beta_1)}{2(1+\beta_1)}
       \nu L+\frac v2.
\]
Thus saddleness controls the curvature shift, while $v$ accounts for Hessian noise. If the compensation is chosen from an estimated eigenvalue, an upper bound on $\extraWD$ must also retain the estimation error. Small saddleness alone does not control $v$.

\begin{proof}[Proof of \Cref{fh:cor:saddleness}]
We first bound the change in a critical LR caused by adding noise to a positive definite mean. Applying this bound to the projected and compensated means leaves an explicit curvature shift and at most $v/2$ from the difference of the two noise contributions.

\medskip\noindent\textbf{Step 1. Bound the noise contribution to the quotient.}\par\nobreak\smallskip
Consider the AGM recurrence with a generic mean $\mC\succ0$ and symmetric centered noise satisfying $\Ebb\mDelta_t^2\preceq v\mC$. Recall that $L_{\mC}=\lambda_{\max}(\mC)$, $c_b=(1-b)/(1+b)$, and $r_{\mC}(b)$ is the maximum of the stability quotient over nonzero positive semidefinite test matrices. For such a matrix $\mX$, the noise numerator is nonnegative because
\[
 \Ebb\tr(\mX\mDelta_t\mX\mDelta_t)
 =\Ebb\norm{\mX^{1/2}\mDelta_t\mX^{1/2}}_F^2\ge0.
\]
For the upper bound, apply the Cauchy--Schwarz inequality in the Frobenius inner product and use the symmetry of $\mX$ and $\mDelta_t$:
\begin{align*}
 \Ebb\tr(\mX\mDelta_t\mX\mDelta_t)
 &=\Ebb\langle\mDelta_t\mX,\mX\mDelta_t\rangle_F
       &&\explain{the Frobenius inner product}\\
 &\le\Ebb\bigl[\norm{\mDelta_t\mX}_F\norm{\mX\mDelta_t}_F\bigr]
       &&\explain{Cauchy--Schwarz}\\
 &=\Ebb\tr(\mX^2\mDelta_t^2)
       &&\explain{the two norms are equal}\\
 &=\tr\bigl(\mX^2\Ebb\mDelta_t^2\bigr)
       &&\explain{linearity of expectation}\\
 &\le v\tr(\mC\mX^2).
       &&\explain{$\Ebb\mDelta_t^2\preceq v\mC$}
\end{align*}
The quotient divides this numerator by $2\tr(\mC\mX^2)>0$. Its noise contribution therefore lies between zero and $v/2$ for every admissible test matrix. The maximum without noise equals $c_bL_{\mC}/2$ by \eqref{fh:eq:noisefree}. Taking maxima gives, for every $-1<b<1$,
\[
 \frac{c_bL_{\mC}}2
 \le r_{\mC}(b)
 \le\frac{c_bL_{\mC}}2+\frac v2.
\]

\medskip\noindent\textbf{Step 2. Compare with the critical LR without noise.}\par\nobreak\smallskip
Let $\lr_{\mathrm{crit}}$ be the critical rate for this noisy recurrence. Recall from \eqref{fh:eq:detthreshold-generic} that the same mean without noise has critical rate
\[
 \lr_{\mathrm{det}}
 =\frac{2(1+\beta_1)}{(1-\beta_1)L_{\mC}+2\beta_1\wdc}.
\]
By the exact stability criterion (\Cref{fh:thm:exact}), adding the nonnegative noise term can only remove strictly stable LRs of the homogeneous recurrence.
Hence $\lr_{\mathrm{crit}}\le\lr_{\mathrm{det}}$. Their effective momenta are
\[
 b_{\mathrm{crit}}\coloneqq\beta_1(1-\lr_{\mathrm{crit}}\wdc),
 \qquad
 b_{\mathrm{det}}\coloneqq\beta_1(1-\lr_{\mathrm{det}}\wdc).
\]
Both lie in $(-1,1)$ by \Cref{fh:thm:decay-exact}, and
\begin{align*}
 b_{\mathrm{crit}}-b_{\mathrm{det}}
 &=\beta_1\wdc(\lr_{\mathrm{det}}-\lr_{\mathrm{crit}})\\
 &\ge0.
\end{align*}
Recall that the critical-rate identity \eqref{fh:eq:criticalidentity} adds the same weight-decay offset to the inverse rate in both systems. Since $r_{\mC}$ is nonincreasing in its momentum argument,
\begin{align*}
 \frac1{\lr_{\mathrm{crit}}}+\frac{\beta_1\wdc}{1-\beta_1}
 &=r_{\mC}(b_{\mathrm{crit}})
       &&\explain{critical-rate identity}\\
 &\le r_{\mC}(b_{\mathrm{det}})
       &&\explain{monotonicity in $b$}\\
 &\le\frac{c_{b_{\mathrm{det}}}L_{\mC}}2+\frac v2
       &&\explain{Step~1}\\
 &=\frac1{\lr_{\mathrm{det}}}+\frac{\beta_1\wdc}{1-\beta_1}+\frac v2.
       &&\explain{boundary without noise}
\end{align*}
Subtracting the common offset and using $\lr_{\mathrm{crit}}\le\lr_{\mathrm{det}}$ gives
\[
 0\le\frac1{\lr_{\mathrm{crit}}}-\frac1{\lr_{\mathrm{det}}}\le\frac v2.
\]
This argument uses monotonicity on the full interval $(-1,1)$ and therefore also covers negative effective critical momenta.

\medskip\noindent\textbf{Step 3. Apply the bound to the two systems.}\par\nobreak\smallskip
The projected system has mean $\mA^+$, noise $\mDelta_t^+$ and largest mean eigenvalue $L$. The compensated system has mean $\mA_{\mathsf P}$, noise $\mDelta_t$ and largest mean eigenvalue $L+\extraWD$. The two inequalities in \eqref{fh:eq:noise-level} allow us to apply Step~2 with the same $v$:
\begin{align*}
 0&\le\frac1{\lrstar}
   -\frac{(1-\beta_1)L+2\beta_1\wdc}{2(1+\beta_1)}
   \le\frac v2,\\
 0&\le\frac1{\lrP}
   -\frac{(1-\beta_1)(L+\extraWD)+2\beta_1\wdc}{2(1+\beta_1)}
   \le\frac v2.
\end{align*}
The difference between the two displayed noise contributions lies in $[-v/2,v/2]$. The deterministic terms differ by
\begin{align*}
 &\frac{(1-\beta_1)(L+\extraWD)+2\beta_1\wdc}{2(1+\beta_1)}
  -\frac{(1-\beta_1)L+2\beta_1\wdc}{2(1+\beta_1)}\\
 &\qquad=\frac{(1-\beta_1)(L+\extraWD)-(1-\beta_1)L}{2(1+\beta_1)}
       &&\explain{cancel the weight-decay terms}\\
 &\qquad=\frac{1-\beta_1}{2(1+\beta_1)}\extraWD.
\end{align*}
To subtract the two noise contributions, use the upper endpoint of one interval and the lower endpoint of the other:
\begin{align*}
 &\frac1{\lrP}-\frac1{\lrstar}
      -\frac{1-\beta_1}{2(1+\beta_1)}\extraWD\\
 &\quad=\left[\frac1{\lrP}
     -\frac{(1-\beta_1)(L+\extraWD)+2\beta_1\wdc}{2(1+\beta_1)}\right]\\
 &\qquad-\left[\frac1{\lrstar}
     -\frac{(1-\beta_1)L+2\beta_1\wdc}{2(1+\beta_1)}\right]\\
 &\quad\le\frac v2-0=\frac v2.
       &&\explain{each bracket lies in $[0,v/2]$}
\end{align*}
Using the opposite endpoints in the same identity gives
\[
 \frac1{\lrP}-\frac1{\lrstar}
      -\frac{1-\beta_1}{2(1+\beta_1)}\extraWD
 \ge0-\frac v2=-\frac v2.
\]
Since the curvature shift is nonnegative, the triangle inequality gives
\begin{align*}
 \left|\frac1{\lrP}-\frac1{\lrstar}\right|
 &\le\frac{1-\beta_1}{2(1+\beta_1)}\extraWD
   +\left|\frac1{\lrP}-\frac1{\lrstar}
          -\frac{1-\beta_1}{2(1+\beta_1)}\extraWD\right|\\
 &\le\frac{1-\beta_1}{2(1+\beta_1)}\extraWD+\frac v2.
\end{align*}
This proves \eqref{fh:eq:saddleness-compensation}.

  \WCupdate{For the expected detector, \Cref{fh:prop:horizon-true} bounds the inverse of every accepted rate below by $1/\lrP-H_T$. Taking their infimum gives the same lower bound for $1/\lrhat(M,T)$. Acceptance of $\lrP$ in \Cref{fh:ass:transient} gives $\lrhat(M,T)\ge\lrP$. Together these imply}
\[
 -H_T\le\frac1{\lrhat(M,T)}-\frac1{\lrP}\le0.
\]
Split the inverse-rate error at $\lrP$ and apply the triangle inequality:
\begin{align*}
 \left|\frac1{\lrhat(M,T)}-\frac1{\lrstar}\right|
 &=\left|\left(\frac1{\lrhat(M,T)}-\frac1{\lrP}\right)
          +\left(\frac1{\lrP}-\frac1{\lrstar}\right)\right|\\
 &\le\left|\frac1{\lrhat(M,T)}-\frac1{\lrP}\right|
      +\left|\frac1{\lrP}-\frac1{\lrstar}\right|\\
 &\le H_T+\frac{1-\beta_1}{2(1+\beta_1)}\extraWD+\frac v2.
\end{align*}
This proves \eqref{fh:eq:saddleness-detector} and completes the proof of \Cref{fh:cor:saddleness}.
\end{proof}

\paragraph{Small saddleness alone is insufficient.}
For $0<s\le1$, take $\mP=\mI_2$, $\beta_1=\wdc=0$,
\[
 \mH_{\batch_t}=\operatorname{diag}(1,-s+\sqrt{2s}\,\xi_t),\qquad
 \Pr(\xi_t=\pm1)=\tfrac12,\qquad \extraWD=2s.
\]
\WCupdate{The saddleness is $s$. The first coordinate, which spans the positive eigenspace, is noiseless, while the compensated second coordinate has mean $s$ and noise variance $2s$. Here $\mA_{\mathsf P}=\operatorname{diag}(1+2s,s)$ and $\mDelta_t^+=0$. To evaluate the covariance mismatch, recall that $\dcov$ maximizes the normalized noise difference in \eqref{fh:eq:delta}. For a symmetric $2\times2$ test matrix with entries $X_{ij}$,}
\begin{align*}
 \Ebb\tr(\mX\mDelta_t\mX\mDelta_t)&=2sX_{22}^2,\\
 \tr(\mA_{\mathsf P}\mX^2)
 &=(1+2s)(X_{11}^2+X_{12}^2)+s(X_{12}^2+X_{22}^2),\\
 \frac{\Ebb\tr(\mX\mDelta_t\mX\mDelta_t)}{2\tr(\mA_{\mathsf P}\mX^2)}
 &=\frac{sX_{22}^2}{(1+2s)(X_{11}^2+X_{12}^2)+s(X_{12}^2+X_{22}^2)}\le1.
\end{align*}
Equality holds at the admissible matrix $\mX=\operatorname{diag}(0,1)$, so $\dcov=1$. The scalar threshold above and the definition of $\Phi$ now give
\begin{align*}
 \frac1{\lrstar}&=\frac12,\qquad \Phi(1)=s,\qquad \dcov=1,\\
 \frac1{\lrP}&=\max\left\{\frac{1+2s}{2},1+\frac s2\right\}
 =1+\frac s2.
\end{align*}
The last equality uses $s\le1$. The inverse-rate difference tends to $1/2$ as $s\to0$. The discarded noise prevents the compensation error from vanishing, despite vanishing saddleness.
\subsection{Finite-Horizon Error}\label{fh:sec:horizon}

The homogeneous compensated recurrence, with mean $\mA_{\mathsf P}=\mA+\extraWD\mI$ and centered noise $\mDelta_t=\mA_{\batch_t}-\mA$, is
\[
 \begin{aligned}
 \vx_{t+1}&=\bigl[(1+b)\mI-\alpha(\mA_{\mathsf P}+\mDelta_t)\bigr]\vx_t-b\vx_{t-1},\\
 b&=\beta_1(1-\lr\wdc),\qquad \alpha=\lr(1-\beta_1),
 \end{aligned}
\]
Its second-moment map on the history state $\vy_t=(\vx_t,\vx_{t-1})^\top$ is $\Tcal_\lr^{\mathsf P}$. The mean and noise law stay fixed as the candidate LR $\lr$ varies.

Acceptance means $\Ebb\norm{\vvareps_t}^2\le M^2$ for every unit start and $0\le t\le T$ in the original affine run.
We convert this condition into a moment-radius bound, then apply \Cref{fh:prop:radius-quotient}.
Recall the uniform radius bound from \eqref{fh:eq:radius},
\[
 \rhoT=\left[\frac{M(1+\beta_1)}{1-\beta_1}\right]^{2/T}.
\]
It depends on $M$ and $T$, not on the candidate LR. The last proposition gives a sufficient cutoff for the original affine run to be accepted at $\lrP$.

\subsubsection{Bounding the Moment Radius}

\begin{lemma}[The expected-displacement criterion bounds the moment radius]
\label{fh:lem:detector-radius}
Under \Cref{fh:ass:comparison}, every LR $\lr\in\mathcal I$ accepted by the detector \eqref{fh:eq:compensated-detector} satisfies
\begin{equation}\label{fh:eq:acceptedradius}
 \rho(\Tcal_{\lr}^{\mathsf P})\le
 \left[\frac{M(1+\beta_1|1-\lr\wdc|)}{1-\beta_1}\right]^{2/T}
 \le\rhoT .
\end{equation}
\end{lemma}
\begin{proof}[Proof of \Cref{fh:lem:detector-radius}]
Fix an accepted LR $\lr$ and write $\rho\coloneqq\rho(\Tcal_\lr^{\mathsf P})$. If $\rho<1$, the claim follows because the proposed upper bound is greater than one. We therefore assume $\rho\ge1$. Our goal is to prove
\[
 \rho^{T/2}\le\frac{M(1+b)}{1-\beta_1},
 \qquad b=\beta_1(1-\lr\wdc).
\]
The detector bounds displacement, whereas $\rho$ describes growth of the full second moment. We connect the two by choosing a nonnegative quadratic form whose expectation grows by the factor $\rho$ at each step. The three steps below remove the affine forcing, construct this quadratic form, and bound its value at time $T$ using the detector.

\medskip\noindent\textbf{Step 1. Bound the homogeneous trajectory.}\par\nobreak\smallskip
For any unit vector $\vv$, run the original affine simulation \eqref{fh:eq:compensated-model} twice, from initial displacements $\vv$ and $-\vv$, using the same mini-batches and the same checkpoint momentum. Denote their positions and momenta by $(\vvareps_t^+,\vm_t^+)$ and $(\vvareps_t^-,\vm_t^-)$. Within this proof, set
\[
 \vvareps_t\coloneqq\frac{\vvareps_t^+-\vvareps_t^-}{2},
 \qquad
 \vp_t\coloneqq\frac{\lr}{2}\mP^{-1}(\vm_t^+-\vm_t^-).
\]
Here $\vp_t$ is the momentum difference scaled into displacement units. Thus $\vvareps_0=\vv$ and $\vp_0=0$. Subtracting the two momentum updates cancels their common batch gradient. Subtracting the two position updates also cancels the checkpoint weight-decay drift. The resulting equations are
\begin{align*}
 \vp_{t+1}
 &=\beta_1\vp_t+\alpha(\mP^{-1}\mH_{\batch_t}+\extraWD\mI)\vvareps_t,\\
 \vvareps_{t+1}
 &=(1-\lr\wdc)\vvareps_t-\vp_{t+1},
 \qquad \alpha=\lr(1-\beta_1).
\end{align*}
Here $\mH_{\batch_t}$ is the raw mini-batch Hessian.
Weight decay remains in the separate position factor $1-\lr\wdc$.
These are the same homogeneous compensated dynamics as before, expressed in the detector's original coordinates.

Both affine runs pass the detector because their initial displacements have unit norm. Consequently, for every $0\le t\le T$,
\begin{align*}
 \Ebb\norm{\vvareps_t}^2
 &=\frac14\Ebb\norm{\vvareps_t^+-\vvareps_t^-}^2\\
 &\le\frac12\Ebb\bigl(\norm{\vvareps_t^+}^2+\norm{\vvareps_t^-}^2\bigr)\\
 &\le\frac12(M^2+M^2)=M^2.
\end{align*}
The second line uses $\norm{\va-\vb}^2\le2\norm\va^2+2\norm\vb^2$. To control the momentum at time $T$, use the preceding position update:
\[
 \vp_T=(1-\lr\wdc)\vvareps_{T-1}-\vvareps_T.
\]
The search window guarantees $0\le1-\lr\wdc\le1$, and $T\ge1$. Minkowski's inequality therefore gives
\begin{align*}
 (\Ebb\norm{\vp_T}^2)^{1/2}
 &\le(1-\lr\wdc)(\Ebb\norm{\vvareps_{T-1}}^2)^{1/2}
       +(\Ebb\norm{\vvareps_T}^2)^{1/2}\\
 &\le(1-\lr\wdc)M+M\\
 &=(2-\lr\wdc)M.
\end{align*}
We have now bounded both components of the homogeneous state, for every unit initial displacement and zero initial momentum.

\medskip\noindent\textbf{Step 2. Measure growth with a quadratic form.}\par\nobreak\smallskip
Write $\vz_t\coloneqq(\vvareps_t,\vp_t)^\top$. The two updates in Step~1 give
\begin{equation}\label{fh:eq:Fhat}
\begin{aligned}
 \vz_{t+1}&=\widehat\mF_t\vz_t,\\
 \widehat\mF_t&\coloneqq
 \begin{pmatrix}
 (1-\lr\wdc)\mI-\alpha(\mP^{-1}\mH_{\batch_t}+\extraWD\mI)&-\beta_1\mI\\
 \alpha(\mP^{-1}\mH_{\batch_t}+\extraWD\mI)&\beta_1\mI
 \end{pmatrix}.
\end{aligned}
\end{equation}
This state is related to the existing position--momentum state of \eqref{fh:eq:model} by $\vx_t=\mP^{1/2}\vvareps_t$ and $\vn_t=\lr^{-1}\mP^{1/2}\vp_t$. This change is invertible for every $\lr>0$. The moment map for $\widehat\mF_t$ thus has spectral radius $\rho$, by the coordinate argument in \Cref{fh:sec:tools}, which also covers $\lr\wdc=1$.

We seek a matrix $\mW\succeq0$, $\mW\ne0$, such that
\[
 \Ebb[\widehat\mF_t^\top\mW\widehat\mF_t]=\rho\mW.
\]
This identity has a direct purpose: it makes the expected value of $\vz_t^\top\mW\vz_t$ multiply by $\rho$ after one step. Such a matrix exists by the Perron--Frobenius theorem for PSD-preserving maps (\Cref{fh:lem:perron}), applied to the adjoint moment map $\mW\mapsto\Ebb[\widehat\mF_t^\top\mW\widehat\mF_t]$. The adjoint has the same spectral radius as the moment map and preserves PSD matrices.

Let $\mW_{11}$ and $\mW_{22}$ be the two diagonal blocks of $\mW$, corresponding to displacement and scaled momentum. We next bound their relative size, so that Step~1 can control this quadratic form. Since $\mW\succeq0$, the Euclidean triangle inequality gives, for any vectors $\vu,\vv$,
\begin{align*}
    \left[
      \begin{pmatrix}\vu\\\vv
      \end{pmatrix}^{\!\top}
      \mW
      \begin{pmatrix}\vu\\\vv
    \end{pmatrix}\right]^{1/2}
    &=\left\|\mW^{1/2}
    \begin{pmatrix}\vu\\0
    \end{pmatrix}
    +\mW^{1/2}
    \begin{pmatrix}0\\\vv
    \end{pmatrix}\right\|\\
 &\le(\vu^\top\mW_{11}\vu)^{1/2}
      +(\vv^\top\mW_{22}\vv)^{1/2}.
\end{align*}
The momentum column of $\widehat\mF_t$ is deterministic: $\widehat\mF_t(0,\vv)^\top=\beta_1(-\vv,\vv)^\top$. Testing the eigenmatrix identity on $(0,\vv)$ therefore gives
\begin{align*}
 \rho\,\vv^\top\mW_{22}\vv
    &=
    \begin{pmatrix}0\\\vv
    \end{pmatrix}^{\!\top}
   \Ebb[\widehat\mF_t^\top\mW\widehat\mF_t]
    \begin{pmatrix}0\\\vv
    \end{pmatrix}\\
    &=\beta_1^2
    \begin{pmatrix}-\vv\\\vv
    \end{pmatrix}^{\!\top}
    \mW
    \begin{pmatrix}-\vv\\\vv
    \end{pmatrix}\\
 &\le\beta_1^2\left[(\vv^\top\mW_{11}\vv)^{1/2}
                    +(\vv^\top\mW_{22}\vv)^{1/2}\right]^2.
\end{align*}

Taking square roots and moving the $\mW_{22}$ term to the left gives, for every $\vv$,
\[
 (\sqrt\rho-\beta_1)(\vv^\top\mW_{22}\vv)^{1/2}
 \le\beta_1(\vv^\top\mW_{11}\vv)^{1/2}.
\]
The displacement block $\mW_{11}$ is nonzero. Otherwise $\sqrt\rho-\beta_1>0$ would force $\mW_{22}=0$. The triangle inequality above would then give $\vz^\top\mW\vz=0$ for every $\vz$, contradicting $\mW\ne0$.

Rescale $\mW$ so that $\norm{\mW_{11}}=1$, where the norm is the spectral norm. The eigenmatrix identity is unchanged by this positive scalar rescaling. Dividing the preceding inequality by $\sqrt\rho-\beta_1>0$ and maximizing over unit $\vv$ gives
\begin{align*}
 \sqrt{\norm{\mW_{22}}}
 &=\max_{\norm\vv=1}(\vv^\top\mW_{22}\vv)^{1/2}\\
 &\le\frac{\beta_1}{\sqrt\rho-\beta_1}
       \max_{\norm\vv=1}(\vv^\top\mW_{11}\vv)^{1/2}\\
 &=\frac{\beta_1}{\sqrt\rho-\beta_1}
 \le\frac{\beta_1}{1-\beta_1}.
\end{align*}
The last line uses $\norm{\mW_{11}}=1$ and $\rho\ge1$.
Applying the same triangle inequality to an arbitrary state $\vz=(\vvareps,\vp)^\top$ now yields
\begin{equation}\label{fh:eq:seminorm}
\begin{aligned}
 (\vz^\top\mW\vz)^{1/2}
 &\le(\vvareps^\top\mW_{11}\vvareps)^{1/2}
      +(\vp^\top\mW_{22}\vp)^{1/2}\\
 &\le\norm{\vvareps}+\sqrt{\norm{\mW_{22}}}\,\norm{\vp}\\
 &\le\norm{\vvareps}+\frac{\beta_1}{1-\beta_1}\norm{\vp}.
\end{aligned}
\end{equation}
\medskip\noindent\textbf{Step 3. Compare exact growth with the detector bounds.}\par\nobreak\smallskip
Choose a unit top eigenvector $\vv$ of $\mW_{11}$. Our normalization gives $\vv^\top\mW_{11}\vv=1$. Use the homogeneous trajectory from Step~1 with this initial displacement and zero initial momentum. Its initial quadratic form is
\[
 \vz_0^\top\mW\vz_0
    =
    \begin{pmatrix}\vv\\0
    \end{pmatrix}^{\!\top}
    \mW
    \begin{pmatrix}\vv\\0
    \end{pmatrix}
 =\vv^\top\mW_{11}\vv=1.
\]
This choice is allowed because acceptance covers every unit initial displacement.

The batch at time $t$ is independent of $\vz_t$. The tower property and the eigenmatrix identity give
\begin{align*}
 \Ebb[\vz_{t+1}^\top\mW\vz_{t+1}]
 &=\Ebb[\vz_t^\top\widehat\mF_t^\top\mW\widehat\mF_t\vz_t]\\
 &=\Ebb\left[\vz_t^\top
       \Ebb[\widehat\mF_t^\top\mW\widehat\mF_t\mid\vz_t]\vz_t\right]\\
 &=\rho\,\Ebb[\vz_t^\top\mW\vz_t].
\end{align*}
Iterating from the initial value one gives $\Ebb[\vz_T^\top\mW\vz_T]=\rho^T$. On the other hand, \eqref{fh:eq:seminorm}, Minkowski's inequality and the bounds from Step~1 imply
\begin{align*}
 \rho^{T/2}
 &=\bigl(\Ebb[\vz_T^\top\mW\vz_T]\bigr)^{1/2}\\
 &\le\left[\Ebb\left(\norm{\vvareps_T}
               +\frac{\beta_1}{1-\beta_1}\norm{\vp_T}\right)^2\right]^{1/2}\\
 &\le(\Ebb\norm{\vvareps_T}^2)^{1/2}
      +\frac{\beta_1}{1-\beta_1}(\Ebb\norm{\vp_T}^2)^{1/2}\\
 &\le M+\frac{\beta_1}{1-\beta_1}(2-\lr\wdc)M\\
 &=\frac{M\bigl[1-\beta_1+2\beta_1-\beta_1\lr\wdc\bigr]}{1-\beta_1}\\
 &=\frac{M(1+b)}{1-\beta_1}.
\end{align*}
Finally, $b=\beta_1(1-\lr\wdc)=\beta_1|1-\lr\wdc|$ and $b\le\beta_1$ on the search window. Taking the power $2/T$,
\begin{align*}
 \rho
 &\le\left[\frac{M(1+b)}{1-\beta_1}\right]^{2/T}\\
 &=\left[\frac{M(1+\beta_1|1-\lr\wdc|)}{1-\beta_1}\right]^{2/T}\\
 &\le\left[\frac{M(1+\beta_1)}{1-\beta_1}\right]^{2/T}
 =\rhoT.
\end{align*}
This completes the proof of \Cref{fh:lem:detector-radius}.
\end{proof}

\begin{remark}[Fixed weighted displacement norms]\label{fh:rem:weighted}
For any fixed $\mW_0\succ0$, \Cref{fh:lem:detector-radius} holds with the same constant when both the unit initialization and the displacement cutoff use $\norm{\vvareps}_{\mW_0}^2=\vvareps^\top\mW_0\vvareps$: replacing $(\vvareps,\vp)$ by $(\mW_0^{1/2}\vvareps,\mW_0^{1/2}\vp)$ preserves the scalar $\beta_1$ and $1-\lr\wdc$ blocks in the proof and changes the moment map by a similarity. This includes the fixed preconditioner norm used by the implemented detector.
\end{remark}

\subsubsection{Bounding the Inverse LR}

\begin{proof}[Proof of \Cref{fh:prop:horizon-true}]
We apply the quantitative stability bound from \Cref{fh:prop:radius-quotient} to the compensated system. Its critical test matrix then converts the quotient difference into the desired inverse-rate difference.

Fix an accepted LR $\lr\in\mathcal I$. If $\lr\le\lrP$, then $1/\lr\ge1/\lrP$, as claimed. Assume $\lr>\lrP$ from now on. Recall the candidate and critical coefficients,
\[
 \begin{aligned}
 b&=\beta_1(1-\lr\wdc),
 &\alpha&=(1-\beta_1)\lr,\\
 \bstarP&=\beta_1(1-\lrP\wdc),
 &\alpha_\star^{\mathsf P}&=(1-\beta_1)\lrP.
 \end{aligned}
\]
The search window and $\lr>\lrP$ imply
\[
 0\le b\le\bstarP\le\beta_1<1,
 \qquad \alpha>\alpha_\star^{\mathsf P}>0.
\]
The matrix $\XstarP\succeq0$ maximizes the compensated quotient at its critical rate, with $\norm{\XstarP}_F=1$. Its curvature is $\mustarP=\tr(\mA_{\mathsf P}(\XstarP)^2)>0$.

\medskip\noindent\textbf{Step 1. Use acceptance to bound the quotient.}\par\nobreak\smallskip
Lemma~\ref{fh:lem:detector-radius} gives $\rho(\Tcal_\lr^{\mathsf P})\le\rhoT$. Apply \Cref{fh:prop:radius-quotient} with the compensated mean $\mC=\mA_{\mathsf P}$, its full noise $\mDelta_t$, radius bound $R=\rhoT$, and test matrix $\mX=\XstarP$. This gives
\begin{equation}\label{fh:eq:accepted-quotient}
 q_{\mA_{\mathsf P},b}(\XstarP)-\frac{1-b}{\alpha}
 \le(\rhoT-1)\left[
       \frac{2b}{(\rhoT+b)\alpha}
       +\frac{(1-b/\rhoT)^2}{2\mustarP\alpha^2}
       \right].
\end{equation}

\medskip\noindent\textbf{Step 2. Convert the quotient into an inverse-rate difference.}\par\nobreak\smallskip
At the critical rate, the quotient maximum equals the normalized inverse LR. Recall \eqref{fh:eq:criticalidentity} and the AGM parameter curve:
\begin{align*}
 q_{\mA_{\mathsf P},\bstarP}(\XstarP)
 &=\frac{1-\bstarP}{\alpha_\star^{\mathsf P}}
 =\frac1{\lrP}+\frac{\beta_1\wdc}{1-\beta_1},\\
 \frac{1-b}{\alpha}
 &=\frac{1-\beta_1+\beta_1\lr\wdc}{(1-\beta_1)\lr}
 =\frac1\lr+\frac{\beta_1\wdc}{1-\beta_1}.
\end{align*}
Subtracting cancels the same weight-decay offset:
\begin{equation}\label{fh:eq:shiftconst}
\begin{aligned}
 q_{\mA_{\mathsf P},\bstarP}(\XstarP)-\frac{1-b}{\alpha}
 &=\left(\frac1{\lrP}+\frac{\beta_1\wdc}{1-\beta_1}\right)
   -\left(\frac1\lr+\frac{\beta_1\wdc}{1-\beta_1}\right)\\
 &=\frac1{\lrP}-\frac1\lr.
\end{aligned}
\end{equation}
Because $b\le\bstarP$, the pointwise momentum inequality \eqref{fh:eq:momentum-monotonicity} gives $q_{\mA_{\mathsf P},\bstarP}(\XstarP)\le q_{\mA_{\mathsf P},b}(\XstarP)$. Hence
\begin{align*}
 \frac1{\lrP}-\frac1\lr
 &=q_{\mA_{\mathsf P},\bstarP}(\XstarP)-\frac{1-b}{\alpha}\\
 &\le q_{\mA_{\mathsf P},b}(\XstarP)-\frac{1-b}{\alpha}\\
 &\le(\rhoT-1)\left[
       \frac{2b}{(\rhoT+b)\alpha}
       +\frac1{2\mustarP}\left(\frac{1-b/\rhoT}{\alpha}\right)^2
       \right].
\end{align*}
The last line uses \eqref{fh:eq:accepted-quotient}. It remains to replace the candidate coefficients by the fixed critical coefficients.

\medskip\noindent\textbf{Step 3. Obtain the uniform bound $H_T$.}\par\nobreak\smallskip
For the first term, $0\le b\le\bstarP$ gives
\[
 \frac{\bstarP}{\rhoT+\bstarP}-\frac b{\rhoT+b}
 =\frac{\rhoT(\bstarP-b)}{(\rhoT+\bstarP)(\rhoT+b)}\ge0.
\]
Combining this with $\alpha>\alpha_\star^{\mathsf P}>0$,
\[
 \frac{b}{(\rhoT+b)\alpha}
 \le\frac{\bstarP}{(\rhoT+\bstarP)\alpha}
 \le\frac{\bstarP}{(\rhoT+\bstarP)\alpha_\star^{\mathsf P}}.
\]
For the second term, we compare the entire ratio $(1-b/\rhoT)/\alpha$. Substituting the AGM expressions for $b$ and $\alpha$,
\begin{align*}
 \frac{1-b/\rhoT}{\alpha}
 &=\frac{1-\beta_1(1-\lr\wdc)/\rhoT}{(1-\beta_1)\lr}\\
 &=\frac1{1-\beta_1}
       \left[\frac{1-\beta_1/\rhoT}{\lr}
             +\frac{\beta_1\wdc}{\rhoT}\right]\\
 &\le\frac1{1-\beta_1}
       \left[\frac{1-\beta_1/\rhoT}{\lrP}
             +\frac{\beta_1\wdc}{\rhoT}\right]\\
 &=\frac{1-\beta_1(1-\lrP\wdc)/\rhoT}{(1-\beta_1)\lrP}
 =\frac{1-\bstarP/\rhoT}{\alpha_\star^{\mathsf P}}.
\end{align*}
The inequality uses $\lr>\lrP$ and $1-\beta_1/\rhoT>0$. Both ratios are positive, so the inequality remains valid after squaring. Substituting these two comparisons gives exactly the definition of $H_T$ from \eqref{fh:eq:HT}:
\[
 \frac1{\lrP}-\frac1\lr
 \le(\rhoT-1)\left[
 \frac{2\bstarP}{(\rhoT+\bstarP)\alpha_\star^{\mathsf P}}
 +\frac{(1-\bstarP/\rhoT)^2}{2\mustarP(\alpha_\star^{\mathsf P})^2}
 \right]=H_T.
\]
\begin{samepage}
For $\beta_1=0$, $b=0$ and $\alpha=\lr$, so the same quotient bound \eqref{fh:eq:accepted-quotient} also gives the direct LR bound:
\begin{align*}
 \frac1{\lrP}-\frac1\lr
 &\le\frac{\rhoT-1}{2\mustarP\lr^2},\\
 \lr^2-\lrP\lr
 &\le\frac{(\rhoT-1)\lrP}{2\mustarP},\\
 \lr
 &\le\frac{\lrP}{2}\left[1+\sqrt{1+\frac{2(\rhoT-1)}{\mustarP\lrP}}\right]\\
 &\le\lrP+\frac{\rhoT-1}{2\mustarP}.
\end{align*}
The second line multiplies the first by $\lrP\lr^2>0$. The third takes the positive root of the resulting quadratic, and the last uses $\sqrt{1+x}\le1+x/2$ for $x\ge0$.
\end{samepage}
This completes the proof of \Cref{fh:prop:horizon-true}.
\end{proof}

For a scalar system without momentum ($\beta_1=0$), with compensated mean $a$ and $\lr>\lrP$, the quantitative bound \eqref{fh:eq:radius-quotient} is an equality when $R$ is the actual radius $\rho(\Tcal_\lr^{\mathsf P})=\Ebb(1-\lr(a+\mDelta_t))^2$. The inverse LR bound then loses only the gap between this radius and $\rhoT$ in \Cref{fh:lem:detector-radius}, and the final relaxation $1/\lr\le1/\lrP$. With momentum, the momentum comparisons can introduce a first-order loss when the noise is small.

\subsubsection{Choosing the Cutoff}

At $\lrP$, $\rho(\Tcal^{\mathsf P}_{\lrP})=1$ by \Cref{fh:thm:decay-exact}, but this does not ensure bounded affine moments or acceptance at a finite cutoff.
For example, without noise, forcing or weight decay, the top eigendirection at $\alpha\lambda_{\max}(\mA_{\mathsf P})=2(1+\beta_1)$ obeys $x_{t+1}=-(1+\beta_1)x_t-\beta_1x_{t-1}$.
From a unit displacement with zero momentum ($x_{-1}=x_0=1$), its solution is
\[
 \begin{aligned}
 x_t&=\frac{1+\beta_1}{1-\beta_1}(-1)^t-\frac{2\beta_1}{1-\beta_1}(-\beta_1)^t,\\
 |x_t|&\longrightarrow\frac{1+\beta_1}{1-\beta_1},
 \end{aligned}
\]
At $\beta_1=0.9$, the limiting amplitude is $19$, so any smaller cutoff $M$ rejects $\lrP$ for sufficiently large $T$ despite bounded displacement.
The affine detector may also fail its cutoff because of batch-gradient forcing, weight-decay drift or checkpoint momentum.

Acceptance at $\lrP$ gives $\lrhat(M,T)\ge\lrP$, which supplies the upper bound in \Cref{fh:thm:main}.
The lower bound only needs a nonempty accepted set and $\lrP\in\mathcal I$, which ensures nonnegative effective momentum at $\lrP$.
\Cref{fh:prop:cutoff} gives a sufficient cutoff that grows polynomially with $T$.

We now fix $\lr=\lrP$ and return to the original affine simulation, with its prescribed checkpoint momentum and forcing. The state below is the state of this single run, not a difference of coupled runs or a rescaled trajectory. Use the original-coordinate state from the proof of \Cref{fh:lem:detector-radius}, $\vz_t=(\vvareps_t,\lrP\mP^{-1}\vm_t)^\top$. With $\alpha=\lrP(1-\beta_1)$, it satisfies
\begin{equation}\label{fh:eq:affine}
 \begin{aligned}
 \vz_{t+1}&=\widehat\mF_t\vz_t+\vf_t,\\
    \vf_t&\coloneqq
    \begin{pmatrix}-\alpha\mP^{-1}\vg_{\batch_t}-\lrP\wdc\param_{t_0}\\ \alpha\mP^{-1}\vg_{\batch_t}
    \end{pmatrix},
 \end{aligned}
\end{equation}
with the homogeneous transition matrix $\widehat\mF_t$ from \eqref{fh:eq:Fhat}, evaluated at $\lr=\lrP$. Thus all quantities in this cutoff calculation use the critical LR. Write $\widehat\mF_{t:s}\coloneqq\widehat\mF_{t-1}\cdots\widehat\mF_s$ for the propagator from time $s$ to time $t$, with $\widehat\mF_{t:t}=\mI$, and define
\begin{equation}\label{fh:eq:cutoffconstants}
 \begin{aligned}
 K_T&\coloneqq\max_{0\le t\le T}\bigl\|\Ebb[\widehat\mF_{t:0}^\top\widehat\mF_{t:0}]\bigr\|,\\
 R_0&\coloneqq\bigl(1+\norm{\lrP\mP^{-1}\vm_{t_0}}^2\bigr)^{1/2},\\
 \sigma_f&\coloneqq\bigl(\Ebb\norm{\vf_t}^2\bigr)^{1/2}.
 \end{aligned}
\end{equation}
$K_T\ge1$ is the worst-case second-moment growth of the homogeneous run over the horizon, $\Ebb\norm{\widehat\mF_{t:0}\vz}^2\le K_T\norm\vz^2$ for $t\le T$. The quantity $R_0$ is the norm of the initial state for a unit displacement with the checkpoint momentum. The size of the forcing is
\[
 \sigma_f^2
 =\Ebb\norm{\alpha\mP^{-1}\vg_{\batch_t}+\lrP\wdc\param_{t_0}}^2
  +\alpha^2\Ebb\norm{\mP^{-1}\vg_{\batch_t}}^2.
\]

\begin{proposition}[An explicit cutoff for acceptance at the compensated critical LR]\label{fh:prop:cutoff}
Under \Cref{fh:ass:comparison}, suppose $\lrP\in\mathcal I$. If
\begin{equation}\label{fh:eq:cutoff}
 M\ge\sqrt{K_T}\,\bigl(R_0+T\sigma_f\bigr),
\end{equation}
then $\lrP$ is accepted by the detector \eqref{fh:eq:compensated-detector}, that is, \Cref{fh:ass:transient} holds.
\end{proposition}
\begin{proof}[Proof of \Cref{fh:prop:cutoff}]
We use $K_T$, the maximal homogeneous second-moment growth over the horizon, to control the initial state and each forcing term separately. Recall that $\norm{\vz_0}=R_0$ for every unit initial displacement with the prescribed checkpoint momentum, and $\Ebb\norm{\vf_j}^2=\sigma_f^2$. Unrolling the affine recursion \eqref{fh:eq:affine},
\[
 \vz_t=\widehat\mF_{t:0}\vz_0+\sum_{j=0}^{t-1}\widehat\mF_{t:j+1}\vf_j.
\]
The propagator $\widehat\mF_{t:j+1}$ depends on the batches $j+1,\ldots,t-1$ and $\vf_j$ on batch $j$, so the two are independent. By the tower property,
\begin{align*}
 \Ebb\norm{\widehat\mF_{t:j+1}\vf_j}^2
 &=\Ebb\bigl[\vf_j^\top\,\Ebb[\widehat\mF_{t:j+1}^\top\widehat\mF_{t:j+1}]\,\vf_j\bigr]\\
 &\le\bigl\|\Ebb[\widehat\mF_{t:j+1}^\top\widehat\mF_{t:j+1}]\bigr\|\,\sigma_f^2\\
 &=\bigl\|\Ebb[\widehat\mF_{t-j-1:0}^\top\widehat\mF_{t-j-1:0}]\bigr\|\,\sigma_f^2\\
 &\le K_T\sigma_f^2,
\end{align*}
where the equality holds because the batches are i.i.d., so that $\widehat\mF_{t:j+1}$ has the law of $\widehat\mF_{t-j-1:0}$. Likewise $\Ebb\norm{\widehat\mF_{t:0}\vz_0}^2\le K_T\norm{\vz_0}^2=K_TR_0^2$ for a unit displacement. Minkowski's inequality in $L^2$ gives, for $t\le T$,
\begin{align*}
 \bigl(\Ebb\norm{\vz_t}^2\bigr)^{1/2}
 &\le\bigl(\Ebb\norm{\widehat\mF_{t:0}\vz_0}^2\bigr)^{1/2}
       +\sum_{j=0}^{t-1}\bigl(\Ebb\norm{\widehat\mF_{t:j+1}\vf_j}^2\bigr)^{1/2}\\
 &\le\sqrt{K_T}\,(R_0+t\sigma_f)\\
 &\le M.
\end{align*}
Since $\norm{\vvareps_t}\le\norm{\vz_t}$, $\Ebb\norm{\vvareps_t}^2\le M^2$ for all $0\le t\le T$ and every unit $\vvareps_0$, which is the acceptance of $\lrP$. This completes the proof of \Cref{fh:prop:cutoff}.
\end{proof}
The condition is not a smallness condition on the noise: $K_T$, $R_0$ and $\sigma_f$ enter with their actual values, and $K_T$ is finite for every $T$. To control its growth, write
\[
 \mathcal U(\mS)\coloneqq\Ebb[\widehat\mF_t\mS\widehat\mF_t^\top],
 \qquad \mS\in\Sbb_{2d},
\]
for the homogeneous moment map at $\lrP$. The critical-rate result (\Cref{fh:thm:decay-exact}) and the change of coordinates above give $\rho(\mathcal U)=1$. The Jordan decomposition therefore bounds the powers of its adjoint, applied to any fixed matrix, by a polynomial. In particular, for some fixed $C'>1$ and integer $0\le k<\dim\Sbb_{2d}$,
\[
 \norm{(\mathcal U^*)^t(\mI_{2d})}_F\le C'(1+t)^k.
\]
Iterating the adjoint map uses the same independence across steps as the forward second-moment recursion. Consequently,
\begin{align*}
 \Ebb[\widehat\mF_{t:0}^\top\widehat\mF_{t:0}]
 &=(\mathcal U^*)^t(\mI_{2d}), &&\explain{iterate the adjoint moment map}\\
 \bigl\|\Ebb[\widehat\mF_{t:0}^\top\widehat\mF_{t:0}]\bigr\|
 &\le\norm{(\mathcal U^*)^t(\mI_{2d})}_F\\
 &\le C'(1+t)^k, &&\explain{Jordan power bound}\\
 K_T&\le C'(1+T)^k. &&\explain{maximize over $0\le t\le T$}
\end{align*}
The resulting cutoff satisfies
\begin{align*}
 M_T&\coloneqq\sqrt{C'(1+T)^k}\,(R_0+T\sigma_f)\\
 &\ge\sqrt{K_T}\,(R_0+T\sigma_f), &&\explain{the sufficient cutoff above}\\
 \frac{1+\log M_T}{T}
 &=O\!\left(\frac{\log T}{T}\right)\longrightarrow0. &&\explain{fixed model and polynomial cutoff}
\end{align*}
Thus $M_T>1$ is admissible at every horizon and yields $H_T=O(\log T/T)$ for a fixed model. The factor $T\sigma_f$ comes from Minkowski's inequality and does not assert linear growth of the actual run.
\subsection{Proofs of the Stability Criterion}\label{fh:sec:criterion-proofs}
We prove the exact criterion, its quantitative form, and the AGM critical-rate statement from \Cref{fh:sec:criterion}. For the exact and quantitative criteria, we use the generic recurrence
\[
 \vx_{t+1}=\bigl[(1+b)\mI-\alpha(\mC+\mDelta_t)\bigr]\vx_t-b\vx_{t-1},
 \qquad \mC\succ0,\quad -1<b<1,\quad\alpha>0.
\]
The AGM proof also treats the remaining LRs, for which $b\le-1$.
Recall that $\Tcal(\mS)=\Ebb[\mF_t\mS\mF_t^\top]$ is its second-moment map, with the history transition $\mF_t$ from \eqref{fh:eq:F}. The quotient $q_{\mC,b}$ and its maximum $r_{\mC}(b)$ are recalled in \Cref{fh:sec:criterion}.

\subsubsection{Proof of the Exact Criterion}\label{fh:sec:exact-proof}

\paragraph{A general random linear recurrence}

Consider the random linear recurrence
\begin{equation}\label{fh:eq:generalrecurrence}
 \vy_{t+1}=\mE_t\vy_t,\qquad \vy_t\in\Rbb^m,
\end{equation}
where the matrices $\mE_t$ are i.i.d.\ with finite second moments and the initial state $\vy_0$ is deterministic. We seek a criterion for $\Ebb\norm{\vy_t}^2\to0$ for every initial state. Separate the mean matrix from its centered perturbation,
\begin{equation}\label{fh:eq:generalsplit}
\begin{aligned}
 \mF&\coloneqq\Ebb\mE_t,\qquad
 \mN_t\coloneqq\mE_t-\mF,\\
 \Tcal_0(\mS)&\coloneqq\mF\mS\mF^\top,\qquad
 \Ncal(\mS)\coloneqq\Ebb[\mN_t\mS\mN_t^\top].
\end{aligned}
\end{equation}
Both $\Tcal_0$ and $\Ncal$ preserve positive semidefinite matrices. The second moment $\mSigma_t\coloneqq\Ebb[\vy_t\vy_t^\top]$ evolves as
\begin{equation}\label{fh:eq:generalT}
\begin{aligned}
 \mSigma_{t+1}
 &=\Ebb[(\mF+\mN_t)\mSigma_t(\mF+\mN_t)^\top]\\
 &=\Tcal_0(\mSigma_t)+\Ncal(\mSigma_t). && \explain{$\Ebb\mN_t=0$}
\end{aligned}
\end{equation}
The first equality uses the independence of $\vy_t$ and $\mE_t$.
Write $\Tcal\coloneqq\Tcal_0+\Ncal$ for this second-moment map.

Suppose first that the mean dynamics are stable, $\rho(\mF)<1$. A perturbation with second moment $\mQ$ injected once produces the later second moments $\mF^j\mQ(\mF^\top)^j$. Their sum is
\begin{equation}\label{fh:eq:K}
 \Kcal(\mQ)\coloneqq\sum_{j\ge0}\mF^j\mQ(\mF^\top)^j,
 \qquad \Kcal=(\Ical-\Tcal_0)^{-1}.
\end{equation}
Here $\Ical$ is the identity operator on symmetric matrices. The series converges and preserves PSD matrices. Thus $\Kcal(\mQ)$ is the cumulative response of the mean dynamics to one injection, and $\Kcal\Ncal$ first forms the second moment of a noise term and then accumulates its response. The next proposition characterizes stability through this composition.

\begin{proposition}[Feedback criterion]\label{fh:prop:feedback}
The recurrence \eqref{fh:eq:generalrecurrence} is strictly mean-square stable if and only if $\rho(\mF)<1$ and $\rho(\Kcal\Ncal)<1$.
\end{proposition}
\begin{proof}[Proof of \Cref{fh:prop:feedback}]
We transfer a Lyapunov inequality between $\Tcal$ and $\Kcal\Ncal$. The inverse $\Kcal$ accounts for the mean dynamics in this transfer.

\medskip\noindent\textbf{Step 1. Separate the mean dynamics.}\par\nobreak\smallskip
Strict mean-square stability implies $\rho(\mF)<1$ by the second-moment characterization (\Cref{fh:lem:covariance}). Assume $\rho(\mF)<1$ from now on, so $\Kcal$ exists. Then
\begin{align*}
 (\Ical-\Tcal_0)(\Ical-\Kcal\Ncal)
 &=\Ical-\Tcal_0-\Ncal && \explain{$(\Ical-\Tcal_0)\Kcal=\Ical$}\\
 &=\Ical-\Tcal. && \explain{$\Tcal=\Tcal_0+\Ncal$}
\end{align*}

\medskip\noindent\textbf{Step 2. Necessity.}\par\nobreak\smallskip
If $\rho(\Tcal)<1$, its Neumann series defines
\[
 \mW\coloneqq(\Ical-\Tcal)^{-1}(\mI_m)
 =\sum_{j\ge0}\Tcal^j(\mI_m)\succeq\mI_m.
\]
The factorization in Step~1 gives
\begin{align*}
 \mW-\Kcal\Ncal(\mW)
 &=\Kcal\bigl(\mW-\Tcal(\mW)\bigr)\\
 &=\Kcal(\mI_m)\succeq\mI_m.
\end{align*}
The first equality multiplies the factorization by $\Kcal$. The second line
uses the Neumann series for $\mW$ and $\Kcal$.
The Lyapunov criterion for PSD-preserving maps (\Cref{fh:lem:positive}) therefore gives $\rho(\Kcal\Ncal)<1$.

\medskip\noindent\textbf{Step 3. Sufficiency.}\par\nobreak\smallskip
Conversely, if $\rho(\Kcal\Ncal)<1$, put
\[
 \mW\coloneqq(\Ical-\Kcal\Ncal)^{-1}\Kcal(\mI_m)
 =\sum_{j\ge0}(\Kcal\Ncal)^j\Kcal(\mI_m)\succeq\mI_m.
\]
Now the same factorization gives
\begin{align*}
 \mW-\Tcal(\mW)
 &=(\Ical-\Tcal_0)\bigl(\mW-\Kcal\Ncal(\mW)\bigr) && \explain{Step~1}\\
 &=(\Ical-\Tcal_0)\Kcal(\mI_m)=\mI_m. && \explain{definition of $\mW$}
\end{align*}
The Lyapunov criterion gives $\rho(\Tcal)<1$, which is strict mean-square stability by \Cref{fh:lem:covariance}. This completes the proof of \Cref{fh:prop:feedback}.
\end{proof}

\paragraph{The momentum recurrence}

We now apply the general criterion to the recurrence from \eqref{fh:eq:genericrecurrence}. Its mean curvature is $\mC\succ0$, its parameters satisfy $-1<b<1$ and $\alpha>0$, and its symmetric noise $\mDelta_t$ is i.i.d.\ with zero mean and finite second moments. The initial pair $(\vx_0,\vx_{-1})$ is deterministic. Recall the recurrence and its history state,
\[
 \vx_{t+1}=\bigl[(1+b)\mI-\alpha(\mC+\mDelta_t)\bigr]\vx_t-b\vx_{t-1},
 \qquad \vy_t\coloneqq(\vx_t,\vx_{t-1})^\top.
\]
Write $\mE\coloneqq(1+b)\mI-\alpha\mC$. The random and mean history matrices, and the injection into the position block, are
\begin{equation}\label{fh:eq:history}
  \mF_t\coloneqq
  \begin{pmatrix}\mE-\alpha\mDelta_t&-b\mI\\\mI&0
  \end{pmatrix},\qquad
  \mF\coloneqq\Ebb\mF_t=
  \begin{pmatrix}\mE&-b\mI\\\mI&0
  \end{pmatrix},\qquad
  \mJ\coloneqq
  \begin{pmatrix}\mI\\0
  \end{pmatrix}.
\end{equation}
This is the general recurrence with state dimension $m=2n$ and random matrix $\mE_t=\mF_t$. We use $\Tcal_0$, $\Ncal$ and $\Kcal$ for this history system from now on. The matrix $\mJ$ embeds a position vector as $(\vx,0)^\top$, while $\mJ^\top\mS\mJ$ extracts the position block of a history second moment.
Let $a_i>0$ be the eigenvalues of $\mC$. Recall that $L_{\mC}=\lambda_{\max}(\mC)$. The characteristic polynomials of $\mF$ along these eigendirections are $x^2-(1+b-\alpha a_i)x+b$. The Jury stability test (\Cref{fh:lem:quadratic}) requires
\[
 |b|<1,\qquad \alpha a_i>0,\qquad 2(1+b)-\alpha a_i>0.
\]
The first two conditions already hold. Thus the mean dynamics are stable if and only if
\begin{equation}\label{fh:eq:meanstable}
 \alpha L_{\mC}<2(1+b).
\end{equation}

The centered perturbation is $\mN_t=\mF_t-\mF=-\alpha\mJ\mDelta_t\mJ^\top$, so it acts only on the position. Define its second-moment map on position matrices by
\begin{equation}\label{fh:eq:operators}
 \Vcal(\mX)\coloneqq\Ebb[\mDelta_t\mX\mDelta_t],\qquad \mX\in\Sbb_n.
\end{equation}
Then the full noise map from \eqref{fh:eq:generalsplit} is
\begin{equation}\label{fh:eq:Nspecial}
 \Ncal(\mS)=\alpha^2\mJ\Vcal(\mJ^\top\mS\mJ)\mJ^\top.
\end{equation}
Under the mean stability condition $\alpha L_{\mC}<2(1+b)$, $\Kcal$ propagates an injection through the mean history matrix $\mF$ above, as in \eqref{fh:eq:K}. Since $\Ncal$ uses only the position block, reducing $\Kcal\Ncal$ requires the position response to a position injection. The following lemma computes that response. Recall $c_b\coloneqq(1-b)/(1+b)$ and $k_b\coloneqq b/(1-b^2)$.

\begin{lemma}[Position block of the cumulative response]\label{fh:lem:positionblock}
Assume \eqref{fh:eq:meanstable}, and define the operator on $\Sbb_n$
\begin{equation}\label{fh:eq:Doperator}
 \Dcal(\mX)\coloneqq\frac{1-b}{\alpha}(\mC\mX+\mX\mC)
 -c_b\mC\mX\mC+k_b[\mC,[\mC,\mX]].
\end{equation}
Then $\Dcal$ is self-adjoint and positive definite. For every $\mQ\in\Sbb_n$,
\begin{equation}\label{fh:eq:responseinverse}
 \mJ^\top\Kcal(\mJ\mQ\mJ^\top)\mJ
 =\alpha^{-2}\Dcal^{-1}(\mQ).
\end{equation}
Consequently, $\Dcal^{-1}$ and $\Dcal^{-1}\Vcal$ preserve PSD matrices, and $\Kcal\Ncal$ and $\Dcal^{-1}\Vcal$ have the same nonzero eigenvalues. In particular, their spectral radii are equal.
\end{lemma}
\begin{proof}[Proof of \Cref{fh:lem:positionblock}]
We solve the Lyapunov equation for a single position injection by eliminating its delayed and cross blocks. This identifies $\Dcal$, while the response series proves that its inverse preserves PSD matrices.

\medskip\noindent\textbf{Step 1. Eliminate the cross and delayed blocks.}\par\nobreak\smallskip
Fix $\mQ\in\Sbb_n$. Since \eqref{fh:eq:meanstable} gives $\rho(\mF)<1$, the cumulative response is
\begin{equation}\label{fh:eq:auxiliaryresponse}
 \mSigma_{\mQ}\coloneqq\Kcal(\mJ\mQ\mJ^\top)
 =\sum_{j\ge0}\mF^j\mJ\mQ\mJ^\top(\mF^\top)^j
    =
    \begin{pmatrix}\mX&\mY\\\mY^\top&\mW
    \end{pmatrix}.
\end{equation}
Here $\mX$ is the cumulative position response, $\mW$ the cumulative delayed-position response, and $\mY$ their cross block. These blocks depend linearly on the injection $\mQ$.
Its Lyapunov equation $\mSigma_{\mQ}=\mJ\mQ\mJ^\top+\mF\mSigma_{\mQ}\mF^\top$ reads
\begin{align*}
 \mX&=\mQ+\mE\mX\mE-b(\mE\mY+\mY^\top\mE)+b^2\mW,\\
 \mY&=\mE\mX-b\mY^\top,\qquad \mW=\mX.
\end{align*}
The cross block $\mY$ need not be symmetric. Its equation and its transpose give
\begin{align*}
 \mY+b\mY^\top&=\mE\mX,\qquad \mY^\top+b\mY=\mX\mE,\\
 (1-b^2)\mY&=\mE\mX-b\mX\mE,\\
 \mE\mY+\mY^\top\mE
 &=\frac{\mE^2\mX+\mX\mE^2-2b\mE\mX\mE}{1-b^2}.
\end{align*}
The second line subtracts $b$ times the transposed equation from the first.
The last line substitutes the resulting expressions for $\mY$ and $\mY^\top$.
Substituting these identities and $\mW=\mX$ into the first block equation,
\begin{align*}
 \mQ&=(1-b^2)\mX-\mE\mX\mE+b(\mE\mY+\mY^\top\mE)\\
 &=(1-b^2)\mX-\frac{1+b^2}{1-b^2}\mE\mX\mE
   +\frac{b}{1-b^2}(\mE^2\mX+\mX\mE^2).
\end{align*}

\medskip\noindent\textbf{Step 2. Identify the position operator.}\par\nobreak\smallskip
Expand $\mE=(1+b)\mI-\alpha\mC$ in the last display. The coefficients of $\mX$ and $\mC\mX+\mX\mC$ simplify, respectively, to
\begin{align*}
 (1-b^2)-\frac{(1+b)^2(1+b^2-2b)}{1-b^2}
 &=(1-b^2)-(1-b^2)=0,\\
 \frac{\alpha(1+b)(1+b^2-2b)}{1-b^2}
 &=\alpha(1-b).
\end{align*}
The quadratic terms are
\begin{align*}
 &-\frac{(1+b^2)\alpha^2}{1-b^2}\mC\mX\mC
   +\frac{b\alpha^2}{1-b^2}(\mC^2\mX+\mX\mC^2)\\
 &\quad=-\frac{(1-b)^2\alpha^2}{1-b^2}\mC\mX\mC
   +\frac{b\alpha^2}{1-b^2}[\mC,[\mC,\mX]]\\
 &\quad=-\alpha^2c_b\mC\mX\mC+\alpha^2k_b[\mC,[\mC,\mX]].
\end{align*}
The second line uses $\mC^2\mX+\mX\mC^2=[\mC,[\mC,\mX]]+2\mC\mX\mC$. Combining the linear and quadratic terms gives
\begin{equation}\label{fh:eq:elimination}
 \mQ=\alpha^2\Dcal(\mX).
\end{equation}

\medskip\noindent\textbf{Step 3. Prove positive definiteness.}\par\nobreak\smallskip
Each term in \eqref{fh:eq:Doperator} is self-adjoint under the Frobenius inner product. For every symmetric $\mX$, the commutator identity gives
\begin{equation}\label{fh:eq:detquadraticbound}
\begin{aligned}
 &c_b\tr(\mX\mC\mX\mC)-k_b\norm{[\mC,\mX]}_F^2\\
 &\quad=c_b\tr(\mC^2\mX^2)
   -\Bigl(\frac{c_b}{2}+k_b\Bigr)\norm{[\mC,\mX]}_F^2\\
 &\quad\le c_b\tr(\mC^2\mX^2)
 \le c_bL_{\mC}\tr(\mC\mX^2).
\end{aligned}
\end{equation}
Here $c_b/2+k_b=(1+b^2)/(2(1-b^2))>0$ for every $-1<b<1$, and $\mC^2\preceq L_{\mC}\mC$. Therefore, for $\mX\ne0$,
\begin{align*}
 \inner\mX{\Dcal\mX}
 &\ge\left(\frac{2(1-b)}{\alpha}-c_bL_{\mC}\right)\tr(\mC\mX^2)\\
 &>0.
\end{align*}
The first inequality follows from \eqref{fh:eq:Doperator} and
\eqref{fh:eq:detquadraticbound}. The second uses
$\alpha L_{\mC}<2(1+b)$ and $\mC\succ0$.
Thus $\Dcal\succ0$. Solving $\mQ=\alpha^2\Dcal(\mX)$ from \eqref{fh:eq:elimination} gives $\mX=\alpha^{-2}\Dcal^{-1}(\mQ)$, the position response in \eqref{fh:eq:responseinverse}. If $\mQ\succeq0$, every summand of the response series \eqref{fh:eq:auxiliaryresponse} is PSD, so $\mX\succeq0$. Hence $\Dcal^{-1}$ preserves PSD matrices. This property follows from the response series, not from positive definiteness of $\Dcal$ as an operator. The composition $\Dcal^{-1}\Vcal$ also preserves PSD matrices.

\medskip\noindent\textbf{Step 4. Reduce the spectrum to the position block.}\par\nobreak\smallskip
The full noise map injects $\alpha^2\Vcal(\mJ^\top\mS\mJ)$ into the position, and the response of a position injection is $\alpha^{-2}\Dcal^{-1}$. Combining these two identities, \eqref{fh:eq:Nspecial} and \eqref{fh:eq:responseinverse}, gives
\begin{equation}\label{fh:eq:actualresponse}
\begin{aligned}
 \Kcal\Ncal(\mS)&=\alpha^2\Kcal\bigl(\mJ\Vcal(\mJ^\top\mS\mJ)\mJ^\top\bigr),\\
 \mJ^\top\Kcal\Ncal(\mS)\mJ&=\Dcal^{-1}\Vcal(\mJ^\top\mS\mJ).
\end{aligned}
\end{equation}
If $\Kcal\Ncal(\mS)=\lambda\mS$ with $\lambda\ne0$, compress to $\mX\coloneqq\mJ^\top\mS\mJ$. Then
\[
 \Dcal^{-1}\Vcal(\mX)=\lambda\mX.
\]
For $\mS\ne0$ this $\mX$ is nonzero, since otherwise the first line of \eqref{fh:eq:actualresponse} would give $\lambda\mS=0$. Conversely, suppose $\Dcal^{-1}\Vcal(\mX)=\lambda\mX$ with $\lambda\ne0$ and $\mX\ne0$. Set $\mS\coloneqq\alpha^2\Kcal(\mJ\Vcal(\mX)\mJ^\top)$. Then
\begin{align*}
 \mJ^\top\mS\mJ&=\Dcal^{-1}\Vcal(\mX)=\lambda\mX\ne0,\\
 \Kcal\Ncal(\mS)
 &=\alpha^2\Kcal\bigl(\mJ\Vcal(\lambda\mX)\mJ^\top\bigr)=\lambda\mS.
\end{align*}
The first line uses \eqref{fh:eq:responseinverse}. The second uses
\eqref{fh:eq:actualresponse} and linearity.
The nonzero eigenvalues coincide, so the spectral radii are equal. This completes the proof of \Cref{fh:lem:positionblock}.
\end{proof}

We can now express stability through the quotient from \eqref{fh:eq:quotient}. For the same mean $\mC$, noise $\mDelta_t$ and fixed momentum $b$, recall
\[
 q_{\mC,b}(\mX)
 =\frac{\Ebb\tr(\mX\mDelta_t\mX\mDelta_t)+c_b\tr(\mX\mC\mX\mC)-k_b\norm{[\mC,\mX]}_F^2}{2\tr(\mC\mX^2)},
 \qquad 0\ne\mX\succeq0.
\]
Here $c_b=(1-b)/(1+b)$ and $k_b=b/(1-b^2)$, and $r_{\mC}(b)$ is the maximum over $\mX\succeq0$ with $\norm{\mX}_F=1$. The first numerator term is $\inner\mX{\Vcal\mX}$, the noise contribution. The other two come from the mean dynamics through $\Dcal$. Thus
\begin{equation}\label{fh:eq:quadraticform}
 \inner\mX{(\Dcal-\Vcal)\mX}
 =2\tr(\mC\mX^2)\left(\frac{1-b}{\alpha}-q_{\mC,b}(\mX)\right).
\end{equation}
This identity holds for every nonzero PSD $\mX$. It follows from $\inner\mX{\Vcal\mX}=\Ebb\tr(\mX\mDelta_t\mX\mDelta_t)$ and $\inner\mX{[\mC,[\mC,\mX]]}=\norm{[\mC,\mX]}_F^2$.

\begin{proof}[Proof of \Cref{fh:thm:exact}]
The general criterion reduces stability to $\Dcal^{-1}\Vcal$. A PSD eigenmatrix at its spectral radius then reduces the test to the quotient in \eqref{fh:eq:quadraticform}.

\medskip\noindent\textbf{Step 1. Reduce to the position operator.}\par\nobreak\smallskip
For a unit top eigenvector $\vu$ of $\mC$, test the quotient with $\vu\vu^\top$:
\begin{equation}\label{fh:eq:rankonetest}
 r_{\mC}(b)\ge q_{\mC,b}(\vu\vu^\top)
 =\frac{c_bL_{\mC}}{2}+\frac{\Ebb(\vu^\top\mDelta_t\vu)^2}{2L_{\mC}}
 \ge\frac{c_bL_{\mC}}{2}>0.
\end{equation}
Thus \eqref{fh:eq:pdtest} implies $\alpha L_{\mC}<2(1+b)$. Strict mean-square stability also implies this mean stability condition. Under it, \Cref{fh:prop:feedback,fh:lem:positionblock} give
\begin{equation}\label{fh:eq:feedbackiff}
 \text{strict mean-square stability}\quad\Longleftrightarrow\quad
 \rho(\Dcal^{-1}\Vcal)<1.
\end{equation}

\medskip\noindent\textbf{Step 2. Test the quadratic form on PSD matrices.}\par\nobreak\smallskip
Write $\rho\coloneqq\rho(\Dcal^{-1}\Vcal)$. If $\rho>0$, the Perron--Frobenius theorem for PSD-preserving maps (\Cref{fh:lem:perron}) gives $\mW\succeq0$, $\mW\ne0$, such that
\[
 \Vcal(\mW)=\rho\Dcal(\mW),\qquad
 \rho=\frac{\inner\mW{\Vcal\mW}}{\inner\mW{\Dcal\mW}}.
\]
Consequently, positivity of $\inner\mX{(\Dcal-\Vcal)\mX}$ for every nonzero PSD $\mX$ implies $\rho<1$. This implication also holds when $\rho=0$.

Conversely, suppose $\rho<1$. We proved in \Cref{fh:lem:positionblock} that $\Dcal$ is a positive definite self-adjoint operator on $\Sbb_n$, so it has a positive square root $\Dcal^{1/2}$. The operator $\Vcal$ is self-adjoint because the noise matrices are symmetric. The similarity
\[
 \Dcal^{1/2}(\Dcal^{-1}\Vcal)\Dcal^{-1/2}
 =\Dcal^{-1/2}\Vcal\Dcal^{-1/2}
\]
therefore transforms $\Dcal^{-1}\Vcal$ into a self-adjoint operator with the same spectral radius $\rho$. Its largest eigenvalue is at most $\rho$. Applying its Rayleigh bound to $\Dcal^{1/2}\mX$, for every symmetric $\mX\ne0$,
\begin{align*}
 \inner\mX{\Vcal\mX}
 &=\inner{\Dcal^{1/2}\mX}{(\Dcal^{-1/2}\Vcal\Dcal^{-1/2})\Dcal^{1/2}\mX}\\
 &\le\rho\norm{\Dcal^{1/2}\mX}_F^2
 =\rho\inner\mX{\Dcal\mX}, && \explain{Rayleigh bound}\\
 \inner\mX{(\Dcal-\Vcal)\mX}&\ge(1-\rho)\inner\mX{\Dcal\mX}>0. && \explain{$\rho<1$ and $\Dcal\succ0$}
\end{align*}
This proves the equivalence between stability and positivity of the quadratic form on nonzero PSD matrices.

\medskip\noindent\textbf{Step 3. Take the maximum of the quotient.}\par\nobreak\smallskip
The identity \eqref{fh:eq:quadraticform} expresses this quadratic form as the positive factor $2\tr(\mC\mX^2)$ times $(1-b)/\alpha-q_{\mC,b}(\mX)$. Its positivity is therefore equivalent to
\[
 q_{\mC,b}(\mX)<\frac{1-b}{\alpha}\qquad
 \text{for every }\mX\succeq0\text{ with }\norm{\mX}_F=1.
\]
The maximum is attained on this compact set, since its denominator $2\tr(\mC\mX^2)$ is at least $2\lambda_{\min}(\mC)>0$. The condition is therefore \eqref{fh:eq:pdtest}. For fixed $b$, its right side is independent of $\alpha$, giving the stated interval. This completes the proof of \Cref{fh:thm:exact}.
\end{proof}

\subsubsection{Proof of the Quantitative Bound}\label{fh:sec:quantitative-proof}
The proof rescales the recurrence to make it strictly stable and then applies the exact criterion. We first prove the shift identity (\Cref{fh:lem:shift}), which compares the rescaled mean with the original one.

\begin{proof}[Proof of \Cref{fh:lem:shift}]
The noise term is unchanged because the two systems use the same $\mDelta_t$. We compute the other terms separately. First,
\begin{align*}
 [\mC+h\mI,\mX]
 &=(\mC+h\mI)\mX-\mX(\mC+h\mI)\\
 &=\mC\mX+h\mX-\mX\mC-h\mX\\
 &=[\mC,\mX]. &&\explain{the scalar-shift terms cancel}
\end{align*}
Next, expanding both factors of the shifted mean gives
\begin{align*}
 &\tr\bigl(\mX(\mC+h\mI)\mX(\mC+h\mI)\bigr)\\
 &\quad=\tr(\mX\mC\mX\mC)
       +h\tr(\mX\mC\mX)+h\tr(\mX^2\mC)+h^2\tr(\mX^2)\\
 &\quad=\tr(\mX\mC\mX\mC)+2h\tr(\mC\mX^2)+h^2\tr(\mX^2)\\
 &\quad=\tr(\mX\mC\mX\mC)+2h\mu+h^2.
\end{align*}
The second equality uses cyclicity of the trace. The last substitutes
$\tr(\mC\mX^2)=\mu$ and $\tr(\mX^2)=1$.
The denominator changes according to
\begin{align*}
 \tr\bigl((\mC+h\mI)\mX^2\bigr)
 &=\tr(\mC\mX^2)+h\tr(\mX^2)\\
 &=\mu+h.
\end{align*}
The original quotient satisfies
\[
 2\mu\,q_{\mC,b}(\mX)
 =\Ebb\tr(\mX\mDelta_t\mX\mDelta_t)
  +c_b\tr(\mX\mC\mX\mC)-k_b\norm{[\mC,\mX]}_F^2.
\]
Thus the shifted numerator adds only $c_b(2h\mu+h^2)$ to this expression. Using the denominator computed above,
\begin{align*}
 q_{\mC+h\mI,b}(\mX)
 &=\frac{2\mu\,q_{\mC,b}(\mX)+c_b(2h\mu+h^2)}{2(\mu+h)}\\
 &=\frac{\mu\,q_{\mC,b}(\mX)+c_bh(\mu+h/2)}{\mu+h}.
\end{align*}
The last equality divides numerator and denominator by $2$.
This is \eqref{fh:eq:shiftidentity} and completes the proof of \Cref{fh:lem:shift}.
\end{proof}

\begin{proof}[Proof of \Cref{fh:prop:radius-quotient}]
Fix a PSD test matrix $\mX$ with $\norm{\mX}_F=1$, and recall $\mu=\tr(\mC\mX^2)>0$. We divide the trajectories by a geometric sequence to apply the exact criterion, compare the resulting quotient with $q_{\mC,b}(\mX)$, and bound the difference by scalar algebra.

\medskip\noindent\textbf{Step 1. Divide out enough growth to obtain stability.}\par\nobreak\smallskip
By assumption, $\rho(\Tcal)\le R$. Choose any real number $s>\sqrt{R}\ge1$. Dividing the state at time $t$ by $s^t$ divides its one-step transition matrix by $s$ and its second-moment map by $s^2$. Indeed, for the existing history state $\vy_t=(\vx_t,\vx_{t-1})^\top$ and transition matrix $\mF_t$,
\begin{align*}
 s^{-(t+1)}\vy_{t+1}
 &=\frac{\mF_t}{s}\,(s^{-t}\vy_t),\\
 \Ebb\left[\frac{\mF_t}{s}\mS\frac{\mF_t^\top}{s}\right]
 &=\frac1{s^2}\Ebb[\mF_t\mS\mF_t^\top]
 =\frac1{s^2}\Tcal(\mS).
\end{align*}
The last identity holds for every symmetric matrix $\mS$. Hence the new moment map has spectral radius
\[
 \rho\!\left(\frac1{s^2}\Tcal\right)
 =\frac{\rho(\Tcal)}{s^2}
 \le\frac{R}{s^2}<1.
\]
By the second-moment characterization (\Cref{fh:lem:covariance}), the divided state is strictly mean-square stable.

We now express this division in the position recurrence so that the exact criterion (\Cref{fh:thm:exact}) applies. Set $\vx'_t\coloneqq s^{-t}\vx_t$. Its history state satisfies
\[
    \begin{pmatrix}\vx'_t\\\vx'_{t-1}
    \end{pmatrix}
    =
    \begin{pmatrix}\mI&0\\0&s\mI
    \end{pmatrix}
    (s^{-t}\vy_t).
\]
This fixed invertible change of coordinates preserves strict mean-square stability. Thus the recurrence for $\vx'_t$ is strictly stable for every initial pair. To find that recurrence, substitute $\vx_t=s^t\vx'_t$ and $\vx_{t-1}=s^{t-1}\vx'_{t-1}$ into the original update:
\begin{align*}
 \vx'_{t+1}
 &=s^{-(t+1)}\left(
     \bigl[(1+b)\mI-\alpha(\mC+\mDelta_t)\bigr]\vx_t
      -b\vx_{t-1}\right)\\
 &=s^{-(t+1)}\left(
     \bigl[(1+b)\mI-\alpha(\mC+\mDelta_t)\bigr]s^t\vx'_t
      -bs^{t-1}\vx'_{t-1}\right)\\
 &=\left[\frac{1+b}{s}\mI
       -\frac\alpha s(\mC+\mDelta_t)\right]\vx'_t
       -\frac b{s^2}\vx'_{t-1}.
\end{align*}
The new momentum coefficient is $b/s^2$. The generic recurrence uses the identity coefficient $1+b/s^2$, whereas the last display has $(1+b)/s$. We absorb their difference into a scalar shift of the mean. Define
\begin{align*}
 u&\coloneqq\frac s\alpha
       \left(1+\frac b{s^2}-\frac{1+b}{s}\right)\\
  &=\frac{s^2+b-s(1+b)}{s\alpha}\\
  &=\frac{(s-1)(s-b)}{s\alpha}\ge0.
\end{align*}
Here both factors in the numerator are positive because $s>1>b$. By construction,
\[
 \left(1+\frac b{s^2}\right)\mI
   -\frac\alpha s(\mC+u\mI+\mDelta_t)
 =\frac{1+b}{s}\mI-\frac\alpha s(\mC+\mDelta_t).
\]
Consequently, the divided positions obey
\[
 \vx'_{t+1}
 =\left[\left(1+\frac b{s^2}\right)\mI
       -\frac\alpha s(\mC+u\mI+\mDelta_t)\right]\vx'_t
       -\frac b{s^2}\vx'_{t-1}.
\]
This is the generic recurrence with mean $\mC+u\mI\succ0$, momentum $b/s^2\in[0,1)$, coefficient $\alpha/s>0$, and the same noise $\mDelta_t$. These parameters satisfy the assumptions of \Cref{fh:thm:exact}.

\medskip\noindent\textbf{Step 2. Compare the quotient at the same test matrix.}\par\nobreak\smallskip
Keep the matrix $\mX$ from the statement fixed. The divided recurrence is strictly stable, so the exact criterion (\Cref{fh:thm:exact}) gives
\begin{equation}\label{fh:eq:testedcriterion}
\begin{aligned}
 \frac{s-b/s}{\alpha}
 &=\frac{1-b/s^2}{\alpha/s}\\
 &>r_{\mC+u\mI}(b/s^2)\\
 &\ge q_{\mC+u\mI,\,b/s^2}(\mX).
\end{aligned}
\end{equation}
The last inequality holds because $\mX\succeq0$ and $\norm{\mX}_F=1$ make it an admissible test matrix. We compare this last quotient with $q_{\mC,b}(\mX)$ by undoing the mean shift and comparing the momenta.

First, recall that $\mu=\tr(\mC\mX^2)$ and $\tr(\mX^2)=1$, so
\[
 \tr\bigl((\mC+u\mI)\mX^2\bigr)
 =\tr(\mC\mX^2)+u\tr(\mX^2)=\mu+u.
\]
The shift identity (\Cref{fh:lem:shift}), with the same noise in both quotients, gives
\[
 q_{\mC+u\mI,\,b/s^2}(\mX)
 =\frac{\mu\,q_{\mC,b/s^2}(\mX)+c_{b/s^2}u(\mu+u/2)}{\mu+u}.
\]
Second, $b\ge0$ and $s>1$ imply $b/s^2\le b$. The pointwise momentum inequality \eqref{fh:eq:momentum-monotonicity} therefore gives
\[
 q_{\mC,b/s^2}(\mX)\ge q_{\mC,b}(\mX).
\]
Substituting the shift identity into \eqref{fh:eq:testedcriterion}, multiplying by $\mu+u>0$, and then applying this comparison,
\begin{equation}\label{fh:eq:tested2}
\begin{aligned}
 \frac{(s-b/s)(\mu+u)}{\alpha}
 &\ge\mu\,q_{\mC,b/s^2}(\mX)+c_{b/s^2}u(\mu+u/2)\\
 &\ge\mu\,q_{\mC,b}(\mX)+c_{b/s^2}u(\mu+u/2).
\end{aligned}
\end{equation}
We have weakened the strict first inequality to a non-strict one. The right side now contains the quotient of the original recurrence at the original momentum.

\medskip\noindent\textbf{Step 3. Bound the quotient difference.}\par\nobreak\smallskip
Moving all other terms in \eqref{fh:eq:tested2} to the right gives
\begin{align*}
 \mu\left(q_{\mC,b}(\mX)-\frac{1-b}{\alpha}\right)
 &\le\frac{(s-b/s)(\mu+u)}{\alpha}
       -\frac{\mu(1-b)}{\alpha}
       -c_{b/s^2}u\left(\mu+\frac u2\right)\\
 &=\mu\left[\frac{s-b/s-(1-b)}{\alpha}-c_{b/s^2}u\right]
       +u\left[\frac{s-b/s}{\alpha}-\frac{c_{b/s^2}u}{2}\right].
\end{align*}
The first numerator simplifies as
\[
 s-\frac bs-(1-b)
 =(s-1)+b\left(1-\frac1s\right)
 =(s-1)\left(1+\frac bs\right).
\]
Dividing by $\mu>0$, we obtain
\[
 q_{\mC,b}(\mX)-\frac{1-b}{\alpha}
 \le\left[\frac{(s-1)(1+b/s)}{\alpha}-c_{b/s^2}u\right]
       +\frac u{\mu}
          \left[\frac{s-b/s}{\alpha}-\frac{c_{b/s^2}u}{2}\right].
\]
We simplify these two terms separately. Recall the scalar shift from Step~1 and the definition of $c_b$:
\[
 u=\frac{(s-1)(s-b)}{s\alpha},
 \qquad c_{b/s^2}=\frac{1-b/s^2}{1+b/s^2}
                    =\frac{s^2-b}{s^2+b}.
\]
For the first term, substituting these expressions and combining denominators gives
\begin{align*}
 \frac{(s-1)(1+b/s)}{\alpha}-c_{b/s^2}u
 &=\frac{s-1}{s\alpha}
       \left[(s+b)-\frac{(s^2-b)(s-b)}{s^2+b}\right]\\
 &=\frac{(s-1)\bigl[(s+b)(s^2+b)-(s^2-b)(s-b)\bigr]}
        {s\alpha(s^2+b)}\\
 &=\frac{(s-1)(2bs^2+2bs)}{s\alpha(s^2+b)}\\
 &=\frac{2b(s-1)(s+1)}{\alpha(s^2+b)}
 =\frac{2b(s^2-1)}{(s^2+b)\alpha}.
\end{align*}
For the second term, first substitute $u$ inside the bracket, then factor its numerator:
\begin{align*}
 &\frac u{\mu}\left[\frac{s-b/s}{\alpha}-\frac{c_{b/s^2}u}{2}\right]\\
 &\quad=\frac{u(s^2-b)}{\mu s\alpha}
         \left[1-\frac{(s-1)(s-b)}{2(s^2+b)}\right]\\
 &\quad=\frac{u(s^2-b)\bigl[2(s^2+b)-(s-1)(s-b)\bigr]}
              {2\mu s\alpha(s^2+b)}\\
 &\quad=\frac{u(s^2-b)\bigl[s^2+(b+1)s+b\bigr]}
              {2\mu s\alpha(s^2+b)}\\
 &\quad=\frac{u(s^2-b)(s+1)(s+b)}{2\mu s\alpha(s^2+b)}\\
 &\quad=\frac{(s-1)(s-b)(s^2-b)(s+1)(s+b)}
              {2\mu s^2\alpha^2(s^2+b)}\\
 &\quad=\frac{(s^2-1)(s^2-b^2)(s^2-b)}
              {2s^2(s^2+b)\mu\alpha^2}.
\end{align*}
The penultimate line substitutes the remaining factor $u$. The last line groups $(s-1)(s+1)$ and $(s-b)(s+b)$. Adding the two terms yields, for every $s>\sqrt{R}$,
\[
 q_{\mC,b}(\mX)-\frac{1-b}{\alpha}
 \le(s^2-1)\left[
       \frac{2b}{(s^2+b)\alpha}
       +\frac{(s^2-b^2)(s^2-b)}{2s^2(s^2+b)\mu\alpha^2}
       \right].
\]
The left side does not depend on $s$, and the right side is continuous at $s=\sqrt{R}$, with positive denominators. Letting $s\downarrow\sqrt{R}$ therefore gives
\begin{equation}\label{fh:eq:exactform}
 q_{\mC,b}(\mX)-\frac{1-b}{\alpha}
 \le(R-1)\left[
 \frac{2b}{(R+b)\alpha}
 +\frac{(R-b^2)(R-b)}{2R(R+b)\mu\alpha^2}
 \right].
\end{equation}
We retain this sharper form and now simplify its second coefficient.

\medskip\noindent\textbf{Step 4. Simplify the second coefficient.}\par\nobreak\smallskip
The second coefficient in \eqref{fh:eq:exactform} is easier to compare after bounding it by a square. To justify this bound, subtract its two relevant factors:
\begin{align*}
 \frac{R-b}{R}-\frac{R-b^2}{R+b}
 &=\frac{(R-b)(R+b)-R(R-b^2)}{R(R+b)}\\
 &=\frac{R^2-b^2-R^2+R b^2}{R(R+b)}\\
 &=\frac{b^2(R-1)}{R(R+b)}\ge0.
\end{align*}
Multiplying by $(R-b)/R>0$ gives
\[
 \frac{(R-b^2)(R-b)}{R(R+b)}
 \le\left(\frac{R-b}{R}\right)^2
 =\left(1-\frac b{R}\right)^2.
\]
Thus \eqref{fh:eq:exactform} implies
\[
 q_{\mC,b}(\mX)-\frac{1-b}{\alpha}
 \le(R-1)\left[
       \frac{2b}{(R+b)\alpha}
       +\frac1{2\mu}\left(\frac{1-b/R}{\alpha}\right)^2
       \right].
\]

This is \eqref{fh:eq:radius-quotient} and completes the proof of \Cref{fh:prop:radius-quotient}.
\end{proof}

\subsubsection{Proof of the AGM Critical LR}\label{fh:sec:decay}
Recall from \Cref{fh:sec:criterion} that $b=\beta_1(1-\lr\wdc)$, $\alpha=\lr(1-\beta_1)$, and
\[
 G_{\mC}(\lr)=\frac1\lr+\frac{\beta_1\wdc}{1-\beta_1}
 -r_{\mC}\bigl(\beta_1(1-\lr\wdc)\bigr).
\]
We prove that this function has one zero, which is the endpoint of the stable LRs, and that the moment spectral radius is strictly greater than one above it.

\begin{proof}[Proof of \Cref{fh:thm:decay-exact}]
To rule out $\rho(\Tcal)\le1$ above the root, we rescale a trajectory by $s^{-t}$ and let $s\downarrow1$.

\medskip\noindent\textbf{Step 1. Apply the fixed-parameter criterion.}\par\nobreak\smallskip
For $\lr<\lr_{\max}$, \Cref{fh:thm:exact} gives strict stability if and only if $G_{\mC}(\lr)>0$. If $\lr_{\max}<\infty$ and $\lr\ge\lr_{\max}$, then $b\le-1$. The roots of the mean recurrence along each eigendirection have product $b$, so one root has modulus at least one. Strict mean-square stability then fails by the mean-radius inequality \eqref{fh:eq:meanradius}. At $\lr\wdc=1$, the effective momentum is simply $b=0$, already covered by the criterion.

\medskip\noindent\textbf{Step 2. Prove monotonicity and continuity.}\par\nobreak\smallskip
Recall the pointwise momentum comparison \eqref{fh:eq:momentum-monotonicity}: for $-1<b'\le b<1$, $q_{\mC,b'}(\mX)\ge q_{\mC,b}(\mX)$ for every nonzero PSD $\mX$. We apply this comparison before taking maxima.

The maximization defining $r_{\mC}$ uses the same set $\{\mX\succeq0:\norm{\mX}_F=1\}$ for every $b$. Choose a maximizer $\mX_b$ at $b$, which exists by \eqref{fh:eq:variational}. It satisfies $\mX_b\succeq0$, $\norm{\mX_b}_F=1$, and $q_{\mC,b}(\mX_b)=r_{\mC}(b)$. Therefore,
\begin{align*}
 r_{\mC}(b')
 &=\max_{\mX\succeq0,\ \norm{\mX}_F=1}q_{\mC,b'}(\mX) && \explain{definition of $r_{\mC}(b')$}\\
 &\ge q_{\mC,b'}(\mX_b) && \explain{$\mX_b$ is also admissible at $b'$}\\
 &\ge q_{\mC,b}(\mX_b) && \explain{the pointwise inequality above}\\
 &=r_{\mC}(b). && \explain{$\mX_b$ maximizes the quotient at $b$}
\end{align*}
Hence $b'\le b$ implies $r_{\mC}(b')\ge r_{\mC}(b)$, which is the claimed nonincreasing dependence on $b$.

Return now to the AGM parameter curve and set $b_i\coloneqq\beta_1(1-\lr_i\wdc)$. Since $\lr_2<\lr_1$, we have $b_2\ge b_1$ and therefore $r_{\mC}(b_1)\ge r_{\mC}(b_2)$. Recall that $G_{\mC}(\lr)$ subtracts $r_{\mC}(b)$ from $1/\lr+\beta_1\wdc/(1-\beta_1)$. The offset $\beta_1\wdc/(1-\beta_1)$ is constant along this curve, giving
\begin{align*}
 G_{\mC}(\lr_2)-G_{\mC}(\lr_1)
 &=\frac1{\lr_2}-\frac1{\lr_1}+r_{\mC}(b_1)-r_{\mC}(b_2)\\
 &\ge\frac1{\lr_2}-\frac1{\lr_1}>0. && \explain{monotonicity of $r_{\mC}$}
\end{align*}
This is \eqref{fh:eq:increment}. The quotient is jointly continuous in $b\in(-1,1)$ and $\mX$ on the compact PSD unit sphere, with denominator at least $2\lambda_{\min}(\mC)$. Its maximum $r_{\mC}(b)$, and hence $G_{\mC}$, is continuous.

\medskip\noindent\textbf{Step 3. Check the endpoints.}\par\nobreak\smallskip
As $\lr\downarrow0$, $b\to\beta_1<1$ and
\[
 G_{\mC}(\lr)=\frac1\lr+\frac{\beta_1\wdc}{1-\beta_1}-r_{\mC}(b)
 \longrightarrow+\infty.
\]
If $\beta_1\wdc>0$, then $b\downarrow-1$ as $\lr\uparrow\lr_{\max}$. The rank-one lower bound gives
\[
 r_{\mC}(b)\ge\frac{1-b}{1+b}\frac{L_{\mC}}2\longrightarrow\infty,
 \qquad G_{\mC}(\lr)\longrightarrow-\infty.
\]
If $\beta_1\wdc=0$, then $b=\beta_1$ is fixed and
\[
 G_{\mC}(\lr)=\frac1\lr-r_{\mC}(\beta_1)
 \longrightarrow-r_{\mC}(\beta_1)<0\qquad(\lr\to\infty).
\]
Continuity and strict monotonicity now give one root, with $G_{\mC}>0$ precisely below it.

\medskip\noindent\textbf{Step 4. Locate the second-moment spectral radius.}\par\nobreak\smallskip
First fix $\lr\in(\lr_{\mathrm{crit}},\lr_{\max})$, so that $-1<b<1$ and $G_{\mC}(\lr)<0$. Suppose for a contradiction that its second-moment map satisfies $\rho(\Tcal)\le1$. For every $s>1$, dividing the state at time $t$ by $s^t$ gives
\[
 \rho(s^{-2}\Tcal)=s^{-2}\rho(\Tcal)\le s^{-2}<1.
\]
Recall the position calculation in the proof of \Cref{fh:prop:radius-quotient}. The divided recurrence has coefficient $\alpha/s$, momentum $b/s^2$, unchanged noise, and mean $\mC+u\mI$, where
\[
 u\coloneqq\frac{(s-1)(s-b)}{s\alpha}>0,
 \qquad
 1+\frac b{s^2}-\frac\alpha s u=\frac{1+b}{s}.
\]
That calculation needs only $-1<b<1$: for every $s>1$ we have $-1<b/s^2<1$, $\alpha/s>0$, and $u>0$ because $s>b$. Thus the shifted mean remains positive definite even when $b<0$. The divided recurrence is strictly stable, so the exact criterion gives, for every $\mX\succeq0$ with $\norm{\mX}_F=1$,
\begin{align*}
 q_{\mC+u\mI,\,b/s^2}(\mX)
 &\le r_{\mC+u\mI}(b/s^2)\\
 &<\frac{1-b/s^2}{\alpha/s}
 =\frac{s-b/s}{\alpha}.
\end{align*}
As $s\downarrow1$, we have $u\to0$ and $b/s^2\to b$. The quotient is continuous in these parameters, with a positive denominator. Hence
\begin{align*}
 q_{\mC,b}(\mX)&\le\frac{1-b}{\alpha}
 &&\explain{take the limit for each fixed $\mX$}\\
 r_{\mC}(b)&=\max_{\mX\succeq0,\ \norm{\mX}_F=1}q_{\mC,b}(\mX)
 \le\frac{1-b}{\alpha},\\
 G_{\mC}(\lr)&=\frac{1-b}{\alpha}-r_{\mC}(b)\ge0.
\end{align*}
This contradicts $G_{\mC}(\lr)<0$. Thus $\rho(\Tcal)>1$ throughout this interval.

If $\lr_{\max}<\infty$ and $\lr>\lr_{\max}$, then $b<-1$. The two roots along any eigendirection of the mean have product $b$, so at least one has modulus greater than one. At $\lr=\lr_{\max}$, we have $b=-1$. For an eigenvalue $a_i>0$ of $\mC$, the characteristic polynomial satisfies
\[
 p_i(x)=x^2+\alpha a_i x-1,\qquad
 p_i(-1)=-\alpha a_i<0,\qquad
 \lim_{x\to-\infty}p_i(x)=+\infty.
\]
It therefore has a root below $-1$. In both cases, the mean-radius inequality \eqref{fh:eq:meanradius} implies $\rho(\Tcal)>1$.

Finally, the entries of the second-moment map are continuous in $\lr$, so its spectral radius is continuous. The stable rates approaching $\lr_{\mathrm{crit}}$ from below give $\rho(\Tcal)\le1$ at the critical rate. Since $G_{\mC}(\lr_{\mathrm{crit}})=0$, the exact criterion excludes strict stability there and gives $\rho(\Tcal)\ge1$. Thus the radius equals one at the critical rate. This completes the proof of \Cref{fh:thm:decay-exact}.
\end{proof}
\subsection{Auxiliary Results}\label{fh:sec:tools}\label{fh:sec:linear}

The proofs use four finite-dimensional lemmas on second-moment maps and on maps that preserve positive semidefinite matrices. We also use the Cauchy--Schwarz and Minkowski inequalities, Jensen's inequality, the Neumann series for $(s\Ical-\mathcal U)^{-1}$ when $|s|>\rho(\mathcal U)$, the Rayleigh quotient characterization of the largest eigenvalue of a self-adjoint operator, and continuity of polynomial roots in their coefficients. The cutoff argument uses the Jordan decomposition: the powers of a finite-dimensional operator with spectral radius one grow at most polynomially.

We distinguish two notions of positivity. A map $\mathcal U:\Sbb_n\to\Sbb_n$ \emph{preserves PSD matrices} if $\mX\succeq0$ implies $\mathcal U(\mX)\succeq0$. A self-adjoint operator $\mathcal D$ on $\Sbb_n$ is \emph{positive definite}, written $\mathcal D\succ0$, if $\inner \mX{\mathcal D \mX}>0$ for every nonzero symmetric $\mX$. Preservation of PSD matrices does not imply positive definiteness as an operator, and the converse fails as well. Second moments are uncentered throughout: $\mSigma_t=\Ebb[\vy_t\vy_t^\top]$, so that $\tr \mSigma_t=\Ebb\norm{\vy_t}^2$, and no zero-mean assumption is imposed on the state.

\begin{samepage}
\begin{lemma}[Second-moment characterization]\label{fh:lem:covariance}
Let $\vy_{t+1}=\mM_t\vy_t$, where the $\mM_t$ are i.i.d.\ with finite second moments and $\vy_0$ is deterministic. Define $\Tcal(\mS)\coloneqq\Ebb[\mM_t\mS \mM_t^\top]$ on $\Sbb_n$.
Then $\Ebb\norm{\vy_t}^2\to0$ for every initial state if and only if $\rho(\Tcal)<1$. Strict mean-square stability also implies $\rho(\Ebb \mM_t)<1$, and in general
\begin{equation}\label{fh:eq:meanradius}
 \rho(\Ebb\mM_t)^2\le\rho(\Tcal).
\end{equation}
\end{lemma}
\end{samepage}
\begin{proof}[Proof of \Cref{fh:lem:covariance}]
Write $\mSigma_t\coloneqq\Ebb[\vy_t\vy_t^\top]$. Since $\vy_t$ depends only on the initial state and on $\mM_0,\ldots,\mM_{t-1}$, it is independent of $\mM_t$, so
\begin{align*}
 \mSigma_{t+1}
 &=\Ebb[\mM_t\vy_t\vy_t^\top \mM_t^\top]\\
 &=\Ebb\bigl[\mM_t\,\Ebb[\vy_t\vy_t^\top\mid \mM_t]\,\mM_t^\top\bigr]\\
 &=\Ebb[\mM_t\mSigma_t\mM_t^\top]=\Tcal(\mSigma_t),
\end{align*}
with $\mSigma_0=\vy_0\vy_0^\top$. The second equality uses the tower property, conditioning on $\mM_t$, and the third uses independence. If $\rho(\Tcal)<1$, the Jordan form of $\Tcal$ shows that $\norm{\Tcal^t}$ is bounded by a polynomial in $t$ times a decaying exponential, so $\Tcal^t\to0$ and $\tr \mSigma_t\to0$.

Conversely, suppose $\Ebb\norm{\vy_t}^2\to0$ for every $\vy_0$. Since $\mSigma_t\succeq0$ we have $\norm{\mSigma_t}_F\le\tr \mSigma_t$, so $\Tcal^t(\vy_0\vy_0^\top)=\mSigma_t\to0$ for every $\vy_0$. Rank-one symmetric matrices span $\Sbb_n$: the diagonal basis elements are $\ve_i\ve_i^\top$, and each off-diagonal symmetric basis element is $(\ve_i+\ve_j)(\ve_i+\ve_j)^\top-\ve_i\ve_i^\top-\ve_j\ve_j^\top$. Hence $\Tcal^t\to0$ as an operator on the finite-dimensional space $\Sbb_n$. An eigenvalue of modulus at least one would contradict this convergence, so $\rho(\Tcal)<1$.

For the statements about the mean, independence gives $\Ebb\vy_{t+1}=(\Ebb\mM_t)\Ebb\vy_t$, and the matrices have the same law at each time. Iterating and applying Jensen's inequality,
\begin{align*}
 \Ebb\vy_t&=(\Ebb\mM_t)^t\vy_0,\\
 \norm{(\Ebb\mM_t)^t\vy_0}^2
 &=\norm{\Ebb\vy_t}^2
 \le\Ebb\norm{\vy_t}^2
 =\tr\Tcal^t(\vy_0\vy_0^\top).
\end{align*}
Under strict mean-square stability the right side tends to zero for every $\vy_0$, so $(\Ebb\mM_t)^t\to0$ and $\rho(\Ebb\mM_t)<1$. To obtain \eqref{fh:eq:meanradius}, apply this implication to the matrices $\mM_t/\sqrt{s}$ for any real $s>\rho(\Tcal)$. Their second-moment map is $\Tcal/s$, so
\begin{align*}
 \rho(\Tcal/s)&=\rho(\Tcal)/s<1, &&\explain{choice of $s$}\\
 \rho\!\left(\Ebb\mM_t/\sqrt{s}\right)&<1, &&\explain{the mean-stability implication above}\\
 \rho(\Ebb\mM_t)^2&<s. &&\explain{rescale the mean matrix}
\end{align*}
Letting $s\downarrow\rho(\Tcal)$ proves \eqref{fh:eq:meanradius}, including when $\rho(\Tcal)=0$. This completes the proof of \Cref{fh:lem:covariance}.
\end{proof}

\paragraph{Connection to the Kronecker second-moment matrix.}
The matrix $\Ebb[\mM_t\otimes\mM_t]$ represents the map $\mS\mapsto\Ebb[\mM_t\mS\mM_t^\top]$ on all $n\times n$ matrices in the column-stacked basis.
For the projected AGM dynamics, this is $\mK^+(\eta)$ in \eqref{eq:agm-local-second-moment-operator}.
The map $\Tcal$ in the preceding lemma is the restriction to symmetric matrices.

Write $\mPhi_t\coloneqq\mM_{t-1}\cdots\mM_0$. For fixed vectors $\vu,\vv$, the Cauchy--Schwarz inequality gives
\begin{align*}
 \norm{\Ebb[\mPhi_t\vu\vv^\top\mPhi_t^\top]}_F
 &\le\Ebb\norm{\mPhi_t\vu\vv^\top\mPhi_t^\top}_F\\
 &=\Ebb\bigl[\norm{\mPhi_t\vu}\norm{\mPhi_t\vv}\bigr]\\
 &\le\bigl(\Ebb\norm{\mPhi_t\vu}^2\,
            \Ebb\norm{\mPhi_t\vv}^2\bigr)^{1/2}.
\end{align*}
If the symmetric second moments decay, the right side tends to zero. Since the matrices $\vu\vv^\top$ span all matrices, the full moment map also has powers tending to zero. The converse follows by restriction. Replacing $\mM_t$ by $\mM_t/\sqrt{s}$ divides both moment maps by $s$. Applying the same argument for every $s>0$ shows that the two spectral radii are below $s$ for exactly the same $s$, hence are equal.

Recall that the history state uses $(\vx_t,\vx_{t-1})$, whereas \eqref{fh:eq:model} uses position and momentum $(\vx_t,\vn_t)$. The position update gives their linear relation
\[
  \begin{pmatrix}\vx_t\\\vn_t
  \end{pmatrix}
  =
  \begin{pmatrix}
   \mI&0\\-\lr^{-1}\mI&\lr^{-1}(1-\lr\wdc)\mI
  \end{pmatrix}
  \begin{pmatrix}\vx_t\\\vx_{t-1}
  \end{pmatrix}.
\]
This change is invertible when $\lr\wdc\ne1$, and an invertible change of state induces a similarity between the moment maps. Their spectral radii are therefore equal there. The entries of both moment maps are polynomials in $\lr$, and the spectral radius is continuous in the entries, so the equality extends to $\lr\wdc=1$ as well.

\begin{lemma}[A Lyapunov criterion for PSD-preserving maps]\label{fh:lem:positive}
If $\mathcal U:\Sbb_n\to \Sbb_n$ preserves PSD matrices, then
\[
 \rho(\mathcal U)<1\quad\Longleftrightarrow\quad
 \text{there exists }\mW\succ0\text{ with }\mW-\mathcal U(\mW)\succ0.
\]
\end{lemma}
\begin{proof}[Proof of \Cref{fh:lem:positive}]
If $\rho(\mathcal U)<1$, the Neumann series $\mW\coloneqq\sum_{j\ge0}\mathcal U^j(\mI )$ converges, satisfies $\mW\succeq \mI $ because every summand is PSD, and $\mW-\mathcal U(\mW)=\mI $.
Conversely, given such a $\mW$, choose $\varepsilon\in(0,1)$ with
$0\preceq\mathcal U(\mW)\preceq(1-\varepsilon)\mW$.
For every symmetric $\mX$ there is a finite $a$ with $-a\mW\preceq \mX\preceq a\mW$. Since $\mathcal U$ preserves the order $\preceq$,
\[
 -a(1-\varepsilon)^j\mW\preceq\mathcal U^j(\mX)\preceq a(1-\varepsilon)^j\mW\qquad(j\ge0).
\]
Applying this to a finite basis of $\Sbb_n$ gives $\mathcal U^j\to0$, hence $\rho(\mathcal U)<1$. This completes the proof of \Cref{fh:lem:positive}.
\end{proof}

\begin{lemma}[Perron--Frobenius theorem for PSD-preserving maps]
\label{fh:lem:perron}
Let $\mathcal U:\Sbb_n\to\Sbb_n$ preserve PSD matrices,
and suppose $R=\rho(\mathcal U)>0$. There is a nonzero
$\mW\in\Sbb_n^+$ with $\mathcal U(\mW)=R\mW$.
\end{lemma}
This is the Perron--Frobenius theorem for a linear map that leaves a proper cone invariant, here the cone $\Sbb_n^+$ of positive semidefinite matrices (\citealp{vandergraft1968spectral}; \citealp[Chapter~1]{berman1994nonnegative}).

\begin{lemma}[Jury stability test for a real quadratic]\label{fh:lem:quadratic}
Both roots of $p(x)=x^2-\sigma x+\pi$, with real $\sigma,\pi$, lie strictly inside the unit disk if and only if
$|\pi|<1$, $1-\sigma+\pi>0$, and $1+\sigma+\pi>0$.
\end{lemma}
This is the Schur--Cohn, or Jury, stability test for a real quadratic \citep{jury1964theory}: the conditions are $|\pi|<1$, $p(1)>0$ and $p(-1)>0$, and they are equivalent to both roots lying in the open unit disk.
\clearpage
\endgroup

\end{document}